\documentclass[aos]{imsart}

\RequirePackage{amsthm,amsmath,amsfonts,amssymb,color}
\RequirePackage[numbers]{natbib}
\RequirePackage[colorlinks,citecolor=blue,urlcolor=blue]{hyperref}
\RequirePackage{graphicx}
\usepackage[arxiv]{optional} %
\usepackage{etoc} %

\startlocaldefs
\theoremstyle{plain}

\newtheorem{theorem}{Theorem}[section]
\newtheorem{lemma}[theorem]{Lemma}
\theoremstyle{definition}
\newtheorem{definition}[theorem]{Definition}
\newtheorem*{example}{Example}

\usepackage{upgreek} %
\usepackage{xspace} %
\usepackage{enumitem} %
\usepackage{thmtools,thm-restate}
\usepackage[capitalize]{cleveref}  

\usepackage{autonum}

\newtheorem{remark}{Remark}

\newtheorem{assumption}{Assumption}
\newtheorem{corollary}{Corollary}
\newtheorem{proposition}{Proposition}

\newtheorem{myexample}{Example}

\newtheorem{myproposition}{Proposition}

\newtheorem{mycorollary}{Corollary}

\newtheorem{mylemma}{Lemma}

\newtheorem{mydefinition}{Definition}

\newtheorem{myremark}{Remark}

\newtheorem{myassumption}{Assumption}

\crefname{appendix}{App.}{App.}
\crefname{subsubsubappendix}{App.}{App.}
\crefname{equation}{}{}
\crefname{lemma}{Lem.}{Lem.}
\crefname{theorem}{Thm.}{Thm.}
\crefname{Corollary}{Cor.}{Cors.}
\crefname{algorithm}{Alg.}{Algs.}

\crefname{section}{Sec.}{Sec.}
\crefname{table}{Tab.}{Tab.}
\crefname{remark}{Rem.}{Rem.}
\crefname{definition}{Def.}{Def.}
\crefname{Proposition}{Prop.}{Prop.}
\crefname{myremark}{Rem.}{Rem.}
\crefname{mylemma}{Lem.}{Lem.}
\crefname{mydefinition}{Def.}{Defs.}
\crefname{myproposition}{Prop.}{Prop.}
\crefname{mycorollary}{Cor.}{Cors.}
\crefname{myassumption}{Assum.}{Assum.}
\crefname{figure}{Fig.}{Fig.}
\crefname{myexample}{Ex.}{Ex.}
\crefname{enumi}{}{}
\crefname{name}{}{} %
    \newcommand{\betadiff}{\beta_{\delta}}
    \newcommand{\ppau}{\probparam_{\n, a}}
 \newcommand{\nhat}[1][\n, a, \threshold]{{\mbf{N}}_{#1}}
  \newcommand{\nstar}[1][\n, a, \threshold]{\mbf{N}^\star_{#1}}
  \newcommand{\nwstar}[1][\n, a, \threshold]{\wtil{\mbf{N}}^\star_{#1}} 
  \newcommand{\nnhat}[1][\n, a, \threshold]{{N}_{#1}}
  \newcommand{\nnstar}[1][\n, a, \threshold]{N^\star_{#1}}
  \newcommand{\nnwstar}[1][\n, a, \threshold]{\wtil{N}^\star_{#1}}    
  \newcommand{\tstopj}[1][\l]{\t_{(#1)}(\n, j)}

\newcommand{\N}{N}
\newcommand{\errt}[1][\n, \t, a]{\mrm{Err}_{#1}}

\newcommand{\pparamn}[1][\n]{\probparam_{#1}}
\newcommand{\pparamu}[1][\lunit]{\pparamn[#1]}
\newcommand{\NT}[1][\T]{\N_{#1}}
\newcommand{\MT}[1][\T]{M_{#1}}
\newcommand{\thresonet}[1][\T]{\threshold_{#1}'}
\newcommand{\threstwot}[1][\T]{\threshold_{#1}''}

\newcommand{\KTset}[1][\T]{\mbf{K}_{#1}}
\newcommand{\KT}[1][\T]{K_{#1}}

\newcommand{\sumN}[1][\n]{\sum_{#1=1}^{\N}}

\newcommand{\trueybar}[1][\T]{\bar{\trueobsvar}^{(a)}_{#1}}
\newcommand{\estybar}[1][\threshold_{\T}]{\hat{\bar{\trueobsvar}}^{(a)}_{#1}}
\newcommand{\estybarall}{\what{\overline{\trueobsvar}}^{(a)}}
\newcommand{\estybaralltwo}[1][\threshold_{\T}, K_{\T}]{\estybar[#1]}
\newcommand{\trueybarall}[1][\T]{\trueybar[#1]}

\newcommand{\errterm}[1][\delta]{e_{#1}}

\newcommand{\squash}[1]{\!{#1}\!}
\newcommand{\Sigv}[1][a]{\Sigma_{\ltime[]}^{(#1)}}

\newcommand{\vconst}{c_{\ltime[]}}
\newcommand{\uconst}{c_{\lunit[]}}
\newcommand{\nconst}{c_{\noise}}

\newcommand{\mat}[1]{\mbf{#1}}

\newcommand{\sumt}{\sum_{\t=1}^{\T}}
\newcommand{\sumn}[1][\n]{\sum_{#1=1}^{\N}}

\newcommand{\pmintnotag}{\underline{\policy}}
\newcommand{\pmint}[1][\T]{\pmintnotag_{a,#1}}
\newcommand{\obsvar}{Y}
\newcommand{\trueobsvar}{\theta}

\newcommand{\miss}[1][\n, \t]{A_{#1}}
\newcommand{\history}[1][\n, \t-1]{\mc F_{#1}}
\newcommand{\obs}[1][\n, \t]{\obsvar_{#1}}
\newcommand{\trueobs}[1][\n,\t]{\trueobsvar_{#1}^{(a)}}
\newcommand{\noiseobs}[1][\n, \t]{\noise_{#1}^{(a)}}
\newcommand{\estobs}[1][\n,\t, \threshold]{\what{\trueobsvar}_{#1}^{(a)}}

\newcommand{\lfun}{f}
\newcommand{\lunit}[1][\n]{u_{#1}}
\newcommand{\ltime}[1][\t]{v_{#1}}

\newcommand{\ybnd}{D}

\newcommand{\aref}[1]{\mrm{A}#1}

\newcommand{\estnbrnotag}{\mbf{\N}}
\newcommand{\nestnbrnotag}{\N}
\newcommand{\estnbr}[1][\n, \t, a, \threshold]{\estnbrnotag_{#1}}
\newcommand{\nestnbr}[1][\n, \t, a, \threshold]{\nestnbrnotag_{#1}}

\newcommand{\distnotag}{\rho}
\newcommand{\estdist}[1][\n, j, a]{\distnotag_{#1}}
\newcommand{\truedist}[1][\n, j, a]{\distnotag^\star_{#1}}
\newcommand{\threshold}{\eta}

\newcommand{\probparam}[1][a]{\Phi}

\newcommand{\suchthat}{\ \big \vert \ }

\newcommand{\nidx}[1][{i}]{#1}
\newcommand{\tidx}[1][{t}]{#1}

\newcommand{\actionset}{\mc A}

\newcommand{\ntimes}[1][T]{#1}

\newcommand{\noise}{\vareps}

\newcommand{\policy}{\uppi}

\newcommand{\T}{\ntimes}

\newcommand{\n}{\nidx}
\renewcommand{\t}{\tidx}

\newcommand{\order}{\mc{O}}

\newcommand{\mfk}{\mathfrak}

\newcommand{\indicator}{\mbf 1}

\newcommand{\snorm}[1]{\Vert #1 \Vert}
\newcommand{\sinfnorm}[1]{\snorm{#1}_\infty}

\newcommand{\sless}[1]{\stackrel{#1}{\leq}}
\newcommand{\sgrt}[1]{\stackrel{#1}{\geq}}
\newcommand{\seq}[1]{\stackrel{#1}{=}}

\newcommand{\x}{x}

\newcommand{\axi}[1][i]{\x_{#1}}

\renewcommand{\l}{\ell}

\newcommand{\eps}{\epsilon}

\newcommand{\vareps}{\varepsilon}

\newcommand{\ltwonorm}[1]{\norm{#1}_{2}}

\newcommand{\eventnotag}{\mc{E}}
\newcommand{\event}[1][]{\eventnotag_{#1}}

\newcommand{\wtil}[1]{\widetilde{#1}}

\renewcommand{\natural}{\mbb{N}}

\newcommand{\pseqxn}[1][n]{(\axi[i])_{i\geq 1}} %
\newcommand{\pseqxnn}[1][n]{(\axi[i])_{i=1}^n} %

\newcommand{\brackets}[1]{\left[ #1 \right]}

\newcommand{\parenth}[1]{\left( #1 \right)}
\newcommand{\sparenth}[1]{( #1 )}
\newcommand{\bigparenth}[1]{\big( #1 \big)}
\newcommand{\biggparenth}[1]{\bigg( #1 \bigg)}
\newcommand{\sbraces}[1]{\{ #1  \}}
\newcommand{\braces}[1]{\left\{ #1 \right \}}

\newcommand{\biggbraces}[1]{\bigg\{ #1 \bigg \}}
\newcommand{\abss}[1]{\left| #1 \right |}
\newcommand{\sabss}[1]{| #1  |}
\newcommand{\angles}[1]{\left\langle #1 \right \rangle}
\newcommand{\sangles}[1]{\langle #1 \rangle}

\newcommand{\tp}{^\top}
\newcommand{\inv}{^{-1}}
\newcommand{\real}{\ensuremath{\mathbb{R}}}

\def\balign#1\ealign{\begin{align}#1\end{align}}
\def\baligns#1\ealigns{\begin{align*}#1\end{align*}}
\def\balignat#1\ealign{\begin{alignat}#1\end{alignat}}
\def\balignats#1\ealigns{\begin{alignat*}#1\end{alignat*}}
\def\bitemize#1\eitemize{\begin{itemize}#1\end{itemize}}
\def\benumerate#1\eenumerate{\begin{enumerate}#1\end{enumerate}}

\newenvironment{talign*}
 {\let\displaystyle\textstyle\csname align*\endcsname}
 {\endalign}
\newenvironment{talign}
 {\let\displaystyle\textstyle\csname align\endcsname}
 {\endalign}

\def\balignst#1\ealignst{\begin{talign*}#1\end{talign*}}
\def\balignt#1\ealignt{\begin{talign}#1\end{talign}}
\newcommand{\qtext}[1]{\quad\text{#1}\quad} 
\newcommand{\stext}[1]{\ \text{#1}\ } 

\let\originalleft\left
\let\originalright\right
\renewcommand{\left}{\mathopen{}\mathclose\bgroup\originalleft}
\renewcommand{\right}{\aftergroup\egroup\originalright}

\def\tinycitep*#1{{\tiny\citep*{#1}}}
\def\tinycitealt*#1{{\tiny\citealt*{#1}}}
\def\tinycite*#1{{\tiny\cite*{#1}}}
\def\smallcitep*#1{{\scriptsize\citep*{#1}}}
\def\smallcitealt*#1{{\scriptsize\citealt*{#1}}}
\def\smallcite*#1{{\scriptsize\cite*{#1}}}

\def\blue#1{\textcolor{black}{{#1}}}

\def\mbf#1{\mathbf{#1}}
\def\mbb#1{\mathbb{#1}}
\def\mc#1{\mathcal{#1}}
\def\mrm#1{\mathrm{#1}}
\def\trm#1{\textrm{#1}}
\def\tbf#1{\textbf{#1}}

\def\N{\mathbb{N}}
\def\<{\left\langle} %
\def\>{\right\rangle}

\def\implies{\quad\Longrightarrow\quad}

\def\defeq{\triangleq} %
\def\half{\frac{1}{2}}
\def\quarter{\frac{1}{4}}
\def\norm#1{\left\|{#1}\right\|} %
\newcommand{\twonorm}[1]{\norm{#1}_2} %
\newcommand{\infnorm}[1]{\norm{#1}_{\infty}} %
\newcommand{\opnorm}[1]{\norm{#1}_{\mathrm{op}}} %
\def\what#1{\widehat{#1}}

\def\E{\mbb{E}} %

\def\P{\mbb{P}} %

\def\Var{\mrm{Var}} %

\def\KL#1#2{\textnormal{KL}({#1}\Vert{#2})}

\newcommand{\iid}{\textrm{i.i.d}\xspace}

\ifdefined\nonewproofenvironments\else
\ifdefined\ispres\else
\newenvironment{proof-sketch}{\noindent\textbf{Proof Sketch}
  \hspace*{1em}}{\qed\bigskip\\}
\newenvironment{proof-idea}{\noindent\textbf{Proof Idea}
  \hspace*{1em}}{\qed\bigskip\\}
\newenvironment{proof-of-lemma}[1][{}]{\noindent\textbf{Proof of Lemma {#1}}
  \hspace*{1em}}{\qed\\}
\newenvironment{proof-of-theorem}[1][{}]{\noindent\textbf{Proof of Theorem {#1}}
  \hspace*{1em}}{\qed\\}
\newenvironment{proof-attempt}{\noindent\textbf{Proof Attempt}
  \hspace*{1em}}{\qed\bigskip\\}

\newcommand{\ox}[1]{\opt{arxiv}{#1}}
\newcommand{\oa}[1]{\opt{aos}{#1}}
\newcommand{\vlf}{\mc V}
\newcommand{\ulf}{\mc U} 
\newcommand{\zalpha}{z_{\alpha/2}}
\newcommand{\seqpolicy}{Sequential policy}
\newcommand{\bilin}{Latent factorization of counterfactuals}
\newcommand{\posdefinite}{\iid latent time factors}
\newcommand{\iidlf}{\iid unit latent factors}
\newcommand{\zeromean}{\iid zero mean noise}
\newcommand{\examplediscrete}{Discrete unit factors\xspace}
\newcommand{\examplecontinuous}{Continuous unit factors\xspace}

\newcommand{\varestimatename}{Consistency of variance estimate}

\newcommand{\LT}[1][\T]{L_{#1}}
\newcommand{\anytimeasympname}{Asymptotic guarantees for $\estobs$}    
\newcommand{\errtwo}[1][\T]{\errterm[#1]}

\newcommand{\unittime}{unit$\times$time\xspace}

\newcommand{\anytimeboundname}{Non-asymptotic error bound for $\estobs$}

\newcommand{\ateconsistencyname}{Asymptotic guarantees for estimating $\trueybarall[]$}
\newcommand{\pparam}[1][\thresonet]{\ppau(#1)}

\newcommand{\goodusers}{filtered users\xspace}
\newcommand{\nonlinearanytimeboundname}{Non-asymptotic error bound for $\estobs$ under a non-linear factor model}
\newcommand{\nonlinearfactor}{Non-linear factorization of counterfactuals}
\newcommand{\nonlineartimefactor}{Lower bounded density of time factors}

\graphicspath{{code/figs/},{../code/figs/},{../figs/},{figs/}}

\renewcommand{\N}{N}
\endlocaldefs

\begin{document}

\begin{frontmatter}
\title{Counterfactual inference in sequential experiments}
\runtitle{Counterfactual inference in sequential experiments}

\begin{aug}
\author[A]{\fnms{Raaz}~\snm{Dwivedi}\ead[label=e1]{dwivedi@cornell.edu}},
\author[B]{\fnms{Katherine}~\snm{Tian}\ead[label=e2]{kattian@alumni.harvard.edu}},
\author[C]{\fnms{Sabina}~\snm{Tomkins}\ead[label=e3]{stomkins@umich.edu}},
\author[C]{\fnms{Predrag}~\snm{Klasnja}\ead[label=e4]{klasnja@umich.edu}},
\author[B]{\fnms{Susan}~\snm{Murphy}\ead[label=e5]{samurphy@harvard.edu}}
\and
\author[D]{\fnms{Devavrat}~\snm{Shah}\ead[label=e6]{devavrat@mit.edu}}
\address[A]{Department of Operations Research and Information Engineering,
Cornell Tech\printead[presep={,\ }]{e1}}

\address[B]{Department of Statistics,
Harvard University\printead[presep={,\ }]{e2,e5}}

\address[C]{School of Information,
University of Michigan\printead[presep={,\ }]{e3,e4}}
\address[D]{Department of Electrical Engineering and Computer Science,
Massachusetts Institute of Technology\printead[presep={,\ }]{e6}}

\end{aug}

\begin{abstract}
We consider after-study statistical inference for sequentially designed experiments wherein multiple units are assigned treatments for multiple time points using treatment policies that adapt over time. Our goal is to provide inference guarantees for the counterfactual mean at the smallest possible scale---mean outcome under different treatments \emph{for each unit and each time}---with minimal assumptions on the adaptive treatment policy. Without any structural assumptions on the counterfactual means, this challenging task is infeasible due to more unknowns than observed data points. To make progress, we introduce a  latent factor model over the counterfactual means that serves as a non-parametric generalization of the non-linear mixed effects model and the bilinear latent factor model considered in prior works. For estimation, we use a non-parametric method, namely a variant of nearest neighbors, and establish a non-asymptotic high probability error bound for the counterfactual mean for each unit and each time. Under regularity conditions, this bound leads to asymptotically valid confidence intervals for the counterfactual mean as the number of units and time points grows to $\infty$ together at suitable rates. We illustrate our theory via several simulations and a case study involving data from a mobile health clinical trial HeartSteps.
\end{abstract}



\end{frontmatter}

\addtocontents{toc}{\protect\setcounter{tocdepth}{0}}

\section{Introduction}
\label{sec:introduction}
There is growing interest in using sequential experiments in a multitude of real-world applications. As examples, the domains and applications using such experiments include mobile health, e.g., to encourage healthy behavior~\cite{yom2017encouraging,liao2020personalized,forman2019can,tomkins2021intelligentpooling}, online education, e.g., to enhance learning,~\cite{liu2014trading,qi2018bandit,shaikh2019balancing,cai2021bandit}, legal context, e.g., to encourage court attendance~\cite{stanford:pre-trialnudge}, public policy~\cite{kasy2021adaptive,caria2020adaptive}, besides personalized ads and news recommendations~\cite{li2010contextual,schwartz2017customer,bakshy2013uncertainty,avadhanula2021stochastic,sawant2018contextual}, and the conventional domains like recommender systems~\citep{ijcai2019sequential}. We focus on sequential experiments where in $\N$ units undergo a sequence of treatments for $\T$ time periods. Via such experiments (also referred to as studies), practitioners aim to learn effective treatment that promote a desired behavior (or outcome) with the units in the experiment. A key task for achieving this aim is the after-study estimation of treatment effect that can inform the design of subsequent studies. Here we develop methods to estimate the effect of various treatments at the unit$\times$time-level---the smallest possible scale---for a general class of sequential experiments.

More precisely, we consider an experimental setting with finitely many treatments $\actionset$ (e.g., $\braces{0, 1}$). Let $\miss$ denote the treatment assigned to unit $\n \in[\N]$ at time $\t\in[\T]$, and let $\obs$ denote the outcome observed for unit $\n$ after the treatment was assigned at time $\t$. We assume the following model:
\begin{talign}
\label{eq:model_mc}
    \obs^{(a)} = \trueobsvar_{\n, \t}^{(a)} + \noiseobs,
    \stext{for} a\in\actionset,
    \qtext{and}
    \obs = \obs^{(\miss)},
\end{talign}
where for each treatment $a\in\actionset$, the tuple $(\obs^{(a)}, \trueobsvar_{\n, \t}^{(a)})$ denotes the counterfactual outcome (also referred to as the potential outcome) and its mean under treatment $\miss=a$, and $\noiseobs$ denotes zero-mean noise. Note that for any at a given time, only one of the outcomes is observed (denoted by $\obs$). The model~\cref{eq:model_mc} follows the classical Neyman-Rubin causal model~\cite{neyman,rubin1976}, and as stated already makes two assumptions: (i) The \emph{consistency assumption}, also known as the stable unit treatment value assumption (SUTVA), i.e., the observed outcome is equal to the potential outcome, and (ii) \emph{No delayed or spill-over effect}, i.e., the counterfactuals at time $\t$ do not depend on any past treatment.\footnote{See \cref{sub:delayed_effects} on possible extensions when spillover effects are present.} We consider settings where the treatments $\sbraces{\miss}_{\n=1}^{\N}$ are assigned using an adaptive treatment policy $\policy_{\t}$ that can depend on all the data observed until time $\t$ (hence the name sequential experiment).  In this work, we consider generic sequential policies (see \cref{assum:policy} for a formal description), e.g., it can be based on a bandit algorithm like Thompson sampling, $\epsilon$-greedy etc. Notably, we allow policies that pool the data across units during the experiment. While such pooled policies exhibit promises for effective treatment delivery during the study in noisy data settings, they introduce non-trivial challenges for after-study inference~\cite{yom2017encouraging,tomkins2021intelligentpooling}.
 
A growing line of recent work has developed methods to estimate average treatment effects in sequential experiments with stochastic adaptive policies. These works primarily leverage on knowledge of the policy and use adjustments via inverse propensity weights (the sampling probabilities), see, e.g., \cite{hadad2021confidence,zhan2021off,bibaut2021post} for off-policy evaluation, \cite{zhang2021statistical} for asymptotic properties of M-estimators in contextual bandits. These works, however, assume \iid potential outcomes---namely independent units are observed at each time (i.e., there is no joint experiment across units). \cite{NEURIPS2020_6fd86e0a,bojinov2021panel} provide strategies to estimate average treatment effects for non-stationary settings that include multiple units within an experiment and allow for non-\iid potential outcomes within units over time. Notably these works assume that the experiments are run independently for each unit so that pooling of units for policy design is \emph{not} allowed.

This work is motivated by the scientific desires of performing multiple primary and secondary analyses after-study~\citep{liu2014trading,erraqabi2017trading}, leading to our goal, namely \emph{counterfactual inference in sequential experiments} at the unit$\times$time level. In particular, we want to learn the counterfactual means  $\Theta\defeq\sbraces{\trueobs, a\in\actionset, \n\in[\N], \t\in[\T]}$ in \cref{eq:model_mc}. This goal is more ambitious than estimating the population-level mean counterfactual mean parameter $\bar{\theta}^{(a)}\!\!=\! \E[\obs^{(a)}]$ (with expectation over units, and time), since we consider heterogeneity in the means across units, time and treatments. These estimates can then be used to (i) estimate the treatment effect for a given unit at a given time (often referred to as the individual treatment effect (ITE)), the variation in ITE across units, and time, etc., (ii) to adjust the sequential policy updates in an ongoing experiment, (iii) to assess the regret of the sequential policy, and (iv) off-policy optimization for the next experiment, etc.

 \paragraph{Our contributions} Without any structural assumption on the counterfactuals $\Theta$, our task is infeasible, since the number of unknowns ($|\actionset|\N\T$) is more than the number of observations ($\N\T$). We consider a non-linear latent factor model on $\Theta$ that (i) is a (statistical) model-free generalization of a broad class of non-linear mixed effects model, (ii) allows heterogeneity in treatment effects across units and time, and (iii) makes counterfactual inference in sequential experiments feasible (see \cref{sub:model} for details). To construct the counterfactual estimates for data from a sequential experiments with a generic sequential policy, sub-Gaussian latent factors, and noise (\cref{assum:policy,assum:bilinear,assum:lambda_min,assum:iid_unit_lf,assum:zero_mean_noise}), we use a non-parametric method, namely a variant of nearest neighbors algorithm (\cref{sec:algorithm}). We establish a high probability error bound for our estimate of $\trueobs$ at any time $\t$ and for any unit $\n$ (\unittime-level; \cref{thm:anytime_bound}) that can be applied to a wide range of factor distributions. For the setting with $M$ distinct types of units (discrete units), the squared error is $\wtil{\order}(\T^{-1/2}+M/N)$ under regularity conditions (\cref{cor:anytime_bound}). We then show that under regularity conditions, our general error guarantee yields an asymptotically valid prediction interval (\cref{thm:anytime_asymp}); which admits a width of $\wtil{\order}(\T^{-1/4})$ for the setting with discrete units (\cref{cor:anytime_asymp}). We also establish asymptotic guarantees for the population-level mean counterfactual using the \unittime-level estimates (\cref{thm:ate_asymp}), with width of the intervals similar to that at the unit-level (\cref{cor:ate_asymp}). Our proof involves a novel sandwich argument to tackle the challenges introduced due to sequential nature of policy (see \cref{sub:proof_sketch}). In \cref{sec:possible_extensions}, we illustrate how our theory extends to settings with (bounded) spillover effects.
Finally, we illustrate our theory via several simulations, and a case study with data from a mobile health trial HeartSteps~\cite{liao2020personalized}.

 Overall, our work compliments the growing literature on statistical inference with adaptively collected data~\cite{hadad2021confidence,bibaut2021post,zhang2020inference,zhang2021statistical} by providing \emph{the first guarantee for counterfactual inference under a general non-parametric model that also allows for adaptive sampling policies that depend arbitrarily on all units' past observed data}. En route to establishing our guarantees, we leverage the connection between latent factor models, and matrix completion, \blue{and provide the first entry-wise guarantee for non-linear factor models, and under (sequential) missing at random (MAR) settings. Overall, our results advance the literature on factor models~\cite{bai2003inferential,bai2024likelihood,athey2021matrix} with entry-wise inference with non-linear factor models under sequential randomization as well as matrix completion results by complimenting the recent work on entry-wise guarantees for missing completely at random settings (purely exploratory experiment)~\cite{li2019nearest,chen2019inference,ieee_matrix_completion_overview} and missing not at random settings (observational study)~\cite{agarwal2021synthetic,agarwal2021causal}.}

\paragraph{Organization} We start with a description of our set-up and algorithm in \cref{sec:problem_algorithm}, followed by main results in \cref{sec:main_results}. We discuss multiple extensions of our work in \cref{sec:possible_extensions} and present our empirical vignettes of our algorithm in \cref{sec:experiments}. We conclude with a discussion of future directions in \cref{sec:discussion}. Proofs and details of experimentsare deferred to the supplement~\ox{\cite{supplement}}\oa{\cite{supplement}}.

\paragraph{Notation}
For $k \in \natural$, $[k] \defeq \braces{1, \ldots, k}$. We use bold symbols (like $\mbf N$, $\mbf T$) to denote sets, and normal font with same letters to denote the size of these sets (e.g, $N= |\mbf N|$ and $T = |\mbf T|$). We write $a_{n} = \order(b_{n})$, $a_n \ll b_n$, or $b_n \gg a_n$ to denote that there exists a constant $c\!>\!0$ such that $a_{n}\! \leq\! c b_{n}$, and $a_{n} \!=\! o(b_{n})$ when $\frac{a_{n}}{b_{n}} \!\to\! 0$ as $n \to \infty$. We use $a_n \precsim b_n$ or $b_n \succsim a_n$ to denote that $a_n = \order((\log n)^c b_n)$ for some constant $c$. We use $\wtil{O}$-notation to hide logarithmic dependencies on the underlying index set. For a sequence of real-valued random variables $X_n$, we write $X_n=\order_{P}(b_n)$ when $\frac{X_n}{b_n}$ is stochastically bounded, i.e., for all $\delta>0$, there exist finite $c_{\delta}$ and $n_{\delta}$ such that $\P(|\frac{X_n}{b_n}|\!>\!c_{\delta})\!<\!\delta$, for all $n>n_{\delta}$; and we write $X_{n} \!=\! o_p(b_{n})$ when $\frac{X_{n}}{b_{n}}\!\to\! 0$ in probability as $n \!\to\! \infty$. We often use a.s. to abbreviate almost surely.

\section{Problem set-up and algorithm}
\label{sec:problem_algorithm}

In \cref{sec:data_generation}, we formally describe the data generating mechanism and the structural assumptions on the counterfactual means, followed by the description of the nearest neighbors algorithm proposed for estimating these means (along with confidence intervals) in \cref{sec:algorithm}. We then elaborate two settings in \cref{sub:two_illustrative} that serve as running illustrative examples for our theoretical results.

Throughout this work, we use the term exogenous to denote that the quantity was generated independent of all other sources of randomness.
\vspace{-2mm}
\subsection{Sequential policy}
\label{sec:data_generation}
First, we describe how the treatments are assigned in the sequential experiment. Let $\policy_{\t}$ denote the treatment policy at time $\t$ that generates the treatments $\sbraces{\miss}_{\n=1}^{\N}$. We use 
$\history[\t\!-\!1] \defeq \braces{\obs[\n, \t'], \miss[\n, \t'], \n \in [\N], \t'\in [\t\!-\!1]}$ to denote the sigma-algebra of the data observed till time $\t\!-\!1$. Our first assumption states that the $\sbraces{\policy_{\t}}_{\t=1}^{\T}$ are sequentially adaptive, and in fact the dependence of the policy can be arbitrary on the past data:
\vspace{-2mm}
\begin{assumption}[\seqpolicy]
    \label{assum:policy}
    The treatment policy $\policy_{\t}$ at time $t$ is measurable with respect to the sigma-algebra, $\history[\t\!-\!1]$, where $\policy_{\t,\n}(a) = \P(\miss = a \vert  \history[\t\!-\!1])$ for unit $\n\in[\N]$ and treatment $a\in\actionset$. Conditioned on $\history[\t-1]$, the variables $\sbraces{\miss[\n,\t]}_{\n=1}^{\N}$ are drawn independently.
    Moreover, for $a \in \actionset$, let $\sbraces{\pmint[\t]}_{\t=1}^{\infty}$ denote a sequence of non-increasing scalars (that can decay to $0$) such that
    \begin{align}
    \label{eq:pmin}
        \inf_{\t'\leq \t, \n\in[\N]}\! \policy_{\t', \n}(a) \geq \pmint[\t]
        \qtext{almost surely for $\t \in [\T]$.}
    \end{align}
\end{assumption}
\vspace{-2mm}
\cref{assum:policy} states that the treatments at time $\t$ are assigned based on a policy that depends on the data observed till time $\t\!-\!1$; and conditional on the history, the treatments are independently assigned at time $\t$. Notably \cref{assum:policy} puts no restrictions on how the policy is updated, e.g., it could have pooled the data across units, it could be based on bandit algorithms like Thompson sampling, $\epsilon$-greedy, multiplicative weights, softmax etc, or it can come from a non-Markovian algorithm that might use past data in a non-trivial way. The term $\pmint[\t]$ captures the rate of decay in the minimum exploration of the underlying sampling policy. E.g., an $\epsilon$-greedy or pooled $\epsilon$-greedy based treatment policy, with $\epsilon$ decaying with $\t$ as $\epsilon_{\t}$, would admit $\pmint[\t]=\frac{\epsilon_{\t}}{|\actionset|}$. \cref{assum:policy} is generic and natural in several sequential experiments. In the mobile health trial considered later in \cref{sub:mobile_health_study}, for the treatment being sending a notification (or not)---whether or not user $\n$ is sent a notification at time $\t$ is independent of whether any other user $j$ is sent a notification, conditional on historical data for all users.

\paragraph{Bounded noise}
We make a standard assumption on the noise variables viz zero mean and bounded variance. To simplify the presentation, we assume bounded noise; and refer to \ox{\cref{sub:sub_gauss}}\oa{\cite[\cref{sub:sub_gauss}]{supplement}} for an extension of our theory for sub-Gaussian noise:
\vspace{-2mm}
\begin{assumption}[\zeromean]
\label{assum:zero_mean_noise}
    The noise $\sbraces{\noiseobs, \n\in[\N], \t \in [\T]}$ are exogenous, independent of each other with $\E[\noiseobs] = 0$, $\E[(\noiseobs)^2]=\sigma_{a}^2$, and $|\noiseobs| \leq \nconst$ a.s. for $a \in \mc A$.
\end{assumption}
\vspace{-2mm}
\vspace{-2mm}
\vspace{-2mm}

\subsection{Non-parametric factor model on counterfactuals}
\label{sub:model}
We impose a non-linear factor model on the collection of mean counterfactual parameters $\Theta$. Our model involves two sets of $d$-dimensional latent factors---$\ulf \defeq \sbraces{\lunit^{(a)}, a\in\actionset, \n \in[\N]}$, and $\vlf \defeq \sbraces{\ltime^{(a)}, a \in \actionset, \t\in[\T]}$. In particular, we assume that in the study, given a treatment $a\in\actionset$, for each unit we can associate a latent unit factor $\lunit^{(a)}$, and with each time we can associate a latent time factor $\ltime^{(a)}$, such that conditioned on these latent factors
\begin{align}
\label{eq:non_linear_lf}
    \trueobs = \lfun^{(a)}(\lunit^{(a)}, \ltime^{(a)}),
\end{align}
for some \emph{fixed but unknown} function $\lfun^{(a)}$. For the sequential experiments motivating our work, the unit latent factors can account for the unobserved unit-specific traits, while the latent time factors encode time-specific causes that might correspond, e.g., to unobserved societal changes. Our model allows the non-linearity, the unit and time factors to vary treatment $a$.
Putting this together with \cref{eq:model_mc}, we conclude that our counterfactual model is given by
\begin{align}
\label{eq:model_mc_complete}
    \obs^{(a)} = \lfun^{(a)}(\lunit^{(a)}, \ltime^{(a)}) + \noiseobs.
\end{align}
In particular, $\lunit^{(a)}$, and $\ltime^{(a)}$ capture the heterogeneity of counterfactuals for a given treatment $a$ across units and time respectively, while $\lfun^{(a)}$ allows for a different non-linear function across treatments.
We highlight that \cref{eq:model_mc_complete} is a flexible non-parametric model as it does not assume the knowledge of the non-linear function $\lfun^{(a)}$ and puts no parametric assumption on the distributions of latent factors and the noise variables. This functional form generalizes the classical (i) non-linear mixed effects model, and (ii) bilinear latent factor model in causal panel data, which we now discuss.
\vspace{-2mm}

\subsubsection{Generalization of prior models}
\label{sub:generalize}
We first discuss how \cref{eq:model_mc_complete} generalizes a class of non-linear mixed effects models, followed by the bilinear factor model in causal panel data.

\paragraph{Relation with non-linear mixed effects model}
Our non-linear latent factor model \cref{eq:model_mc_complete} is a (statistical) model-free generalization of non-linear mixed effect model with no observed covariates. For such a model, we note that the corresponding equation for the counterfactuals is given by~\citep[Eqn.~2.1]{vonesh1992mixed}, \citep[Eqn.~3.1]{young1997fieller}):
\begin{align}
\label{eq:mixed_effect}
    \obs^{(a)} = g(\beta, \wtil{\lunit}^{(a)}, \t, a) + \noiseobs.
\end{align}
Here the function $g$ takes a \emph{known} non-linear functional form and is parameterized by the unknown parameter $\beta$, also called the \emph{fixed effect}; $\wtil{\lunit}^{(a)}$ is a zero mean random variable specific to unit $\n$ and treatment $a$, and is referred to as unit-specific \emph{random effect} assumed be drawn \iid from a known family of distributions (typically Gaussian distribution with unknown variance); and $\noiseobs$ denotes \iid mean zero noise, that is also typically assumed to be Gaussian with unknown variance.  The noise and random effects are assumed to be drawn independent of each other. Some common examples of \cref{eq:mixed_effect} include the univariate linear mixed effects model:
\begin{align}
\label{eq:additive_re}
    g(\beta, \wtil{\lunit}^{(a)}, \t, a) = \beta_{a} +\wtil{\lunit}^{(a)} + \wtil{\ltime}^{(a)},
\end{align}
where the scalars $\wtil{\ltime}^{(a)}$ capture the variation of $g$ with time $\t$ and $a$. A multivariate linear mixed effects model is given by
\begin{align}
    \label{eq:mult_re}
    g(\beta, \wtil{\lunit}^{(a)}, \t, a)
    = \angles{\lunit^{(a)}, \ltime^{(a)}},
\end{align}
where $\lunit^{(a)}$ denotes the unit-specific random effect, and $\ltime^{(a)}$ denotes the time-specific random effect.
We highlight that \cref{eq:additive_re} can be obtained as a special case of \cref{eq:mult_re} by choosing $\lunit^{(a)} \gets (\beta_a, \wtil{\lunit}^{(a)}, 1)\tp$, and $\ltime^{(a)} \gets (1, 1, \wtil{\ltime}^{(a)})\tp$. For these models, the quantities of interest include the fixed effect (parameter) $\beta$, and the unit-specific random effect (random variable) $\wtil{\lunit}^{(a)}$ for each unit $\n$. Typical estimation strategies include Expectation-Maximization for maximum likelihood for estimating $\beta$, followed by empirical Bayes for estimating $\wtil{\lunit}^{(a)}$ when the random effects and noise are Gaussian; or Markov chain Monte Carlo for random effects and noise from a general parametric family~\cite{laird1982random,young1997fieller} (see the book~\cite{davidian2017nonlinear} for a thorough treatment for random effects models).

\paragraph{Tradeoff between assumptions and estimands} Given the discussion above, we conclude that our model~\cref{eq:model_mc_complete} is a generalization of the non-linear mixed effects model~\cref{eq:mixed_effect} since we assume \emph{no knowledge} of (i) the non-linear function $\lfun^{(a)}$, (ii) the distribution  of $\lunit^{(a)}$, and (iii) the distribution of noise $\noiseobs$. We tradeoff this increase in the flexibility of the model with a restriction on the quantities of interest from the unknown parameter and random effects to the conditional mean of counterfactuals. E.g., in the setting~\cref{eq:additive_re} instead of estimating each of the (fixed and random) effects $\beta, \wtil{\lunit}^{(a)}, \wtil{\ltime}^{(\t)}$, we directly estimate the unit-time level conditional mean $\beta_{a} +\wtil{\lunit}^{(a)} + \wtil{\ltime}^{(a)}$. Note that this change of estimand also implies that our latent factors are not necessarily mean zero. This tradeoff in the favor of a weaker assumptions with a weaker estimand is motivated by our main goal in this paper, namely estimating the \unittime-level treatment effect. This choice is well-suited for numerous applications in healthcare, behavioral science and social science,  where the functional form of treatment effect is not known due to the lack of well-established mechanistic models between the unobserved random effects and the observed outcomes. (In the mobile health study considered later in this work, the functional form of relationship between the unobserved user traits and the step-counts over time is not well-understood.) 

\paragraph{Relation with causal panel data} We note that special cases of model \cref{eq:model_mc_complete} have appeared in prior work in panel data settings, synthetic control and synthetic interventions with a bilinear $\lfun^{(a)}$ and shared unit latent factors across treatments. In particular, the synthetic control literature~\cite{abadie2,xu2017generalized,arkhangelsky2019synthetic} and the causal panel data literature~\cite{athey2021matrix,bai2021matrix} imposes a bilinear factor model as the counterfactual means under control (treatment $0$) as $\trueobsvar_{\n, \t}^{(0)} = \sangles{\lunit, \ltime}$ as their primary estimands are the control counterfactuals $\trueobsvar_{\n, \t}^{(0)} $. On the other hand, in synthetic interventions~\cite{agarwal2021synthetic}, all counterfactuals are modeled via bilinear factor model with shared unit latent factors across treatments, that is, $\trueobs = \sangles{\lunit, \ltime^{(a)}}$. Thus our model~\cref{eq:model_mc_complete} is a non-linear generalization of the counterfactual mean model assumed in these works.

We finally state our non-parametric model as a formal assumption for ease of reference throughout the paper. For simplifying the presentation of results for the most of the paper, we state most of our guarantees primarily for a bilinear latent factor model. We highlight that this assumption is \emph{not} necessary for our analysis and we present a more general result for the non-linear factor model with Lipschitz $\lfun^{(a)}$ in \cref{sub:non_linear_lf}.
\vspace{-2mm}
\begin{assumption}[\bilin]
\label{assum:bilinear}
    Conditioned on the latent factors $\ulf$ and $\vlf$, the mean counterfactuals satisfy $\trueobsvar_{\n, \t}^{(a)}\! 
    \defeq \E[\obs^{(a)}\vert \ulf,\vlf] = \!\sangles{\lunit^{(a)}, \ltime^{(a)}}$ for $a\in \actionset$.
\end{assumption}
\vspace{-2mm}
We make a final remark before describing the overall data generating mechanism.
\vspace{-2mm}
\begin{remark}[Dimensionality reduction]
   The factorization assumption~\cref{eq:model_mc_complete} implicitly serves the purpose of dimensionality reduction as it reduces the effective degrees of freedom for our counterfactual inference task with $\NT$ observed outcomes of estimating $|\actionset|\N\T$ unknowns in $\Theta$ to estimating $|\actionset|(\N+\T)d$ unknowns in $(\ulf, \vlf)$. Throughout this work, we treat the dimension $d$ of latent factors as fixed while $\N, \T$ get large.  
\end{remark}
\vspace{-2mm}
\vspace{-2mm}
\vspace{-2mm}

\subsubsection{Distributional assumptions on latent factors, and data generating mechanism}
\label{sub:lf}
Next, we state assumptions on how the latent factors are generated---put simply, we assume that all the latent factors are exogenous and drawn \iid. For simplicity, we also assume that the latent factors are bounded, and discuss how our results can be easily extended to the setting with sub-Gaussian latent factors in \ox{\cref{sub:sub_gauss}}\oa{\cite[\cref{sub:sub_gauss}]{supplement}}.
\vspace{-2mm}
\vspace{-2mm}
\begin{assumption}[\iidlf]
\label{assum:iid_unit_lf}
    The unit latent factors $\ulf \defeq \sbraces{\lunit^{(a)}, a\in\actionset}_{\n=1}^{\N}$ are exogenous, and drawn \iid from a distribution with bounded support and mean not necessarily zero, and $\Vert{\lunit^{(a)}}\Vert_2 \leq \uconst$ a.s. 
\end{assumption}
\vspace{-2mm}
\vspace{-2mm}
\vspace{-2mm}
\vspace{-2mm}
\begin{assumption}[\posdefinite]
\label{assum:lambda_min}
The time latent factors $\vlf \defeq \sbraces{\ltime^{(a)}, a\in\actionset}_{\t=1}^{\T}$ are exogenous, and for each $a \in \actionset$, they are drawn \iid from a distribution with bounded support and mean not necessarily zero, $\Vert\ltime^{(a)}\Vert_2 \leq \vconst$ a.s., and $\Sigv\defeq \E[\ltime^{(a)}(\ltime^{(a)})\tp] \succeq \lambda_{a} \mat I_{d}$ for some constants $\vconst, \lambda_a>0$. 
\end{assumption}
\vspace{-2mm}
Before discussing \cref{assum:iid_unit_lf,assum:lambda_min}, we note that putting together \cref{assum:policy,assum:zero_mean_noise,assum:lambda_min,assum:iid_unit_lf,assum:bilinear} imply the following data generating mechanism:
\begin{enumerate}[leftmargin=*]
    \item Generate $\N|\actionset|$ \iid latent factors $\ulf$, $\T|\actionset|$ \iid time latent factors $\vlf$ and $\N\T$ \iid  noise variables $\sbraces{\noiseobs}$, independently of each other. (These random variables determine the $\N\T|\actionset|$ counterfactual mean parameters $\Theta$, and the counterfactual outcomes in \cref{eq:model_mc_complete}.)
    \item Initialize the treatment policy at time $\t=1$ to some vector $\policy_1 \in [0, 1]^{\N}$.
    \item For $t=1, 2, \ldots, \T$:
    \begin{enumerate}[label=(\roman*),leftmargin=*]
        \item Generate treatments $\miss\sim\mrm{Multinomial}(\policy_{\t,\n})$, independently across all units $\n \in [\N]$.
        \item Observe noisy outcomes $\sbraces{\obs =\trueobsvar^{(\miss)}_{\n, \t} + \noiseobs}_{\n=1}^{\N}$ are observed for all units $\n \in [\N]$.
        \item Update the treatment policy to $\policy_{\t+1} \in [0, 1]^{\N}$ using all the observed data so far.
    \end{enumerate}
\end{enumerate}
This mechanism serves as a simplified but representative setting for several sequential experiments, especially the ones on digital platforms (like mobile health, online education).

Now we discuss \cref{assum:iid_unit_lf,assum:lambda_min}. The \iid assumption on the unit latent factors (\cref{assum:iid_unit_lf}) is a reasonable assumption when the units in the experiment are drawn in an \iid manner from some superpopulation. This assumption is standard in the literature on experiments (as well as its analog that the unit-specific random effects are drawn \iid in the literature on mixed-effects models). \blue{Such a distributional assumption on the unit factors allows us to easily characterize the neighborhood (number of neighbors, smoothness in how that number varies) around a fixed unit factor as a function of the distance from that unit factor. These distributional assumptions can be relaxed and instead regularity conditions that provide a tight control on the number of neighbors as a function of radius.}

The \iid assumption on the time latent factors however can be a bit restrictive, as the exogenous time-varying causes might be correlated. We discuss possible extensions to such settings in \ox{\cref{sub:non_iid}}\oa{\cite[\cref{sub:non_iid}]{supplement}}. Moreover, the positive definiteness of $\Sigv$ can also be relaxed (\ox{\cref{sub:non_psd}}\oa{\cite[\cref{sub:non_psd}]{supplement}}).
The boundedness assumption on latent factors is for simplicity, and can be extended to sub-Gaussian latent factors; see \ox{\cref{sub:sub_gauss}}\oa{\cite[\cref{sub:sub_gauss}]{supplement}}. Finally, we remark that even though unit and time latent factors are \iid, the counterfactual means in $\Theta$ are not \iid, and are correlated across units and time. E.g., for unit $\n$, the mean parameters $\sbraces{\trueobs[\n, \t]}_{\t=1}^{\T}$ across time are coupled via the unit latent factor $\lunit$. We discuss two illustrative examples of latent factor distributions in \cref{sub:two_illustrative}.

\begin{figure}[t]
    \centering
    \resizebox{\textwidth}{!}{
    \begin{tabular}{cc}
    \includegraphics[width=0.51\linewidth,trim={0 3.5cm 0 0cm},clip]{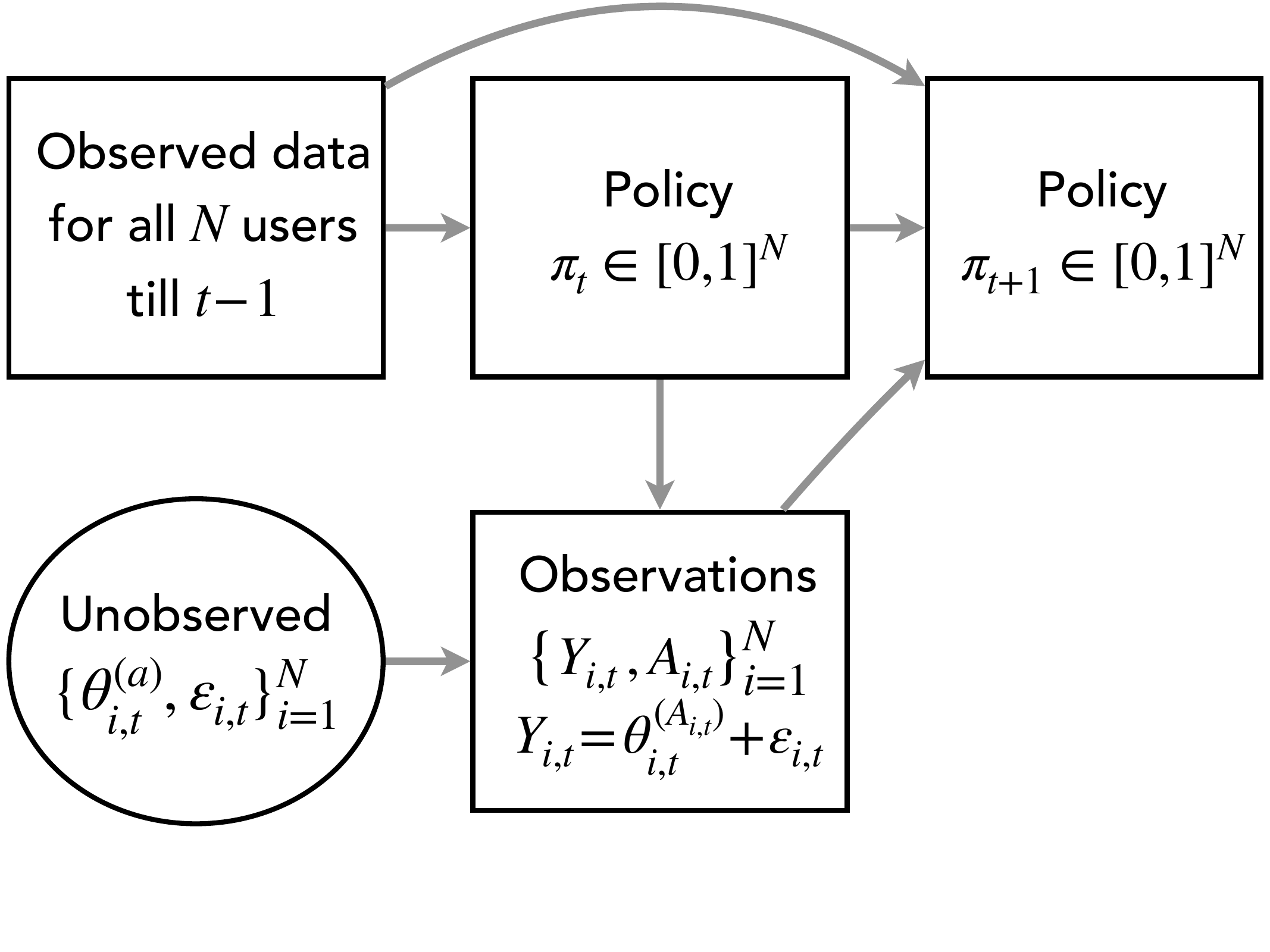} &
    \includegraphics[width=0.44\linewidth,trim={0 0cm 6cm 0},clip]{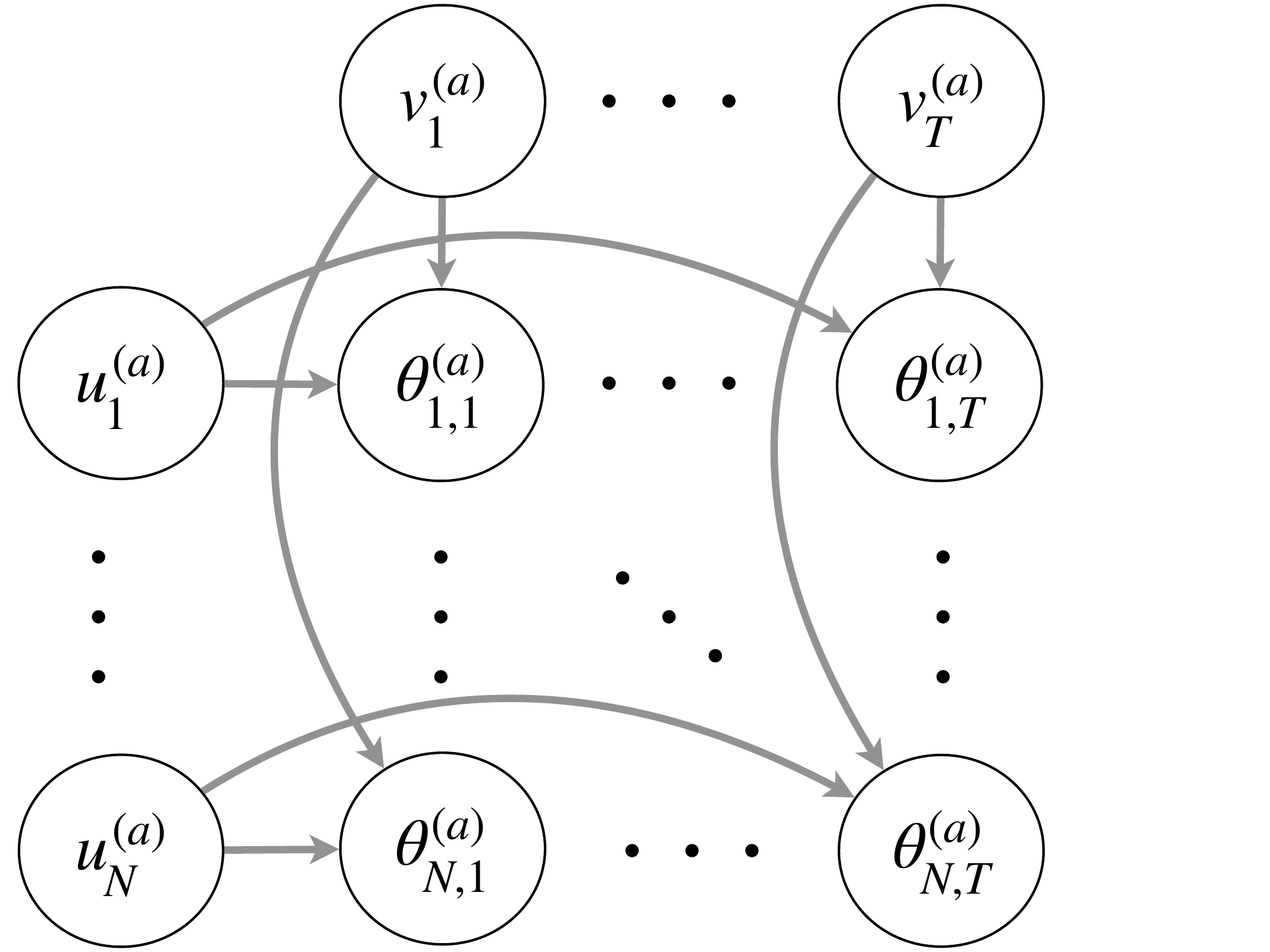}  \\
    (a) Sequential policy and observations
    &
    (b) Latent factor model on counterfactual means
    \end{tabular}
    }
    \caption{Graphical model for the sequential experiment design considered here. Circles and rectangles denote unobserved and observed variables respectively. \mbox{\tbf{Panel (a)}} illustrates the observation model~\cref{eq:model_mc}, and the sequential nature of policy (\cref{assum:policy}).
    \mbox{\tbf{Panel (b)}} visualizes the latent factor model assumed on the counterfactual mean parameters $\trueobs$ in terms of the (treatment-specific) unit and time latent factors $\lunit^{(a)}$ and $\ltime^{(a)}$ (\cref{assum:bilinear,eq:non_linear_lf}).
    }

    \label{fig:sample_split_graph}
    \vspace{-3mm}
\end{figure}

\subsubsection{Estimands and terminology}
\label{sub:estimands}
Our main goal is estimating the \unittime-level counterfactual mean $\trueobs$ (conditioned on $\lunit^{(a)}, \ltime^{(a)}$) for an arbitrary tuple $(\n, \t, a)$, and we are interested in quantifying non-asymptotic performance of our estimate as well as constructing asymptotically valid prediction intervals for the estimand.
Our secondary goal is to estimate the population-level counterfactual mean---which due to \cref{assum:bilinear} can be rewritten as $\trueybarall[]  \defeq \E[\lunit\tp\ltime^{(a)}]$ for $a \in \actionset$---and establish asymptotic consistency of the estimate and asymptotic validity of the constructed confidence intervals. Note that while $\trueobs$ is a random variable akin to the \unittime-level random effect in \cref{eq:additive_re}, $\trueybarall[]$ is a parameter akin to the fixed effect in \cref{eq:additive_re}. Consequently, we use the terms \emph{prediction interval} and \emph{confidence interval} for the intervals constructed to quantify the uncertainty in the estimation of $\trueobs$, and $\trueybarall[]$ respectively. Moreover, it is conventional to refer estimating the random variable $\trueobs$ as a prediction problem, and the corresponding estimate as a predictor. However, for consistency in the language, we continue to use the terms \emph{estimation} and \emph{inference} for our task of estimating both these quantities, and quantifying the associated uncertainty.

\subsection{Algorithm}
\label{sec:algorithm}
We solve the task at hand via a non-parameteric method, namely a variant of nearest neighbors. To estimate $\trueobs$ for a given tuple $(\n, \t, a)$, our nearest neighbors estimate is constructed in two simple steps: (i) use all the available data under treatment $a$ to estimate a set of nearest neighbors for $\n$, and (ii) compute an average of the observations across these neighbors that also have $\miss=a$. To perform step~(i), we make use of a hyperparameter $\threshold$ that needs to tuned (that we discuss in \ox{\cref{sec:exp_details}}\oa{\cite[\cref{sec:exp_details}]{supplement}}). We now state the algorithm for construction our estimates and uncertainty intervals.

\subsubsection{Estimates and uncertainty quantification}
\label{sub:estimates}
We start with the unit level estimates, and then describe the population-level estimates.

    \paragraph{Estimate and prediction interval for $\trueobs$} For any $\n, j \in [\N]$ and $a\in\actionset$ define
    \begin{align}
    \label{eq:time_nbr_dist}
    \estdist \defeq \displaystyle  \frac{\sum_{s=1}^{\T}\indicator(\miss[\n, s]=a) \indicator(\miss[j, s]=a) (\obs[\n, s]\!-\!\obs[j, s])^2}{\sum_{s=1}^{\T} \indicator(\miss[\n, s]=a) \indicator(\miss[j, s]=a)}.
    \end{align}
    In words, $\estdist$ denotes the mean squared distance of observed outcomes over time points when both unit $\n$ and $j$ were assigned treatment $a$. Next, the set of available nearest neighbors for the tuple $(\n, \t, a, \threshold)$ is defined as
     \begin{align}
        \label{eq:reliable_nbr}
        \estnbr \defeq 
        \braces{j \in [\N] \backslash\braces{\n} \suchthat \estdist \leq \threshold, 
        \miss[j, \t] = a}.
    \end{align}
    Denote $\nestnbr = |\estnbr|$. Our estimate for $\trueobs$ is given by a simple outcome average:
    \begin{align}
    \label{eq:obs_estimate}
    \stext{if} 
         \nestnbr > 0
         \quad 
        \estobs \defeq 
         \displaystyle\frac{1}{\nestnbr} \sum_{j \in 
         \estnbr[\n, \t, a, \threshold]} \obs[j, \t].
    \end{align}
    When the number of available neighbors $\nestnbr=0$, the estimate is as follows:
    \begin{align}
    \label{eq:no_nbr_est}
      \stext{if} \miss=a \quad \estobs = \obs, \qtext{else}
       \estobs = \frac{\sum_{j=1}^{\N}\indicator(\miss=a)\obs[j, \t]}{\sum_{j=1}^{\N}\indicator(\miss=a)}.
    \end{align}
    That is, when $\nestnbr=0$, we return the observed outcome if $\miss=a$, else we simply return the averaged over all observed outcomes corresponding to treatment $a$ at time $\t$. Notably, our guarantees are established for \cref{eq:obs_estimate} (and not \cref{eq:no_nbr_est}), and we establish regularity conditions such that $\nestnbr=0$ happens with small probability.\footnote{\blue{While theoretically unit $\n$ is allowed to be its own neighbor, in practice we do not include it since the test/validation error in real settings is measured with respect to held out observations and is severely impacted if we allow the unit to be its own neighbor.}} \blue{Notably, the estimates above do not require any knowledge of the sequential policy $\policy_t$ although when the policy is known, one can instead consider an inverse-propensity-weighted version (see \ox{\cref{sub:non_iid}}\oa{\cite[\cref{sub:non_iid}]{supplement}}).}

    Next, we construct the prediction interval for $\trueobs$ conditioned on $(\lunit^{(a)}, \ltime^{(a)})$. For a given level $\alpha\!>\!0$, our estimate for the $(1\!-\!\alpha)$ confidence interval for $\trueobs$ is provided only if $\nestnbr\!>\!0$ and is given by
    \begin{align}
    \label{eq:unit_ci}
        \biggparenth{\estobs \!\!-\!\! \frac{\zalpha\,\what{\sigma}_{a}}{\sqrt{\nestnbr}},\,\,
        \estobs \!\!+\!\! \frac{\zalpha\,\what{\sigma}_{a}}{\sqrt{\nestnbr}}
        },
    \end{align}
    where $z_{\frac{\alpha}{2}}$ denotes the $1\!-\!\frac{\alpha}{2}$ quantile of standard normal random variable, and $\what{\sigma}_{a}^2$ is the estimated noise variance (see \ox{\cref{sec:exp_details}}\oa{\cite[\cref{sec:exp_details}]{supplement}}).

    \paragraph{Estimate and confidence interval for $\trueybarall[]$}
    Our estimate for the counterfactual mean averaged over all units and time is
    \begin{align}
    \label{eq:ybart_estimate}
         \estybar[\threshold] \!\defeq\! \frac{1}{\N\T} \sum_{\t=1}^{\T}\sum_{\n=1}^{\N} \bigparenth{\indicator(\miss=a) \obs \!+\! \indicator(\miss\neq a)   \estobs[\n, \t, \threshold]}.
    \end{align}
    Finally, to construct the $(1-\alpha)$ confidence interval for $\trueybarall[]$, we randomly sub-sample $K$ indices (denoted by $\mbf K$) from the set $[\N] \times [\T]$ and construct the interval
    \begin{align}
        \label{eq:sigma_nn}
        &\biggparenth{\estybaralltwo[\threshold, K] \!-\!  \frac{\zalpha\,\what{\sigma}_{a,\mbf{K}}}{\sqrt{K}},\,\, \estybaralltwo[\threshold, K] \!+\!  \frac{\zalpha\,\what{\sigma}_{a,\mbf{K}}}{\sqrt{K}} },
        \qtext{where} \\
        &\estybaralltwo[\threshold, K] \!\defeq\! \frac{1}{K}\! \!\sum_{(\n,\t)\in\mbf K} \!\!\estobs[\n, \t, \threshold]
        \qtext{and}
        \what{\sigma}^2_{a,\mbf{K}}\!\defeq\! \frac{1}{K\!-\!1}\!\! \sum_{(\n,\t)\in\mbf K} \!\!(\estobs[\n, \t, \threshold]\!-\!\estybaralltwo[\threshold, K])^2.
    \end{align}

\subsection{Two illustrative examples}
\label{sub:two_illustrative}
Our results for the nearest neighbor algorithm rely on the probability of sampling a nearest neighbor for a given unit. 

Given a latent unit factor $\lunit[]$ under treatment $a$, we define the function $\probparam_{\lunit[], a}: \real_{+} \to [0, 1]$ that characterizes the probability of sampling a neighboring unit of $\lunit[]$:
\begin{align}
\label{def:phi}
   \probparam_{\lunit[], a}(r) &\defeq  \P\brackets{ (\lunit[]'-\lunit[])\tp \Sigv (\lunit[]'-\lunit[]) \leq r},
   \qtext{for}r \geq 0,
\end{align}
where $\Sigv$ is the covariance matrix of the latent time factors as defined in \cref{assum:lambda_min}, and the probability is taken only over the randomness in the draw of the unit factor $\lunit[]'$ from the corresponding distribution assumed in \cref{assum:iid_unit_lf}.

Next, we describe two distinct examples for the latent factor distributions that are later used to illustrate our theoretical results. We assume $\Sigv = \mat I_{d}$ to simplify calculations:
\begin{example}[\examplediscrete]
    \label{example:finite}
    Here, we consider a uniform latent unit distribution that over a finite set of $M$ distinct $d$-dimensional vectors (where $M$ can scale with $\T$; see \cref{item:finite_asymp}) so that for any $\lunit[]$ in the set, $\probparam_{\lunit[], a}(0) = \frac1M$.
\end{example}

\begin{example}[\examplecontinuous]
    \label{example:continuous}
    In this setting, the latent unit factors are distributed uniformly over $[0, 1]^d$ so that $\probparam_{\lunit[], a}(\threshold) \asymp \threshold^{\frac{d}{2}}$ for all $\threshold \leq 1$, and any $\lunit[] \in [0, 1]^{d}$.
\end{example}

These two examples cover two different classes of models. The set-up in \cref{example:finite} is analogous to a parametric model with degree of freedom $M$, while the set-up in \cref{example:continuous} is representative of a non-parametric model  in $d$ dimensions (the uniformity of distribution in these examples is to simplify computations; see \cref{rem:non_unif}).

\section{Main results}
\label{sec:main_results}

    In this section, we present our main results for counterfactual inference. We begin with non-asymptotic guarantee at \unittime-level (\cref{thm:anytime_bound}) in \cref{sub:non_asymp}, followed by an asymptotic version (\cref{thm:anytime_asymp}) in \cref{sub:anytime_asymp}. We discuss asymptotic guarantees at the population-level (\cref{thm:ate_asymp}) in \cref{sub:ate_asymp}. We conclude this section with a proof sketch of \cref{thm:anytime_bound} in \cref{sub:proof_sketch}, while deferring the discussion on variance estimate to \ox{\cref{sub:variance_estimate}}\oa{\cite[\cref{sub:variance_estimate}]{supplement}.}

    \subsection{Non-asymptotic guarantee at the \unittime-level}
    \label{sub:non_asymp}
     We use the following shorthands to simplify the presentation of our results:
    \begin{align}
        \label{eq:eta_chi}
        \threshold' \defeq \threshold\!-\!2\sigma_{a}^2\!\!-\!\!\errterm
        \qtext{where}
        \errterm \!\defeq\! \frac{4 ( \vconst \uconst \!+\! \nconst)^2\!\sqrt{2\log(\!\frac{4\N\T}{\delta}\!)} } {\pmint \sqrt{\T} }.
    \end{align}
    Recall the definition~\cref{def:phi} of $\ppau$, and note $C$ denotes a universal constant. We are now well-equipped to state our first main result for \unittime-level estimates.
    \begin{theorem}[\anytimeboundname]
    \label{thm:anytime_bound}
    Consider a sequential experiment with $\N$ units and $\T$ time points satisfying \cref{assum:policy,assum:zero_mean_noise,assum:bilinear,assum:lambda_min,assum:iid_unit_lf}. Given a fixed $\delta \in (0, 1)$, suppose that for the sequential policy $ \pmint \!\geq\!\T^{-\half} \sqrt{8\log(2/\delta)}$. Then for any fixed unit with latent factor $\lunit$, any fixed $\t \in [\T]$, $a \in \actionset$, and any fixed scalar $\threshold$ such that $\pmint[\t](\N\!-\!1) \ppau(\threshold') \geq 1$, conditional on $(\lunit^{(a)}, \ltime^{(a)})$ the nearest neighbor estimate $\estobs$~\cref{eq:obs_estimate} satisfies
    \begin{align}
    \label{eq:nn_bnd}
        (\estobs-\trueobs)^2 &\leq 2(\wtil{\mbb B} + \wtil{\mbb V}) \qtext{where}
        \wtil{\mbb B} \defeq \frac{\vconst^2}{\lambda_{a}}
    \parenth{\threshold -2\sigma_{a}^2} +  \frac{\vconst^2\errterm}{\lambda_{a}},
    \\ 
    \wtil{\mbb V} &\defeq 
    \frac{24\sigma_{a}^2\log(\frac{8}{\delta}) \max\braces{1,\, \frac{4\nconst^2\log(\frac{8}{\delta})/\sigma_{a}^2}{3\pmint[\t]\ppau( \threshold')\N}}}{\pmint[\t] \ppau(\threshold') (\N\!-\!1)}
     \\
        &\qquad+ 
         \frac{72\nconst^2}{\pmint[\t]^2} \max\braces{\parenth{\frac{\ppau(\threshold'\!+\!2\errterm)}{\ppau(\threshold')}\!-\!1}^2, \frac{C\log^2(\frac{8}{\delta})}{(\ppau(\threshold') (\N\!-\!1))^2}},
        \label{eq:nn_bnd_v}
\end{align}
    and $\ppau \!\defeq\! \probparam_{\lunit^{(a)}, a}$, 
    with probability at least $1\!-\!\delta\!-\!2e^{-\frac{\pmint[\t]\ppau( \threshold')(\N\squash{-}1)}{16}}$.
    \end{theorem}
    See \ox{\cref{sec:proof_of_anytime_bound}}\oa{~\cite[\cref{sec:proof_of_anytime_bound}]{supplement}} for its proof.
    \cref{thm:anytime_bound} provides a high probability error bound on the nearest neighbor estimate \blue{for generic sequential experiments, while leveraging entire data. Before we return to discussing the consequences of such a generic guarantee (in \cref{subsub:robustness}), let us} first unpack this general guarantee. We start with a corollary for the two examples discussed earlier~\cref{example:finite,example:continuous} with a scaling of $\pmint[\T] = \Omega(\T^{-\beta})$ for $\beta \in[0, \half)$, i.e., the minimum exploration is decaying with a rate bounded below by $\T^{-\beta}$ at time $\T$. E.g., in an $\epsilon$-greedy or pooled $\epsilon$-greedy with $\epsilon$ decaying with $\t$ as $\epsilon_{\t}$, this scaling asserts that $\frac{\epsilon_{\T}}{|\actionset|} = \Omega(\T^{-\beta})$ and thus as $\beta$ increases, $\epsilon_{\t}$ decreases to $0$ more rapidly, which translates to faster decay in the exploration rate, and faster increase in the exploitation rate over time.
    This corollary (proof in \ox{\cref{sub:proof_of_cor:anytime_bound}}\oa{\cite[\cref{sub:proof_of_cor:anytime_bound}]{supplement}}) treats $\t$ as fixed (while $\T$ gets large) so that $\pmint[\t]$ is a constant (while $\pmint[\T]$ does decay to $0$), and tracks dependency of the error only on $(\N, \T)$ in terms of $\wtil{O}$ notation that hides logarithmic dependencies on $\N, \T$, and constants that depend on $(d, \sigma_{a}^2, \uconst,\vconst,\nconst)$.
    \begin{corollary}
        \label{cor:anytime_bound}
        Consider a sequential experiment satisfying the assumptions of \cref{thm:anytime_bound}, such that the minimum exploration satisfies $\pmint[\T] = \Omega(\T^{-\beta})$ for some $\beta \in[0, \half)$. Let $\N, \T$ are large such that the maximum for each of the last two terms in \cref{eq:nn_bnd} is achieved by the respective first argument, then the following statements hold true.
        \begin{enumerate}[label=(\alph*)]
            \item \label{item:finite} For the setting in \cref{example:finite}, if $\N$ and $\T$  are large enough such that $\errterm = o(\min_{\lunit[] \neq \lunit[]'} \norm{\lunit[]\!-\!\lunit[]'}^2_{2})$, then for $\threshold = 2\sigma_{a}^2 + \errterm$, with probability at least $1-\delta$, we have
            \begin{align}
            \label{eq:finite_non_asymp_bnd}
               (\estobs-\trueobs)^2 =  \wtil{\order}\parenth{\frac{1}{\sqrt{\T^{1-2\beta}}}\!\!+ \frac{M}{\N}}.
            \end{align}
            \item \label{item:highd} For the setting in \cref{example:continuous}, with $\threshold=2\sigma_{a}^2 + (1+\T^{\frac{\omega}{2}})\errterm$ and a suitably chosen $\omega \in (0, 1)$, with probability at least $1-\delta$, we have
            \begin{align}
                \label{eq:highd_non_asymp_bnd}
            (\estobs-\trueobs)^2 = 
            \begin{cases}
               \wtil{\order}\parenth{\T^{-\frac{1-2\beta}{3}}} \ \stext{if} \N^{\frac{3}{d+2}} \gg \sqrt{\T^{1-2\beta}},\\[2mm]
               \wtil{\order}\parenth{\N^{-\frac2{d+2}}}\, \ \stext{otherwise.}
           \end{cases}
            \end{align}
        \end{enumerate}
    \end{corollary}
    When there are finitely many unit factors, \cref{cor:anytime_bound}\cref{item:finite} recovers the parametric rate of order $M/\N$ when $\T$ is large. On the other hand, when $\T$ is small, the error rate is of order $\T^{-\half}$, when the exploration rate does not decay ($\beta=0$). For continuous factors, \cref{cor:anytime_bound}\cref{item:highd} demonstrates that our non-parametric approach suffers from the curse of dimensionality and would be effective only when the latent factors lie in low-dimensions. Surprisingly, \cref{cor:anytime_bound} recovers the same error rate with an arbitrary sequential policy satisfying \cref{assum:policy,eq:pmin} as that for a non-adaptive, constant policy that satisfies $\policy_{\t, \n}(a) \equiv \pmint$ for all $\n,\t$ (also see \cref{rem:mcar_comparison}).

    \subsection{Discussion on the generality of \cref{thm:anytime_bound}}
    We make a few remarks about the robustness (\cref{subsub:robustness}), bias-variance-adaptivity tradeoff (\cref{sub:bias_variance_tradeoff_eta}) for nearest neighbor in our setting followed by a discussion on the assumptions on latent factor distributions (\cref{rem:non_unif}).

    \subsubsection{Robustness of (vanilla) nearest neighbors}
    \label{subsub:robustness}
    Recall that our nearest neighbors algorithm is both model and policy agnostic, i.e., it requires no knowledge of the latent factor model (and dimension $d$) and the treatment assignment policies. Yet our guarantees apply to a wide range of factor models and sequential policies. 
    Overall  \cref{thm:anytime_bound,cor:anytime_bound} (and the results and experiments in the sequel), perhaps surprisingly, immediately reveal two remarkable robustness properties of nearest neighbors.
    
     First, as long as the policies have some sufficient exploration (which can decay over time) and the neighborhood around the unit in consideration is smooth, (vanilla) nearest neighbors already provides non-trivial estimates. Such a property is revealed by two technical innovations in our proof: (i) a martingale-based analysis to show the concentration of distance and (ii) a careful sandwich argument to show that the neighbors estimated at time $t$ are not too biased despite using observations from time points $>t$. \cref{sub:bias_variance_tradeoff_eta} elaborates the bias-variance tradeoff as well as the additional tradeoff introduced by the bias arising due to adaptivity as a function of $\threshold$. Also, see \cref{sub:proof_sketch} for a sketch of the technical argument that allows us to establish a non-trivial error for the nearest neighbors.

    Second, note that while the guarantees in \cref{cor:anytime_bound} depend on the latent dimension $d$ of the factors, running the NN algorithm requires no knowledge about $d$. Put simply, \cref{thm:anytime_bound,cor:anytime_bound} demonstrate that the error guarantee of NN is adaptive to the underlying smoothness structures in the latent factors. In practice, while cross-validation is an easy way to gauge the performance of NN, one may additionally do a spectral value decomposition (see \cref{sub:mobile_health_study,fig:eda_plot}) of the outcome matrix and check if the eigenvalues decay fast enough to provide a sanity check on whether or not $d$ can be assumed to be small.

    \subsubsection{The bias-variance and \blue{adaptivity} tradeoff in \cref{thm:anytime_bound}} 
    \label{sub:bias_variance_tradeoff_eta} The general error bound~\cref{eq:nn_bnd} comprises of four terms: (i) The first term of order $\threshold-2\sigma_{a}^2$ denotes the squared bias arising due to the threshold hyper-parameter $\threshold$ and it decays to $0$ as $\threshold$ decreases to the noise variance $2\sigma_{a}^2$. (ii) The second term scales as $\wtil{\order}(1/(\pmint\sqrt{\T}))$ and denotes the second source of bias arising due to imperfect estimation of distance $\estdist$ between two units using
    common treatment points that are lower bounded with high probability by $\pmint^2\T$.\footnote{The bound for this term is tight (up to log-factors) without further assumptions, as can be verified with a non-adaptive sampling policy that satisfies $\policy_{\t, \n}(a) \equiv \pmint[\T]$.}
   . (iii) The third term corresponds to the effective noise variance due to the averaging over ``good set'' of neighbors that can be lower bounded by $\pmint\ppau(\threshold') (\N\!-\!1)$ with high probability. 
   (iv) The last term denotes the worst-case inflation in noise variance---due to the sequentially adaptive policy, since the neighbors $\estnbr$ are estimated using entire data, the noise $\sbraces{\noiseobs}$ at time $\t$ does not remain independent of the estimated set of neighbors.\footnote{Indeed, this term would be zero if the policy $\policy_{\t}$ was non-adaptive.} 

     Overall, the first term of order $\threshold-2\sigma_a^2$ in $\wtil{\mbb{B}}$~\cref{eq:nn_bnd} and the first term of order $1/(\pmint\ppau(\threshold') (\N\!-\!1))$ in $\wtil{\mbb V}$~\cref{eq:nn_bnd_v} characterize the natural bias-variance tradeoff for our estimator as a function of $\threshold$---as $\threshold$ gets large, the bias increases since the distance of neighboring units for unit $\n$ increases, but the variance decreases since the number of neighboring units increases. The second term appearing in the definition of $\wtil{\mbb{B}}$ does not depend on $\threshold$, has a very mild logarithmic dependence on $\N$, and decreases polynomially with $\T$ for sampling policies that have a suitable decay rate in their exploration; this bias gets worse as the minimum exploration rate of the policy decays faster.

     The second term appearing in \cref{eq:nn_bnd_v} introduces another tradeoff for our estimator---and this term arises purely in the presence of adaptive/sequential policies (see \cref{rem:mcar_comparison}). This term is small if the neighbor sampling probability function $\probparam_{\lunit[], a}(r)$ varies smoothly with $r$ around $r=\threshold'$; and is thereby small whenever $\threshold$ is large. This term captures the worst-case inflation in error caused due to the fact that while constructing our estimates at time $t$, we used the data from future time points ($>t$) which is correlated with the observed data at time $\t$ due to sequentially adaptive policies. Put simply, despite sequential adaptivity (and despite any adaptive/importance sampling-based reweighting), we control this error by sandwiching the estimated set of neighbors in an annulus induced by two close radii. So if the number of neighbors does not change too abruptly around the considered value of $\threshold$, this worst-case inflation in error due to adaptivity can be well controlled.

    \begin{remark}[Non-adaptive policies]
        \label{rem:mcar_comparison}        
        If the policy is non-adaptive and constant, i.e., $\policy_{\t,\n}(a)\equiv p$ for some constant $p$, \blue{our analysis can be easily adapted so that the last term in \cref{eq:nn_bnd_v} is actually zero}. This adaptation of \cref{thm:anytime_bound} recovers the prior result \citep[Thm.~1]{li2019nearest} (up to constants) that establishes an entry-wise expected error bound for nearest neighbors for matrix completion with entries missing completely at random (MCAR) with probability $p$.
    \end{remark}
    \vspace{-2mm}
    \vspace{-2mm}
    \vspace{-2mm}
    \vspace{-2mm}
    \begin{remark}[No neighbors means no guarantee]
        \label{rem:no_nbr}
        \cref{thm:anytime_bound} applies to the estimate~\cref{eq:obs_estimate} and is meaningful when there are enough number of neighbors and $\threshold$ is well tuned. It does not provide any guarantee for the estimate \cref{eq:no_nbr_est}, i.e., when there are no neighbors. Indeed, if any of $(\threshold, \N,\T, \pmint[\t])$ are too small, we may have either no neighbors, or no reliable neighbors, in which case the bound \cref{eq:nn_bnd} (or the probability of its validity) would be vacuous.
    \end{remark}

    \subsubsection{Non-uniform factor distributions}
    \label{rem:non_unif}
      The uniformity assumption in \cref{example:finite,example:continuous} is for convenience in simplifying the presentation of \cref{cor:anytime_bound}). For instance, if probability mass function for unit factors in \cref{example:finite} is non-uniform, we can simply replace $M$ by the minimum value of the probability mass function in the associated expressions in \cref{cor:anytime_bound}\cref{item:finite}. Similarly, when the distribution in non-uniform with bounded support in \cref{example:continuous}, the conclusion of \cref{cor:anytime_bound}\cref{item:highd} continue to hold as long as the density remains bounded away from $0$---up to constants that depend on the density, and dimension $d$. (Also see the next paragraph for distributions that are not uniformly lower-bounded;  and \ox{\cref{sub:sub_gauss}}\oa{\cite[\cref{sub:sub_gauss}]{supplement}} for sub-Gauss distributions.) When the matrix $\Sigv$ is not identity, the associated constants would scale with dimension $d$, the condition number, and the minimum eigenvalue of the matrix $\Sigv$---so that the conclusions continue to hold when all these quantities are treated as constants as $\N$ and $\T$ grow.   
    When the distribution or the density of the latent unit factors is not lower bounded, i.e., $\inf_u\probparam_{\lunit[], a}=0$, the arguments from the \cite[proof of Cor.~3]{li2019nearest} show that for a general class of distributions, $\probparam_{\lunit^{(a)}, a}$ can be lower bounded suitably with high probability over the draw of $\lunit$. For this setting, our proof with minor modifications yields that (a suitably modified version of) the guarantee in \cref{thm:anytime_bound} yields an analogous high probability error bound for a uniformly drawn estimate from the set $\sbraces{\estobs}_{\n=1}^{\N}$ across units for a fixed time $\t$.

    \subsection{Asymptotic guarantee at the \unittime-level}
    \label{sub:anytime_asymp}

    We now turn attention to our asymptotic guarantees for the estimate $\estobs$. In this section, we establish (i) an asymptotic consistency guarantee, and (ii) an asymptotic normality guarantee (that can be used to construct confidence intervals). These guarantees are established for a sequence of experiments indexed by the number of time points $\T$ in the experiment, such that the
     $\T$-th experiment is run with $\NT$ units for $\T$ time points, and the nearest neighbor estimates are constructed with hyper-parameter $\threshold_{\T}$. Our guarantees are established under suitable scalings of $(\threshold_{\T}, \NT, \T)$ with appropriate dependence on the nearest neighbor probability $\probparam_{\lunit[], a}$ defined in \cref{def:phi}. To simplify the presentation, we use the following shorthand:
    \begin{align}
    \label{eq:errtwo}
        \errtwo \defeq \frac{8(\vconst\uconst+\nconst)^2\sqrt{\log(\NT\T)}}{\pmint\sqrt{\T}}.
    \end{align}
    For our result, we assume that there exists a common sequence of scalars $\sbraces{\pmint[\t]}_{\t=1}^{\infty}$ that satisfies \cref{eq:pmin} for all the experiments. We treat $\pmint[\t]$ for a fixed $\t$ as a constant (i.e., it does not scale with $\T$). Moreover, to establish asymptotic normality of our estimator, we make use of a suitable sequence  $\sbraces{\LT}_{\T=1}^{\infty}$ that diverges to $\infty$, and serves as a deterministic upper bound on the number of neighbors $\nestnbr[\n, \t, \threshold_{\T}]$ in the $\T$-th experiment. This constraint can be easily enforced by randomly sub-sampling the available nearest neighbors whenever $\nestnbr[\n, \t, \threshold_{\T}]$ is larger than $\LT$, and then averaging over the sampled subset in \cref{eq:obs_estimate} to compute $\estobs$.     We now state our asymptotic guarantee (see \ox{\cref{sec:proof_of_thm_anythime_asymp}}\oa{\cite[\cref{sec:proof_of_thm_anythime_asymp}]{supplement}} for the proof).

    \begin{theorem}[\anytimeasympname]
    \label{thm:anytime_asymp}
    Given a tuple $(\n, \t, a)$, consider a series of sequential experiments with $N_{\T}$ units and $T$ time points indexed by $\T \in \mbb N$ satisfying \cref{assum:policy,assum:zero_mean_noise,assum:bilinear,assum:lambda_min,assum:iid_unit_lf} that include the unit~$\n$ with fixed latent factor $\lunit$, time latent factor $\ltime^{(a)}$ at time $\t$, and policy satisfying $ \pmint \!\geq\!\T^{-\half} \sqrt{8\log(2\T)}$. For any fixed $t \in \mbb \N$, let $\estobs[\n, \t, \threshold_{\T}]$ denote the nearest neighbor estimate~\cref{eq:obs_estimate} with threshold $\threshold_{\T} \geq 2\sigma_{a}^2 + \errtwo$ in the $\T$-th experiment, and let $\thresonet \defeq \threshold_{\T}-2\sigma_{a}^2\!-\errtwo$, and $\threstwot \defeq \thresonet+2\errtwo$ (where $\errtwo$ is defined in \cref{eq:errtwo}).

    \begin{enumerate}[label=(\alph*), leftmargin=*]
        \item\label{item:consistency} Asymptotic consistency: If the sequence $\braces{\threshold_{\T},\NT, \pmint}_{\T=1}^{\infty}$ satisfies
         \begin{align}
         \label{eq:regularity_consistency}
            \max\braces{\frac{\log \T}{\NT \pparamu(\thresonet)}, \,
            \threstwot} \seq{(a)} o_{\T}(1), \qtext{and}
            \frac{\pparamu(\threstwot)}{\pparamu(\thresonet)}\!-\!1 \seq{(b)} o_{\T}(1),
        \end{align}
        then $\estobs[\n, \t, \threshold_{\T}] \stackrel{p}{\longrightarrow} \trueobs$ as $\T\to \infty$.
        \item\label{item:clt} Asymptotic normality: If the number of nearest neighbors in \cref{eq:obs_estimate} is capped at $\LT$, and the sequence $\braces{\threshold_{\T}, \LT, \NT, \pmint}_{\T=1}^{\infty}$ satisfies $\LT\to\infty$,
        \begin{align}
        \label{eq:regularity_clt}
            \max\braces{\frac{1}{\NT\pparamu(\thresonet)},
            \LT\threstwot
            }
             \!\seq{(a)}\! o_{\T}(1), 
            \qtext{and}
            \sqrt{\LT} \abss{\frac{\pparamu(\threstwot)}{\pparamu(\thresonet)}\!-\!\!1}
            \!\!\seq{(b)}\! o_{\T}(1),
        \end{align}
        then $\sqrt{\min\sbraces{\nestnbr[\n, \t, \threshold_{\T}], \LT}} (\estobs[\n, \t, \threshold_{\T}]-\trueobs) \stackrel{d}{\Longrightarrow} \mc N(0, \sigma_{a}^2)$ as $\T\to \infty$.
    \end{enumerate}
    \end{theorem}

    In simple words, \cref{eq:regularity_consistency,eq:regularity_clt} state the regularity conditions, and the scaling of $\threshold$ with $\T$ for asymptotic consistency or normality respectively for $\estobs[\n, \t, \threshold_{\T}]$. The condition~\cref{eq:regularity_consistency} requires that (i) $\threshold_{\T}\downarrow 2\sigma_{a}^2$ so that the bias goes to $0$, while still ensuring that (ii) the number of neighbors within a ball of size $\threshold_{\T}\!-\!2\sigma_{a}^2\!-\!\errtwo$ grows to $\infty$ so that the variances goes to $0$, and (iii) the probability of sampling a neighbor within radius $\threshold_{\T}\!-\!2\sigma_{a}^2\! +\! \errtwo$ is close to that within radius $\threshold_{\T}\!-\!2\sigma_{a}^2\!-\!\errtwo$, so that the worst-case noise variance inflation goes to $0$.
     The condition (i) exhibits a tradeoff with (ii) and (iii), since the latter two require that $\threshold_{\T}$ does not decay to $0$ too rapidly. The conditions \cref{eq:regularity_clt}(a) and \cref{eq:regularity_clt}(b) represent the stronger analog of \cref{eq:regularity_consistency}(a) and \cref{eq:regularity_consistency}(b) to obtain the stronger normality guarantee. We cap the number of neighbors at $\LT$ to avoid artificially introducing an upper bound constraint on the number of units $\NT$; without such a constraint, the number of neighbors can be large when $\NT$ grows, so that the bias scaled by the number of neighbors may not decay to $0$. We remark that the conditions~\cref{eq:regularity_consistency,eq:regularity_clt} suffice to provide the first asymptotic guarantee for \unittime-level inference in sequential experiments via \emph{our} estimator, and identifying the weakest such conditions for an \emph{arbitrary} estimator is an interesting future direction.

    Next, we state the asymptotic analog of \cref{cor:anytime_bound} and obtain asymptotic guarantees from \cref{thm:anytime_asymp} for \cref{example:finite,example:continuous}. See \ox{\cref{sub:proof_of_cor:anytime_asymp}}\oa{\cite[\cref{sub:proof_of_cor:anytime_asymp}]{supplement}} for the proof.
    \begin{corollary}
    \label{cor:anytime_asymp}
    Consider the set-up as in \cref{thm:anytime_asymp} with $\pmint[\T] = \Omega(\T^{-\beta})$ for $\beta \in[0, \half]$. Then the following statements hold true as $\T, \NT, \LT \to \infty$.
    \begin{enumerate}[label=(\alph*)]
        \item \label{item:finite_asymp} For the setting in \cref{example:finite} with $\MT$ number of distinct unis in the $\T$-th experiment, if $\T$ is large enough such that $\errtwo = o(\min_{\lunit[] \neq \lunit[]'} \norm{\lunit[]\!-\!\lunit[]'}^2_{2})$, then with $\threshold_{\T} = 2\sigma_{a}^2 + \errtwo$, 
        \begin{align}
        \label{eq:countable_consistency}
            1 &=  o\parenth{\frac{\T^{1-2\beta}}{\log(\NT\T)}}
            \ \ \stext{\&} \ 
            \MT= o\parenth{\frac{\NT}{\log \T}}
            \implies \stext{\cref{eq:regularity_consistency} holds,} \qtext{and} \\
        \label{eq:countable_clt}
            \LT^2 &= o\parenth{\frac{\T^{1-2\beta}}{\log(\NT\T)}}
            \stext{\&} \ 
            \MT = o(\NT)  \qquad 
            \implies \stext{\cref{eq:regularity_clt} holds.} 
        \end{align}
        \item \label{item:highd_asymp} 
        For the setting in \cref{example:continuous}, with $\threshold_{\T}=2\sigma_{a}^2 + (1+ \T^{\frac{1-2\beta}{6}})\errtwo$, 
        \begin{align}
        \label{eq:highd_consistency}
            1 &\!=\! o\bigparenth{\frac{\T^{\frac{1-2\beta}3}}{\sqrt{\log(\NT\T)}}}
            \stext{\&}
            \bigparenth{ \frac{\T^{1+\beta}}{\log^{3/2}(\NT\T)} }^{d/6}\!=\! o(\frac{\NT}{\log\T}) \ \ \Longrightarrow \ \stext{\cref{eq:regularity_consistency} holds,}  \stext{and}
             \\
            \LT  &\!=\! o\bigparenth{\frac{\T^{\frac{1-2\beta}3}}{\sqrt{\log(\NT\T)}}} \stext{\&} \bigparenth{ \frac{\T^{1+\beta}}{\log^{3/2}(\NT\T)} }^{d/6}\!=\! o(\NT)
            \ \ \  \ \  \Longrightarrow \ \stext{\cref{eq:regularity_clt} holds.} 
            \label{eq:highd_clt}
        \end{align}
    \end{enumerate}
\end{corollary}
    For the simplified setting of $\beta=0$ and with large enough $\NT$, \cref{cor:anytime_asymp} provides a CLT at the rate (i) $\T^{\quarter-\alpha}$ for  for discrete units (with constant $\MT$), and (ii) $\T^{\frac16-\alpha}$ for continuous units, where $\alpha>0$ denotes an arbitrary small positive number. We once again note that with $\real^d$-valued unit factors, the requirement on $\NT$ grows exponentially with $d$, so that our strategy would be effective only with low-dimensional latent factors.
    We highlight that \cref{cor:anytime_asymp} allows the minimum exploration rate to decay as fast as $\T^{-\half}$ up to logarithmic factors.

    \begin{remark}[Variance estimate and data-dependent $\threshold$]
        \label{rem:cons_sigma}
        \cref{thm:anytime_asymp}\cref{item:clt} implies asymptotic validity of the confidence interval~\cref{eq:sigma_nn} whenever the noise variance estimate $\what{\sigma}^2$ is consistent; see \ox{\cref{sub:variance_estimate}}\oa{\cite[\cref{sub:variance_estimate}]{supplement}} for one such estimate under regularity conditions.
         \blue{While $\threshold$ would be tuned in practice, we do note that currently our theoretical analysis for asymptotics in \cref{thm:anytime_asymp} assumes that the sequence of $\threshold$ is pre-specified. Deriving a central limit theorem for a data-driven sequence of hyperparameter is a very exciting future direction.}
    \end{remark}
    
    \begin{remark}[Side product: Guarantees for matrix completion]
        \label{rem:prior_ci}
        
        \cref{thm:anytime_asymp,rem:cons_sigma} immediately provide an asymptotically valid prediction interval for each entry in a matrix completion task where the entries of a given matrix are missing completely at random (MCAR) with probability $p$. While the matrix completion literature is rather vast (see the survey \cite{ieee_matrix_completion_overview} for an overview), as noted in~\cite{chen2019inference} the literature for the statistical inference and  uncertainty quantification for the matrix entries is relatively sparse. \blue{While \citep[Thm.~2]{chen2019inference} does provide a prediction interval for bilinear factor model with MCAR patterns, our work advances such results to a vast number of sequentially missing patterns while also allowing for non-linear factor models (see \cref{sub:non_linear_lf}).}
    \end{remark}

\subsection{Asymptotic guarantees at the population-level}
\label{sub:ate_asymp}
    Next, we state our guarantees for estimating $\trueybarall[]$. With the same series of experiments as in \cref{thm:anytime_asymp} and $\sbraces{\KT}_{\T=1}^{\infty}$ denoting a suitable sequence diverging to $\infty$, we use
     $\estybar$
    and $(\estybaralltwo, \what{\sigma}^2_{\KT})$,
    to denote the analogs of $\estybar[\threshold]$ defined in \cref{eq:ybart_estimate}, and $(\estybaralltwo[\threshold, K], \what{\sigma}^2_{K})$ defined in \cref{eq:sigma_nn} for the $\T$-th experiment. Recall that $\probparam_{\lunit[], a}$ was defined in \cref{def:phi}, $\pmint$ in \cref{eq:pmin}, and $\errtwo$ in \cref{eq:errtwo}. In the next result, we abuse notation for $\thresonet,\threstwot$, as their definitions in \cref{thm:ate_asymp} are very slightly different ($\errtwo$ replaced by $\sqrt{2}\errtwo$) than those in \cref{thm:anytime_asymp}. The proof is provided in \ox{\cref{sec:proof_of_thm:ate_asymp}}\oa{\cite[\cref{sec:proof_of_thm:ate_asymp}]{supplement}}.

    \begin{theorem}[\ateconsistencyname]
    \label{thm:ate_asymp}
    Consider the set-up from \cref{thm:anytime_asymp}, with $\threshold_{\T} \geq 2\sigma_{a}^2\!+\!\sqrt{2}\errtwo$, $\thresonet \defeq \threshold_{\T}\!-\!2\sigma_{a}^2\!-\!\sqrt{2}\errtwo$, and $\threstwot \defeq \thresonet\!+\! 2\sqrt{2}\errtwo$.
    \begin{enumerate}[label=(\alph*), leftmargin=*]
        \item\label{item:consistency_ate} Asymptotic consistency: If the sequence $\braces{\threshold_{\T}', \N_{\T}, \pmint}_{\T=1}^{\infty}$ satisfies
         \begin{align}
        &\max\bigg\{
        \threstwot, \,\,
         \NT \sup_{\lunit[]} \sum_{\t=1}^{\T} e^{-\frac{1}{64}\pmint[\t]\probparam_{\lunit[], a}(\thresonet)\NT\!}, 
             \\
           &\qquad
        \frac{\sqrt{\log(\NT\T)}\sum_{\t=1}^{\T}\pmint[\t]^{-\half}/\T}
                 {\inf_{\lunit[]} \sqrt{\NT\probparam_{\lunit[], a}(\thresonet)}}, \, \, \,
           \sup_{\lunit[]} \abss{\frac{\probparam_{\lunit[], a}(\threstwot)}{\probparam_{\lunit[], a}(\thresonet)}\!-\!1}\, \frac{\sum_{\t=1}^{\T}\pmint[\t]\inv}{\T}
           \bigg\} = o_{\T}(1),
           \label{eq:ate_regularity_consistency}
         \end{align}
        then $\estybar \stackrel{p}{\longrightarrow} \trueybarall[]$ as $\T \to \infty$.
        \item\label{item:clt_ate} Asymptotic normality: If the sequence $\braces{\threshold_{\T}', \KT, \N_{\T}, \pmint}_{\T=1}^{\infty}$ satisfies
        \begin{align}
         &\max\bigg\{\KT\threstwot, \,\,\,
             \NT \sup_{\lunit[]} \sum_{\t=1}^{\T} e^{-\frac{1}{64}\pmint[\t]\probparam_{\lunit[], a}(\thresonet)(\NT\!-\!1)},  \\
           &\ 
        \frac{\sqrt{\KT\log(\NT\T)}\sum_{\t=1}^{\T}\pmint[\t]^{-\half}/\T}
                 {\inf_{\lunit[]} \sqrt{\NT\probparam_{\lunit[], a}(\thresonet)}}, \, \, 
           \sqrt{\KT}\sup_{\lunit[]} \abss{\frac{\probparam_{\lunit[], a}(\threstwot)}{\probparam_{\lunit[], a}(\thresonet)}\!-\!1}\, \frac{\sum_{\t=1}^{\T}\pmint[\t]\inv}{\T}
           \bigg\} \!=\! o_{\T}(1),
           \label{eq:ate_regularity_clt}
        \end{align}
        and $\what{\sigma}_{\KT} \stackrel{p}{\longrightarrow} c>0$,
        then $\frac{\sqrt{\KT}}{\what{\sigma}_{\KT}}(\estybaralltwo-\trueybarall[]) \ \stackrel{d}{\Longrightarrow} \ \mc N(0, 1)$ as $\T\to \infty$.
    \end{enumerate}
    \end{theorem}

    In simple words, \cref{eq:ate_regularity_consistency,eq:ate_regularity_clt} denote the regularity conditions to ensure asymptotic consistency of the estimate~\cref{eq:ybart_estimate} and asymptotic validity of the confidence interval~\cref{eq:sigma_nn} for the averaged counterfactual mean. Notably, our current analysis does not yield $\sqrt{\NT\T}$-rate CLT for $\estybar$ due to the $\Omega(\T^{-\quarter})$ bias term (and the lack of decay in bias even after averaging), which diverges after getting multiplied by $\sqrt{\T}$. Consequently, we use subsampling~\cref{eq:sigma_nn} to construct an estimate with CLT at rate slower than $\T^{-\quarter}$ (as the examples below demonstrate). 
    In a nutshell, the four terms on the LHS of \cref{eq:ate_regularity_consistency,eq:ate_regularity_clt} exhibit a tradeoff over the choice of $\threshold_{\T}$---as $\threshold_{\T}\downarrow2\sigma_{a}^2$, the first term decays to $0$, while the remaining three terms grow. The second term in both displays denotes the upper bound on the failure probability---for the event of observing enough available nearest neighbors---accumulated over all time and all units. Furthermore, the first (bias), third and fourth terms in \cref{eq:ate_regularity_clt} constrain how quickly $\KT$ can grow; we note that a large $\KT$ is preferred to obtain narrower intervals for $\trueybarall[]$.

    Next, we present a corollary (with proof in \ox{\cref{sub:proof_of_cor:ate_asymp}}\oa{\cite[\cref{sub:proof_of_cor:ate_asymp}]{supplement}}) that states sufficient conditions in simplified terms to apply \cref{thm:ate_asymp} for \cref{example:finite,example:continuous}.
    \begin{corollary}
    \label{cor:ate_asymp}
    Consider the set-up as in \cref{thm:ate_asymp} with $\pmint[\T] = \Omega(\T^{-\beta})$ for $\beta \in[0, \half]$. Then the following statements hold true as $\T, \NT, \KT \to \infty$.
    \begin{enumerate}[label=(\alph*)]
        \item \label{item:finite_ate_asymp} For the setting in \cref{example:finite} with $\MT$ number of distinct unis in the $\T$-th experiment, if $\errtwo \!=\! o(\min_{\lunit[] \neq \lunit[]'} \norm{\lunit[]\!-\!\lunit[]'}^2_{2})$ for all $\T$ large enough, then with $\threshold_{\T} \!=\! 2\sigma_{a}^2 \!+\! \sqrt{2} \errtwo$, 
        \begin{align}
        \label{eq:countable_ate_consistency}
            1 &=  o\parenth{\frac{\T^{1-2\beta}}{\log(\NT\T)}}
            \ \stext{\&} \ 
            \MT= o\parenth{\frac{\NT}{\T^{\beta}\log(\NT\T)} } 
            \ \ \Longrightarrow \ \ \stext{\cref{eq:ate_regularity_consistency} holds,} \qtext{and} \\
        \label{eq:countable_ate_clt}
            \KT^2 &= o\parenth{\frac{\T^{1-2\beta}}{\log(\NT\T)}}
            \stext{\&} \ 
            \MT = o\parenth{\frac{\NT}{\T^{\half}\log(\NT\T)} }
            \ \ \Longrightarrow \ \ \stext{\cref{eq:ate_regularity_clt} holds.} 
        \end{align}
        \item \label{item:highd_ate_asymp} 
        For the setting in \cref{example:continuous}, with $\threshold_{\T}\!=\!2\sigma_{a}^2 \!+\! \sqrt{2}(1\!+\!\T^{\frac{1+2\beta}{6}}) \errtwo$, if $\beta\leq\quarter$, and
        \begin{align}
        \label{eq:highd_ate_consistency}
            &1 \!=\! o\bigparenth{\frac{\T^{\frac{1-4\beta}3}}{\sqrt{\log(\NT\T)}}}
            \stext{\&}
            \bigparenth{ \frac{\T^{1-\beta}}{\log^{\frac32}(\NT\T)} }^{\frac{d}{6}} \T^{\beta}\!=\! o(\frac{\NT}{\log(\NT\T)}) \Longrightarrow \stext{\cref{eq:ate_regularity_consistency} holds,}  \stext{and}
             \\
            &\KT \!=\! o\bigparenth{\frac{\T^{\frac{1-4\beta}3}}{\sqrt{\log(\NT\T)}}} \stext{\&} \bigparenth{\!\frac{\T^{1-\beta}}{\log^{\frac32}(\NT\T)}\! }^{\frac{d}{6}} \T^{\frac{1-\beta}{3}}\!=\! o(\frac{\NT}{\log(\NT\T)})
            \Longrightarrow \stext{\cref{eq:ate_regularity_clt} holds.} 
            \label{eq:highd_ate_clt}
        \end{align}
    \end{enumerate}
\end{corollary}
        When $\beta=0$, \cref{cor:ate_asymp} yields the same CLT rates for the $\trueybar$ as that from \cref{cor:anytime_asymp} for a single $\trueobs$in both \cref{example:finite,example:continuous} while requiring comparatively larger number of units $\NT$ than those needed in \cref{cor:anytime_asymp}. These constraints are consequences of (i) the need to ensure enough number of neighbors are available throughout all time points, (ii) lack of improvement in the bias due to $\threshold$ (first term in \cref{eq:nn_bnd}) after averaging, (iii) the scaling of the bias due to distance computation (second term in \cref{eq:nn_bnd}) averaged over time as ${\sum_{\t=1}^{\T}\pmint[\t]^{\inv}}/{\T}$) in contrast to just as $\pmint[\t]\inv$ (treated as a constant with $\t$ fixed) in \cref{cor:anytime_asymp}. 
        In \cref{cor:ate_asymp}\cref{item:highd_ate_asymp}, the minimum exploration rate is constrained to be $\Omega(\T^{-\quarter})$, which is much slower than the $\Omega(\T^{-\half})$ constraint in \cref{cor:anytime_asymp}\cref{item:highd_asymp}. This worse scaling is a consequence of a poorer scaling of the worst-case inflation in noise variance averaged over time (the fourth term in \cref{eq:nn_bnd}) with $\T$ due to the dependence on $\pmint[\T]$ (rather than just $\pmint[\t]$ treated as a constant for a fixed $\t$ in \cref{cor:anytime_asymp}). Moreover, the lower bound on $\NT$ getting smaller as $\beta$ gets larger is unconventional due to the unusual choice of $\threshold_{\T}$ which increases with $\beta$ to ensure that the aforementioned variance term can decay to $0$. We leave the sharpening of constraints on $(\beta, \NT)$ here as an interesting future direction.

        \blue{We note that while \cref{thm:ate_asymp,cor:ate_asymp} derive asymptotic guarantees for estimating the average treatment effect, the same machinery can be easily extended to estimating unit-specific treatment effects averaged over time, or time-specific treatment effects averaged over units.}

\subsection{Proof sketch}
\label{sub:proof_sketch}

We now provide a sketch of the proof of \cref{thm:anytime_bound}.
A key highlight of our proof is a novel sandwich argument with nearest neighbors that may be of independent interest for dealing with adaptively collected data. \blue{In particular, this sandwich argument is the key technique that allows us to establish the robustness of nearest neighbors for sequential policies without requiring any knowledge of the policy (as discussed in \cref{subsub:robustness}).}

Overall, our proof proceeds in several steps. First, we decompose the error in terms of a bias ($\mbb B$) and variance term ($\mbb V$) using basic inequalities~\ox{\cref{eq:det_final}}\oa{\cite[Eq.~\cref{eq:det_final}]{supplement}}:
\begin{align}
\label{eq:bias_var_decomp}
	\mbb B \defeq \parenth{ \frac{\sum_{j \in \estnbr} |\trueobs[\n,\t]- \trueobs[j, \t]|}{\nestnbr}  }^2,
	\qtext{and}
     \mbb V \defeq \parenth{ \frac{\sum_{j \in \estnbr} \noiseobs[j, \t]}{\nestnbr}  }^2.
\end{align}
Next, we discuss our strategy to bound the bias $\mbb B$ in \cref{sub:bias} using martingales, followed by the strategy to bound the variance $\mbb V$ in \cref{sub:sandwich} using a sandwich argument.

\subsubsection{Controlling the bias via martingales}
\label{sub:bias}
We show that the bias $\mbb B$ bounded by the worst-case distance between unit factor $\norm{\lunit-\lunit[j]}_2$ over $j$ in the nearest neighbor set $\estnbr$~\ox{\cref{eq:average_error}}\oa{\cite[Eq.~\cref{eq:average_error}]{supplement}}. To show that this distance is small, we use suitably constructed martingales and establish a tight control on the estimated distance $\estdist$ leveraging the exogenous nature of latent factors and noise variables despite the sequential nature of policy. In particular, we show that $\estdist$ concentrates around its suitable expectation denoted by $\truedist$ (see \ox{\cref{lemma:dist_noise_conc}}\oa{\cite[\cref{lemma:dist_noise_conc}]{supplement}}), 
which in turn helps us to control the distance between unit factors in terms of $\threshold, \sigma_{a}^2$ and $\lambda_a$ 
using the fact that $\estdist\leq \threshold$. To provide this tight control on the separation between $\estdist$ and $\truedist$, 
we use the condition~\cref{eq:pmin} and a standard Binomial concentration result and establish a lower bound of $\Omega(\pmint^2\T)$ on the denominator in \cref{eq:time_nbr_dist}. Overall, we prove that the first two terms on the RHS of the display \cref{eq:nn_bnd} upper bound the bias $\mbb B$~\cref{eq:bias_var_decomp}.

\subsubsection{Controlling the variance via sandwiching the estimated neighbors}
\label{sub:sandwich}
Establishing a bound on $\mbb V$ is more involved due to the sequential nature of policy in \cref{assum:policy} which induces correlations between the neighbor set $\estnbr$ and the noise terms $\sbraces{\noiseobs[j,\t]}$.  For this part, we use the novel sandwich argument with the estimated set of neighbors. Without loss of generality we set $a=1$, so that we can use the simplified notation $\indicator(\miss=a) = \miss$. The complete argument is provided in \ox{\cref{lemma:dist_noise_conc}}\oa{\cite[\cref{lem:var_bnd,lem:nset_bnd}]{supplement}}. %

We first show that there exists two sets of units denoted by $\nstar$ and $\nwstar$ that are contained in the sigma-algebra of the unit factors $\ulf$, and there exists a high probability event conditioned on $\ulf$ (corresponding to the tight control on $\estdist$ discussed above), such that the estimated set of neighbors $\nhat \defeq \sbraces{j \in [\N]\suchthat\estdist\leq \threshold}$ is sandwiched as
\begin{align}
\label{eq:n_sandwich}
    \nstar \subseteq \nhat \subseteq \nwstar.
\end{align}
This sandwich relation helps us to provide a tight control on the variance $\mbb V$ despite the arbitrary dependence between the observed noise at time $\t$ with the neighbors  $\nhat$ due to (i) the sequential nature of the treatment policy, and (ii) the fact that entire data is used to estimate the neighbors $\nhat$. In particular, \cref{eq:n_sandwich} allows to decompose the variance $\mbb V$ in two terms, each of which are easy to control as we now elaborate.

Note that by definition~\cref{eq:reliable_nbr}, we have $\estnbr  = \sbraces{j\in\nhat\suchthat \miss[j, \t]=1}$, which when combined with the sandwich condition~\cref{eq:n_sandwich} implies that the variance $\mbb V$ can be decomposed further in two terms: %
\begin{align}
\label{eq:var_decomp}
	\mbb V_1 \!\defeq\! \biggparenth{\frac{\sum_{j\in\nstar}\miss[j, t]\noiseobs[j, \t]}{\sum_{j\in\nhat}\miss[j, t]}}^2
	\stext{and} 
    \mbb V_2 \!\defeq\! \biggparenth{\frac{\sum_{j\in\nhat\backslash\nstar}\miss[j, t]\noiseobs[j, \t]}{\sum_{j\in\nhat}\miss[j, t]}}^2.
\end{align}
In simple words, $\mbb V_1$ corresponds to the variance of averaged noise for the units in $\nstar$ that are assigned treatment $a$ at time $\t$, and  $\mbb V_2$ corresponds to the averaged noise in $\nhat\backslash \nstar$ that are assigned treatment $a$ at time $\t$. 

\paragraph{Bounding $\mbb V_1$ and $\mbb V_2$}
We argue that the variance $\mbb V_1$ behaves (roughly) like the variance of an average of \iid noise variables due to \cref{assum:policy,assum:iid_unit_lf,assum:zero_mean_noise}, so that this averaged noise is of the order $1/(\pmint[\t]|\nstar|)$ and hence decays with the number of neighbors in $\nstar$ increases. 

To bound $\mbb V_2$, first we note that since the policy is sequential, the noise terms at time $\t$ affect the observed data in the future, and thereby the noise terms in $\nhat\backslash \nstar$ are correlated as their realized values affect the future policy, and thereby the observed data and subsequently which units are included in the set $\nhat$. Analyzing this term without further assumptions remains non-trivial, and we overcome this challenge by using a worst-case bound on this term in terms of the bound on noise, and the fact that $|\nhat\backslash \nstar| \leq |\nwstar\backslash \nstar|$ due to \cref{eq:n_sandwich}. Overall, we prove that $\mbb V_2$ is of the order of $|\nwstar\backslash \nstar|^2/(\pmint[\t]^2|\nstar|^2)$, and thereby would decay if the relative difference between the two neighbor sets $\nwstar$ and $\nstar$ decays.  See \ox{\cref{sec:proof_of_anytime_bound}}\oa{\cite[\cref{sec:proof_of_anytime_bound}]{supplement}} for the complete argument. %

\paragraph{Bounding the size of sets $\nstar$ and $\nwstar\backslash\nstar$}
Finally, we prove a lower bound on $|\nstar|$, and an upper bound on $|\nwstar\backslash \nstar|$ using \iid sampling of units (\cref{assum:iid_unit_lf}), the definition of the neighbor sampling probability~\cref{def:phi}, and a standard Binomial concentration.

Putting the pieces together, we conclude that that the third term and fourth term on the RHS of \cref{eq:nn_bnd} respectively upper bound $\mbb V_1$, and $\mbb V_2$ defined in \cref{eq:var_decomp}, which when combined with \cref{eq:bias_var_decomp} yields the bound~\cref{eq:nn_bnd} established in \cref{thm:anytime_bound}.

    \section{Extensions}
    \label{sec:possible_extensions}
    We now discuss on extension of our theoretical results to non-linear factor model in \cref{sub:non_linear_lf} and a model with bounded spillover effects in \cref{sub:delayed_effects}. We refer the reader to \ox{\cref{sub:other_extensions}}\oa{\cite[\cref{sub:other_extensions}]{supplement}} for additional extensions to sub-Gaussian factors, time factors that are non-iid or have singular covariance, and other spillover settings where the spillover effect depends on the entire sequence of treatments.
    
    \subsection{Non-linear latent factor model}
    \label{sub:non_linear_lf}
    We now describe how our results do not critically hinge on the bilinearity assumption made in \cref{assum:bilinear}. In particular, we consider the non-linear latent factor model from \cref{eq:non_linear_lf} now stated as an assumption with some regularity conditions. We also simultaneously state a different version of \cref{assum:lambda_min} on the time factor distributions needed for the analysis in this section.
    \begin{assumption}[\nonlinearfactor]
    \label{assum:nonlinear}
    Conditioned on the latent factors $\ulf$ and $\vlf$, the mean counterfactuals satisfy
    \begin{align}
        \trueobsvar_{\n, \t}^{(a)}\! 
    \defeq \E[\obs^{(a)}\vert \ulf,\vlf] =\lfun^{(a)}(\lunit^{(a)}, \ltime^{(a)})
    \label{eq:non_linear_f}
    \end{align}
     for $a\in \actionset$. Moreover, the non-linearity $\lfun^{(a)}$ is bounded by $G_1$, is $G_2$-Lipschitz in $\infnorm{\cdot}$-norm in the second coordinate, and $G_3$-Lipschitz in the first coordinate. That is, $\sinfnorm{\lfun^{(a)}} \leq G_1$ and the following bounds hold for all $\lunit[], \lunit[]', \ltime[], \ltime[]'$:
     \begin{align}
     \sabss{\lfun^{(a)}(\lunit[], \ltime[]) -\lfun^{(a)}(\lunit[], \ltime[]')} \leq G_2 \sinfnorm{\ltime[]-\ltime[]'} \qtext{and}
     \sabss{\lfun^{(a)}(\lunit[], \ltime[]) -\lfun^{(a)}(\lunit[]', \ltime[])} \leq G_3 \sinfnorm{\lunit[]-\lunit[]'}.
     \label{eq:non_linear_lip}
     \end{align}
    \end{assumption}
    \vspace{-2mm}
    \vspace{-2mm}
    \vspace{-2mm}
    \begin{assumption}[\nonlineartimefactor]
\label{assum:time_factor_lower_bound}
The time latent factors $\vlf \defeq \sbraces{\ltime^{(a)}, a\in\actionset}_{\t=1}^{\T}$ are exogenous, and for each $a \in \actionset$, they are drawn \iid from a distribution with support $[0, 1]^d$ with a density that is lower bounded by a constant $c_{\mc P, v}>0$. 
\end{assumption}
   We note that the choice of $\infnorm{\cdot}$ for Lipschitzness in \cref{assum:nonlinear} and $[0, 1]^d$ for the support of time factors in \cref{assum:time_factor_lower_bound} is for convenience in the analysis. The two conditions can respectively be generalized to any $\norm{\cdot}_p$-norm for $p\geq 1$ and a compact subset of $\real^d$, and would impact merely the $d$-dependent constants implicit in the following results. The lower bound on density of time factors serves as the analog of assuming a lower bound on the minimum eigenvalue of the covariance matrix for the time factors in the bilinear setting. We note that the Lipschitzness with respect to the first argument is not required explicitly for the main result below, but useful to argue that any given unit will have a non-trivial number of neighbors.

    Given the non-linear model~\cref{eq:non_linear_f}, the function $\probparam_{\lunit[], a}: \real_{+} \to [0, 1]$ from \cref{def:phi}, that characterizes the probability of sampling a neighboring unit of $\lunit[]$ is suitably redefined as
\begin{align}
\label{def:non_linear_phi}
   \probparam^{\trm{non-lin}}_{\lunit[], a}(r) &\defeq  \P_{u'}\brackets{ \E_{\ltime[]^{(a)}}(\lfun^{(a)}(\lunit[], \ltime[]^{(a)})-\lfun^{(a)}(\lunit[]', \ltime[]^{(a)}) )^2 \leq r},
   \qtext{for}r \geq 0,
\end{align}
where the expectation inside is computed only over the latent time factor distribution assumed in \cref{assum:time_factor_lower_bound} and like \cref{def:phi}, the probability is computed only over the randomness in the draw of the unit factor $\lunit[]'$ from the corresponding distribution assumed in \cref{assum:iid_unit_lf}. One can immediately verify that for \cref{example:finite}, we still have $\probparam^{\trm{non-lin}}_{\lunit[], a}(r) = \frac{1}{M}$. Moreover, for \cref{example:continuous}, we can also conclude that $\probparam^{\trm{non-lin}}_{\lunit[], a}(r) \leq c_d r^{d/2}$ since under \cref{assum:nonlinear}, $\infnorm{\lunit[]-\lunit[]'} \leq \sqrt{r} $ implies that $|\lfun^{(a)}(\lunit[], \ltime[]^{(a)})-\lfun^{(a)}(\lunit[]', \ltime[]^{(a)})| \leq G_3 \sqrt{r}$.

We are now ready to state our generalization of \cref{thm:anytime_bound}, namely, the non-asymptotic error guarantee, for the non-linear factor model. We state a new error term in \cref{eq:nonlinear_eta_chi}, which is the analog of $\errterm$ from \cref{eq:eta_chi} that uses the bound $G_1$ on $\sinfnorm{\lfun^{(a)}}$:
\begin{align}
\label{eq:nonlinear_eta_chi}
    \errterm' \defeq \frac{4(G_1+\nconst)^2\sqrt{32\log(4\N\T/\delta)}}{\pmint\sqrt{\T}}.
\end{align}
    
    \begin{theorem}[\nonlinearanytimeboundname]
    \label{thm:non_linear_anytime_bound}
    Consider a sequential experiment with $\N$ units and $\T$ time points satisfying \cref{assum:policy,assum:zero_mean_noise,assum:nonlinear,assum:time_factor_lower_bound,assum:iid_unit_lf}. Given a fixed $\delta \in (0, 1)$, suppose that for the sequential policy $ \pmint \!\geq\!\T^{-\half} \sqrt{8\log(2/\delta)}$. Then for any fixed unit with latent factor $\lunit$, any fixed $\t \in [\T]$, $a \in \actionset$, and $\threshold$ such that $\pmint[\t](\N\!-\!1) \ppau(\threshold') \geq 1$ and $\threshold-2\sigma_{a}^2+\errterm' \to 0$, conditional on $(\lunit^{(a)}, \ltime^{(a)})$ the nearest neighbor estimate $\estobs$~\cref{eq:obs_estimate} satisfies
    \begin{align}
        (\estobs-\trueobs)^2 &\leq2(\wtil{\mbb B}^{\trm{non-lin}} + \wtil{\mbb V}) \qtext{where}
        \wtil{\mbb B}^{\trm{non-lin}} \defeq  c_dG_2^{\frac{2d}{d+2}} \parenth{\frac{\threshold-2\sigma_{a}^2 + \errterm'}{c_{\mc P, v}}}^{\frac{2}{d+2}},
    \label{eq:nonlinear_bias}
    \end{align}
    $\wtil{\mbb V}$ is as in \cref{thm:anytime_bound}, with probability at least $1\!-\!\delta\!-\!2e^{-\frac{\pmint[\t]\ppau( \threshold')(\N\squash{-}1)}{16}}$. In this result, we use $\ppau\defeq\probparam^{\trm{non-lin}}_{\lunit, a}$ (defined in \cref{def:non_linear_phi}), $\threshold' = \threshold-2\sigma_{a}^2 - \errterm'$, and $c_d$ denotes a $d$-dependent constant.
    \end{theorem}
    See \ox{\cref{sec:proof_non_linear}}\oa{\cite[\cref{sec:proof_non_linear}]{supplement}} for the proof. The main difference between \cref{thm:anytime_bound} and \cref{thm:non_linear_anytime_bound} lies in the bias~$\mbb B$ term. While, the bias in \cref{eq:nonlinear_bias} suffers from the curse of dimensionality with respect to $\threshold, \T$, the bias in \cref{thm:anytime_bound} does not. Such a difference arises since for the bilinear model the expected error and entry-wise error are linearly related (see \ox{\cref{eq:u_bound}}\oa{\cite[Eq.~\cref{eq:u_bound}]{supplement}}). On the other hand, for the non-linear model upper bounding entry-wise error with the expected error is not straightforward. However, both results have identical scaling for the variance~$\mbb V$ term. 

    We highlight that the guarantee~\cref{eq:nonlinear_bias} is the first guarantee on entry-wise recovery of a non-linear Lipschitz latent factor model obtained for a general class of sequential policies specified in \cref{assum:nonlinear}. 
    To our knowledge, no entry-wise guarantee in the matrix completion literature with no-linear Lipschitz latent factor model is known even under complete randomization case. Prior works with non-linear factor model~\cite{LeeLiShahSong16,yu2022nonparametric} establish an expected squared error guarantee, where the expectation is taken over the time factor. Besides allowing for a wider range of missing patterns (as in \cref{assum:policy}), \cref{thm:non_linear_anytime_bound} innovates on such prior guarantees by converting the expected error guarantee to an entry-wise guarantee (via a technical result~\ox{\cref{lem:inf_lip_bound}}\oa{\cite[\cref{lem:inf_lip_bound}]{supplement}}). Note that the inference results from \cref{thm:anytime_asymp} can also be generalized to the non-linear setting; see \ox{\cref{sec:asymp_nonlinear}}\oa{\cite[\cref{sec:asymp_nonlinear}]{supplement}} for a discussion.

    \subsection{Inference with spillover effects}
    \label{sub:delayed_effects}
    So far, we assumed that there is no spillover effect since the observed outcome for unit $\n$ at time $\t$ depends only on the treatment assigned to that unit at that time $\miss$, i.e., $\obs = \obs^{(\miss)}$ in \cref{eq:model_mc}.
    Next, we show how nearest neighbor algorithm and the results can be generalized to several settings with spillover effects.

    \paragraph{Bounded spillover effects}
    Consider a setting, where the potential outcomes at time $t$ depend on a bounded sequence of treatments. In particular, suppose there exists an integer $K$, such that for any $(i, t)$ and any treatment sequence $\overline{a} \in \sbraces{0, 1}^K$, we have the following potential outcome and observation model:
    \begin{align}
            \obs^{(\overline{a})} = \trueobsvar_{\n, \t}^{(\overline{a})} + \noise_{\n, \t}^{(\overline{a})}
            \stext{and}
            \obs = \obs^{(\miss[\n, \t-K+1:\t])},
            \stext{where}
            \miss[\n, s_1:s_2] \defeq (\miss[\n, s_2], \miss[\n, s_2-1], \miss[\n, s_1]), s_2 \geq s_1,
            \label{eq:bnd_spill}
    \end{align}    
    where $\trueobsvar_{\n, \t}^{(\overline{a})}$ denotes the mean counterfactual for unit $i$ at time $t$ for the sequence of treatments $\overline{a}$, and  $\noise_{\n, \t}^{(\overline{a})}$ denotes exogenous zero-mean noise. This model allows \emph{bounded spillover/lagged effects} in time since (i) the counterfactual mean parameters are parameterized by $K$ treatments and (ii) the outcome observed at time $\t$ is the potential outcome corresponding to the sequence of treatments assigned for the last $K$ time points. 
    The presence of lagged effects necessitates us to define the treatment effects for unit $\n$ at time $\t$ as a function of past $K$ treatments.  For this model, \ox{\cref{sub:excursion_effects}}\oa{\cite[\cref{sub:excursion_effects}]{supplement}} shows how a direct extension of our NN algorithm would yield consistent estimates for  $\trueobsvar_{\n, \t}^{(\overline{a})}$ for any $\overline{a} \in \sbraces{0, 1}^K$ when $\policy_{\n,\t}\in[p, 1-p]$ and  $\min(p^K \T, p^{K} \N) \to \infty $.

    In \ox{\cref{sub:excursion_effects}}\oa{\cite[\cref{sub:excursion_effects}]{supplement}}, we also consider a setting with unbounded $K$, i.e., where the lagged effects at time $\t$ can depend on the entire sequence of treatments till time $\t$, but are moderated by the lagged outcome $\obs[\n, \t-1]$.

    The two extensions above, while generalizing the factor model set-up to spillover effects, is also complimentary to the growing literature in reinforcement learning for off-policy evaluation~\cite{kallus2020double,uehara2022review,bian2023off}, which are naturally equipped to handle spillover effects. Recently, \cite{bian2023off} leverage a factor model to account for the unobserved confounding in the treatment assignment and outcome models and provides entry-wise guarantees while allowing for spillover effects. However, to our knowledge all these works assume that after conditioning on latent factors and past unit-specific data, the treatments are assigned independently across units---a setting violated for the pooled sequential policies (\cref{assum:policy})  motivating this work. 

\section{Experiments}
\label{sec:experiments}
We now present two vignettes to compliment our theoretical results established so far, one with simulated data and another with data from a mobile health study.

\begin{figure}[!t]
    \centering
    \begin{tabular}{c}
    \includegraphics[width=\linewidth]{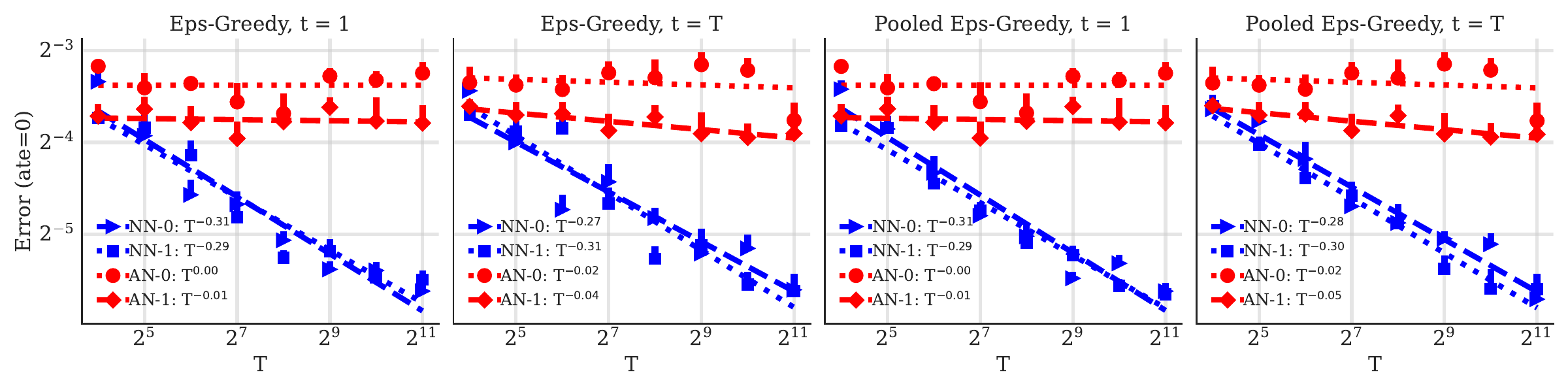} \\[-2mm]
    \includegraphics[width=\linewidth]{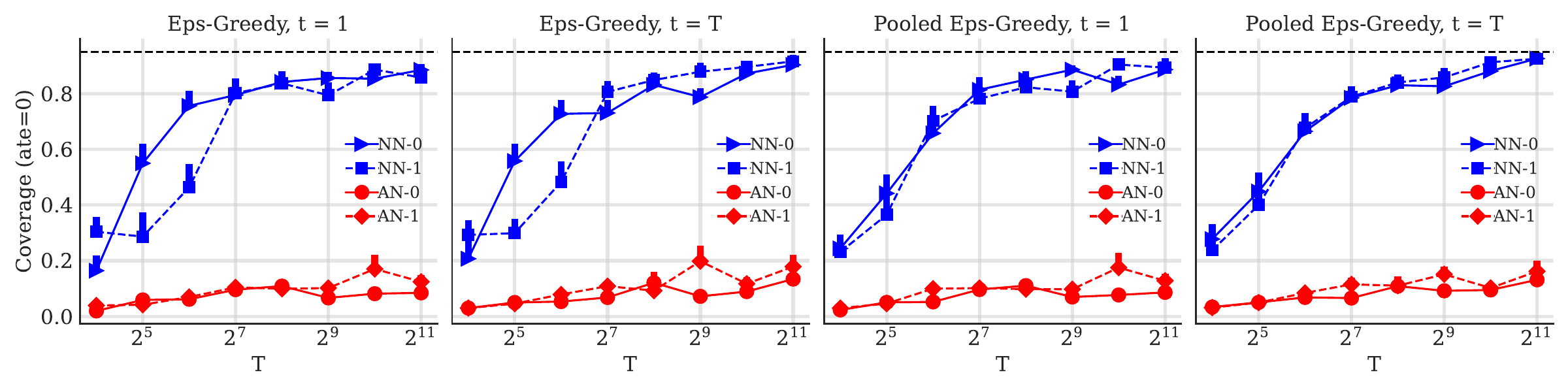} \\[-2mm]
    \includegraphics[width=\linewidth]{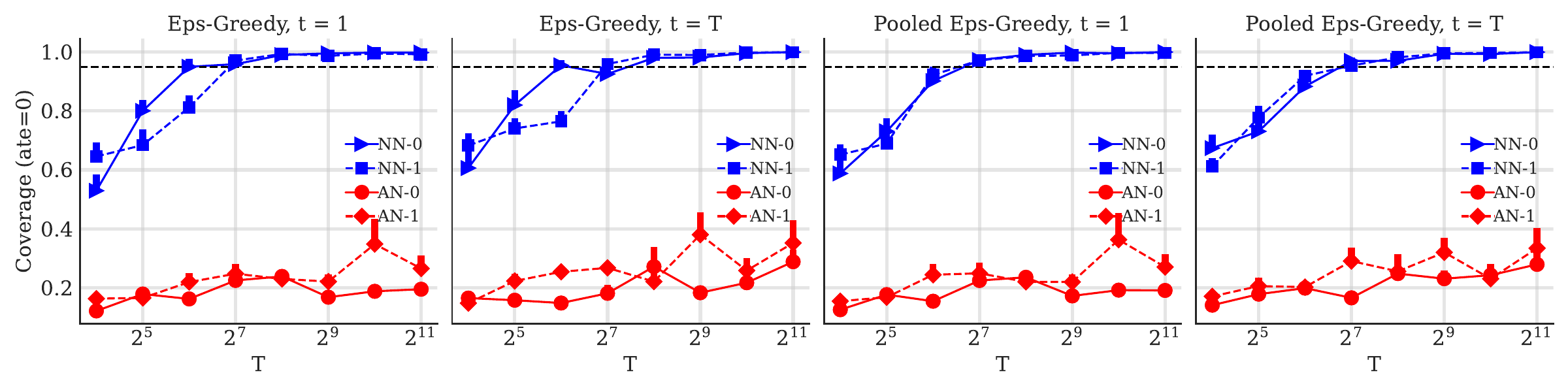} \vspace{-5mm}
    \end{tabular}
    \caption{\tbf{Error (top) and coverage for asymptotic prediction intervals using $\what{\sigma}^2 \!=\! \frac{\eta}{2}$ (middle) and $\what{\sigma}^2$ with finite sample correction (bottom),  as a function of total time points $\T$, when $\mrm{ATE}=0$ in \cref{eq:sims}}. We plot the results for estimates using nearest neighbors (NN-a)~\cref{eq:obs_estimate} and a baseline that treats all units as neighbors (AN-a), for treatments $a\in\braces{0, 1}$ at time $\t \in \{1, \T\}$. When the total number of time points in the trial is equal to the value on the x-axis, the corresponding value on the y-axis for error plots denotes the mean of the errors $|\estobs\!-\!\trueobs|$ across all $\N\!=\!512$ users for treatment $a$ at time $\t$ in the error plots, and that for the coverage plots denotes the fraction of all $N\!=\!512$ users for which the interval~\cref{eq:unit_ci} with appropriate variance estimate covers the true counterfactual mean $\trueobs$ for treatment $a$ at time~$t$. 
    }
    \label{fig:ate_0_plot}
    \vspace{-3mm}
\end{figure}

\subsection{Simulations}
\label{sec:simulations}

We generate data involving two treatments $a\in\sbraces{0, 1}$ and $\N$ user latent factors and $\T$ time latent factors per treatment along with \iid noise in the observations. The latent factors and noise are generated in an (exogenous) \iid manner as follows:
\begin{align}
\label{eq:sims}
    \wtil{u}_{\n}^{(a)} \sim \mrm{Unif}[-0.5, 0.5]^{d}, 
    \ 
    \wtil{v}_{\t}^{(a)} \sim \mrm{Unif}[-0.5, 0.5]^{d}
    \stext{for} a \in \sbraces{0, 1},
    \stext{and}
    \noiseobs  \sim \mc N(0, \sigma_{\varepsilon}^2).
\end{align}
We generate the counterfactual parameters using the model:
\begin{align}
\label{eq:sim_model}
    \trueobsvar^{(a)}_{\n, \t} = \sangles{\wtil{u}_{\n}^{(a)}, \wtil{v}_{\t}^{(a)}} + \trm{ATE} \cdot \mathbf{1}(a=1),
    \stext{for} \trm{ATE} \in \sbraces{0, 0.1}.
\end{align}
We sample treatments using two sequential policies obtained by different variants of $\epsilon$-greedy algorithm: (i) when it is run independently for each unit and (ii) when it is run by pooling the data across all units. See \ox{\cref{sub:sampling}}\oa{\cite[\cref{sub:sampling}]{supplement}} for more details on these sampling algorithms. 
By default, we set $\N= 512, d=2, \sigma_{\varepsilon} = 0.1$, and $\epsilon=0.5$, and vary $\T$ in \cref{fig:ate_0_plot}. In \cref{fig:other_things} of \ox{\cref{sec:exp_details}}\oa{\cite[\cref{sec:exp_details}]{supplement}}, we present results that show robustness of the performance to more general settings with different values of $d, N, \epsilon$ as well as to setting with correlated time factors.

\paragraph{Results} 
We plot the mean of the errors $|\estobs-\trueobs|$ across all $\N$ users for a given treatment $a$ at time $\t$ as a function of the total time points $\T$. Weplot our results with respect to $\T$ for a fixed $\N$ to verify the sharpness of our analysis with respect to $\T$ as its dimension free nature makes it easier to verify the results. We plot a least squares fit for the log error onto the log of study duration $\T$ and display the resulting fit and its slope, e.g., when the slope is $-0.3$, we report an empirical decay rate of $\T^{-0.3}$ for the mean error. For coverage, we report the fraction of users for which the interval~\cref{eq:unit_ci} covers the underlying mean ($\trueobs$ for a given $a, \t$). We abbreviate our nearest neighbors strategy as NN and a simple baseline where all users are treated as neighbors as AN (equivalent to NN with $\threshold = \infty$). The results are averaged over $8$ independent runs along with $1$ standard error bars (which are generally too small to be noticeable) for \cref{fig:ate_0_plot} and 20 runs for \ox{\cref{fig:other_things}}\oa{\cite[\cref{fig:other_things}]{supplement}}. Tuning of $\threshold$ and estimation of variance estimate is discussed in \ox{\cref{sub:variance_estimate}}\oa{\cite[\cref{sub:variance_estimate}]{supplement}}. To improve coverage of our intervals in finite samples, we add the within neighbors variance to our asymptotic variance estimate \ox{\cref{eq:unit_ci_adj}}\oa{(\cite[Eq.~\cref{eq:unit_ci_adj}]{supplement})} and plot these results in the bottom row of \cref{fig:ate_0_plot}. Notably, we use the same estimate for noise variance for both NN and AN.

\paragraph{Results and coverage} We plot the error decay results for the simulations with $\trm{ATE}=0$ in \cref{fig:ate_0_plot}. From the top row we observe that across (i) both the treatments, (ii) both the sampling policies (with or without pooling of users), and (iii) both the first and last time points of the experiment, i.e., $\t\in\sbraces{1, \T}$, the absolute error decays with $\T$ at rate close to $\T^{-0.3}$. This decay rate is roughly consistent with our theory as for the absolute error the decay rate from \cref{thm:anytime_bound} is $\wtil{O}(\T^{-0.25})$. In contrast, the baseline method (AN) that effectively uses sample average over all users provides a trivial error that does not exhibit any decay with $\T$ (as expected). Similar trends are observed for NN and AN in \ox{\cref{fig:ate_1_plot}}\oa{\cite[\cref{fig:ate_1_plot}]{supplement}} where the data is generated with $\trm{ATE}=0.1$, with one noticeable difference: The error and coverage for $a=1$ is better across all settings when compared with $a=0$. This difference can be explained by the sampled frequency of treatments $0$ and $1$. Due to non-zero ATE, the (pooled) $\eps$-greedy samples $\miss = 1$ much more frequently than $\miss=0$, so that the bias of our nearest neighbor estimates would be smaller for $a=1$ than that for $a=0$. Roughly speaking, a difference is also implied by our theoretical bound~\cref{eq:nn_bnd} since we effectively have $\pmintnotag_{1, \T} = 0.75$ and $\pmintnotag_{0, \T} = 0.25$ (see \cref{assum:policy}) for this simulation.

With the asymptotic intervals, the coverage generally improves as $\T$ increases and the 95\% intervals provide $\geq80\%$ coverage for $\T \geq 128$ and approach $90-95\%$ coverage when $\T= 2048$. Moreover, the finite sample adjustments (\ox{\cref{rem:simple_eta}}\oa{\cite[\cref{rem:simple_eta}]{supplement}}) do provide a significant boost to coverage and the intervals now provide over-coverage for all $\T \geq 128$. For the baseline approach, although the coverage does improve due to finite sample adjustments, overall it remains poor since the intervals have width $\Omega(\N^{-1/2})$ as all users are averaged, thereby providing an overconfident estimate. The trends in \ox{\cref{fig:ate_1_plot}}\oa{\cite[\cref{fig:ate_1_plot}]{supplement}} are also similar, except for generally improved performance of NN for $a=1$ since there are more observations for that treatment due to the non-zero ATE as noted above.

\subsection{HeartSteps case study}
\label{sub:mobile_health_study}

We illustrate our approach using data from HeartSteps, a mobile health clinical trial aimed at improving physical activity in users with stage I hypertension~\cite{liao2020personalized}. In this trial, users were equipped with a Fitbit tracker and a mobile app that used a contextual bandit variant of the Thompson Sampling (TS) algorithm~\cite{DBLP:journals/ftml/RussoRKOW18} to decide whether to send activity messages at five daily decision points. The outcome was the logarithm of the user’s step count within 30 minutes after a decision. Treatment $a=0$ denotes not sending a message, while $a=1$ denotes sending one.

We focus on data from 45 participants over a 90-day period from September 14, 2019, to December 5, 2019, where a treatment decision was made 5 times per day. Of the 45 users, 35 were assigned treatments via an individualized TS, a reinforcement learning algorithm based on a Bayesian Gaussian linear model, updated independently for each user using their own data. These 35 are referred to as non-pooled users.
The remaining 10 users participated in a feasibility study and were assigned treatments using a pooled variant of TS called IntelligentPooling (IP-TS)~\cite{tomkins2021intelligentpooling}. IP-TS learns both person-specific and population-level parameters to balance the need for individualization in a heterogeneous population with the need for quick learning from limited data. These users are referred to as pooled users. The IP-TS algorithm was initialized with data from 10 non-pooled users prior to September 14, 2019, and updated nightly using data from all pooled users and the 10 non-pooled users.

\subsubsection{Exploratory data analysis} 
We refer the reader to \ox{\cref{sub:heartsteps}}\oa{\cite[\cref{sub:heartsteps}]{supplement}} for data processing details,  after which we obtain a 
 data with $\N=35$ \goodusers (with 27 non-pooled and 8 pooled) and $\T=450$ decision times. \cref{fig:eda_plot}(a) plots the frequency of the treatments across users and we find that $a=0$ was assigned between $20$ to $125$ decision times, and $a=1$ between 0 to 50 times. The points are colored by the average \emph{dosage}, i.e. the fraction of times with $a=1$, which dosage primarily lies below $0.25$ and always below $0.5$. Overall, we find that $a=1$is less frequent than $a=0$ (consistent with the goal of HeartSteps team to minimize burden from activity messages). \cref{fig:eda_plot}(b) plots the histogram of the outcomes colored by treatment, and shows that the range of outcome values is similar across the two treatments. 
 To check whether the observed data is low-dimensional, the left y-axis of \cref{fig:main_results}(c) plots the (ordered) singular values $\sigma_1 \geq \cdots \geq \sigma_{35}$ of the matrix of observed outcomes for the two treatments. To compute these values, the missing values under treatment $a$ due to non-availability or other reasons (e.g., user not wearing their fitbit tracker) with the user-specific mean of the observed values under treatment $a$. The right y-axis in \cref{fig:eda_plot}(c) plots the fraction of variance explained given by $\frac{\sum_{\l=1}^{k}\sigma_\l^2}{\sum_{\l=1}^{35}\sigma_\l^2}$, over index $k$. In both cases, the first component accounts for more than $99\%$ of the variance. The high fraction of non-availability across users (maximum available time points for a user is $164$ out of possible 450) makes the singular values sensitive to the choice of imputing the missing outcomes (see \ox{\cref{fig:svd_plot}}\oa{\cite[\cref{fig:svd_plot}]{supplement}}).

\begin{figure}[t!]
    \centering
    \resizebox{\linewidth}{!}{
    \begin{tabular}{ccc}
    \includegraphics[width=0.33\linewidth]{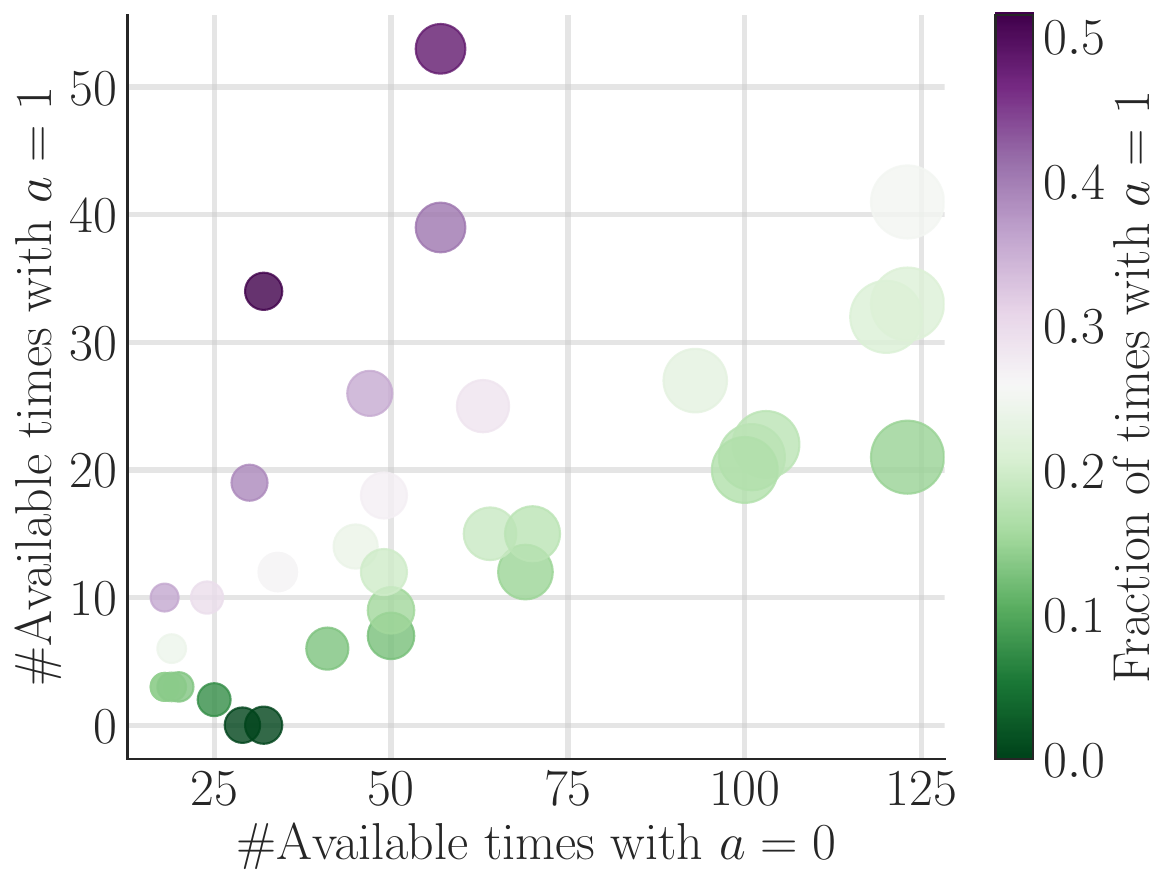} \hspace{1mm} &
    \includegraphics[width=0.33\linewidth]{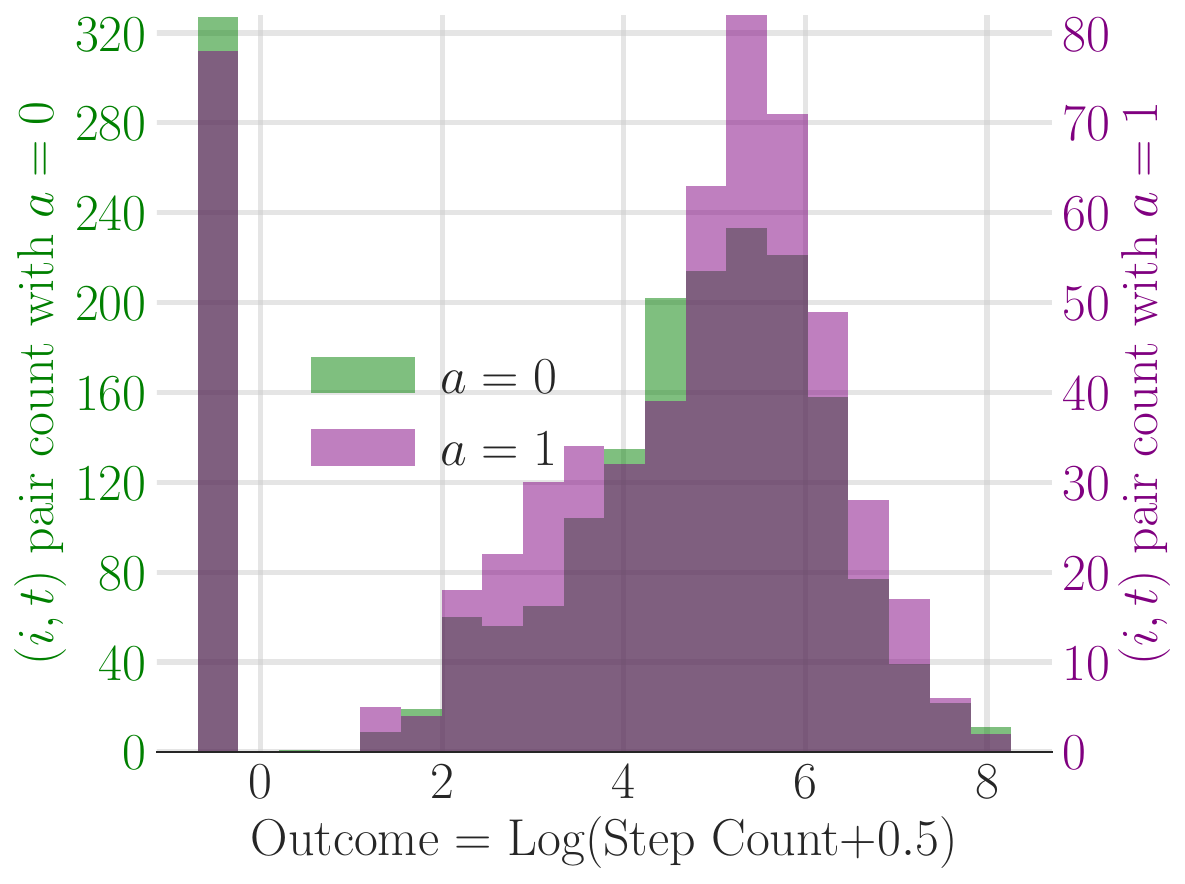}\hspace{1mm} &
    \includegraphics[width=0.33\linewidth]{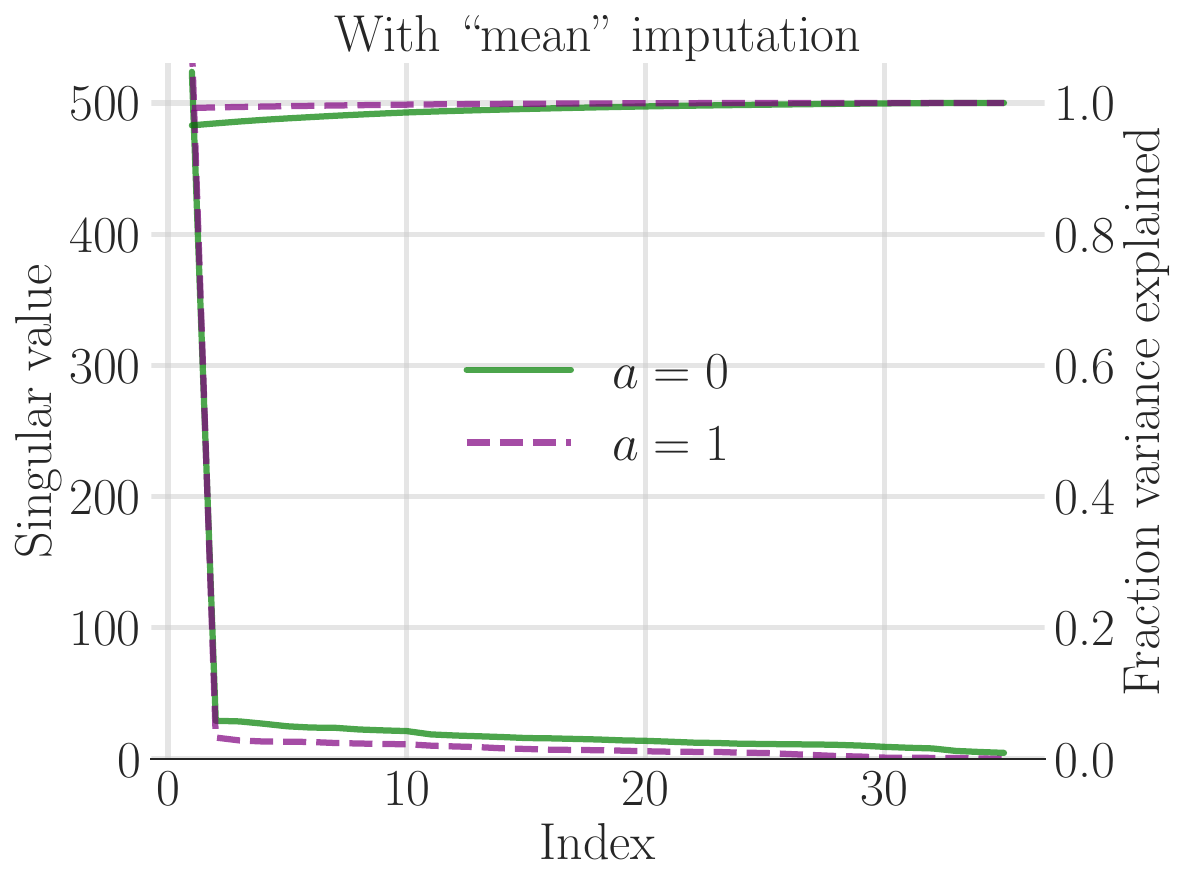}  \\
    (a) & (b) & (c)
    \end{tabular}
    }
    \caption{\tbf{An overview of frequency of treatments and rewards, and the singular values in HeartSteps data.} Panel (a) displays the number of times a user was sent an activity message ($a=1$) on the y-axis, the number of times a user was not sent any message ($a=0$) on the x-axis (and in marker size); the color represents the fraction of times with $a=1$. Panel (b) plots a histogram of outcomes with counts for $a=1$ and $a=0$ respectively on the left and right y-axes. Panel (c) shows the singular values for outcomes under two treatments in order with index on the x-axis, values on the left y-axis, and the fraction of variance explained on the right y-axis.
    }
    \label{fig:eda_plot}
    \vspace{-3mm}
\end{figure}

\subsubsection{Our goal and the challenges posed by limited data} 
Our goal is to estimate counterfactuals for each user across all decision times. In this work, we ignore context information, reducing the signal-to-noise ratio, and treat this as a proof-of-concept for our methodology in inference with adaptively collected data and pooled policies.

The dataset's limited number of users and decision times presents challenges. Often, there are too few decision times where both users in a pair receive treatment $1$, making the distance $\estdist$ for $a=1$ unreliable. Therefore, we estimate neighbors only under treatment $0$ and use the same neighbors to estimate counterfactuals for $a=1$. This approach aligns with our framework if counterfactuals follow the latent factor model $\trueobs=\lfun^{(a)}(\lunit, \ltime^{(a)})$.

Since we don't know the true counterfactuals, we randomly hold out 100 decision times from the 450 as a test set and use the remaining 350 to train our nearest neighbor estimator. We estimate the neighbor set $\sbraces{j: \estdist \leq \threshold}$ for each user and tune $\threshold$. Results are presented by comparing the estimated counterfactuals to the held-out outcomes and evaluating how often the 95\% prediction interval covers the observed outcomes across test decision times.

\vspace{-3mm}
\begin{figure}[ht!]
    \centering
    \resizebox{\textwidth}{!}{
    \begin{tabular}{ccc}
    \includegraphics[width=0.35\textwidth]{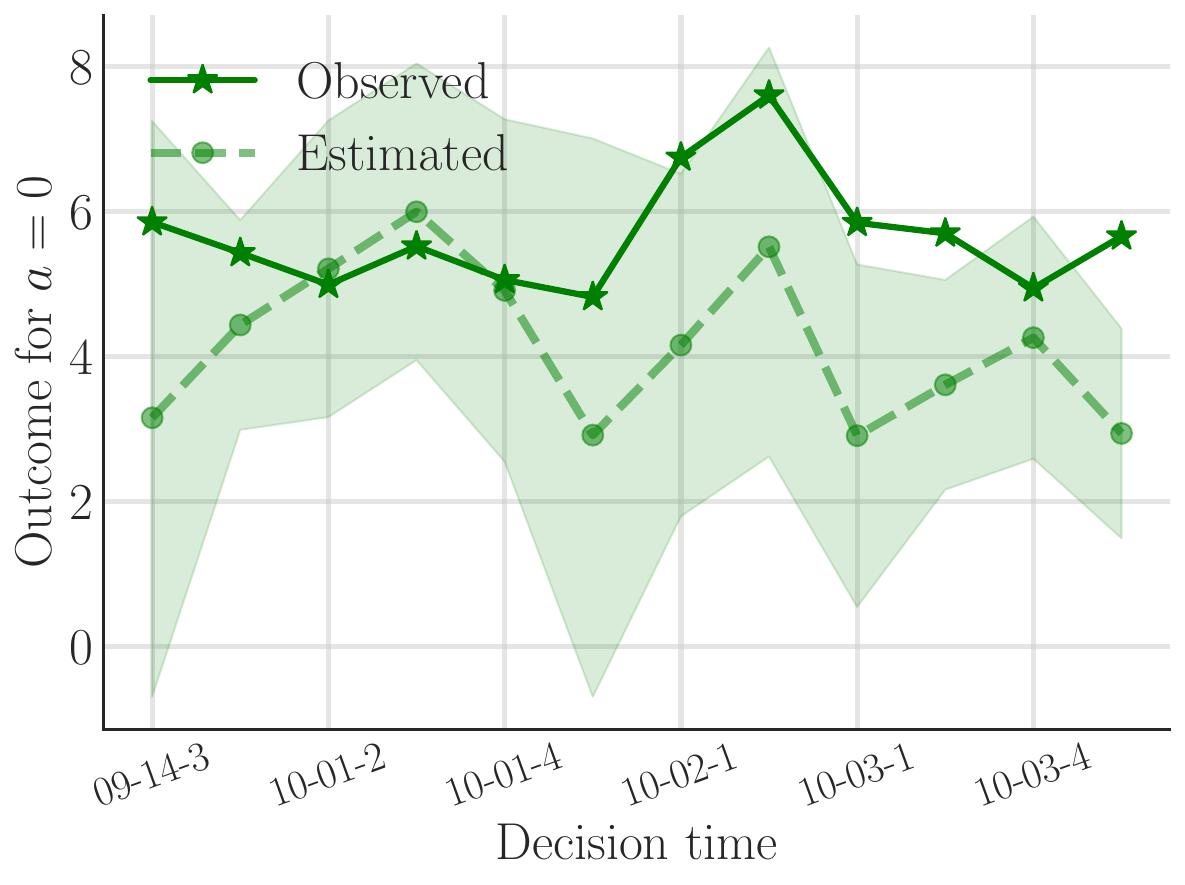}& 
    \includegraphics[width=0.35\textwidth]{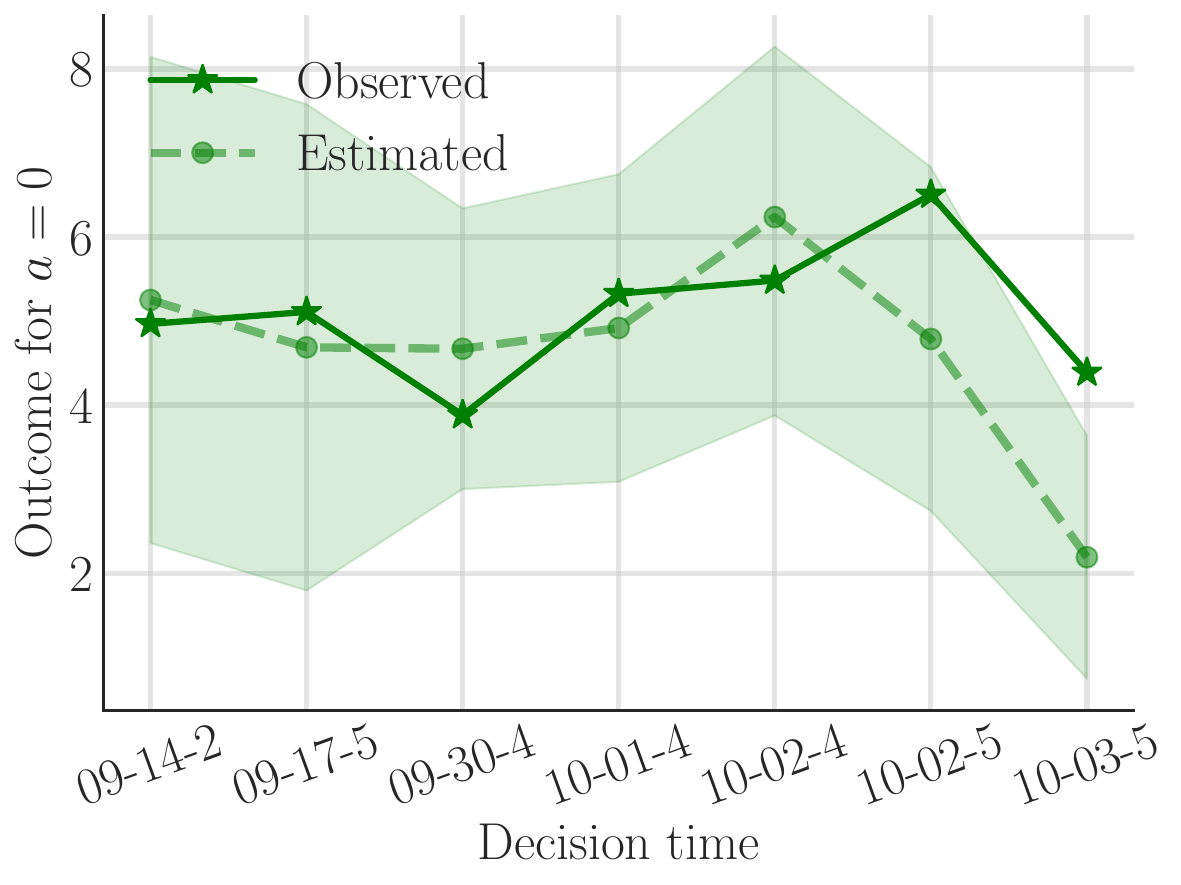} & 
    \includegraphics[width=0.35\textwidth]{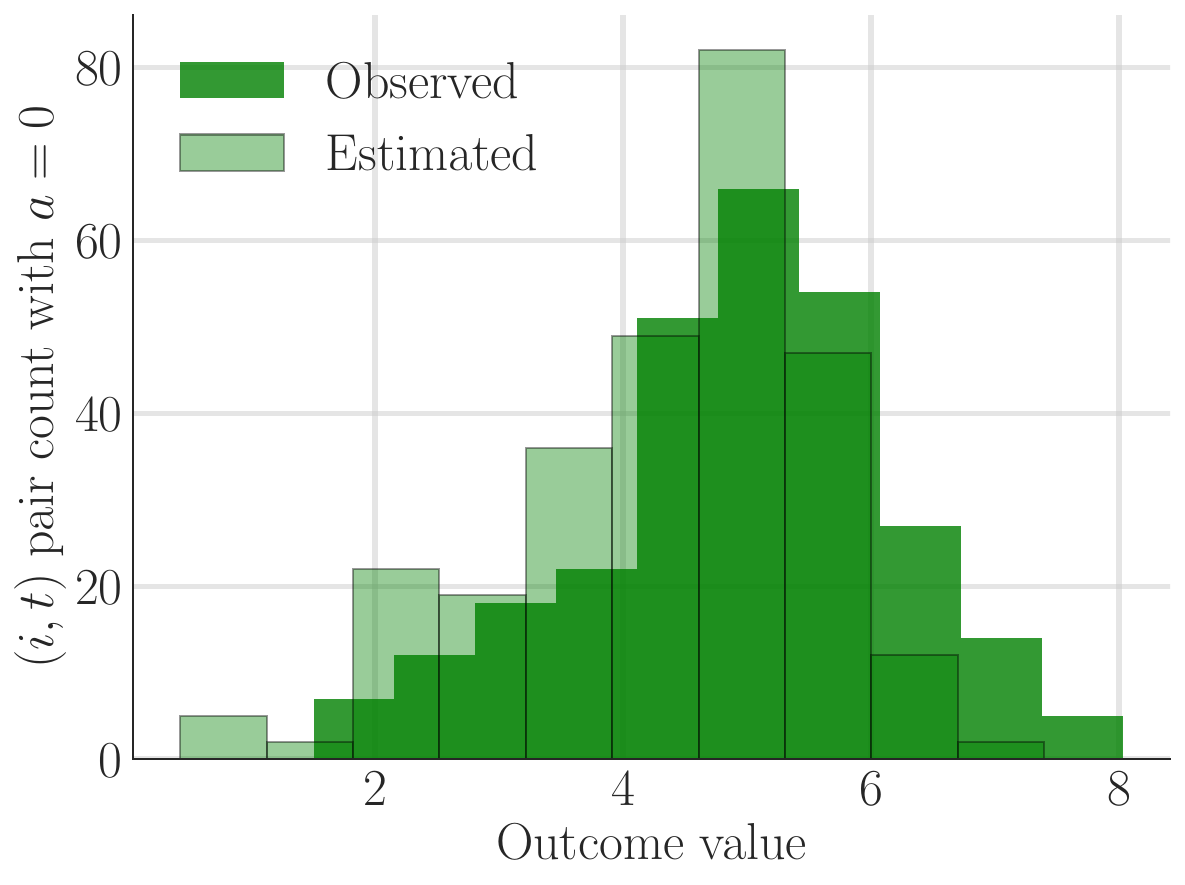} \\[-1mm]
    (a) Pooled user $1$ & (b)  Non-pooled user $1$ & (c) All users
     \\[1mm] \hline \\[-1mm]
    \includegraphics[width=0.35\textwidth]{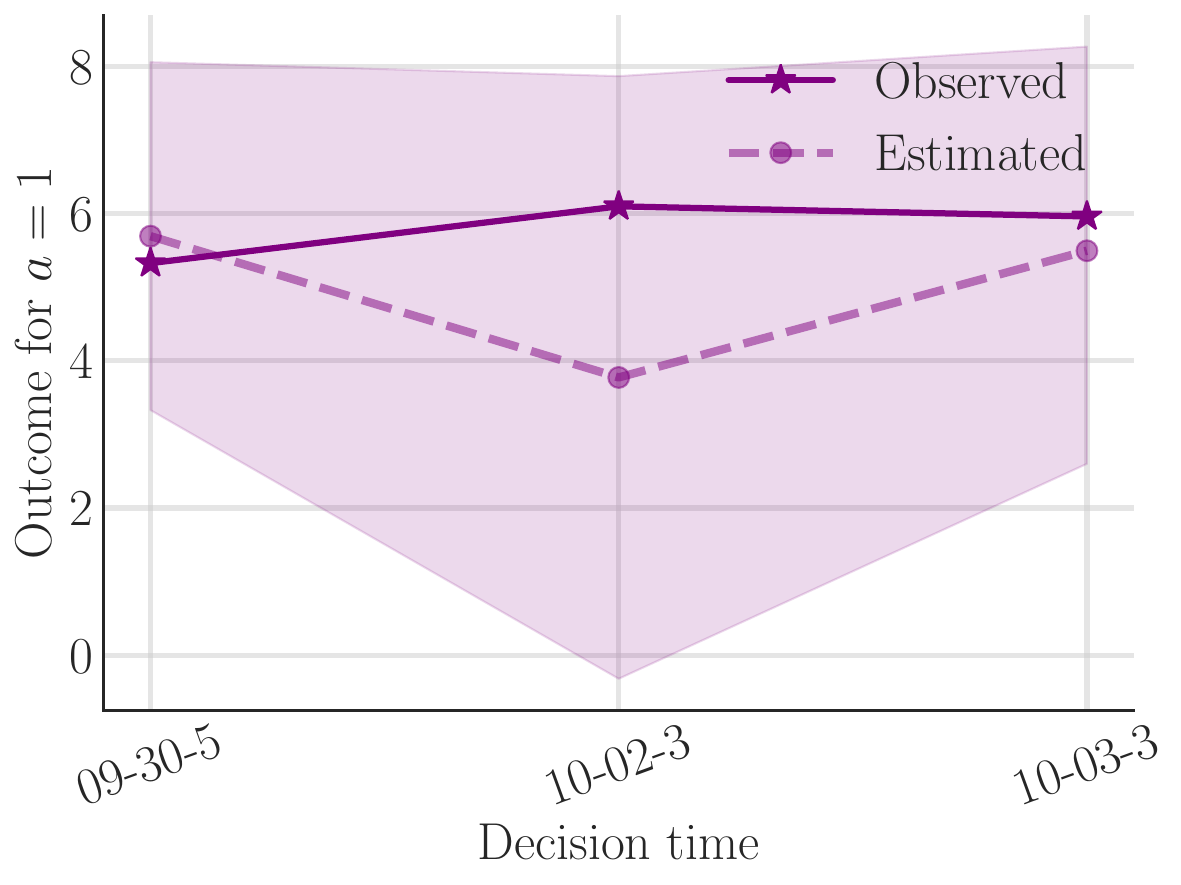} & 
    \includegraphics[width=0.35\textwidth]{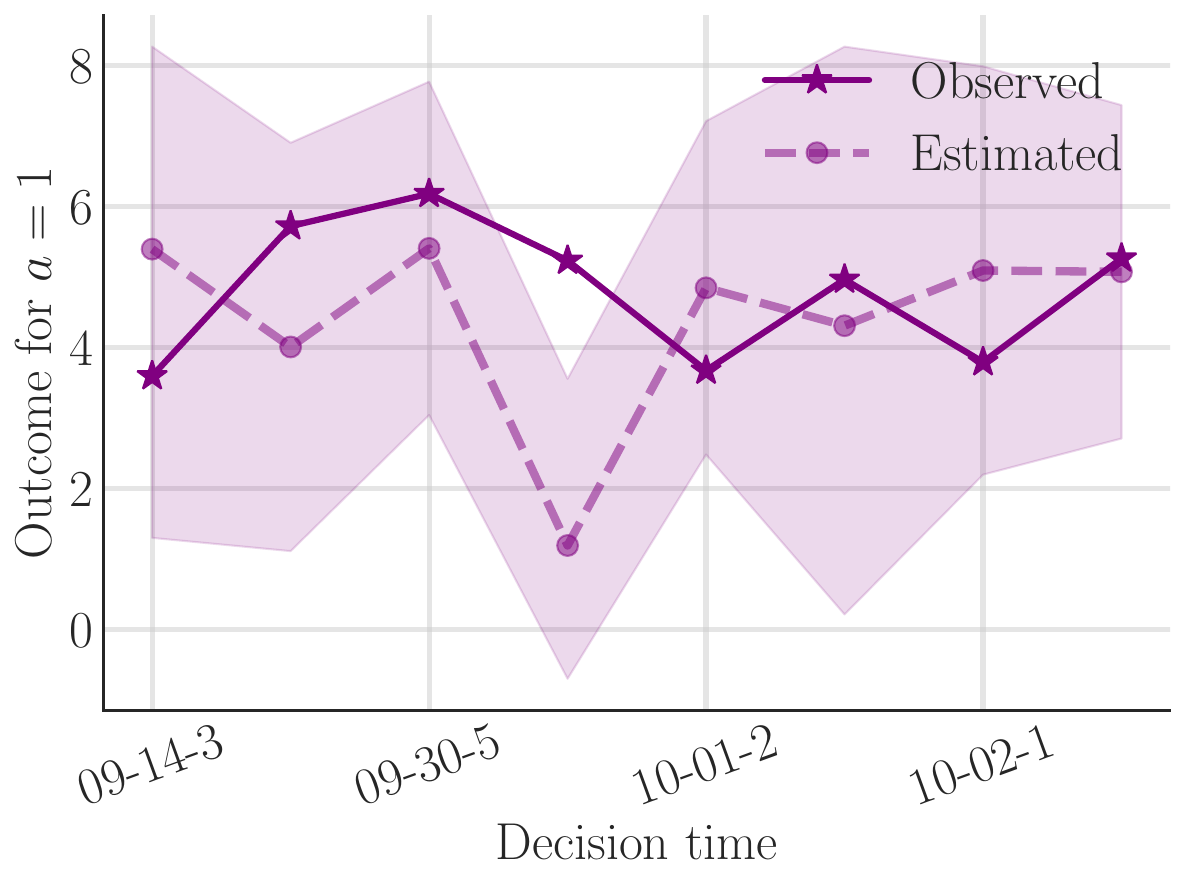} & 
    \includegraphics[width=0.35\textwidth]{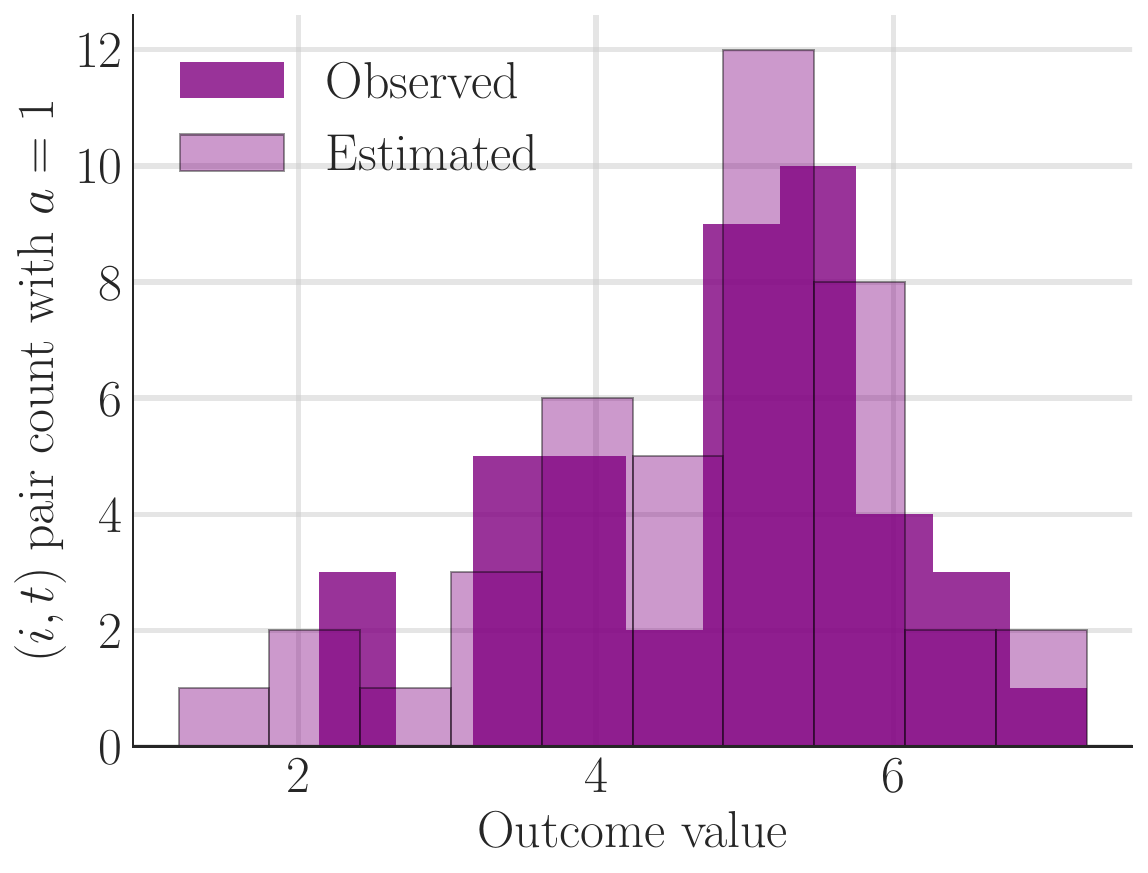} \\[-1mm]
    (d) Pooled user $1$ & (e)  Non-pooled user $1$ & (f) All users 
    \\[1mm] \hline \\[-1mm]
    \includegraphics[width=0.35\textwidth]{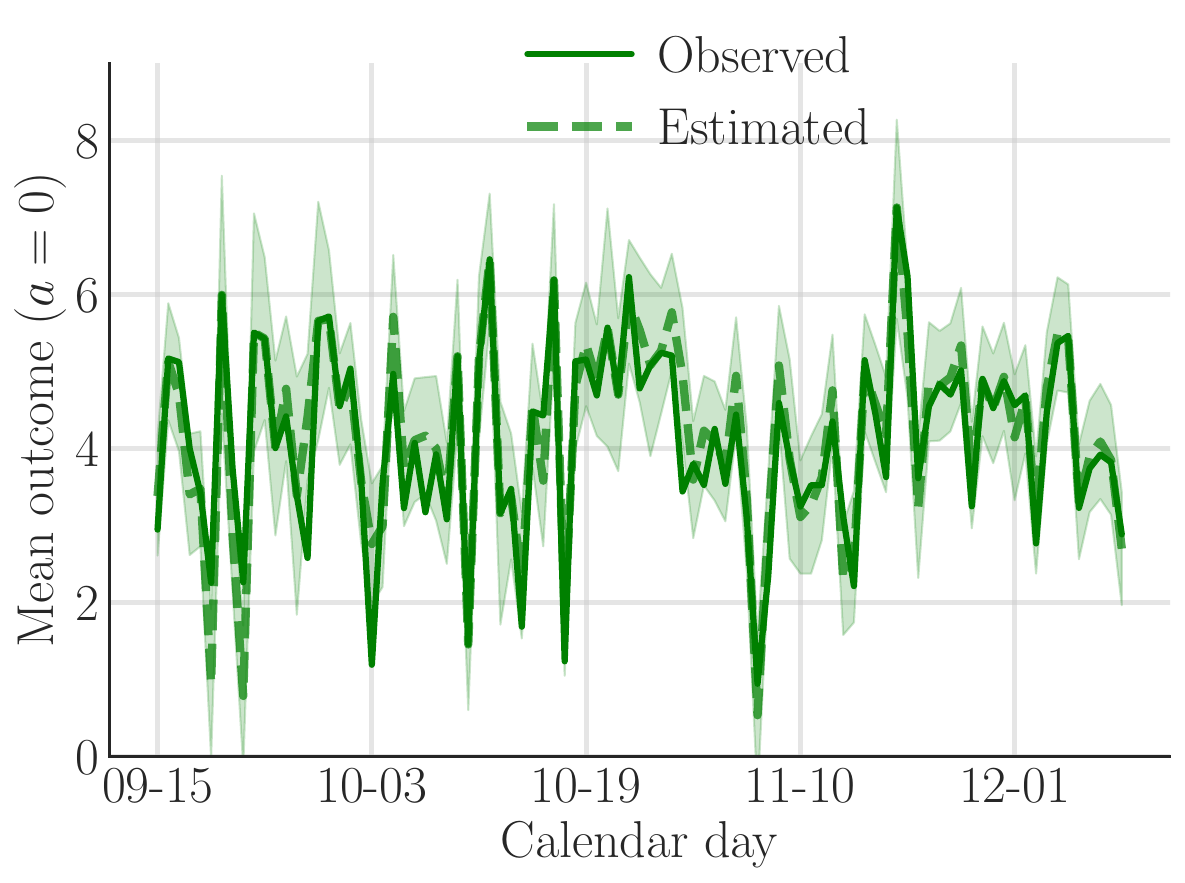}
    & \includegraphics[width=0.35\textwidth]{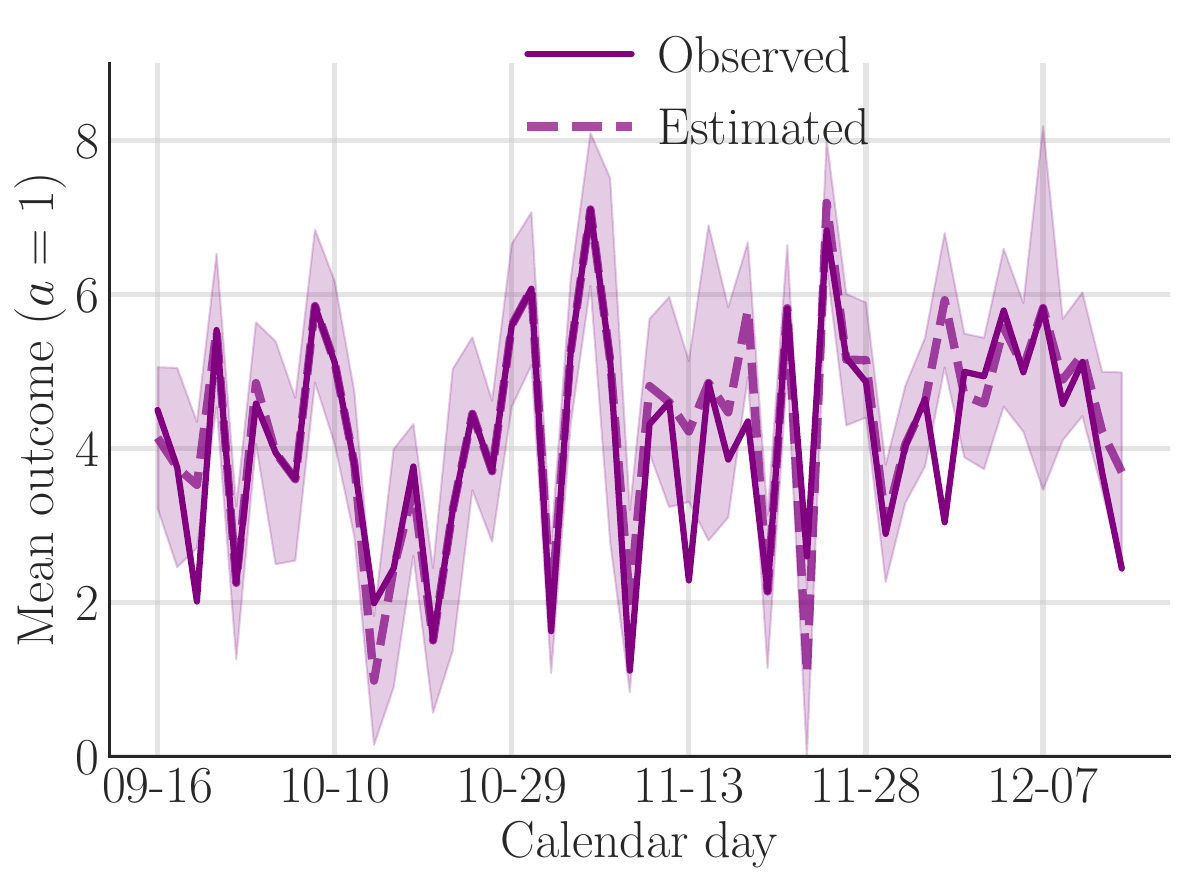}
    & \includegraphics[width=0.35\textwidth]{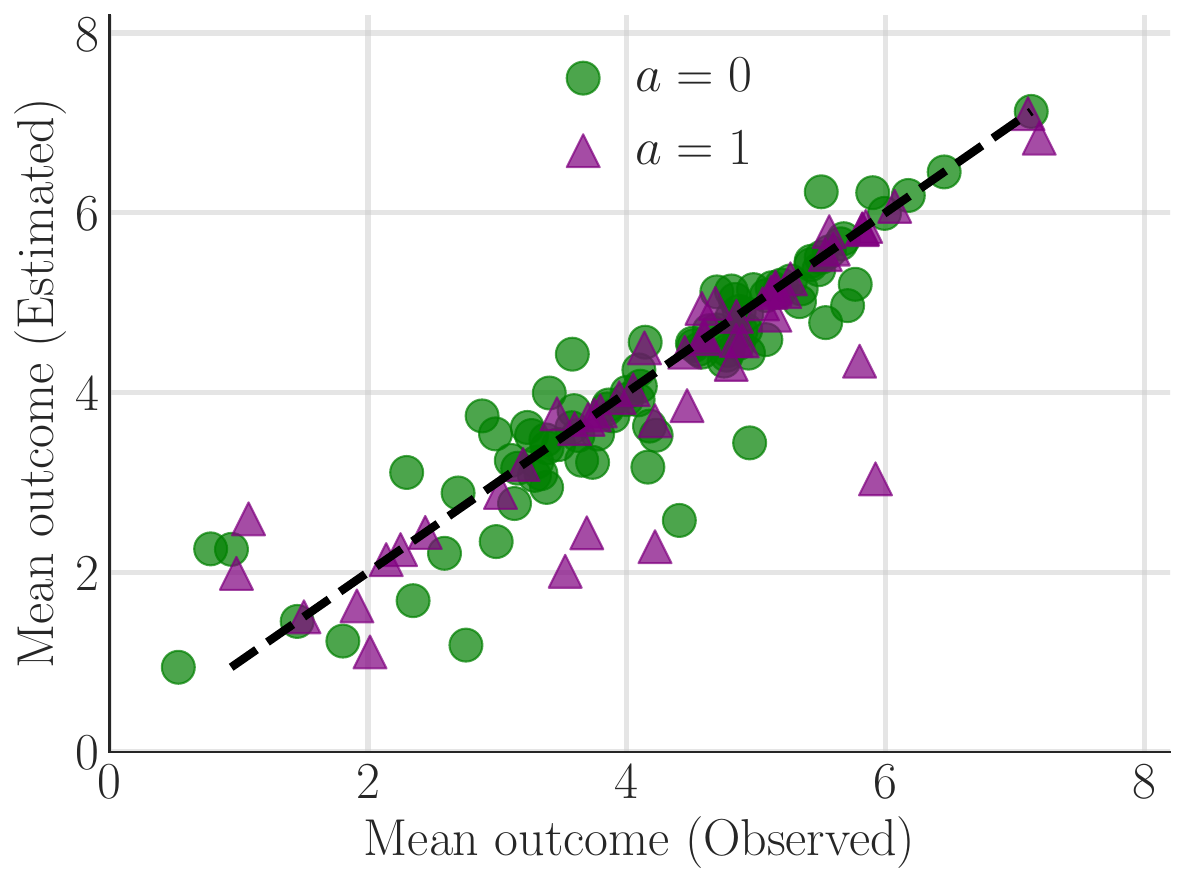} \\ 
    (g) Mean across users $(a=0)$ & (h) Mean across users $(a=1)$ & (i) Another visual for (g) and (h)
    \end{tabular}
    }
    \caption{\tbf{Nearest neighbors results for HeartSteps.} 
    The top and middle rows of the plot show the results for treatment $a=0$ and $a=1$, respectively. Panels (a, d) display the results for a specific pooled, and panels (b, e) for a specific non-pooled user; the x-axis represents decision time, dark-colored stars denote the observed outcomes, and light-colored circles denote the estimated outcomes. The confidence intervals are displayed as shaded regions. Panels (c) and (f) show histograms of observed and estimated outcomes for all $(\n,\t)$ pairs from the test time for $a=0$ and $a=1$, respectively. In panels (g) and (h), we present the results for outcomes averaged across users, with time on the x-axis and values on the y-axis. Observed values are shown in dark colors, estimated values in light colors, and the confidence intervals as light-colored shaded regions. Finally, panel (i) is a scatter plot of the mean outcomes from panels (g) and (h). Observed values are on the x-axis, and estimated values are on the y-axis. Values for $a=0$ are represented by green circles, and values for $a=1$ by purple triangles.
    }
    \label{fig:main_results}
    \vspace{-3mm}
\end{figure}

\subsubsection{Results}
Due to low availability of the users, the number of decision times with held out estimates is much less than $100$, and often we do not have enough neighbors to provide a reliable estimate; the situation is further exacerbated for treatment $1$ (due to low dosage as described above). In the sequel, when considering treatment $a$, we limit our attention to test decision times $\t$ such that the (i) the user has a positive value of outcome for treatment $a$ at $\t$, and (ii) we have at least two nearest neighbors with treatment equal to $a$ at time $\t$. %
We highlight that to add truth-in-advertising, we do not allow a user to be its own neighbor. This filtering and the sparsity of sending activity message led to 6 users having any valid test decision time for $a=1$, while 30 users have non-zero valid test times for $a=0$.

We present our results in \cref{fig:main_results}. First, we note that the tuned $\eta=8.72$, which leads to a variance estimate of $\what\sigma= 2.08$ which is rather large compared to the scale of outcomes, which lie between 0 and 8; see \ox{\cref{sub:heartsteps}}\oa{\cite[\cref{sub:heartsteps}]{supplement}} for further details on these choices. We plot the unit$\times$time-specific results for two users, one pooled user and one non-pooled user respectively, for treatment $0$ in panels (a) and (b) of \cref{fig:main_results}. For the same users, we present the results for treatment $1$ in panels (d) and (e) of \cref{fig:main_results} respectively.
We observe an empirical coverage---defined as the percentage of decision times where the (user-specific) confidence interval covers the observed outcome---of 67\%, 86\%, 100\%, and 88\% respectively in panels (a), (b), (d), and (e).
The relatively high value of $\what \sigma$ and the fact that we often have few neighbors leads to relatively wide unit$\times$time-specific intervals in panels (a, b, d, e) of \cref{fig:main_results}. In panels~(c) and (f), we plot the histogram of all held out unit$\times$time outcome values and the corresponding estimates, respectively for treatment $0$ and $1$. In panels (c) and (f), we omit the time points where the observed step count is $0$. We observe that the estimates for $a=0$ provide a good approximation for the test outcomes, and that of $a=1$ have more discrepancies with respect to the corresponding test outcomes. 

Finally, in the bottom row of \cref{fig:main_results}, we present the results for the mean outcome averaged across all users for the two treatments. For treatment $a$ at time $t$, we compute the mean observed outcome as $\overline{Y}_t \defeq \frac{\sum_{i=1}^{35} \obs \indicator(\miss=a)}{\sum_{i=1}^{35} \indicator(\miss=a)}$ and the mean estimated outcome as $\overline{\what\theta}^{(a)}_t \defeq \frac{\sum_{i=1}^{35} \estobs \indicator(\nestnbr\geq 2)}{\sum_{i=1}^{35} \indicator(\nestnbr\geq 2)}$. We plot these two quantities over time in panels (g) and (h) respectively for treatment $0$ and $1$, along with a $95\%$ confidence interval $\overline{\what\theta}^{(a)}_t\pm \frac{\what \sigma}{\sqrt{\sum_{i=1}^{35} \indicator(\nestnbr\geq 2)}}$ as a shaded region. We observe an empirical coverage of 91\% and 88\% respectively for $a=0$ and $a=1$. We also provide a scatter plot of the observed means and estimated means in \cref{fig:main_results}(i), where we observe a correlation of $0.91$ and $0.89$ (each with p-value $<10^{-10}$ for a t-test) between the observed and estimated means, for treatment 0 and 1 respectively.

\vspace{-3mm}
\begin{remark}
    For the pooled users, overtime the IP-TS algorithm should learn to depend less on global data for those users who are most unique. Thus, we expect that for the \emph{unique users}, the NN algorithm would find the fewer number of neighbors. We check this hypothesis by inspecting the correlation between the number of neighbors found by the NN algorithm for a given user, and that user's random effect (analogous to latent factor $\lunit$ in our set-up) estimated by IP-TS by the end of the feasibility study. We expect that as a user's random effect grows, the user should be more unique and their number of neighbors should decrease. 
    In \ox{\cref{fig:corr_plot}}\oa{\cite[\cref{fig:corr_plot}]{supplement}}, we find a correlation of $-0.69$ between the user's number of neighbors and the user's estimated random effects across the 8 pooled users (with a p-value of $0.06$ for a t-test).  This exploratory assessment provides some evidence that both the pooling and the NN algorithms are potentially capturing user heterogeneity across the pooled users.
\end{remark}

\section{Conclusion}
\label{sec:discussion}
In this work, we introduced a non-parametric latent factor model for counterfactual inference at \unittime-level in experiments with sequentially adaptive treatment policies. Using a variant of nearest neighbors algorithm, we estimate each of the counterfactual means and provide a non-asymptotic error bound, and asymptotic normality results for these means. We also use these estimates to construct confidence interval for the average treatment effect. We illustrate our theory via several examples, and then illustrate its benefits via simulations, and two case studies.

Our work naturally opens door to several future directions. While this paper focused on counterfactual inference in sequential experimental design, our \unittime-level results also can be easily extended to provide guarantees for a broad class of panel data settings (like synthetic control~\cite{xu2017generalized}, synthetic interventions~\cite{agarwal2021synthetic}, including those with staggered adoption~\cite{athey2021matrix}. While our model generalizes the non-linear mixed effects model, and the bilinear factor model considered in prior work, it does not cover the settings where time-varying features for units are available at each decision time (e.g., in contextual bandits). Our model also does not allow spill-over effects of treatment which arises in several real-world settings. 

\blue{Notably, our algorithm can be applied to tackle two settings easily: (i) categorical context variables by treating each possible covariate category separately and (ii) context variables that themselves follow a factor model structure, by simply treating the variables as latent. In either case, not leveraging the additional known information would necessarily lead to loss of statistical power, and thus developing a better strategy that allows us to learn from contexts is desirable. It would also be interesting to characterize the downstream gains of our \unittime-level estimates and the associated guarantees with adaptively collected data, e.g., in off-policy evaluation, or optimization, when the data is collected using a learning algorithm (like in \cref{assum:policy}).} 

Finally, we note that our work can also be useful for data imputation strategies. Our guarantees can help provide theoretical justification (with a generic non-parameteric factor model) for the success of nearest neighbor observed empirically in various missing data settings and multiple imputation strategies~\cite{beretta2016nearest,bertsimas2017personalized,bertsimas2017predictive,rubin2004multiple}--- when the data is missing at random in a sequential manner.

\begin{acks}[Acknowledgments]
The authors would like to thank Isaiah Andrews, Avi Feller, Natesh Pillai, Qingyuan Zhao, and Rina Barber for their helpful comments and suggestions that improved the
quality of this paper.
\end{acks}

\begin{funding}
RD, SM and DS acknowledge support by National Science Foundation under Grant No. DMS2023528 for the Foundations of Data Science Institute (FODSI). DS acknowledges support by NSF DMS-2022448, and DSO National Laboratories grant DSO-CO21070. PK acknowledges support by NIH NHLBI R01HL125440 and 1U01CA229445. SM was supported by NIH grants P50DA054039, P41EB028242 and UH3DE028723.  SM holds concurrent appointments at Harvard University and as an Amazon Scholar. This paper describes work performed at Harvard University and is not associated with Amazon.
\end{funding}

\renewcommand {\thepage} {\arabic{page}}
\ox{
\renewcommand {\theequation} {S\arabic{equation}}
\setcounter{equation}{0}}
\counterwithin{mylemma}{section}
\counterwithin{mydefinition}{section}
\counterwithin{myproposition}{section}
\counterwithin{mycorollary}{section}

\addtocontents{toc}{\protect\setcounter{tocdepth}{2}}
\begin{center}
	\bf \large
	\ox{Table of Contents for the Appendix}
\end{center}
\makeatletter
\@starttoc{toc}
\makeatother

\begin{appendix}

    \section{Proof of Thm.~\ref{thm:anytime_bound}: \anytimeboundname}
    \label{sec:proof_of_anytime_bound}
    We refer the reader to \cref{sub:proof_sketch} for an overview of our proof strategy, and proceed here directly. We have
    \begin{align}
        (\trueobs-\estobs)^2
        &=
        \parenth{\trueobs-\frac{\sum_{j \in
             \estnbr} \obs[j, \t]}{\nestnbr} }^2 \\
        &\seq{\cref{eq:model_mc}}\parenth{ \frac{\sum_{j \in \estnbr} (\trueobs[\n,\t]- \trueobs[j, \t]) + \noiseobs[j, \t] }{\nestnbr}  }^2 \\
        &\sless{(i)} 
        2 \parenth{ \frac{\sum_{j \in \estnbr} |\trueobs[\n,\t]- \trueobs[j, \t]|}{\nestnbr}  }^2
        + 2 \parenth{ \frac{\sum_{j \in \estnbr} \noiseobs[j, \t]}{\nestnbr}  }^2
        \defeq 2(\mbb B + \mbb V)
        \label{eq:det_final}
    \end{align}
    where step~(i) follows from the facts that $(a+b)^2 \leq 2a^2 + 2b^2$ and  $(a+b)^2 \leq (|a|+|b|)^2$ for all scalars $a, b$. We use $\mbb B$ to denote the bias term and $\mbb V$ to denote the variance term in \cref{eq:det_final}.

    To control the bias term, using \cref{assum:bilinear,assum:lambda_min} and Cauchy-Schwarz's inequality, we obtain that
    \begin{align}
     \frac{\sum_{j \in \estnbr} |\trueobs[\n,\t]- \trueobs[j, \t]|}{\nestnbr} 
     \seq{(\aref{\ref{assum:bilinear}})} \frac{\sum_{j \in \estnbr} |\angles{\lunit[\n], \ltime[\t]}- \angles{\lunit[j], \ltime[\t]}|}{\nestnbr} 
     &\leq \ltwonorm{\ltime} \frac{\sum_{j \in \estnbr} \ltwonorm{\lunit[\n]- \lunit[j]}}{\nestnbr} \\
     &\sless{(\aref{\ref{assum:lambda_min}})} \vconst \max_{j \in \estnbr}\ltwonorm{\lunit[\n]- \lunit[j]}.
     \label{eq:average_error} 
    \end{align}    
    Recall the notations~\cref{eq:eta_chi}. Define
    \begin{align}
    \label{eq:truedist}
        \truedist \defeq \norm{\lunit\!-\!\lunit[j]}^2_{\Sigv} + 2\sigma_{a}^2.
    \end{align}
    Fix $\delta\in(0, 1]$, and define the event $\event[1]$ such that
    \begin{align}
        \abss{\estdist\!-\!\truedist} \leq \frac{(\vconst\uconst+\nconst)^2\sqrt{32\log(4\N\T/\delta)}}{\pmint\sqrt{\T}} \seq{\cref{eq:eta_chi}} \errterm, \stext{for all} j \in [\N]\backslash\braces{\n}.
        \label{eq:event_dist}
    \end{align}
    Notably, $\event[1]$ denotes the event that the estimated distance between unit $\n$ and all other units $j$ concentrates around the expected value $\truedist$. As a direct consequence of \cref{assum:lambda_min}, we obtain that on event $\event[1]$,
    \begin{align}
        \max_{j \in \estnbr}\ltwonorm{\lunit[\n]- \lunit[j]}^2 \leq \max_{j \in \estnbr}\frac{\truedist-2\sigma_{a}^2}{\lambda_{\min}(\Sigma_{\ltime[]})} &\leq \frac{\estdist-2\sigma_{a}^2 + \errterm}{\lambda} \\
        &\sless{(i)} \frac{\threshold-2\sigma_{a}^2+\errterm}{\lambda}
        \seq{\cref{eq:eta_chi}} \frac{\threshold' + 2 \errterm}{\lambda}, 
    \label{eq:u_bound}
    \end{align}
    where step~(i) follows from the definition~\cref{eq:reliable_nbr} of $\estnbr$. Overall, on event $\event[1]$, putting together \cref{eq:average_error,eq:truedist,eq:event_dist}, we conclude that
    \begin{align}
    \label{eq:bias_bnd}
        \mbb B \leq \frac{ \vconst^2}{\lambda_a} (\threshold' + 2 \errterm)
        \seq{\cref{eq:eta_chi}} \frac{\vconst^2}{\lambda_a} (\threshold-2\sigma_{a}^2 + 2\errterm)
        \seq{\eqref{eq:nn_bnd}} \wtil{\mbb B}
    \end{align}

    Next, define the event $\event[2]$ that provides a direct bound on the variance:
    \begin{align}
         \mbb V &\leq \frac{24\sigma_{a}^2\log(\frac{8}{\delta})}{\pmint[\t]\ppau(\threshold') (\N\!\!-\!\!1)} \cdot \max\braces{1, \frac{4\nconst^2\log(\frac{8}{\delta}) /\sigma_{a}^2}{3\pmint[\t]\ppau(\threshold') (\N\!\!-\!\!1)}}
         \\ 
        &\qquad +\frac{72\nconst^2}{\pmint[\t]^2} \parenth{\frac{\ppau(\threshold'\!+\!2\errterm)}{\ppau(\threshold')}-1}^2  \max\braces{1, \frac{C\log^2(4/\delta)}{(\ppau(\threshold') (\N\!\!-\!\!1))^2 }} \\ 
        &\seq{\eqref{eq:nn_bnd_v}} \wtil{\mbb V}.
        \label{eq:event_noise}
    \end{align}
    Finally, our next lemma shows that the two events $\event[1]$ and $\event[2]$ hold with probability:
    \begin{lemma}
        \label{lemma:dist_noise_conc}
        Under the setting of \cref{thm:anytime_bound}, the events $\event[1]$~\cref{eq:event_dist}, and  $\event[2]$~\cref{eq:event_noise} satisfy
        \begin{align}
        \P(\event[1]\vert \ulf) \sgrt{(a)}  1-\frac{\delta}{2},
        \qtext{and}
        \P(\event[2], \event[1] \vert \lunit) \sgrt{(b)} 1-\frac{\delta}{2}-2\exp\parenth{-\frac{\pmint[\t]\ppau( \threshold')(\N\squash{-}1)}{8}}.
        \label{eq:event_dist_noise_conc}
    \end{align}
    \end{lemma}
    Taking \cref{lemma:dist_noise_conc} as given at the moment, we now proceed to finish the proof. Note that the bound~\cref{eq:det_final} when put together with \cref{eq:bias_bnd,eq:event_noise} implies the bound~\cref{eq:nn_bnd} under events $\event[1]\cap\event[2]$, and
    \begin{align}
        \P(\event[1] \cap \event[2] \vert \lunit)
        = \P(\event[1]\vert \lunit) \cdot \P(\event[2]\vert \event[1], \lunit)
        \sgrt{\cref{eq:event_dist_noise_conc}} 1-\delta-2\exp\parenth{-\frac{\pmint[\t]\ppau( \threshold')(\N\squash{-}1)}{8}},
    \end{align}
    which immediately yields the claimed high probability result of \cref{thm:anytime_bound}.
    
    It remains to prove the two high probability bounds in \cref{lemma:dist_noise_conc}, which we do one-by-one. In particular, we first establish \cref{eq:event_dist_noise_conc}(a) in \cref{sub:proof_of_distance_conc}, and then \cref{eq:event_dist_noise_conc}(b) in \cref{sub:proof_of_noise_conc}.

    We also make use of the following high probability bound for binomial random variables throughout our proofs:
   \begin{lemma}
   \label{lem:bc_bound}
       Given $X_{\l} \sim \mrm{Bernoulli}(p_{\l})$ for  $\l\in [\N]$, we have
       \begin{align}
           \P\parenth{\sum_{\l=1}^{N} X_{\l} \leq \frac{\sum_{\l=1}^{\N} p_{\l}}{2} } \leq \exp\parenth{-\frac{\sum_{\l=1}^{\N} p_{\l}} {8}}.
       \end{align}
   \end{lemma}
   \begin{proof}
   The standard Binomial-Chernoff bound~\citep[Thm.~4.2]{motwani1995randomized} for independent bernoulli variables $\braces{X_{\l}}$ with $\mu \defeq \sum_{\l=1}^{\N} \E[X_{\l}] =\sum_{\l=1}^{\N} p_{\l} $ is given by
   \begin{align}
       \P\parenth{\sum_{\l=1}^{\N} X_{\l} \leq (1-\delta)\mu }
       \leq \exp\parenth{-\frac{\delta^2\mu}{2}}
       \qtext{for} \delta  \in (0, 1).
   \end{align}
   Substituting $\delta=\half$ yields the claim.
   \end{proof}

    \subsection{Proof of \cref{lemma:dist_noise_conc}: Proof of inequality \cref{eq:event_dist_noise_conc}(a)}
    \label{sub:proof_of_distance_conc}
    Without loss of generality we can assume $a=1$, so that we can use the simplified notation $\indicator(\miss=a) = \miss$.
    Our proof makes use of a carefully constructed Martingale argument, so we start with the statement of the concentration bound.

    \paragraph{Azuma Martingale Concentration} 
    \cite[Thm.~16]{chung2006concentration} states the following high probability bound: Given a sequence of variables $\braces{X_i}_{i=1}^{n}$ adapted to the filtration $\braces{\history[i]}_{i=1}^n$, we can construct the Martingale $S_i \defeq \sum_{j=1}^{i} (X_j - \E[X_j\vert \history[j-1]])$ for $i \geq 1$, and $S_0 \defeq 0$. If we have $|X_i-\E[X_i\vert \history[i-1]]|\leq M$ almost surely for all $i$, then we have
    \begin{align}
    \label{eq:hpb_bound_martingale}
        \sabss{S_n} =  \sabss{\sum_{i=1}^n X_i - \sum_{i=1}^n \E[X_i\vert \history[i-1]]} \leq M \sqrt{n\log(2/\delta)},
    \end{align}
    with probability at least $1-\delta$.

    \paragraph{Proof} We proceed in three steps: (i) Constructing a suitable martingale, (ii) relating it to the distance $\estdist$, and (iii) applying martingale concentration.

    \paragraph{Step~(i): Constructing the martingale}
    For any fixed $j \in [\N]\backslash\braces{\n}$, define the sequence $\braces{\tstopj}_{\l=0}^{\infty}$ as follows: Set $\tstopj[0]\defeq 0$, and for $\l = 1, 2, \ldots$, define
    \begin{align}
    \label{eq:tstop}
        \tstopj[\l] = \begin{cases}
            \min\braces{\t: \tstopj[\l-1] < \t \leq \T, \miss\miss[j,\t] = 1 } \stext{if feasible,} \\
            \T+1 \qtext{otherwise.}
        \end{cases}
    \end{align}
    That is, $\tstopj[\l]$ denotes the time point such that the units $\n$ and $j$ simultaneously receive the treatment $1$ for the $\l$-th time. 
    We note that $\tstopj$ is a stopping time with respect to the filtration $\braces{\history[\t]}_{\t=1}^{\T}$, i.e., $\braces{\tstopj \leq \t} \in \history[\t]$. Let $\mc H_{\l}$ denote the sigma-field generated by the stopping time $\tstopj$. 

    Now, recall the definition~\cref{eq:truedist} of $\truedist$, and define the sequence
    \begin{align}
    \label{eq:wl}
        W_{\l} \defeq \indicator(\tstopj[\l] \leq \T) \parenth{ (\obs[\n, \tstopj] - \obs[j, \tstopj])^2 \!-\! \truedist}, \qquad \l = 1, 2, \ldots.
    \end{align}
    Introduce the shorthand $\ybnd=\vconst\uconst+\nconst$. Then clearly have $|W_{\l}|\leq 4\ybnd^2$ for all $\l$, almost surely. Furthermore, we also have
    \begin{align}
        \E[W_{\l}\vert \mc H_{\l}, \ulf] 
         &\!=\! \E\brackets{\indicator(\tstopj[\l] \leq \T) \parenth{(\obs[\n, \tstopj] \!-\! \obs[j, \tstopj])^2 \!-\! \truedist} \vert \mc H_{\l}, \ulf}\\ 
         &\!=\!\indicator(\tstopj[\l] \leq \T)  \cdot  \E\brackets{\parenth{(\obs[\n, \tstopj] \!-\! \obs[j, \tstopj])^2 \!-\! \truedist }\vert \mc H_{\l}, \ulf}  \\ 
        &\!\seq{(i)}\! \indicator(\tstopj[\l] \leq \T) \cdot (\truedist - \truedist) =0,
        \label{eq:zero_mean_martingale}
    \end{align} 
    where step~(i) when $\indicator(\tstopj[\l] \leq \T)=1$ follows from the fact that conditioned on $\tstopj=\t \leq \T$, the distribution of $(\ltime[\t], \noiseobs, \noiseobs[j, \t])$ is independent of the treatments $\braces{\miss[j, \t]}_{j\in[\N]}$ at that time since the policy is sequential~(\cref{assum:policy}), and the latent time factors at time $\t$, and noise variables at time $\t$ are exogenous and drawn independently of the policy at time $\t$ (\cref{assum:lambda_min,assum:zero_mean_noise}), so that they are independent of the event $\tstopj = \t$ and thereby the sigma-algebra $\mc H_{\l}$, which in turn yields that
    \begin{align}
        &\E\brackets{(\obs[\n, \tstopj] \!-\! \obs[j, \tstopj])^2\vert \tstopj=\t, \mc H_{\l}, \ulf} \\ 
        &= (\lunit-\lunit[j])\tp \E[\ltime\ltime\tp\vert \tstopj=\t, \mc H_{\l}, \ulf](\lunit-\lunit[j]) + \E[(\noiseobs-\noiseobs[j, \t])^2\vert \tstopj=\t, \mc H_{\l}, \ulf] \\
        &\qquad+ (\lunit-\lunit[j])\tp \E[\ltime(\noiseobs-\noiseobs[j, \t])\vert \tstopj=\t, \mc H_{\l}, \ulf] \\ 
        &= (\lunit-\lunit[j])\tp \Sigv (\lunit-\lunit[j]) + 2\sigma_{a}^2 +0
        = \truedist.
    \end{align}
    Putting the pieces together, we conclude that $\braces{W_{\l}}_{\l=0}^{\infty}$ is a bounded Martingale difference sequence with respect to the filtration $\braces{\mc H_{\l}}_{\l=0}^{\infty}$ conditioned on the unit factors $\ulf$. 

    \paragraph{Step~(ii): Relating the martingale to $\estdist$}
    Define $\T_{\n, j} \defeq \sumt \miss \miss[j, \t]$. Under the event $\T_{\n, j}>0$, we have
    \begin{align}
        \frac{\sumt\miss\miss[j, \t] (\obs\!\!-\!\!\obs[j, \t])^2}{\sumt \miss \miss[j, \t]} - \truedist 
        &=
        \frac{\sumt\miss\miss[j, \t] ((\obs\!\!-\!\!\obs[j, \t])^2-\truedist)}{\T_{\n, j}} \\
        &\seq{\cref{eq:tstop}} \frac{\sum_{\l=1}^{\T_{\n, j}} \parenth{(\obs[\n, \tstopj]\!\!-\!\!\obs[j, \tstopj])^2-\truedist}}{\T_{\n, j}}  \\
        &\seq{(i)} \frac{\sum_{\l=1}^{\T_{\n, j}}  \indicator(\tstopj[\l] \leq \T) \parenth{(\obs[\n, \tstopj]\!\!-\!\!\obs[j, \tstopj])^2-\truedist}}{\T_{\n, j}}  \\
        &\seq{\cref{eq:wl}}\frac{\sum_{\l=1}^{\T_{\n, j}} W_{\l}}{\T_{\n, j}},
        \label{eq:w_rho_reln}
    \end{align}
    where step~(i) follows from the fact that for any $\l \leq \T_{\n, j}$, we have $\tstopj[\l] \leq \T$ by definition.

    \paragraph{Step~(iii): Applying martingale concentration}
    Fix $\delta>0$. Now, applying Azuma Hoeffding concentration bound~\cref{eq:hpb_bound_martingale} for the sequence $\sbraces{W_{\l}}_{\l=0}^{\infty}$ adapted to the filtration $\sbraces{\mc H_{\l}}_{\l=0}^{\infty}$, we find that for any fixed $K>0$, we have
    \begin{align}
        \P\brackets{\abss{\frac{\sum_{\l=1}^{K} W_{\l}}{K}} \leq \frac{4\ybnd^2 \sqrt{\log(2/\delta)} }{\sqrt{K}} \vert \ulf} \geq 1-\delta. 
    \end{align}
    Applying a union bound over $K \in [\T]$, we find that
    \begin{align}
    \label{eq:k_union_bnd}
        \P\brackets{\abss{\frac{\sum_{\l=1}^{K} W_{\l}}{K}} \leq \frac{4\ybnd^2\sqrt{\log(2/\delta)}}{\sqrt{K}} \stext{for all} K \in [\T] \vert \ulf}  \geq 1-\T\delta.
    \end{align}
    Note that the bound~\cref{eq:k_union_bnd} holds uniformly for all $1 \leq K \leq \T$. 
    By definition $\T_{\n, j} \leq \T$, 
    so that
    \begin{align}
        &\P\brackets{\abss{\frac{\sumt\miss\miss[j, \t] (\obs\!\!-\!\!\obs[j, \t])^2}{\sumt \miss \miss[j, \t]} - \truedist } \leq  \frac{4\ybnd^2\sqrt{\log(2/\delta)}}{\sqrt{\T_{\n, j}}} , \T_{\n, j}>0 \vert \ulf  } \\
        &\seq{(i)}\sum_{K'=1}^{\T}\P\brackets{\abss{\frac{\sumt\miss\miss[j, \t] (\obs\!\!-\!\!\obs[j, \t])^2}{\sumt \miss \miss[j, \t]} - \truedist } \leq  \frac{4\ybnd^2\sqrt{\log(2/\delta)}}{\sqrt{\T_{\n, j}}} , \T_{\n, j} = K', \vert \ulf  } \\
        &\seq{\cref{eq:w_rho_reln}}\sum_{K'=1}^{\T}\P\brackets{\abss{\frac{\sum_{\l=1}^{K'} W_{\l}}{K'}} \leq \frac{4\ybnd^2\sqrt{\log(2/\delta)}}{\sqrt{K'}}, \T_{\n, j} = K', \vert \ulf  } \\
        &\geq\sum_{K'=1}^{\T}\P\brackets{\abss{\frac{\sum_{\l=1}^{K} W_{\l}}{K}} \leq \frac{4\ybnd^2\sqrt{\log(2/\delta)}}{\sqrt{K}} \stext{for all} K \in [\T], \T_{\n, j}=K' \vert \ulf} \\
        &\seq{(ii)}\P\brackets{\abss{\frac{\sum_{\l=1}^{K} W_{\l}}{K}} \leq \frac{4\ybnd^2\sqrt{\log(2/\delta)}}{\sqrt{K}} \stext{for all} K \in [\T], \T_{\n, j}>0 \vert \ulf} \\
         &\sgrt{(iii)} \P\brackets{\abss{\frac{\sum_{\l=1}^{K} W_{\l}}{K}} \leq \frac{4\ybnd^2\sqrt{\log(2/\delta)}}{\sqrt{K}} \stext{for all} K \in [\T] \vert \ulf} +
         \P\brackets{ \T_{\n, j}>0 \vert \ulf}-1 \\
         &\sgrt{\cref{eq:k_union_bnd}} \P\brackets{ \T_{\n, j}>0 \vert \ulf} - \T\delta,
         \label{eq:second_last_bound_distance}
    \end{align}
    where steps~(i) and (ii) follow from the law of marginal probability and step~(iii) from the fundamental fact that $\P(S_1 \cap S_2) =  \P(S_1) + \P(S_2) - \P(S_1 \cup S_2) \geq \P(S_1) + \P(S_2) - 1 $ for any two events $S_1$ and $S_2$. 

    Now, \cref{assum:policy,eq:pmin} imply that that
    \begin{align}
        \E[\miss\miss[j,\t] \vert \ulf] = \E[\E[\miss\miss[j,\t]\vert \ulf, \history[\t-1]]\vert \ulf] \seq{\aref{\ref{assum:policy}}}
        \E[\E[\miss\vert \history[\t-1], \ulf] \E[\miss[j,\t]\vert \history[\t-1], \ulf] \vert \ulf]
        \sgrt{\cref{eq:pmin}} \pmint[\t]^2,
    \end{align}
    and that $\pmint[\t]\geq\pmint[\T]$ for all $\t\in[\T]$.
    Thus applying \cref{lem:bc_bound} with $p_{\l}=\pmint^2$, we conclude that for any fixed pair $(\n, j)$, we have
    \begin{align}
    \label{eq:tij_bnd}
        \P\brackets{\T_{\n, j} \geq \frac{\pmint^2\T}{2} \vert \ulf} \geq 1-\delta, \qtext{if} \pmint^2\geq \frac{8\log(1/\delta)}{\T}.
    \end{align}
    Putting the pieces together, and taking a union bound over $j \in [\N]\backslash\braces{\n}$, 
    \begin{align}
        &\P\brackets{\abss{\frac{\sumt\miss\miss[j, \t] (\obs\!\!-\!\!\obs[j, \t])^2}{\sumt \miss \miss[j, \t]} - \truedist } \leq \frac{\ybnd^2\sqrt{32\log(2/\delta)}}{\pmint\sqrt{\T}}, \stext{for all} j \in [\N]\backslash\braces{\n} \vert \ulf}  \\ &\qquad\geq 1\!-\!(\T\!+\!1)(\N\!-\!1)\delta,
    \end{align}
    which in turn implies that
    \begin{align}
        \P\brackets{\abss{\estdist\!-\!\truedist} \leq \frac{\ybnd^2\sqrt{32\log(2/\delta)}}{\pmint\sqrt{\T}}, \stext{for all} j \in [\N]\backslash\braces{\n} \vert \ulf}  &\!\geq\! 1\!-\!(\T\!+\!1)(\N\!-\!1)\delta,
        \label{eq:event_dist_conc}
    \end{align}
    when $\pmint^2 \geq \frac{8\log(1/\delta)}{\T}$. Replacing $\delta$ by $\frac{\delta}{2\N\T}$ yields the inequality (a) in \cref{eq:event_dist_noise_conc}.

    \subsection{Proof of \cref{lemma:dist_noise_conc}: Proof of inequality \cref{eq:event_dist_noise_conc}(b)}
    \label{sub:proof_of_noise_conc}    
    We refer the reader to \cref{sub:sandwich} for a proof sketch of this part. Without loss of generality, we will assume $a=1$ throughout this proof, and use the simplified notation $\indicator(\miss=a) = \miss$. We use $\mbf S \Delta \mbf T$ to denote the set difference between two sets $\mbf S$ and $\mbf T$.
    
    We proceed in four steps: (I) The sandwich argument to decompose the variance term, (II) upper bounding the variance terms in terms of neighborhood sizes, (III) lower and upper bounding the neighborhood sizes, and (IV) putting the pieces together.

    \paragraph{Step~(I): The sandwich argument}
    Conditional on $\ulf$, define the following sets:
    \begin{align}
    \label{eq:nstar}
        \nstar &\defeq \braces{j \neq \n: \truedist + \errterm \leq \threshold },\qquad \nnstar = |\nstar|, \\
        \nwstar &\defeq \braces{j \neq \n: \truedist - \errterm \leq \threshold },  \qquad \nnwstar = |\nwstar|,
        \label{eq:nwstar}
        \\
        \nhat &\defeq \braces{j \neq \n: \estdist \leq \threshold},\qquad \ \ 
        \qtext{and} \nnhat = |\nhat|.
        \label{eq:nhat}
    \end{align}
    Note that the sets $\nstar$, and $\nwstar$ are contained (and $\nhat$ is not) in the sigma-algebra of $\ulf$. And, by definition,
    \begin{align}
    \label{eq:n_set_sandwich}
        \nstar \subseteq \nwstar, \qtext{and} \estnbr \seq{\cref{eq:reliable_nbr}} \sbraces{j \in \nhat\suchthat \miss[j, \t]=1} \subseteq \nhat
        \qtext{almost surely.}
    \end{align}
    We claim that under the event $\event[1]$, the sets $\nstar$, and $\nwstar$ sandwich $\nhat$ as follows:
    \begin{align}
    \label{eq:n_reln}
        \nstar \subseteq \nhat\subseteq \nwstar.
    \end{align}
    This claim follows immediately by making the following two observations:
    \begin{align}
        &j \in \nstar \implies
        \estdist \sless{\cref{eq:event_dist_conc}} 
        \truedist + \errterm 
        \sless{\cref{eq:nstar}} \threshold
        \stackrel{\cref{eq:nhat}}{\implies} j \in \nhat,
        \qtext{and}\\
        &j \in \nhat \implies
          \truedist -\errterm
         \sless{\cref{eq:event_dist_conc}}  \estdist 
        \sless{\cref{eq:nhat}} \threshold
        \stackrel{\cref{eq:nwstar}}{\implies} j \in \nwstar.
    \end{align}
    Now, under the event $\event[1]$ and conditional on $\lunit$, using the definitions~\cref{eq:reliable_nbr,eq:nhat}, we find that
    \begin{align}
    \label{eq:noise_decomp}
        \mbb V = \parenth{ \frac{\sum_{j \in \estnbr} \noiseobs[j, \t]}{\nestnbr}  }^2&=\parenth{\frac{\sum_{j\in\nhat}\miss[j, t]\noiseobs[j, \t]}{\sum_{j\in\nhat}\miss[j, t]}}^2\\
        &\sless{(i)} 2\parenth{\frac{\sum_{j\in\nstar}\miss[j, t]\noiseobs[j, \t]}{\sum_{j\in\nhat}\miss[j, t]}}^2
        +  2\parenth{\frac{\sum_{j\in\nhat\backslash\nstar}\miss[j, t]\noiseobs[j, \t]}{\sum_{j\in\nhat}\miss[j, t]}}^2 \\
        &\sless{(ii)} 2\parenth{\frac{\sum_{j\in\nstar}\miss[j, t]\noiseobs[j, \t]}{\sum_{j\in\nstar}\miss[j, t]}}^2
        +  2\frac{\bigparenth{\sum_{j\in\nhat\backslash\nstar}\miss[j, t]\noiseobs[j, \t]}^2 }{(\sum_{j\in\nstar}\miss[j, t])^2} \\
        &\sless{(iii)} 2\parenth{\frac{\sum_{j\in\nstar}\miss[j, t]\noiseobs[j, \t]}{\sum_{j\in\nstar}\miss[j, t]}}^2
        +  \frac{2|\nhat\Delta \nstar|^2\nconst^2}{(\sum_{j\in\nstar}\miss[j, t])^2}\\
        &\sless{(iv)} 2\parenth{\frac{\sum_{j\in\nstar}\miss[j, t]\noiseobs[j, \t]}{\sum_{j\in\nstar}\miss[j, t]}}^2
        +  \frac{2|\nwstar\Delta \nstar|^2\nconst^2}{(\sum_{j\in\nstar}\miss[j, t])^2}
        \defeq 2(\mbb V_1 + \mbb V_2),
        \label{eq:v1_v2}
    \end{align}
    where step~(i) follows from the fact that $(a+b)^2\leq 2(a^2+b^2)$ for all scalars $a, b$, and step~(ii) from the fact that $\sum_{j\in\nhat}\miss[j, t] \geq \sum_{j\in\nstar}\miss[j, t]$ since $\nstar\subseteq\nhat$~\cref{eq:n_reln} under the event $\event[1]$, and step~(iii) from the boundedness of the noise variables assumed in \cref{assum:zero_mean_noise}, and step~(iv) from the relation $\nhat\subseteq\nwstar$~\cref{eq:n_reln} under event $\event[1]$.

    \paragraph{Step~(II): Upper bounding the variance terms $\mbb V_1$ and $\mbb V_2$}
    The following lemma provides a high probability bound on the variance quantities in terms of suitable neighborhood sizes, and notably the probability is established conditional on all the unit latent factors $\ulf$. See \cref{sub:proof_of_lem:var_bnd} for its proof.
    \begin{lemma}[Upper bounds on the variance terms]
        \label{lem:var_bnd}
        For any fixed $\delta \in [0, 1)$, and the event $\event[3]$ such that
        \begin{align}
        \label{eq:var_bnds}
            \mbb V_1 \leq \frac{6\sigma_{a}^2\log(\frac{8}{\delta})}{\pmint[\t]\nnstar} \cdot \max\braces{1, \frac{2\nconst^2\log(\frac{8}{\delta})}{3\sigma_{a}^2\pmint[\t]\nnstar}}
            \qtext{and}
            \mbb V_2 \leq \frac{4|\nwstar\Delta \nstar|^2\nconst^2}{\pmint[\t]^2(\nnstar)^2},
        \end{align}
        we have $\P(\event[3] \vert\ulf) \geq 1-\frac{\delta}{4}-\exp\parenth{-\frac{\pmint[\t]\nnstar}{8}}$.
    \end{lemma}

    \paragraph{Step~(III): Bounding the neighborhood sizes} Given \cref{lem:var_bnd}, our next step is to establish lower bound on the number of neighbors $\nnstar$, and an upper bound on $ |\nwstar\Delta \nstar|$ conditional on just the unit $\n$'s latent factor $\lunit^{(a)}$. The following lemma (with proof in \cref{sub:proof_of_lem:nset_bnd}) provides a high probability bound for these quantities:
    \begin{lemma}[Lower bound on the neighborhood sizes]
        \label{lem:nset_bnd}
        For any fixed $\delta \in [0, 1)$, let $\betadiff \defeq \ppau\parenth{\threshold'\!+\!2\errterm}\!-\!\ppau\parenth{\threshold'}$, and $\event[4]$ denote the event such that
        \begin{align}
        \label{eq:nset_bnd}
            \nnstar \sgrt{(a)} \frac{\ppau(\threshold') (\N\!\!-\!\!1)}{2} 
            \qtext{and}
            |\nwstar\Delta \nstar| \sless{(b)} \max\braces{ \frac{3\betadiff (\N\!\!-\!\!1)}{2} , C \log(\frac{4}{\delta})},
        \end{align}
        for some universal constant $C>0$. Then $\P(\event[4]\vert\lunit) \geq 1 - \frac{\delta}{4} - \exp\bigparenth{-\frac{\ppau(\threshold')(\N-1)}{8}}$.
    \end{lemma}
    \begin{remark}
    \label{rem:zero_beta}
        We note in the proof of \cref{lem:nset_bnd} (in \cref{sub:proof_of_lem:nset_bnd}) that if $\betadiff=0$, a direct argument yields that $|\nstar \Delta \nwstar| =0$ almost surely. This fact is used directly in \cref{cor:anytime_bound}\cref{item:finite}, \cref{cor:anytime_asymp}\cref{item:finite_asymp}, and \cref{cor:ate_asymp}\cref{item:finite_ate_asymp}.
    \end{remark}

    \paragraph{Step~(IV): Putting it together}
    Collecting all steps together, we conclude that the event $\cap_{\l=1}^{4}\event[\l]$ conditional on $\lunit$, we have
    \begin{align}
        \mbb V &\leq \frac{24\sigma_{a}^2\log(\frac{8}{\delta})}{\pmint[\t]\ppau(\threshold') (\N\!\!-\!\!1)} \cdot \max\braces{1, \frac{4\nconst^2\log(\frac{8}{\delta}) /\sigma_{a}^2}{3\pmint[\t]\ppau(\threshold') (\N\!\!-\!\!1)}}
         \\ 
        &\qquad +\frac{72\nconst^2}{\pmint[\t]^2}  \max\braces{ \parenth{\frac{\ppau(\threshold'\!+\!2\errterm)}{\ppau(\threshold')}-1}^2, \frac{C\log^2(4/\delta)}{(\ppau(\threshold') (\N\!\!-\!\!1))^2 }},
    \end{align}
    with probability at least
    \begin{align}
        &1-\frac{\delta}{2} - \exp\parenth{-\frac{\pmint[\t]\ppau(\threshold')(\N-1)}{16}} -\exp\bigparenth{-\frac{\ppau(\threshold')(\N-1)}{8}}\\
        &\geq 1-\frac{\delta}{2}-2 \exp\parenth{-\frac{\pmint[\t]\ppau(\threshold')(\N-1)}{16}},
    \end{align}
    as claimed, and we are done.

    \subsection{Proof of \cref{lem:var_bnd}}
    \label{sub:proof_of_lem:var_bnd}
    First, applying \cref{lem:bc_bound} we find that
    \begin{align}
    \label{eq:a_in_nstar_bound}
          \P\parenth{ \sum_{j\in\nstar}\!\!\!\! \miss[j, \t] \geq \frac{\pmint[\t] \nnstar}{2} \, \vert\,  \mc U}
        &\geq 1 - \exp\parenth{-\frac{\pmint[\t] \nnstar}{8}}.
    \end{align}
    Next, we note that the noise variables $\braces{\noiseobs[j, \t]}_{j\in[\N]}$ are independent of the latent unit factors $\ulf$ due to \cref{assum:zero_mean_noise} and hence also of $\nstar$, and also of the treatments $\braces{\miss[j, t]}$ due to \cref{assum:policy}. Consequently, an application of Berstein's concentration inequality\footnote{If $|X_i-\mu_i|<b$ a.s., $\Var(X_i) \leq \sigma_{a}^2$, then $\P(|\frac1n|\sum_{i=1}^n( X_i-\mu_i)| \geq t)| \leq 2\exp(-nt^2/(\sigma_{a}^2+bt/3))$} yields that
    \begin{align}
        \P\brackets{\abss{\frac{\sum_{j\in\nstar}\!\!\miss[j, t]\noiseobs[j, \t]}{\sum_{j\in\nstar}\!\!\miss[j, t]}} \!\!\leq\! \max\biggbraces{\frac{\sigma\sqrt{3\log(\frac8\delta)}}{\sqrt{K}}, \frac{2\nconst\log(\frac8\delta)}{K} } \,\bigg\vert \,\ulf, \braces{\miss[j, \t]}_{j=1}^{\N}, \!\!\!\sum_{j\in\nstar}\!\!\!\!\!\!\miss[j, \t] \!=\!K}  \\
        \geq 1-\frac{\delta}{4}.
        \label{eq:bernstein}
    \end{align}
    Since the bound is independent of $ \braces{\miss[j, \t]}_{j=1}^{\N}$ in the above display, we can remove the conditioning, and moreover replace the condition $\sum_{j\in\nstar}\!\!\!\!\miss[j, \t]\! =\!K$ with $\sum_{j\in\nstar}\!\!\!\miss[j, \t]  \!\geq\! K$ (since it only makes the bound inside the probability display worse), and combine it with \cref{eq:a_in_nstar_bound} and the definitions~\cref{eq:v1_v2} of $\mbb V_1$ and $\mbb V_2$ to conclude that \cref{eq:var_bnds} holds with probability at least $1-\frac{\delta}{4}$ conditional on $\ulf$. The proof is now complete.

    \subsection{Proof of \cref{lem:nset_bnd}}
    \label{sub:proof_of_lem:nset_bnd}
    First, applying \cref{lem:bc_bound} with $p_{\l} \equiv \ppau(\threshold')$ and noting that $\ulf$ are drawn \iid (\cref{assum:iid_unit_lf}), we find that
    \begin{align}
    \label{eq:nstar_bound}
        \P\parenth{\nnstar \geq \frac{\ppau(\threshold') (\N-1)}{2} \vert \lunit}
        &\geq 1 - \exp\parenth{-\frac{\ppau(\threshold')(\N-1)}{8}},
    \end{align}
    which establishes the bound~\cref{eq:nset_bnd}(a).

    Next, to establish the bound on $|\nstar\Delta \nwstar|$, first note that under \cref{assum:iid_unit_lf}, conditional on $\lunit$, the random variables $\sbraces{\indicator(j \in \nstar, j \not\in \nwstar)}_{j\neq \n}$ are \iid Bernoulli random variables with success probability $\betadiff$ defined in \cref{lem:nset_bnd}. Next applying \cref{lem:bc_bound} with $p_{\l} \equiv \betadiff$, we conclude that
    \begin{align}
    \label{eq:bc_bnd_ndiff}
        \P\parenth{|\nstar \Delta \nwstar| \leq  \frac{3\betadiff (\N\!\!-\!\!1)}{2} } \geq 1 - \exp\bigparenth{-\frac{\betadiff \N}{8}}
        \qtext{if}
        \betadiff\N = \Omega(1).
    \end{align}
     For the case with $\betadiff\N=\order(1)$, we use the following entropy-based concentration bound (see~\citep[Exer.~2.9]{wainwright2019high} for a reference): For $X_{\l} \stackrel{\iid}{\sim} \mrm{Bernoulli}(p)$, we have 
    \begin{align}
    \label{eq:kl_bound}
        &\P\parenth{\sum_{\l=1}^{\N} X_{\l} \leq (p+\epsilon)\N } \leq \exp(-\N\cdot\KL{p+\eps}{p}, \\
        \qtext{where}
        &\KL{p+\eps}{p} \defeq (p+\eps)\log\bigparenth{1+\frac{\eps}{p}}+(1-p-\eps)\log\bigparenth{1-\frac{\eps}{1-p}}.
    \end{align}
    Some algebra with the bound~\cref{eq:kl_bound} implies that 
    \begin{align}
    \label{eq:kl_bnd_c}
        \P\parenth{\sum_{\l=1}^{\N} X_{\l} \leq C\log(1/\delta)} \geq 1-\delta
        \qtext{if}
        pN = \order(1),
    \end{align}
    which in turn implies that
    \begin{align}
    \label{eq:kl_bnd_ndiff}
        \P\parenth{|\nstar \Delta \nwstar| \leq C\log(4/\delta) } \geq 1 - \frac{\delta}{4}
        \qtext{if}
        \betadiff\N = \order(1).
    \end{align}
    Putting the bounds~\cref{eq:bc_bnd_ndiff,eq:kl_bnd_ndiff} together, we find that
    \begin{align}
        \P\parenth{|\nstar \Delta \nwstar| \!<\! \max\braces{ \frac{3\betadiff (\N\!\!-\!\!1)}{2} , C \log(\frac{4}{\delta})} \vert \lunit}
        &\geq 1\!-\!\max\braces{\frac{\delta}{4}, \exp\parenth{-\frac{\betadiff(\N\!\!-\!\!1)}{8}}} \\
        &\geq 1\!-\!\frac{\delta}{4}.
        \label{eq:delta_nstar}
    \end{align}
    thereby yielding the bound~\cref{eq:nset_bnd}(b). Note that if $\betadiff=0$, a direct argument yields that $|\nstar \Delta \nwstar| =0$ almost surely. Putting \cref{eq:nstar_bound,eq:delta_nstar} together finishes the proof.

\section{Proof of Thm.~\ref{thm:anytime_asymp}: \anytimeasympname}
\label{sec:proof_of_thm_anythime_asymp}
We prove each part separately. We repeatedly use the following fact for a sequence of random variables $\sbraces{X_{\T}}$ and deterministic scalars $\sbraces{b_{\T}}$: 
\begin{align}
\label{eq:op_ot_reln}
    X_{\T} = \order_{P}(b_{\T})
    \qtext{and}
    b_{\T} = o_{\T}(1)
    \implies
    X_{\T} = o_{P}(1)
    \qtext{as} \T \to \infty.
\end{align}

\subsection{Proof of part~\cref{item:consistency}: Asymptotic consistency}
To prove the claimed result, we need to show that under the conditions~\cref{eq:regularity_consistency}
\begin{align}
    \sabss{\estobs[\n, \t, \threshold]-\trueobs} = o_P(1) \qtext{as} \T \to \infty.
\end{align}
Applying the bound~\cref{eq:nn_bnd} from \cref{thm:anytime_bound} with $\delta = \frac{1}{\T}$, we find that
\begin{align}
    (\estobs[\n, \t, \threshold]-\trueobs)^2 &\!\leq\! \frac{2\vconst^2}{\lambda}
    \biggparenth{\threshold -2\sigma_{a}^2 + \errtwo
    }
    \!+\!\frac{48\sigma_{a}^2\log(8\T) }{\pmint[\t] \probparam_{\lunit}(\threshold'_{\T}) (\NT\!-\!1)}
     +  \frac{6\nconst^2}{\pmint[\t]^2} \parenth{\frac{\probparam_{\lunit}(\threshold'_{\T}\!+\!2\errtwo)}{\probparam_{\lunit}(\threshold'_{\T})}\!-\!1}^2,
    \end{align}
    with probability at least $1-\frac{1}{\T}-\exp\parenth{-\frac{\probparam_{\lunit}( \threshold_{\T}')(\NT\squash{-}1)}{8}}$. Using \cref{eq:regularity_consistency}(a), we find that
    \begin{align}
        1-\frac{1}{\T}-\exp\parenth{-\frac{\probparam_{\lunit}( \threshold_{\T}')(\NT\squash{-}1)}{8}} \to 1 \qtext{as} \T \to \infty,
    \end{align}
     so that 
    \begin{align}
        (\estobs[\n, \t, \threshold]-\trueobs)^2&=\order_{P}\parenth{\threshold_{\T}'+2\errtwo + \frac{\log \T}{\pmint[\t] \probparam_{\lunit}(\threshold'_{\T}) \NT } +  \frac{1}{\pmint[\t]^2} \parenth{\frac{\probparam_{\lunit}(\threshold'_{\T}\!+\!2\errtwo)}{\probparam_{\lunit}(\threshold'_{\T})}\!-\!1}^2},
    \end{align}
    and since $\t$ is fixed, $\pmint[\t]>c$ for some constant $c$, so that \cref{eq:regularity_consistency}(b) implies that the term inside the parentheses in the previous display is $o_{\T}(1)$, thereby yielding the claim due to \cref{eq:op_ot_reln}.

\subsection{Proof of part~\cref{item:clt}: Asymptotic normality}
We have
\begin{align}
    \sqrt{\nestnbr[\n, \t, \threshold_{\T}]} (\estobs[\n, \t, \threshold]-\trueobs) 
    &= \sqrt{\nestnbr[\n, \t, \threshold_{\T}]} \parenth{ \frac{1}{\nestnbr[\n, \t, \threshold_{\T}]} \sum_{j \in \estnbr[\n, \t, \threshold_{\T}]} (\obs[j, \t] - \trueobs[\n, \t])}  \\
    &= \frac{\sum_{j \in \estnbr[\n, \t, \threshold_{\T}]} (\trueobs[j, \t]-\trueobs[\n, \t]) }{\sqrt{\nestnbr[\n, \t, \threshold_{\T}]}}
    + \frac{\sum_{j \in \estnbr[\n, \t, \threshold_{\T}]} \noiseobs[j, \t] }{\sqrt{\nestnbr[\n, \t, \threshold_{\T}]}} \\
    &= \frac{\sum_{j \in \estnbr[\n, \t, \threshold_{\T}]}\angles{\lunit[j]-\lunit, \ltime} }{\sqrt{\nestnbr[\n, \t, \threshold_{\T}]}}
    + \frac{\sum_{j \in \estnbr[\n, \t, \threshold_{\T}]} \noiseobs[j, \t] }{\sqrt{\nestnbr[\n, \t, \threshold_{\T}]}}.
\end{align}
Next, we note that
\begin{align}
    \abss{\frac{\sum_{j \in \estnbr[\n, \t, \threshold_{\T}]}\angles{\lunit[j]-\lunit, \ltime} }{\sqrt{\nestnbr[\n, \t, \threshold_{\T}]}}}
    &\leq \twonorm{\ltime} \sqrt{\nestnbr[\n, \t, \threshold_{\T}]} \max_{j\in\estnbr[\n, \t, \threshold_{\T}]} \twonorm{\lunit[j]-\lunit}   \\
    &\leq \frac{\vconst}{\sqrt{\lambda}} \sqrt{\nestnbr[\n, \t, \threshold_{\T}]} \max_{j\in\estnbr[\n, \t, \threshold_{\T}]} \norm{\lunit[j]-\lunit}_{\Sigv}. 
\end{align}
From \cref{eq:u_bound}, we know that
\begin{align}
    \max_{j\in\estnbr[\n, \t, \threshold_{\T}]} \norm{\lunit[j]-\lunit}_{\Sigv}^2= \order_{P}\parenth{\threshold_{\T}'+2\errtwo},
\end{align}
and from \cref{eq:n_reln}, we know that
\begin{align}
    \nestnbr[\n, \t, \threshold_{\T}] \leq \nnhat \leq \nnwstar \sless{\mrm{\cref{lem:bc_bound},\cref{eq:nwstar}}} \order_P\parenth{\N \probparam_{\lunit}(\threshold_{\T}'+2\errtwo)}.
\end{align}
Putting the pieces together, we find that
\begin{align}
     \abss{\frac{\sum_{j \in \estnbr[\n, \t, \threshold_{\T}]}\angles{\lunit[j]-\lunit, \ltime} }{\sqrt{\nestnbr[\n, \t, \threshold_{\T}]}}} 
     =  \order_{P}\biggparenth{\sqrt{(\threshold_{\T}'+2\errtwo) \N \probparam_{\lunit}(\threshold_{\T}'+2\errtwo)}}
     \sless{\cref{eq:regularity_clt}} \order_P(o_{\T}(1)) \seq{\cref{eq:op_ot_reln}} o_{P}(1).
\end{align}
Noting this bound, and the next lemma immediately yields the claimed central limit theorem.
\begin{lemma}
\label{lem:anytime_clt_noise}
    Under the setting of \cref{thm:anytime_asymp}\cref{item:clt}, we have
    \begin{align}
        \frac{\sum_{j \in \estnbr[\n, \t, \threshold_{\T}]} \noiseobs[j, \t] }{\sqrt{\nestnbr[\n, \t, \threshold_{\T}]}} \vert \lunit \stackrel{d}{\implies} \mc N(0, \sigma_{a}^2)
        \qtext{as} \T \to \infty.
    \end{align}
\end{lemma}
It remains to establish \cref{lem:anytime_clt_noise}, which we do in \cref{sub:proof_of_lem:anytime_clt_noise}.

\subsection{Proof of \cref{lem:anytime_clt_noise}}
\label{sub:proof_of_lem:anytime_clt_noise}
Adapt the definitions~\cref{eq:nhat,eq:nstar,eq:nwstar} with $\threshold$ replaced by $\threshold_{\T}$, and $\errterm$ by $ \errterm[\NT, \T, 1/(\NT \T)]$; and still denote the corresponding sets $\nstar,\nhat$ and $\nwstar$. We have
\begin{align}
    \frac{\sum_{j \in \estnbr[\n, \t, \threshold_{\T}]} \noiseobs[j, \t] }{\nestnbr[\n, \t, \threshold_{\T}]}
    &= \frac{\sum_{j\in\nhat}\miss[j, t]\noiseobs[j, \t]}{\sum_{j\in\nhat}\miss[j, t]} \\
        & = 
        \frac{\sum_{j\in\nstar}\miss[j, t]\noiseobs[j, \t]}{\sum_{j\in\nhat}\miss[j, t]} 
        + 
        \frac{\sum_{j\in\nhat\backslash\nstar}\miss[j, t]\noiseobs[j, \t]}{\sum_{j\in\nhat}\miss[j, t]}\\
        & = 
        \frac{\sum_{j\in\nstar}\miss[j, t]\noiseobs[j, \t]}{\sum_{j\in\nstar}\miss[j, t]}
        + 
        \parenth{\frac{\sum_{j\in\nstar}\miss[j, t]\noiseobs[j, \t]}{\sum_{j\in\nhat}\miss[j, t]} -
        \frac{\sum_{j\in\nstar}\miss[j, t]\noiseobs[j, \t]}{\sum_{j\in\nstar}\miss[j, t]}}
        \\ 
        &\qquad+ 
        \frac{\sum_{j\in\nhat\backslash\nstar}\miss[j, t]\noiseobs[j, \t]}{\sum_{j\in\nhat}\miss[j, t]}.
    \end{align}
Next, we claim that conditional to $\lunit$,
\begin{align}
 &\frac{\sum_{j\in\nstar}\miss[j, t]\noiseobs[j, \t]}{\sqrt{\sum_{j\in\nstar}\miss[j, t]}} \Longrightarrow \mc N(0, \sigma_{a}^2),
      \label{eq:clt_noise} \\
   &\max\braces{ \parenth{\frac{\sum_{j\in\nstar}\miss[j, t]\noiseobs[j, \t]}{\sum_{j\in\nhat}\miss[j, t]}\! -\!
        \frac{\sum_{j\in\nstar}\miss[j, t]\noiseobs[j, \t]}{\sum_{j\in\nstar}\miss[j, t]}}, 
        \! \frac{\sum_{j\in\nhat\backslash\nstar}\miss[j, t]\noiseobs[j, \t]}{\sum_{j\in\nhat}\miss[j, t]} } 
        \\&\qquad\qquad= o_P\parenth{\frac{1}{\sqrt{\nestnbr[\n, \t, \threshold_{\T}]}}}, 
        \label{eq:bias_noise}
         \\
         &\qtext{and}  
         \frac{\nestnbr[\n, \t, \threshold_{\T}]}{\sum_{j\in\nstar}\miss[j, t]}
          \stackrel{p}{\longrightarrow} 1.
           \label{eq:nn_ratio}
\end{align}
With these claims at hand, we immediately have
\begin{align}
     \frac{\sum_{j \in \estnbr[\n, \t, \threshold_{\T}]} \noiseobs[j, \t] }{\sqrt{\nestnbr[\n, \t, \threshold_{\T}]}} \vert \lunit \stackrel{d}{\implies} \mc N(0, \sigma_{a}^2)
        \qtext{as} \T \to \infty.
\end{align}
as claimed in the lemma. We now prove the three claims~\cref{eq:clt_noise,eq:bias_noise,eq:nn_ratio} one-by-one.

\newcommand{\ntstar}{\mbf{N}^\star_{\n, \t}}
\newcommand{\nntstar}{N^\star_{\n, \t}}
\paragraph{Proof of \cref{eq:clt_noise}}
Let $\psi_{\mfk X}(r) \defeq \E[e^{\mbf{i}r\mfk{X}} ]$ denote the characteristic function of a random variable $\mfk X$ where $\mbf{i}$ in the exponent here denotes the complex number $\sqrt{-1}$, and let $\ntstar \defeq \braces{j \in \nstar: \miss[j, \t]=1}$, and $\nntstar=|\ntstar|$.
Now, to prove \cref{lem:anytime_clt_noise}, it suffices to show that conditional on $\lunit$, we have
\begin{align}
    \mfk X_{\t} \stackrel{d}{\Longrightarrow} \mc N(0, 1)
    \qtext{for} \mfk X_{\t} \defeq \frac{1}{\sigma}\displaystyle \frac{\sum_{j\in\ntstar}\noiseobs[j, \t]}{\sqrt{\nntstar}},
\end{align}

and in turn applying \citep[Thm.~3.3.17]{durrett2019probability}, it suffices to show that the characteristic function of $\mfk X_{\t}$ converges to that of $\xi \sim \mc N(0, 1)$, i.e., 
\begin{align}
\label{eq:psi_convergence}
    \psi_{\mfk X_{\t} \vert \lunit}(r) = \E[e^{\mbf{i}r\mfk{X}_{\t}} \vert \lunit ]  \quad \stackrel{\T \to \infty}{\longrightarrow} \quad \psi_{\xi}(r) = e^{-\frac{r^2}{2}}  \qtext{for each} r \in \real.
\end{align}
We now prove the claim~\cref{eq:psi_convergence}.
 
Since $\E[\noiseobs[j, \t]] = 0$ and $\E[(\noiseobs[j, \t])^2 = \sigma_{a}^2]$, \cite[Thm.~3.3.20]{durrett2019probability} yields that
\begin{align}
\label{eq:psi_noise}
    \psi_{\mrm{noise}}(r) \defeq \psi_{\noiseobs[j, \t]/\sigma}(r) = 1- \frac{r^2}{2} + o(r^2).
\end{align}
Then we have
\begin{align}
    \psi_{\mfk X_{\t} \vert \lunit}(r)
    = \E\brackets{\exp\parenth{\mbf{i}r \frac{\sum_{j \in \ntstar} \noiseobs[j, \t] / \sigma }{\sqrt{\nntstar}}} \vert\lunit} 
    &= \E\brackets{\E\brackets{\exp\parenth{\mbf{i}r \frac{\sum_{j \in \ntstar } \noiseobs[j, \t] / \sigma }{\sqrt{\nntstar}}} \vert \ntstar,\lunit}\vert\lunit } \\
    &\seq{(i)} \E\brackets{ \prod_{j \in \ntstar}\psi_{\noiseobs[j, \t]/\sigma\vert \ntstar,\lunit}\parenth{\frac{r}{\sqrt{\nntstar}}} \vert \lunit } 
\end{align}
where step~(i) follows from the fact that the set $\ntstar$ depends only on $\ulf$, and the observed data till time $\T-1$ under \cref{assum:policy}, and hence independent of the noise variables $\braces{\noiseobs[j, \t]}$ generated at time $\T$, and thus under \cref{assum:zero_mean_noise} the noise variables $\braces{\noiseobs[j, \t]}$ remain independent even conditional on $\ntstar$. Next, we have
\begin{align}
    \psi_{\mfk X_{\t} \vert \lunit}(r)
    = \E\brackets{ \prod_{j \in \ntstar}\psi_{\mrm{noise}}\parenth{\frac{r}{\sqrt{\nntstar}}} \vert \lunit } 
    &= \E\brackets{ \parenth{\psi_{\mrm{noise}}\parenth{\frac{r}{\sqrt{\nntstar}}}}^{\nntstar} \vert \lunit } 
    \\
    &= \sum_{k=0}^{\infty} \P(\nntstar = k  \vert \lunit ) \psi^{k}_{\mrm{noise}}\parenth{\frac{r}{\sqrt{k}}},
\end{align}
where we also use the fact the noise variables are independent of latent unit factors.
Moreover, noting that $\sabss{\psi_{\mfk X}(r)} \leq 1$ for all $r$, and any random variable $\mfk X$, we obtain that
\begin{align}
    \abss{\psi_{\mfk X_{\t} \vert \lunit}(r) - e^{-\frac{r^2}{2}}} &= \abss{\sum_{k=0}^{\infty} \P(\nntstar = k\vert \lunit) \psi^{k}_{\noiseobs[]}\parenth{\frac{r}{\sqrt{k}}} -  e^{-\frac{r^2}{2}} }
    \\
    &\leq \sum_{k=0}^{K_{\T}} \P(\nntstar = k\vert \lunit) + \abss{\psi^{K_{\T}}_{\noiseobs[]}\parenth{\frac{r}{\sqrt{K_{\T}}}} -  e^{-\frac{r^2}{2}}} \\
    &\sless{\cref{eq:psi_noise}}  \P(\nntstar \leq K_{\T}\vert \lunit) \!+\!
    \abss{\brackets{1\!-\!\frac{r^2}{2K_{\T}} \!+\! o\parenth{\frac{r^2}{K_{\T}}} }^{K_{\T}} \!-\! 
     e^{-\frac{r^2}{2}}},
     \label{eq:char_bound}
\end{align}
for any arbitrary choice of $K_{\T}$.

Repeating the arguments around \cref{eq:a_in_nstar_bound,eq:nstar_bound} from \cref{sub:proof_of_noise_conc} yields that
\begin{align}
     \P(\nntstar \leq  K_{\T}\vert \lunit) &\leq \exp\parenth{-\frac{\pmint[\t] \probparam_{\n}(\threshold_{\T}') (\NT-1)}{32}}+\exp\parenth{-\frac{\probparam_{\n}(\threshold_{\T}')(\NT-1)}{8}}, \\
     \qtext{where}
     K_{\T} &\defeq \frac{\probparam_{\n}(\threshold_{\T}')}{4} \cdot (\NT-1).
     \label{eq:k_t}
\end{align}
For this choice of $K_{\T}$~\cref{eq:k_t}, invoking \cref{eq:regularity_clt} (treating $\pmint[\t]$ as a constant as $\T$ scales) we have
\begin{align}
\label{eq:nn_star_bound}
    \P(\nntstar \leq  K_{\T}\vert \lunit) \to 0 \qtext{and} K_{\T} \to \infty \qtext{as} \T \to \infty.
\end{align}
Since $K_{\T} \to \infty$, for small $r$, using \cref{eq:psi_noise}, we find that
\begin{align}
\abss{\brackets{1\!-\!\frac{r^2}{2K_{\T}} \!+\! o\parenth{\frac{r^2}{K_{\T}}} }^{K_{\T}} \!-\! 
     e^{-\frac{r^2}{2}}}
     &\leq 
    \abss{\brackets{1\!-\!\frac{r^2}{2K_{\T}} \!+\! o\parenth{\frac{r^2}{K_{\T}}} }^{K_{\T}} \!-\! 
    \parenth{1\!-\!\frac{r^2}{2K_{\T}} }^{K_{\T} } }
    \!+\! \abss{\parenth{1-\frac{r^2}{2K_{\T}} }^{K_{\T}} \!-\! e^{-\frac{r^2}{2}}} \\
    &\leq K_{\T} \cdot o\parenth{\frac{r^2}{K_{\T}}} \!+\! \abss{\parenth{1-\frac{r^2}{2K_{\T}} }^{K_{\T}} \!-\! e^{-\frac{r^2}{2}}}
    \to 0 \qtext{as} \T \to \infty.
\end{align}
Morever, for a large $r$, we can replace the $o(r^2)$ approximation of \cref{eq:psi_noise} by $\psi_{\mrm{noise}}(r) \leq 1-r^2/2 + 2\abss{r}^3$ (see \cite[Proof of Thm.~3.3.20, Eq.~3.3.3]{durrett2019probability}),
and repeat the arguments above for large $|r|$ to find that 
\begin{align}
    \abss{\brackets{1\!-\!\frac{r^2}{2K_{\T}} \!+\! o\parenth{\frac{r^2}{K_{\T}}} }^{K_{\T}} \!-\! 
     e^{-\frac{r^2}{2}}}
     &\leq K_{\T} \cdot \order\parenth{\frac{|r|^3}{K_{\T}^{3/2}}} \!+\! \abss{\parenth{1-\frac{r^2}{2K_{\T}} }^{K_{\T}} \!-\! e^{-\frac{r^2}{2}}} \\
     &\leq \order\parenth{\frac{|r|^3}{K_{\T}^{1/2}}} \!+\! \abss{\parenth{1-\frac{r^2}{2K_{\T}} }^{K_{\T}} \!-\! e^{-\frac{r^2}{2}}},
\end{align}
which still goes to $0$ for any fixed $r$ when $K_{\T} \to \infty$. Putting the pieces together, we conclude that $\sabss{\psi_{\mfk X_{\t}\vert \lunit}(r) - e^{-\frac{r^2}{2}}}  \to 0$ as $\T \to \infty$ for all $r$, as required. The proof is now complete.

\paragraph{Proof of \cref{eq:bias_noise}}
Define $\mfk B_{\t} \defeq \sqrt{\nestnbr[\n, \t, \threshold_{\T}]} \cdot \frac{\sum_{j\in\nstar}\miss[j, t]\noiseobs[j, \t]}{\sum_{j\in\nstar}\miss[j, t]} $.
Then, we have
\begin{align}
    \sqrt{\nestnbr[\n, \t, \threshold_{\T}]}
        \parenth{\frac{\sum_{j\in\nstar}\miss[j, t]\noiseobs[j, \t]}{\sum_{j\in\nhat}\miss[j, t]} -
        \frac{\sum_{j\in\nstar}\miss[j, t]\noiseobs[j, \t]}{\sum_{j\in\nstar}\miss[j, t]}}
        = \mfk B_{\t} \cdot \frac{\sum_{j\in\nhat}\miss[j, t]-\sum_{j\in\nstar}\miss[j, t]}{\sum_{j\in\nhat}\miss[j, t]}.
\end{align}
Note that \cref{eq:nn_ratio,eq:clt_noise} implies that $\mfk B_{\t}=\order_{P}(1)$. Next, we have
\begin{align}
    \abss{\frac{\sum_{j\in\nhat}\miss[j, t]-\sum_{j\in\nstar}\miss[j, t]}{\sum_{j\in\nhat}\miss[j, t]}} \seq{\cref{eq:n_reln}} \order_{P}\parenth{\frac{\sum_{j\in\nwstar}\miss[j, t]\!-\!\sum_{j\in\nstar}\miss[j, t]}{\sum_{j\in\nstar}\miss[j, t]}}
    \!\seq{\cref{eq:a_in_nstar_bound}}\! \order_{P}\parenth{\frac{|\nwstar\Delta\nstar|}{\pmint[\t] \nnstar}}.
    \label{eq:a_diff}
\end{align}
Moreover, we have
\begin{align}
    \frac{|\nwstar\Delta\nstar|}{\pmint[\t] \nnstar}
    \seq{\cref{eq:delta_nstar},\cref{eq:nn_star_bound}} \frac{1}{\pmint[\t]} \parenth{\frac{\probparam_{\lunit}(\threshold'_{\T}+2\errterm)}{\probparam_{\lunit}(\threshold'_{\T})}-1}
    \seq{\cref{eq:regularity_clt}} o_{\T}(1).
    \label{eq:a_diff_ot}
\end{align}
Similarly, 
\begin{align}
    \abss{\sqrt{\nestnbr[\n, \t, \threshold_{\T}]} \frac{\sum_{j\in\nhat\backslash\nstar}\miss[j, t]\noiseobs[j, \t]}{\sum_{j\in\nhat}\miss[j, t]} }
    = \order_P\parenth{\frac{|\nwstar\Delta\nstar|} {\sqrt{\nestnbr[\n, \t, \threshold_{\T}]}}}
    = \order_P\parenth{\frac{|\nwstar\Delta\nstar|}{\sqrt{\pmint[\t] \nnstar}}}
\end{align}
and we have
\begin{align}
    \frac{|\nwstar\Delta\nstar|}{\sqrt{\nnstar}}
    = \order_{P}\parenth{\parenth{\frac{\probparam_{\lunit}(\threshold'_{\T}+2\errterm)}{\probparam_{\lunit}(\threshold'_{\T})}-1} \sqrt{\probparam_{\lunit}(\threshold'_{\T}) \NT}}
    \seq{\cref{eq:regularity_clt}} o_{\T}(1).
\end{align}

\paragraph{Proof of \cref{eq:nn_ratio}}
We have
\begin{align}
     \abss{\frac{\nestnbr[\n, \t, \threshold_{\T}]}{\sum_{j\in\nstar}\miss[j, t]}-1}
     &=
     \abss{\frac{\sum_{j\in\nhat}\miss[j, t]-\sum_{j\in\nstar}\miss[j, t]}{\sum_{j\in\nstar}\miss[j, t]}}
    \seq{\cref{eq:a_diff}} \order_{P}\parenth{\frac{|\nwstar\Delta\nstar|}{\pmint[\t] \nnstar}}
    \seq{\cref{eq:a_diff_ot}}
    \order_P(o_{\T}(1)),
\end{align}
as desired.

  \section{Proof of Thm.~\ref{thm:ate_asymp}: \ateconsistencyname}
    \label{sec:proof_of_thm:ate_asymp}

    We prove the two parts separately, and once again use \cref{eq:op_ot_reln} repeatedly.

    \subsection{Proof of part \cref{item:consistency_ate}: Asymptotic consistency}

    We have
    \begin{align}
    \label{eq:ate_decomp}
       \estybarall_{\threshold_{\T}} - \trueybarall
       = \frac1{\NT\T} \sum_{\t=1}^{\T}\sum_{\n=1}^{\NT} \miss \noiseobs 
       + \frac1{\NT\T} \sum_{\t=1}^{\T}\sum_{\n=1}^{\NT} (1-\miss)(\trueobs-\estobs[\n,\t, \threshold_{\T}])
    \end{align}
    To control the first term we make use of the following lemma based on martingale arguments similar to that in the proof of \cref{lemma:dist_noise_conc}(a) in \cref{sub:proof_of_distance_conc} (see \cref{sub:proof_of_lem:noise_avg} for the proof):
    \begin{lemma}
    \label{lem:noise_avg}
        Under \cref{assum:policy,assum:zero_mean_noise}, for any $\delta\in(0, 1]$, we have
        \begin{align}
             \P\brackets{|\sum_{\n=1}^{\N} \sum_{\t=1}^{\T} \miss \noiseobs| \leq 
             \nconst \sum_{\n=1}^{\N} \biggparenth{\log(\N\T/\delta) \sum_{\t=1}^{\T} \miss}^{\half} }\geq 1-\delta.
        \end{align}
    \end{lemma}
    Invoking \cref{lem:noise_avg} with $\delta=\delta_{\T} = \frac{1}{\NT\T}$, we conclude that
    \begin{align}
    \label{eq:ate_term_1}
        \sum_{\t=1}^{\T}\sum_{\n=1}^{\N} \miss \noiseobs &= \order_{P}\parenth{ \sum_{\n=1}^{\N}  \biggparenth{  \log(\NT\T)\sum_{\t=1}^{\T}\miss }^{\half} }
        = \order_{P}\parenth{\sigma\NT \sqrt{ \T \log(\NT\T)}}
    \end{align}
    since $\sum_{\t=1}^{\T}\miss  \leq \T$ deterministically, which in turn implies that $\frac1{\NT\T} \sum_{\t=1}^{\T}\sum_{\n=1}^{\NT} \miss \noiseobs  = o_{P}(1)$.
    Next, applying \cref{thm:anytime_bound} with $\delta = \frac{1}{\NT\T^2}$ and a union bound, we find that
    \begin{align}
         &\frac1{\NT\T} \sum_{\t=1}^{\T}\sum_{\n=1}^{\NT} (1-\miss)(\trueobs-\estobs[\n,\t, \threshold_{\T}]) \\
         &\leq \frac1{\T}\sum_{\t=1}^{\T}\max_{\n: \miss=0} \abss{\trueobs-\estobs[\n,\t, \threshold_{\T}]} \\
         &\leq \order\parenth{\sqrt{\threstwot} +
         \frac{\sqrt{\log(\NT\T)}\sum_{\t=1}^{\T}\pmint[\t]^{-\half}/\T}
         {\sup_{\lunit[]} \sqrt{\probparam_{\lunit[], a}(\thresonet) (\NT\!-\!1)}}
            +
          \sup_{\lunit[]} \abss{\frac{\probparam_{\lunit[], a}(\threstwot)}{\probparam_{\lunit[], a}(\thresonet)}\!-\!1}\, \frac{\sum_{\t=1}^{\T}\pmint[\t]\inv}{\T }}
          \label{eq:ate_err}
     \end{align} 
     with probability at least
     \begin{align}
     \label{eq:ate_err_1}
         1- \frac{1}{\T} - \sum_{\t=1}^{\T}\sum_{\n=1}^{\NT} \sup_{\lunit[]} \exp\parenth{-\probparam_{\lunit[], a}(\thresonet)(\NT\!-\!1)},
     \end{align}
     which goes to $1$ as $\T\to \infty$ due to \cref{eq:ate_regularity_consistency}.
     Furthermore, \cref{eq:ate_regularity_consistency} also implies that the sum inside the parentheses in display~\cref{eq:ate_err} is $o_{\T}(1)$. Putting the pieces together, we conclude that
     \begin{align}
     \label{eq:ate_err_2}
         \frac1{\NT\T} \sum_{\t=1}^{\T}\sum_{\n=1}^{\NT} (1-\miss)(\trueobs-\estobs[\n,\t, \threshold_{\T}]) = o_{P}(1),
     \end{align}
     which when combined with \cref{eq:ate_decomp,eq:ate_term_1} yields the claimed consistency result.

     \subsection{Proof of part \cref{item:clt_ate}: Asymptotic normality}
     Define the averaged quantities
     \begin{align}
         \what{\overline{\theta}}^{(a)}_{\threshold_{\T}, \mrm{full}} \defeq  \frac1{\NT\T} \sum_{\t=1}^{\T}\sum_{\n=1}^{\NT} \estobs[\n,\t, \threshold_{\T}],
         \qtext{and}
         \trueybarall \defeq  \frac1{\NT\T} \sum_{\t=1}^{\T}\sum_{\n=1}^{\NT}\trueobs.
     \end{align}
     Now we note that
     \begin{align}
     \label{eq:clt_decomp}
         \sqrt{\KT}(\estybaralltwo-\trueybarall[]) &= 
         \sqrt{\KT}(\estybaralltwo-\what{\overline{\theta}}^{(a)}_{\threshold_{\T}, \mrm{full}}) + 
         \sqrt{\KT}(\what{\overline{\theta}}^{(a)}_{\threshold_{\T}, \mrm{full}}-\trueybarall)  \\
         &\qquad+\sqrt{\KT}(\trueybarall-\trueybarall[]).
     \end{align}
     Since the set $\KTset$ is drawn independently of the data and the estimates are bounded $\braces{\estobs[\n,\t, \threshold_{\T}]}$, a standard triangular array central limit theorem yields that
     \begin{align}
         \frac{\sqrt{\KT}}{\what{\sigma}_{\KTset}} (\estybaralltwo-\what{\overline{\theta}}^{(a)}_{\threshold_{\T}, \mrm{full}}) \stackrel{d}{\implies} \mc N(0, 1),
         \label{eq:ate_clt_avg}
     \end{align}
     as $\KT\to\infty$, where we note that a standard argument also yields that $\what{\sigma}^2_{a, \KTset}$~\cref{eq:sigma_nn} is a consistent estimate of the variance $\sigma_{a,nn}^2\defeq \lim_{\T\to\infty}\frac1{\NT\T} \sum_{\t=1}^{\T}\sum_{\n=1}^{\NT} (\estobs[\n,\t, \threshold_{\T}]-\overline{Y}_{\threshold_{\T}})^2$.
     Furthermore, using arguments similar to that for \cref{eq:ate_err,eq:ate_err_1,eq:ate_err_2}, we find that
     \begin{align}
          |\sqrt{\KT} (\what{\overline{\theta}}^{(a)}_{\threshold_{\T}, \mrm{full}}-\trueybarall)|
         &= \sqrt{\KT} \cdot \frac1{\NT\T} \abss{\sum_{\t=1}^{\T}\sum_{\n=1}^{\NT}(\trueobs-\estobs[\n,\t, \threshold_{\T}])} \\
         &\leq \sqrt{\KT}  \max_{\n, \t}  |\trueobs-\estobs[\n,\t, \threshold_{\T}]| \\ 
         &=\order_{P} \Big(\sqrt{\KT\threstwot} +
         \frac{\sqrt{\KT\log(\NT\T)}\sum_{\t=1}^{\T}\pmint[\t]^{-\half}}
         {\sup_{\lunit[]} \sqrt{\probparam_{\lunit[], a}(\thresonet) (\NT\!-\!1)}}
            \\
            &\qquad\qquad\qquad+ 
          \sqrt{\KT} \sup_{\lunit[]} \abss{\frac{\probparam_{\lunit[], a}(\threstwot)}{\probparam_{\lunit[], a}(\thresonet)}\!-\!1}\, \frac{\sum_{\t=1}^{\T}\pmint[\t]\inv}{\T } \Big)\\
        &\seq{\cref{eq:ate_regularity_clt}} \order_{P}(o_{\T}(1)) \seq{\cref{eq:op_ot_reln}} o_P(1).
        \label{eq:ate_clt_err}
     \end{align}
     \newcommand{\ubar}[1][\T]{\overline{u}_{#1}}
     \newcommand{\vbar}[1][\T]{\overline{v}_{#1}^{(a)}}
     Define
     \begin{align}
         \ubar \defeq \frac{1}{\NT}\sum_{\n=1}^{\NT} \lunit
         \qtext{and}
         \vbar \defeq \frac{1}{\T}\sum_{\t=1}^{\T} \ltime^{(a)}.
     \end{align}
     Then, 
     \begin{align}
         \trueybarall\!-\!\trueybarall[] \!=\!
         \angles{\ubar,\vbar} \!-\! \angles{\E[\lunit], \E[\ltime^{(a)}]}
         \!=\! \angles{\ubar\!-\!\E[\lunit],\vbar} \!+\! \angles{\E[\lunit], \vbar\!-\! \E[\ltime^{(a)}]},
     \end{align}
     and the \iid sampling of the bounded latent factors (\cref{assum:iid_unit_lf,assum:lambda_min}) implies that
     \begin{align}
         \twonorm{\ubar-\E[\lunit]} = \order_P\parenth{\frac{1}{\sqrt{\NT}}}
         \qtext{and}
         \twonorm{\vbar-\E[\ltime]} = \order_P\parenth{\frac{1}{\sqrt{\T}}}
     \end{align}
     so that the condition~\cref{eq:ate_regularity_clt} implies that
     \begin{align}
          \sqrt{\KT}(\trueybarall-\trueybarall[]) = \order_P\parenth{\sqrt{\frac{\KT}{\NT}} + \sqrt{\frac{\KT}{\T}}  } = \order_P(o_{\T}(1)) = o_{P}(1).
          \label{eq:bias_two_mean}
     \end{align}

     Putting \cref{eq:clt_decomp,eq:ate_clt_avg,eq:ate_clt_err,eq:bias_two_mean} together, and noting the assumption $\sigma^2_{a, \mrm{nn}} >c$, we conclude that
     \begin{align}
        \frac{\sqrt{\KT}}{\what{\sigma}_{\KTset}}(\estybaralltwo-\trueybarall[])
          \seq{\cref{eq:ate_clt_avg}} 
         \frac{\sqrt{\KT}}{\what{\sigma}_{\KTset}} (\estybaralltwo-\overline{Y}_{\threshold_{\T}})
         + o_{P}(1)
         \stackrel{d}{\implies} \mc N(0, 1),
     \end{align}
      which is the desired claim.

      \subsection{Proof of \cref{lem:noise_avg}}
    \label{sub:proof_of_lem:noise_avg}
    We make use of a martingale concentration argument similar to that in the proof of \cref{lemma:dist_noise_conc}(a) in \cref{sub:proof_of_distance_conc}.
    Without loss of generality we can assume $a=1$, so that we can use the simplified notation $\indicator(\miss=a) = \miss$ henceforth.

    \newcommand{\tstop}[1][\l]{\t_{(#1)}(\n)}
    \paragraph{Step~(i): Constructing the martingale}
    For any fixed $\n\in[\N]$, define the sequence $\braces{\tstop}_{\l=0}^{\infty}$ as follows: Set $\tstop[0]\defeq 0$, and for $\l = 1, 2, \ldots$, define
    \begin{align}
    \label{eq:tstop_i}
        \tstop[\l] = \begin{cases}
            \min\braces{\t: \tstop[\l-1] < \t <\T, \miss = 1 } \stext{if this problem is feasible,} \\
            \T \qtext{otherwise}
        \end{cases}
    \end{align}
    That is, $\tstop[\l]$ denotes the time point such that the unit $\n$ received the treatment $1$ for the $\l$-th time. Clearly, $\tstop \leq \T$ for all $\l$ deterministically, since  $\T$ is the total number of time points. Moreover, $\tstop$ is a stopping time with respect to the filtration $\braces{\history[\t]}_{\t=1}^{\T}$, i.e., $\braces{\tstop \leq \t} \in \history[\t]$. Let $\mc H_{\l}$ denote the sigma-field generated by the stopping time $\tstop$. Define the sequence
    \begin{align}
    \label{eq:wl_eps}
        W_{\l} \defeq \noiseobs[\n, \tstop], \qquad \l = 1, 2, \ldots, \T.
    \end{align}
    We clearly have $|W_{\l}|\leq \nconst$ for all $\l$, almost surely, as well as 
    \begin{align}
    \label{eq:wl_eps_zero}
        \E[W_{\l}\vert \mc H_{\l}] \!=\! \E\brackets{\noiseobs[\n, \tstop] \vert \mc H_{\l}} &\!=\! \E\brackets{\noiseobs[\n, \tstop]\vert \mc H_{\l}} 
        \!\seq{(i)}\!0,
    \end{align} 
    so that $\braces{W_{\l}}_{\l=0}^{\T}$ is a bounded Martingale difference sequence with respect to the filtration $\braces{\mc H_{\l}}_{\l=0}^{\T}$; here step~(i) follows from the fact that conditioned on $\tstop=\t$, the distribution of $\noiseobs$ is independent of the treatment $\miss$ at that time since the policy is sequential~(\cref{assum:policy}), and the latent time factors at time $\t$, and noise variables at time $\t$ are exogenous and drawn independently of the policy at time $\t$ (\cref{assum:zero_mean_noise}), so that the noise at time $\t$ is independent of the event $\tstop = \t$ (and thereby the sigma-algebra $\mc H_{\l}$).

    \paragraph{Step~(ii): Relating the martingale to $\sum_{\t=1}^{\T} \miss \noiseobs$}
    Define $\T_{\n} \defeq \sumt \miss $. Then we have
    \begin{align}
    \label{eq:wsum_eps}
        \sum_{\t=1}^{\T} \miss \noiseobs
        \seq{\cref{eq:tstop_i}} \sum_{\l=1}^{\T_{\n}} \noiseobs[\n, \tstop] 
         \seq{\cref{eq:wl_eps}} \sum_{\l=1}^{\T_{\n}} W_{\l}.
    \end{align}

    \paragraph{Step~(iii): Applying martingale concentration}
    Fix $\delta>0$. Now, applying Azuma Hoeffding concentration bound~\cref{eq:hpb_bound_martingale}, we find that for any fixed $K>0$, we have
    \begin{align}
        \P\brackets{\abss{\sum_{\l=1}^{K} W_{\l}} \leq \nconst \sqrt{K\log(2/\delta)} } \geq 1-\delta. 
    \end{align}
    Note that if $K=0$, the bound insider the parentheses above applies deterministically.
    Applying a union bound over $K \in [\T]$, and noting that the event $\T_{\n}>0$ implies $\T_{\n} \in [\T]$, we conclude that 
    \begin{align}
        \P\brackets{\abss{ \sum_{\l=1}^{\T_{\n}} W_{\l}} \leq  \nconst\sqrt{ \T_{\n}\log(2/\delta)}}  \geq 1-\T\delta.
    \end{align}
    
    Putting the pieces together, and taking a union bound over $\n \in [\N]$, we obtain that
    \begin{align}
        \P\brackets{\abss{ \sum_{\l=1}^{\T_{\n}} W_{\l}} \leq  \nconst \sqrt{\T_{\n}\log(2/\delta)} \stext{for all} \n \in[\N] }  \geq 1-\N\T\delta,
    \end{align}
    Finally, noting that $|\sum_{\n=1}^{\N} \sum_{\t=1}^{\T} \miss \noiseobs| \leq \sum_{\n=1}^{\N} |\sum_{\t=1}^{\T} \miss \noiseobs|$ along with \cref{eq:wsum_eps} we conclude that
    \begin{align}
        \P\brackets{|\sum_{\n=1}^{\N} \sum_{\t=1}^{\T} \miss \noiseobs| \leq 
             \nconst \sum_{\n=1}^{\N} \sqrt{\log(\N\T/\delta) \sum_{\t=1}^{\T} \miss} }\geq 1-\delta,
    \end{align}
    where we have also substituted $\delta$ by $\delta/(\N\T)$ in the bounds above. The claim now follows.

\section{Proofs of the corollaries}
\label{sec:example_details}

In this appendix, we collect the proofs related to all the corollaries presented in the main paper. We remind the reader that the big-$\order$ and $\Omega$ notation is taken with respect to $\T$.

    \subsection{Proof of \cref{cor:anytime_bound}}
    \label{sub:proof_of_cor:anytime_bound}

    For brevity, we use the shorthand $\errt \defeq (\estobs-\trueobs)^2$, and use $a \precsim b$ to denote $a = \wtil{O}(b)$.
    	\subsubsection{Proof of part \cref{item:finite}}
        Let $\N$ and $\T$ be large enough such that $\min_{\lunit[] \neq \lunit[]', \lunit[],\lunit[]' \in \mc U_{\mrm{fin}}} \norm{\lunit[]-\lunit[]'}^2_{2}$ is much larger that $2\errterm$.
        Then choosing $\threshold = 2\sigma_{a}^2 + \errterm$, we find that $\ppau(\threshold')\geq \frac1M$ (since $\ppau(0) =\frac1M$), and $\ppau(\threshold'\!+\!2\errterm)= \ppau(\threshold')$, so that the bound~\cref{eq:nn_bnd} simplifies to 
        \begin{align}
        \label{eq:finite_non_asymp_bnd}
            \errt \leq \frac{12\vconst^2 (\vconst \uconst \!+\! \nconst)^2 \sqrt{2\log(\frac{4\N\T}{\delta})}}{\lambda \pmint\sqrt{\T}} + \frac{12\sigma_{a}^2M\log(\frac{8}{\delta})} {\pmint[\t] (\N\!-\!1)} 
            \precsim \parenth{\frac{1}{\sqrt{\T^{1-2\beta}}}\!\!+ \frac{M}{\N}}.
        \end{align}
        In simple words, when (i) $\N$ and $\T$ are large enough such that the error term $\errterm$ is smaller than the minimum separation between the distinct unit factors, (ii) $\threshold$ is chosen suitably, the error guarantee from \cref{thm:anytime_bound} decreases as $\N$ and $\T$ increase, and naturally gets worse when $M$ gets large. In this case, the property (i) ensures that the inflation in the noise variance (fourth term in \cref{eq:nn_bnd}) can be reduced to zero (see \cref{rem:zero_beta}). Moreover, the bound gets worse as the exploration rate decays faster (i.e., as $\beta$ gets larger). And, to ensure a vanishing error with $\T$, we allow $\beta\leq \half$ assuming that the exploration rate decreases slower than $\T^{-\half}$ by poly-logarithmic factors.
        \subsubsection{Proof of part \cref{item:highd}}
        For readers convenience, we state the results first for $d=1$ and then a general $d$ with the choice $\threshold=2\sigma_{a}^2 + \errterm + \T^{\frac{\omega}{2}}\errterm$ and a suitably chosen $\omega \in (0, 1)$. For $d=1$, the bound~\cref{eq:nn_bnd} can be simplified to be of the following order:
        \begin{align}
           \errt 
           \precsim\frac{1}{\sqrt{\T^{1-\omega-2\beta}}} + \frac{\T^{\frac{1-\omega-2\beta}{4}} } {\N} + \frac{1}{\T^{\omega}} 
           &\stackrel{(i)}{\precsim} \begin{cases}
               \T^{-\frac{1-2\beta}{3}} \quad \stext{if} \N \gg \sqrt{\T^{1-2\beta}}, \\[2mm]
               \N^{-\frac23} \quad\ \  \ \ \stext{otherwise,}
           \end{cases}
           \label{eq:final_simple}
        \end{align}
        where step~(i) is obtained by choosing $\omega\!=\!\frac{1-2\beta}{3}$ for the first case, and $\omega = 1\!-\!2\beta \!-\! \delta_1$ with $\delta_1\! =\! \frac{4\log\N}{3\log\T}$ for the second case assuming that $1\!-\!2\beta\!>\!\frac{3\delta_1}{2}$.
        For a general $d$, similar algebra yields that
        \begin{align}
        \label{eq:highd_non_asymp_bnd}
            \errt\! 
            \precsim\frac{1}{\sqrt{\T^{1-\omega-2\beta}}} + \frac{\T^{d(\frac{1-\omega-2\beta)}{4}} } {\N} + \frac{d^2}{\T^{\omega}}  
            &\stackrel{(ii)}{\precsim} \begin{cases}
               \T^{-\frac{1-2\beta}{3}} \ \stext{if} \N^{\frac{3}{d+2}} \gg \sqrt{\T^{1-2\beta}},\\[2mm]
               \N^{-\frac2{d+2}}\, \ \stext{otherwise,}
           \end{cases}
        \end{align}
         where step~(ii) is obtained by choosing $\omega\!=\!\frac{1-2\beta}{3}$ for the first case, and $\omega = 1\!-\!2\beta \!-\! \delta_{d}$ with $\delta_{d}\! \defeq\! \frac{4\log\N}{(d+2)\log\T}$ for the second case assuming that $1\!-\!2\beta\!>\!\frac{3\delta_{d}}{2}$. When $\N$ is exponentially (in $d$) large compared to $\T$, the error decays as $\T^{-\frac{1-2\beta}{3}}$ so that the exploration rate has to decay slower than $\T^{-\half}$ to get a decaying error with $\T$. On the other hand, when $\T$ is large enough (and we can take $\delta_{d} \to 0$ so that the exploration rate can still be allowed to be close to $\T^{-\half}$), the squared error decays as $\N^{-\frac2{d+2}}$. 

	\subsection{Proof of \cref{cor:anytime_asymp}}
	\label{sub:proof_of_cor:anytime_asymp}
	For brevity, we use $a_{\T}\gg b_{\T}$ to denote that $b_{\T} = o(a_{\T})$, in this proof. 
	We note that to ensure the requirement \cref{eq:regularity_consistency} for asymptotic consistency, it suffices to ensure that
	\begin{align}
	\label{eq:suff_regularity_consistency}
		\threstwot
		\stackrel{(a)}{\ll} 1,
		\quad 
		\NT \stackrel{(b)}{\gg}  \frac{\log \T}{\pparam}
		\qtext{and}
		\abss{ \frac{\pparam[\threstwot]}{\pparam[\thresonet]}\!-\!1} \stackrel{(c)}{\ll} 1,
	\end{align}
	and to ensure the requirement \cref{eq:regularity_clt} for asymptotic normality, it suffices to ensure that
	\begin{align}
	\label{eq:suff_regularity_clt}
		\threstwot \LT
		\stackrel{(a)}{\ll} 1,
		\quad 
		\NT \stackrel{(b)}{\gg}  \frac{1}{\pparam}
		\qtext{and}
		\LT\abss{ \frac{\pparam[\threstwot]}{\pparam[\thresonet]}\!-\!1} \stackrel{(c)}{\ll} 1.
	\end{align}

    \subsubsection{Proof of part~\cref{item:finite_asymp}}
    \label{sub:proof_for_example_finite_asymp}
    First note that $\pparam[0] = \frac1{\MT}$. Furthermore, under the assumption $ \min_{\lunit[] \neq \lunit[]', \lunit[], \lunit[]' \in \mc U_{\mrm{fin}, \T} } \norm{\lunit[]\!-\!\lunit[]'}^2_{\Sigv} \gg 2\errtwo$, with $\threshold_{\T} = 2\sigma_{a}^2 + \errtwo$ we find that
\begin{align}
\label{eq:countable_param_ratio}
	\pparam[\threstwot] = \pparam[\thresonet]
	= \pparam[0]  = \frac1{\MT}
	\implies \frac{\pparam[\threstwot]}{\pparam[\thresonet]} =1,
\end{align}
since $\thresonet=0$ and $\threstwot=2\errtwo$. Moreover, the condition $1\geq \pmint = \Omega(\T^{-\beta})$ implies that
\begin{align}
\label{eq:thres_scaling_finite}
	\threstwot = \order\parenth{\sqrt{\frac{\log(\NT\T)}{\T^{1-2\beta}}}}
	\qtext{and}
	\frac{1}{\threstwot} = \order\parenth{\sqrt{\frac{\T} {\log(\NT\T)}}}.
\end{align}

\paragraph{Consistency conditions}
 Now, we can directly verify that \cref{eq:suff_regularity_consistency}(c) is satisfied due to \cref{eq:countable_param_ratio}, and \cref{eq:suff_regularity_consistency}(a) is satisfied when $\errtwo = \order(\pmint\inv\sqrt{\T^{-1}\log(\NT\T)}) \ll 1$, which in turn is satisfied when
\begin{align}
\label{eq:countable_pmin_log}
	\T^{-\beta}\sqrt{\T} \gg \sqrt{\log(\NT\T)}
	\qtext{or}
	\T^{1-2\beta} \gg \log(\NT\T)
\end{align}
since $1\geq \pmint = \Omega(\T^{-\beta})$; this condition is feasible only if $\beta<\half$. Finally, a direct verification yields that \cref{eq:suff_regularity_consistency}(b) is satisfied when 
\begin{align}
\label{eq:countable_N_log}
	\NT \gg \MT \log\T.
\end{align}
Putting together the pieces yields that the stated consistency conditions imply \cref{eq:suff_regularity_consistency} (which in turn implies \cref{eq:regularity_consistency}).

\paragraph{Normality conditions} Next, we establish that the stated CLT conditions imply \cref{eq:suff_regularity_clt}, which in turn implies \cref{eq:regularity_clt} as desired. Note that \cref{eq:countable_param_ratio} implies that \cref{eq:suff_regularity_clt}(b) holds when $\NT\gg \MT$, and \cref{eq:countable_pmin_log} directly implies that  \cref{eq:suff_regularity_clt}(c) holds. Finally, a direct verification using \cref{eq:thres_scaling_finite} shows that the condition~\cref{eq:suff_regularity_clt}(a) holds when
\begin{align}
\label{eq:lt_bound_finite}
	\T^{1-2\beta} \gg \LT^2 \log(\NT\T).
\end{align}
 Putting together the pieces yields the desired claim.

 \subsubsection{Proof for part~\cref{item:highd_asymp}}
\label{sub:proof_for_example_highd_asymp}
Use the shorthand $\omega\defeq\frac{1-2\beta}{2}$. Then with $\threshold_{\T}=2\sigma_{a}^2 \!+\! \errtwo \!+\! \T^{\omega/2} \errtwo$, we find that
\begin{align}
	\thresonet = \T^{\omega/2} \errtwo,
	\qquad
	\threstwot = (\T^{\omega/2}+2) \errtwo,
	\qtext{and}
	\frac{\threstwot}{\thresonet} = 1+\frac{2}{\T^{\omega/2}},
	\label{eq:eta_high_d}
\end{align}
which when combined with \cref{eq:errtwo} and the fact that $1 \geq \pmint = \Omega(\T^{-\beta})$ yields
\begin{align}
	\max\braces{\thresonet, \threstwot} \seq{(i)}  \order\parenth{\frac{\sqrt{\log(\NT\T)}}{\sqrt{T^{1-\omega-2\beta}}}},
	\stext{and}
	\max\braces{\frac{1}{\thresonet}, \frac{1}{\threstwot}} \seq{(ii)} \order\parenth{\frac{{\sqrt{T^{1-\omega}}}}{\sqrt{\log(\NT\T)}}},
	\label{eq:high_d_threshold_scaling}
\end{align}
Next, using \cref{eq:eta_high_d}, we find that
\begin{align}
\label{eq:phi_ratio_highd}
	\frac{\pparam[\threstwot]}{\pparam[\thresonet]}
	\asymp \parenth{\frac{\threstwot}{\thresonet}}^{\frac{d}{2}}
	\seq{\cref{eq:eta_high_d}} \parenth{1 + \frac{2}{\T^{\omega/2}}}^{\frac{d}{2}},
\end{align}

\paragraph{Consistency conditions} The display~\cref{eq:phi_ratio_highd} directly implies that
\begin{align}
\label{eq:highd_param_ratio}
	\frac{\pparam[\threstwot]}{\pparam[\thresonet]}-1
	= \order\parenth{\frac{d}{\T^{\omega/2}}}
	= o_{\T}(1) \qtext{if} \omega>0,
\end{align}
so that \cref{eq:suff_regularity_consistency}(c) is satisfied. Now, using~\cref{eq:high_d_threshold_scaling}(i), we find that \cref{eq:suff_regularity_consistency}(a) is satisfied if 
\begin{align}
\label{eq:condition_on_T}
	\T^{1-\omega-2\beta} \gg \log\T.
\end{align}
Next, using $\pparam[\thresonet]\asymp(\thresonet)^{d/2}$ and \cref{eq:high_d_threshold_scaling}(ii), we find that \cref{eq:suff_regularity_consistency}(b) is satisfied when
\begin{align}
\label{eq:highd_N_log}
	\NT \gg \parenth{ \frac{\T^{1-\omega}}{\log(\NT\T)} }^{\frac{d}{4}} \log(\NT\T).
\end{align}
Putting the pieces together, and substituting $\omega = \frac{1-2\beta}{3}$ yields the desired claim.

\subsection{Proof of \cref{cor:ate_asymp}}
\label{sub:proof_of_cor:ate_asymp}
For brevity, we use $a_{\T}\gg b_{\T}$ to denote that $b_{\T} = o(a_{\T})$ in this proof. 
Next, we note that the slight abuse of notation in the definitions of $\thresonet,\threstwot$ in \cref{thm:anytime_asymp,thm:ate_asymp}, where in a factor of $\sqrt{2}$ is multiplied to $\errtwo$ in \cref{thm:ate_asymp} has no real consequence on the discussion to follow. Moreover, several algebraic simplifications done here follow from our arguments in \cref{sub:proof_of_cor:anytime_asymp}.

Since $\pmint = \Omega(\T^{-\beta})$, we have
\begin{align}
	\frac{\sumt \pmint^{-\half}}{\T} \leq c \T^{\frac{\beta}{2}},
	\qquad
	\frac{\sumt \pmint\inv}{\T} &\leq c' \T^{\beta},  \\
	\qtext{and}
	\sum_{\t=1}^{\T} \exp\biggparenth{-\frac{1}{16}\pmint[\t]\probparam_{\lunit[], a}(\thresonet)\NT\!}
	&\leq \T \exp\parenth{-c''\T^{-\beta}\probparam_{\lunit[], a}(\thresonet)\NT\!}
\end{align}
for some universal constants $c, c', c''$.
Consequently, the following conditions suffice for \cref{eq:ate_regularity_consistency}:
\begin{align}
	\label{eq:suff_ate_regularity_consistency}
	\threstwot
	\stackrel{(a)}{\ll} 1,
	\quad 
	\NT \stackrel{(b)}{\gg}  \frac{\T^{\beta}\log(\NT\T)}{\inf_{\lunit[]} \probparam_{\lunit[], a}(\thresonet) }
	\qtext{and}
	\T^{\beta}\abss{ \frac{\pparam[\threstwot]}{\pparam[\thresonet]}\!-\!1} \stackrel{(c)}{\ll} 1.
\end{align}
Notably, \cref{eq:suff_ate_regularity_consistency}(a) is identical to \cref{eq:suff_regularity_consistency}(a), while \cref{eq:suff_ate_regularity_consistency}(b) has an extra factor of $\T^{\beta}$ on the RHS when compared to the RHS of \cref{eq:suff_regularity_consistency}(b), and \cref{eq:suff_ate_regularity_consistency}(c) has an extra factor of $\T^{\beta}$ on the LHS when compared to the LHS of \cref{eq:suff_regularity_consistency}(c).

Similarly, the following conditions suffice to ensure \cref{eq:ate_regularity_clt}:
\begin{align}
	\label{eq:suff_ate_regularity_clt}
	\KT\threstwot
	\stackrel{(a)}{\ll} 1,
	\quad 
	\NT \stackrel{(b)}{\gg}  \frac{\KT\T^{\beta}\log(\NT\T)}{\inf_{\lunit[]} \probparam_{\lunit[], a}(\thresonet) }
	\qtext{and}
	\sqrt{\KT}\T^{\beta}\abss{ \frac{\pparam[\threstwot]}{\pparam[\thresonet]}\!-\!1} \stackrel{(c)}{\ll} 1.
\end{align}

\subsubsection{Proof for \cref{item:finite_ate_asymp}: \examplediscrete}
\label{sub:proof_for_example_finite_ate_asymp}
We repeat some arguments from \cref{sub:proof_for_example_finite_asymp}. 

\paragraph{Consistency conditions}
For $\threshold_{\T}=2\sigma_{a}^2+\sqrt{2}\errtwo$, we have $\thresonet = 0$, and $\threstwot=2\sqrt{2}\errtwo$. Now repeating arguments from \cref{sub:proof_for_example_finite_asymp}, we conclude that \cref{eq:suff_ate_regularity_consistency}(a) and (c) are immediately satisfied under the scaling assumed in this part. Next, noting that $\inf_{\lunit[]} \probparam_{\lunit[], a}(\thresonet) = \frac1{\MT}$, we conclude that \cref{eq:suff_ate_regularity_consistency}(b) is satisfied if
\begin{align}
	\NT \gg \MT\T^{\beta}\log(\NT\T),
\end{align}
thereby yields that the conditions assumed for consistency imply the conditions in \cref{eq:suff_ate_regularity_consistency}. 

\paragraph{Normality conditions} First, note that \cref{eq:countable_param_ratio} immediately implies that \cref{eq:suff_ate_regularity_clt}(c) is satisfied. Next, argument similar to that for \cref{eq:lt_bound_finite} implies that \cref{eq:suff_ate_regularity_clt}(a) is satisfied if
\begin{align}
\label{eq:kt_bound_finite}
	\T^{1-2\beta} \gg \KT^2\log(\NT\T),
\end{align}
which is feasible only if $\beta\leq\half$. Finally, we can directly verify that \cref{eq:suff_ate_regularity_clt}(b) is satisfied if
\begin{align}
	\NT \gg \MT \sqrt{\T}\log(\NT\T).
\end{align}
where we assert that  \cref{eq:kt_bound_finite} holds. Putting the pieces together yields the desired claim.

\subsubsection{Proof for \cref{item:highd_ate_asymp}: \examplecontinuous}
\label{sub:proof_for_example_highd_ate_asymp}
In the discussion below, we make use of some arguments from \cref{sub:proof_for_example_highd_asymp}.

\paragraph{Consistency conditions}
 For $\threshold_{\T}\!=\!2\sigma_{a}^2 \!+\! \sqrt{2}\errtwo \!+\! \sqrt{2}\T^{\frac{\omega}{2} } \! \errtwo$, we have
 \begin{align}
	\max\braces{\thresonet, \threstwot} \seq{(i)}  \order\parenth{\frac{\sqrt{\log(\NT\T)}}{\sqrt{T^{1-\omega-2\beta}}}},
	\qtext{and}
	\max\braces{\frac{1}{\thresonet}, \frac{1}{\threstwot}} \seq{(ii)} \order\parenth{\sqrt{\frac{\T^{1-\omega}}{\log(\NT\T)}}},
	\label{eq:high_d_ate_threshold_scaling}
\end{align}
Then, using \cref{eq:highd_param_ratio}, we conclude that \cref{eq:suff_ate_regularity_consistency}(c) is satisfied if 
\begin{align}
	\beta<\omega/2.
\end{align}
Next, noting that $\inf_{\lunit[]} \probparam_{\lunit[], a}(\thresonet) \asymp (\thresonet)^{\frac{d}{2}}$ and \cref{eq:high_d_ate_threshold_scaling}(ii), \cref{eq:suff_ate_regularity_consistency}(b) is satisfied if
\begin{align}
	\NT \gg  \parenth{ \frac{\T^{1-\omega}}{\log(\NT\T)} }^{\frac{d}{4}} \T^{\beta}\log(\NT\T).
\end{align}
Finally, using \cref{eq:high_d_ate_threshold_scaling}(i), we find that \cref{eq:suff_ate_regularity_consistency}(a) is satisfied if
\begin{align}
	\T^{1-\omega-2\beta} \gg \log(\NT\T).
\end{align}
Putting the pieces together, and substituting $\omega = \frac{1+2\beta}{3}$ with $\beta<\quarter$ leads to the sufficient conditions stated in this part.

\paragraph{Normality conditions}
Next, we verify that the conditions assumed for normality imply the conditions in \cref{eq:suff_ate_regularity_clt}. Using similar arguments as above, we find that \cref{eq:suff_ate_regularity_clt}(c) is satisfied if 
\begin{align}
	\KT \ll \T^{\omega-2\beta}.
\end{align}
Using \cref{eq:high_d_ate_threshold_scaling}(i), we find that \cref{eq:suff_ate_regularity_clt}(a) is satisfied if 
\begin{align}
	\KT \ll \frac{\T^{\frac{1-\omega-2\beta}{2}}}{ \sqrt{\log(\NT\T)} }.
\end{align}
Moreover, \cref{eq:suff_ate_regularity_clt}(b) is satisfied if
\begin{align}
	\NT \gg \parenth{ \frac{\T^{1-\omega}}{\log(\NT\T)} }^{\frac{d}{4}} \KT \T^{\beta}\log(\NT\T).
\end{align}
Substituting $\omega = \frac{1+2\beta}{3}$, with $\beta<\quarter$ yields the desired claim.

    \section{Detailed deferred from experiments}

    We first discuss our strategy for hyper-parameter tuning and variance estimation followed by additional details about the experiments.

    \subsection{Tuning $\threshold$ and estimating $\sigma$} 
    \label{sub:tune_eta}
    We do a data-split to tune $\threshold$. Consider a two-way split of time points $[\T]$---denoted by $\mc T_1$ and $\mc T_2$. First, we compute the distances $\rho_{\mrm{vec}, a}\defeq\braces{\estdist}_{i\neq j}$~\cref{eq:time_nbr_dist} for all pairs of units and for each treatment $a$ using only the data in time $\mc T_1$. Next, given a threshold $\threshold$, let $\estobs[\n, \t, \threshold]$ denote the estimate for $\trueobs$ using neighbors using the distances in $\rho_{\mrm{vec}, a}$. Now, we compute a grid of $k$ suitably chosen percentiles from the s the distance vector $\rho_{\mrm{vec}, a}$, and call this the candidate set $\eta_{\mrm{set}, a} \defeq \braces{\threshold_1, \ldots, \threshold_{k}}$. Then, for each treatment $a$, we tune $\threshold$ as
    \begin{align}
    \label{eq:eta_tune}
        \threshold_{\trm{tuned}, a} = \arg\min_{\eta \in \eta_{\mrm{set}, a} } \what\sigma^2_{a, \eta}
        \qtext{where} \what\sigma^2_{a, \eta} \defeq\frac{\sum_{\n\in\N}\sum_{\t\in\mc T_2} \indicator(\miss=a) (\obs-\estobs[\n, \t, \threshold])^2}{\sum_{\n\in\N}\sum_{\t\in\mc T_2} \indicator(\miss=a) }.
    \end{align}
    Note that $\what\sigma^2_{a, \eta}$ is effectively the mean squared error on the validation set.

\subsection{Consistent estimate for the noise variance}
\label{sub:variance_estimate}

Our next result establishes the consistency of the noise variance estimate obtained using the procedure described above.

\begin{theorem}[\varestimatename]
\label{thm:var_est}
Consider the set-up from \cref{thm:ate_asymp} such that the sequence $\braces{\threshold_{\T}', \N_{\T}, \pmint}_{\T=1}^{\infty}$ satisfies \cref{eq:ate_regularity_consistency} after substituting $\pmint^{-1/2}, \pmint\inv$ in place of $\frac1{\T}\sum_{\t=1}^{\T}\pmint[\t]^{-1/2}, \frac1{\T}\sum_{\t=1}^{\T}\pmint[\t]^{-1/2}$ respectively. Then the estimate $\what{\sigma}^2_{\threshold_{\T}}$~\cref{eq:eta_tune} using $\mc T_1 = [\T/2]$, and $\mc T_{2} = \sbraces{\frac{\T}{2}\!+\!1,\! \ldots, \!\T}$  satisfies $\what{\sigma}^2_{\threshold_{\T}} \stackrel{p}{\longrightarrow} \sigma_{a}^2$ as $\T\to\infty$.
\end{theorem}

\paragraph{Interpretation} \cref{thm:var_est} establishes the consistency of our variance estimate $\what{\sigma}$ assuming that a good choice of $\threshold_{\T}$ is already known. However, a good choice of $\threshold_{\T}$ is close to $2\sigma_{a}^2$. This requirement is similar in spirit to the analysis of iterative algorithms that require a good initialization. However, this requirement also appears circular as the knowledge of a good $\threshold$ would imply the knowledge of $\sigma$. Nevertheless, we can construct a practical procedure that builds on this circular reasoning to both tune $\threshold$ and estimate $\sigma$: (i) Initialize $\threshold$, (ii) estimate $\sigma$ using \cref{eq:eta_tune}, (iii) update $\threshold$ to $2\what{\sigma}^2$, and (iv) iterate steps (ii) and (iii) till the estimate for $\sigma$ converges. In full generality, our result do not immediately establish that this procedure contracts towards $2\sigma_{a}^2$ due to unknown constants that arise in the analysis. However, if the multiplicative constant in front of $\threshold-2\sigma_{a}^2$ in display~\cref{eq:nn_bnd} from \cref{thm:anytime_bound} satisfies $\frac{2\vconst^2}{\lambda_{a}}\leq C$, then a straightforward modification of the proof of \cref{thm:var_est} would yield that this iterative procedure contracts towards $\sigma_{a}^2$ if the starting guess is bounded by $\frac{2C-1}{C-1}\sigma_{a}^2$ under suitable regularity conditions like that in \cref{thm:var_est}.
 We note that the constant $\vconst^2/\lambda_{a}$ is necessarily greater than or equal to $1$ due to their definitions (see \cref{assum:lambda_min}).

\begin{remark}
    \label{rem:simple_eta}
    The condition~\cref{eq:ate_regularity_consistency} implies that $\threshold_{\T} \to 2\sigma_{a}^2$ thereby motivating a practical heuristic (with an asymptotic justification) of setting the variance estimate as $\what{\sigma}^2 = \frac{\threshold_{\trm{tuned}}}{2}$, where the hyperparameter $\threshold$ is tuned as per the discussion above and in \cref{sub:tune_eta}.  In practice, we can also set our estimate for noise variance as $\what{\sigma}^2_{\threshold_{\trm{tuned}}}$ given by \cref{eq:eta_tune} since asymptotically, for our generative model and under regularity conditions on the decay of $\eta_{\T}$, we would have that $\what{\sigma}^2_{\threshold_{\trm{tuned}}} \approx \frac{\threshold_{\trm{tuned}}}{2}$.
    We highlight that all the steps are repeated separately for each treatment $a$ in our simulations.
    To improve finite sample coverage, we add to our asymptotic variance estimate the within neighbor variance estimate. That is, we replace the interval \cref{eq:unit_ci} by
    \begin{align}
    \label{eq:unit_ci_adj}
        \biggparenth{\estobs \!\!-\!\! \frac{\zalpha\,\wtil{\sigma}_a}{\sqrt{\nestnbr}},\,\,
        \estobs \!\!+\!\! \frac{\zalpha\,\wtil{\sigma}_a}{\sqrt{\nestnbr}}
        }
        \qtext{where}
       \wtil{\sigma}_a =\what{\sigma}_a + \what{\sigma}_{\n, \t, a}
    \end{align}
    and $\what{\sigma}_{\n, \t, a}^2 =\frac{\sum_{j=1}^{\N}\indicator(\miss=a)(\obs[j, \t]-\estobs)^2}{\sum_{j=1}^{\N}\indicator(\miss=a)}$, where $\what \sigma_a^2$ is the noise variance estimate. 
    See \cref{sec:exp_details} for empirical results for tuning $\eta$ in our simulations and HeartSteps.
\end{remark}

\paragraph{Examples} We reconsider the examples \cref{item:finite_ate_asymp,item:highd_ate_asymp}, with $\pmint = \Omega(\T^{-\beta})$. In this case, we find that both $\pmint^{-1/2}$ and $\frac1{\T}\sum_{\t=1}^{\T}\pmint[\t]^{-1/2}$ are $\order(\T^{\frac{\beta}{2}})$, and both $\pmint^{-1}$ and $\frac1{\T}\sum_{\t=1}^{\T}\pmint[\t]^{-1}$ are $\order(\T^{\beta})$; these upper bounds are what we used in establishing the sufficient conditions \cref{thm:ate_asymp}. Thus, we conclude that for examples \cref{item:finite_ate_asymp,item:highd_ate_asymp}, with the assumed conditions, \cref{thm:var_est} implies that $\what{\sigma}$ is consistent estimator for $\sigma$.

  \subsection{Proof of \lowercase{\Cref{thm:var_est}}: \varestimatename}
    \label{sec:proof_of_thm:var_est}
    Without loss of generality, we can assume $a=1$, and we use the simplified notation $\indicator(\miss=a) = \miss$ in the proof.
    We have
    \begin{align}
        \sumN \sum_{\t> \T/2}\miss (\obs\!-\!(\estobs)^{\dagger})^2 
        = \sumN \sum_{\t> \T/2}\miss  \parenth{(\noiseobs)^2 
        + 2\noiseobs(\trueobs\!-\!(\estobs)^{\dagger}) + (\trueobs\!-\!(\estobs)^{\dagger})^2}.
    \end{align}
    Next, we make use of the following lemma which can be proven essentially by mimicking arguments from the proof of \cref{lem:noise_avg} (see the end of this section for a proof sketch).
    \begin{lemma}
        \label{lem:noise_sq_avg}
        Under \cref{assum:policy,assum:zero_mean_noise}, for any $\delta\in(0, 1]$, we have
        \begin{align}
             \P\brackets{|\sum_{\n=1}^{\N} \sum_{\t>\T/2} \miss ((\noiseobs)^2-\sigma_{a}^2)| \leq
             \nconst^2 \sum_{\n=1}^{\N} \biggparenth{\log(\N\T/\delta) \sum_{\t=1}^{\T} \miss}^{\half} }\geq 1-\delta.
             \label{eq:noise_sq}
        \end{align}
        Furthermore, if \cref{assum:bilinear,assum:lambda_min,assum:iid_unit_lf} also hold, then
         \begin{align}
             \P\brackets{|\sum_{\n=1}^{\N} \sum_{\t> \T/2}\miss\noiseobs(\trueobs\!-\!(\estobs)^{\dagger})| \leq
             c \sum_{\n=1}^{\N} \biggparenth{\log(\N\T/\delta) \sum_{\t=1}^{\T} \miss}^{\half} }\geq 1-\delta.
             \label{eq:noise_cross}
        \end{align}
        for $c=2\nconst (\uconst\vconst+\nconst)$.
    \end{lemma}
    Moreover, \cref{lem:bc_bound} and a standard union bound implies that
    \begin{align}
    \label{eq:tij_bnd}
        \P\brackets{ \sum_{\t> \T/2}\miss  \geq \frac{\pmint\T}{4}, \stext{for all} \n \in [\N]} \geq 1-\frac{1}{\N\T}, \qtext{if} \pmint\geq \frac{16\log(\N\T)}{\T}.
    \end{align}
    Putting together \cref{lem:noise_sq_avg,eq:tij_bnd}  with $\delta= \delta_{\T}= \frac{1}{\NT\T}$, we conclude that
    \begin{align}
        \abss{\frac{\sumN \sum_{\t> \T/2}\miss (\noiseobs)^2}{\sumN \sum_{\t> \T/2}\miss}-\sigma_{a}^2} &= \order_P\parenth{\frac{\log(\NT\T)}{\pmint\sqrt{\T}}} \seq{(i)} o_{P}(1),
        \\
        \abss{\frac{\sumN \sum_{\t> \T/2}\miss\noiseobs(\trueobs\!-\!(\estobs)^{\dagger})}{\sumN \sum_{\t> \T/2}\miss}} &= \order_P\parenth{\frac{\log(\NT\T)}{\pmint\sqrt{\T}}} \seq{(ii)} o_P(1).
    \end{align}
    where steps~(i) and (ii) follow from the the variant of \cref{eq:ate_regularity_consistency} assumed in \cref{thm:var_est}.
    Moreover, applying \cref{thm:anytime_bound} and mimicking steps from \cref{eq:ate_err,eq:ate_err_1}, we find that
    \begin{align}
         &\frac{\sumN \sum_{\t> \T/2}\miss(\trueobs\!-\!(\estobs)^{\dagger})^2}{\sumN \sum_{\t> \T/2}\miss}\\
          &\leq \max_{\n,\t>\T/2} (\trueobs\!-\!(\estobs)^{\dagger})^2 \\
         &=\order_{P}\biggparenth{\threstwot +
         \frac{\log(\NT\T)}
         {\pmint\sup_{\lunit[]} \probparam_{\lunit[], a}(\thresonet) (\NT\!-\!1)}
            +
          \frac{1}{\pmint^2}\sup_{\lunit[]} \biggparenth{\frac{\probparam_{\lunit[], a}(\threstwot)}{\probparam_{\lunit[], a}(\thresonet)}\!-\!1}^2}
          &\seq{(i)} o_P(1),
    \end{align}
    where step~(i) follows from the the variant of \cref{eq:ate_regularity_consistency} assumed in \cref{thm:var_est}.
    Putting the pieces together yields the claim.

    \paragraph{Proof of \cref{lem:noise_sq_avg}}
    The proof follows essentially using the same stopping time construction as in the proof of \cref{lem:noise_avg}, so we only sketch out the definitions of $W_{\l}$ that replace~\cref{eq:wl_eps}. For establishing \cref{eq:noise_sq}, we use $W_{\l} = (\noiseobs[\n, \tstop])^2-\sigma_{a}^2$ and repeat arguments from \cref{sub:proof_of_lem:noise_avg} by noting that $|W_{\l}| \leq \nconst^2$ almost surely, and $\E[W_{\l}\vert \mc H_{\l}]=0$ using the same argument as in \cref{eq:wl_eps_zero}. Next, to establish \cref{eq:noise_cross}, we use similar arguments with $W_{\l} = \noiseobs[\n, \tstop](\trueobs[\n, \tstop]\!-\!(\estobs[\n, \tstop, \threshold_{\T}])^{\dagger})$, after noting that $|W_{\l}| \leq 2\nconst|\vconst\uconst+\nconst|$ almost surely, and that $\E[W_{\l}\vert \mc H_{\l}] = 0$ since by construction the estimate $\estobs[\n, \tstop, \threshold_{\T}]$ does not depend on $\noiseobs[\n, \tstop]$ (due to the data split).

\section{Deferred details and results from Sec.~\ref{sec:experiments}}
\label{sec:exp_details}
In this appendix, we collect supplementary information for the experimental results presented in \cref{sec:experiments}.

\subsection{Deferred details for simulations from \cref{sec:simulations}}
\label{sub:sampling}
We start by presenting additional set of results in our simulations in \cref{fig:ate_1_plot}, which is an analog of \cref{fig:ate_0_plot} except that we now choose a non-zero value of ATE. 
    \begin{figure}[!t]
        \centering
        \begin{tabular}{c}
        \includegraphics[width=\linewidth]{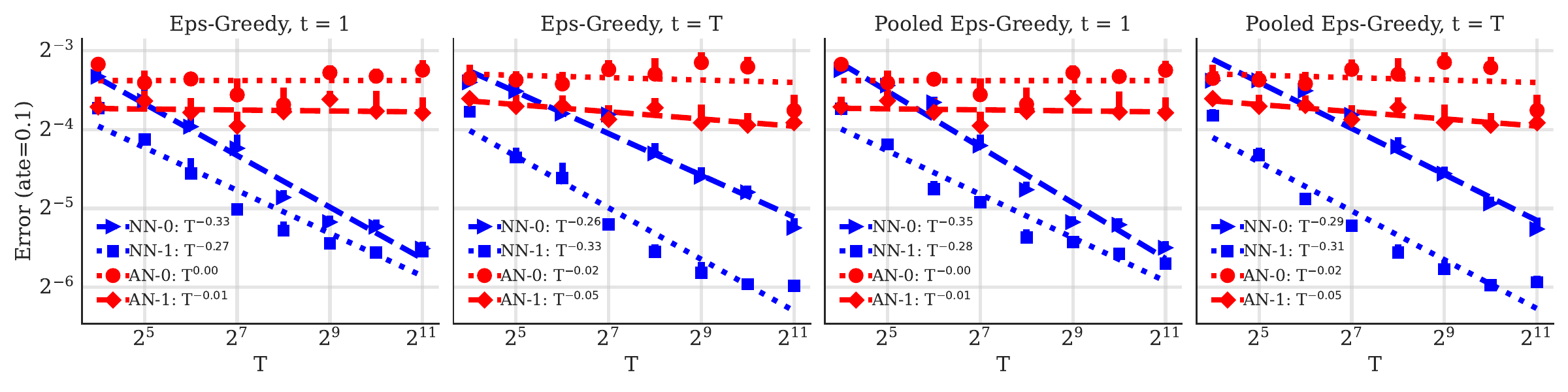} \\
        \includegraphics[width=\linewidth]{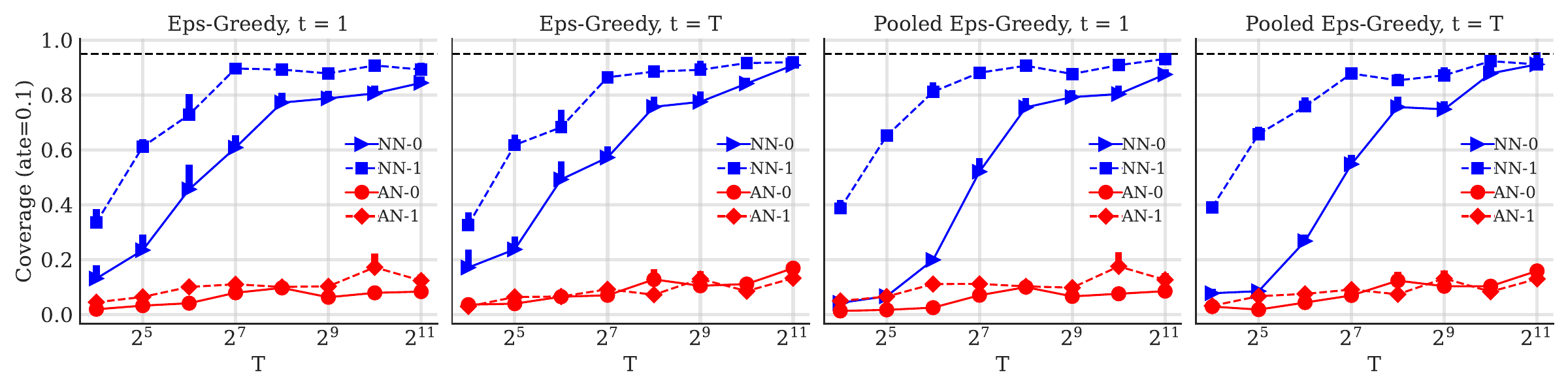} \\
        \includegraphics[width=\linewidth]{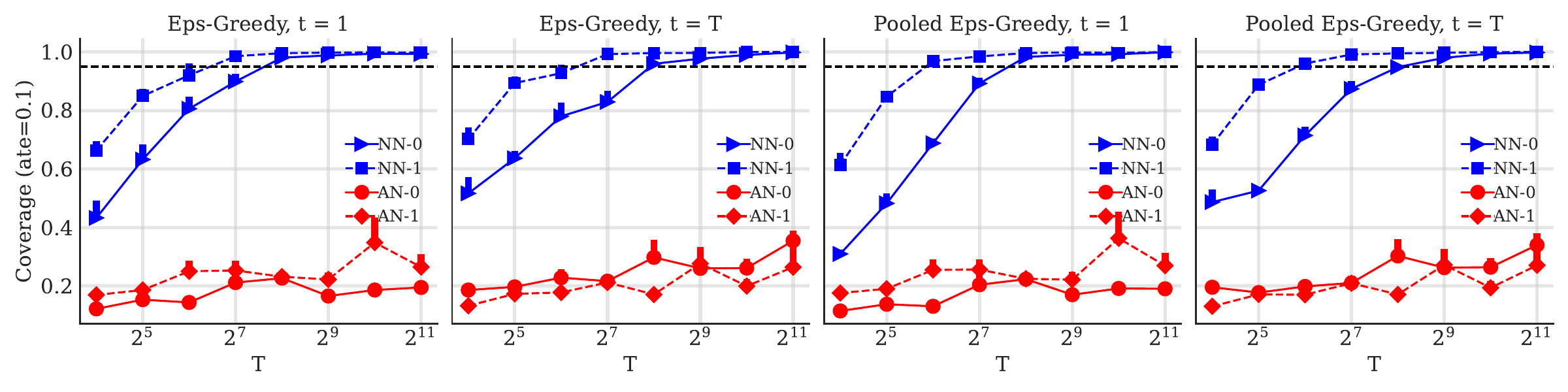} \vspace{-5mm}
        \end{tabular}
        \caption{\tbf{Error (top) and coverage for asymptotic prediction intervals using $\what{\sigma}^2 \!=\! \frac{\eta}{2}$ (middle) and $\what{\sigma}^2$ with finite sample correction~\cref{eq:unit_ci_adj} (bottom),  as a function of total time points $\T$, when $\mrm{ATE}=0.1$ in \cref{eq:sims}}. We plot the results for estimates using nearest neighbors (NN-a)~\cref{eq:obs_estimate} and a baseline that treats all units as neighbors (AN-a), for treatments $a\in\braces{0, 1}$ at time $\t \in \{1, \T\}$. When the total number of time points in the trial is equal to the value on the x-axis, the corresponding value on the y-axis for error plots denotes the mean of the errors $|\estobs\!-\!\trueobs|$ across all $\N\!=\!512$ users for treatment $a$ at time $\t$ in the error plots, and that for the coverage plots denotes the fraction of all $N\!=\!512$ users for which the interval~\cref{eq:unit_ci} with appropriate variance estimate covers the true counterfactual mean $\trueobs$ for treatment $a$ at time~$t$. Note that the only difference in the generative model for \cref{fig:ate_0_plot} and this figure is in the value of $\mrm{ATE}$.}
        \label{fig:ate_1_plot}
    \end{figure}

\subsubsection{Sampling algorithm details}
For binary treatments, the sampling policy for the $\eps$-greedy algorithm relies on a running average of mean outcomes observed for the two treatments, and chooses the better treatment (with higher mean outcome estimate) with probability $0.5(1+\eps)$ and the worst treatment  with probability $0.5(1-\eps)$. In particular, the individualized (unit-specific) $\eps$-greedy's sampling policy for user $\n$ is given by
\begin{align}
    \policy_{\t, \n}(1) = \begin{cases}
        \frac{1+\eps}{2} \qtext{if} \frac{\sum_{\t'<\t} \obs[\n,\t'] \miss[\n,\t']}{\sum_{\t'<\t} \miss[\n,\t'] } > \frac{\sum_{\t'<\t} \obs[\n,\t'] (1-\miss[\n,\t'])}{\sum_{\t'<\t} (1-\miss[\n,\t']) } \\ 
        \frac{1-\eps}{2} \qtext{otherwise.}
    \end{cases}
\end{align}
For the pooled-variant of $\eps$-greedy used by us, the policy $\policy_{\t, \n}$ is common across all users is given by:
\begin{align}
    \policy_{\t, \n}(1) = \displaystyle\begin{cases}
        \frac{1+\eps}{2} \qtext{if} \frac{\sumn[j]\sum_{\t'<\t} \obs[j,\t'] \miss[j,\t']}{\sumn\sum_{\t'<\t} \miss } > \frac{\sumn[j]\sum_{\t'<\t}\obs[j,\t'] (1-\miss[j,\t'])}{\sumn[j]\sum_{\t'<\t} (1-\miss[j,\t']) } \\ 
        \frac{1-\eps}{2} \qtext{otherwise.}
    \end{cases}
\end{align}
In each case, we set $\policy_{\t, \n}(0)=1-\policy_{\t, \n}(1)$.

\subsubsection{Bias-variance tradeoff with $\eta$ in simulations}
    \cref{fig:eta_tune} illustrates the practical approach used in our experiments to tune $\threshold$. We present the results from the tuning procedure for one of the runs for the trial corresponding to the rightmost plot of top row of \cref{fig:ate_1_plot} with $\N=\T=512, d=2, \sigma_{\vareps}=0.1, \mrm{ATE}=0.1$ and $\eps=0.5$ with treatments assigned using pooled $\epsilon$ greedy. In all cases, we exclude the data for test time points $t=1, \T$. 
    We do a $80\%$-$20\%$ split for the remaining time points (for the above case $\sbraces{2, \ldots, \T-1}$) into 402 (80\%) training and validation 108 (20\%) time points.
     and we compute the grid of possible values for $\eta_{\mrm{set}, a}$ using the $\sbraces{0.5, 1, 2, 5, 10, 15, 25, 30, 40, 50}$ percentile values for the pairwise distances $\sbraces{\estdist}_{i\neq j}$.  The two panels denote the results (a) using the naive estimate~\cref{eq:no_nbr_est} for all $(\n, \t)$ pairs with $\nestnbr=0$ while computing $\sigma^2_{a, \eta}$ in \cref{eq:eta_tune}, and (b) omitting all the $(\n, \t)$ pairs with $\nestnbr=0$ while computing $\sigma^2_{a, \eta}$ in \cref{eq:eta_tune}. 

    We observe in \cref{fig:eta_tune} that the validation error (same as $\sigma^2_{a, \eta}$~\cref{eq:eta_tune}) exhibits a $U$-shape curve with $\eta$ as suggested by our theory (see discussion just before \cref{rem:no_nbr}). On the right axis, we also plot the fraction of estimates in the validation set that have a non-zero count of neighbors, i.e., $\nestnbr>0$, and as expected this curve is monotone in $\eta$. 
    \begin{figure}[t!]
        \centering
        \begin{tabular}{ccc}
        \includegraphics[width=0.48\linewidth]{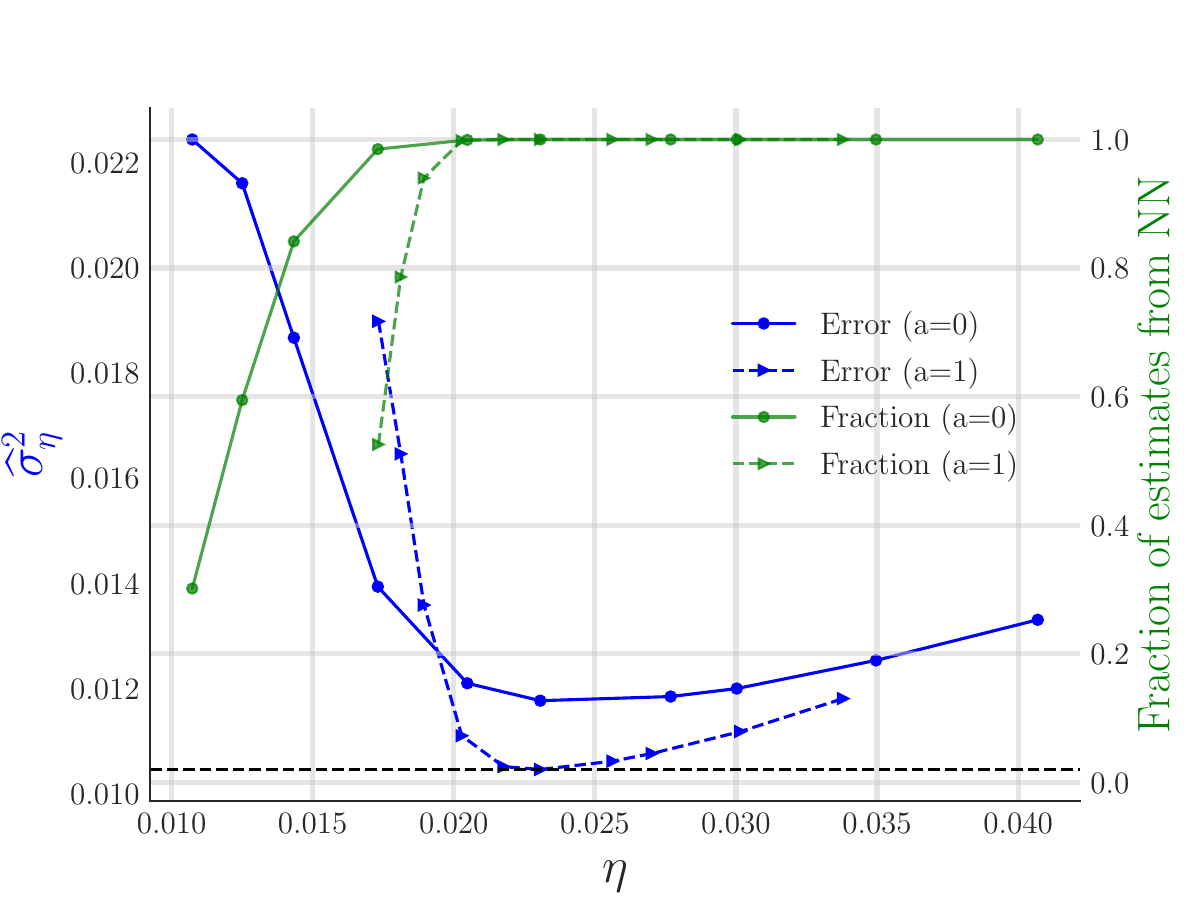} & \hspace{0.2cm}
        &
        \includegraphics[width=0.48\textwidth]{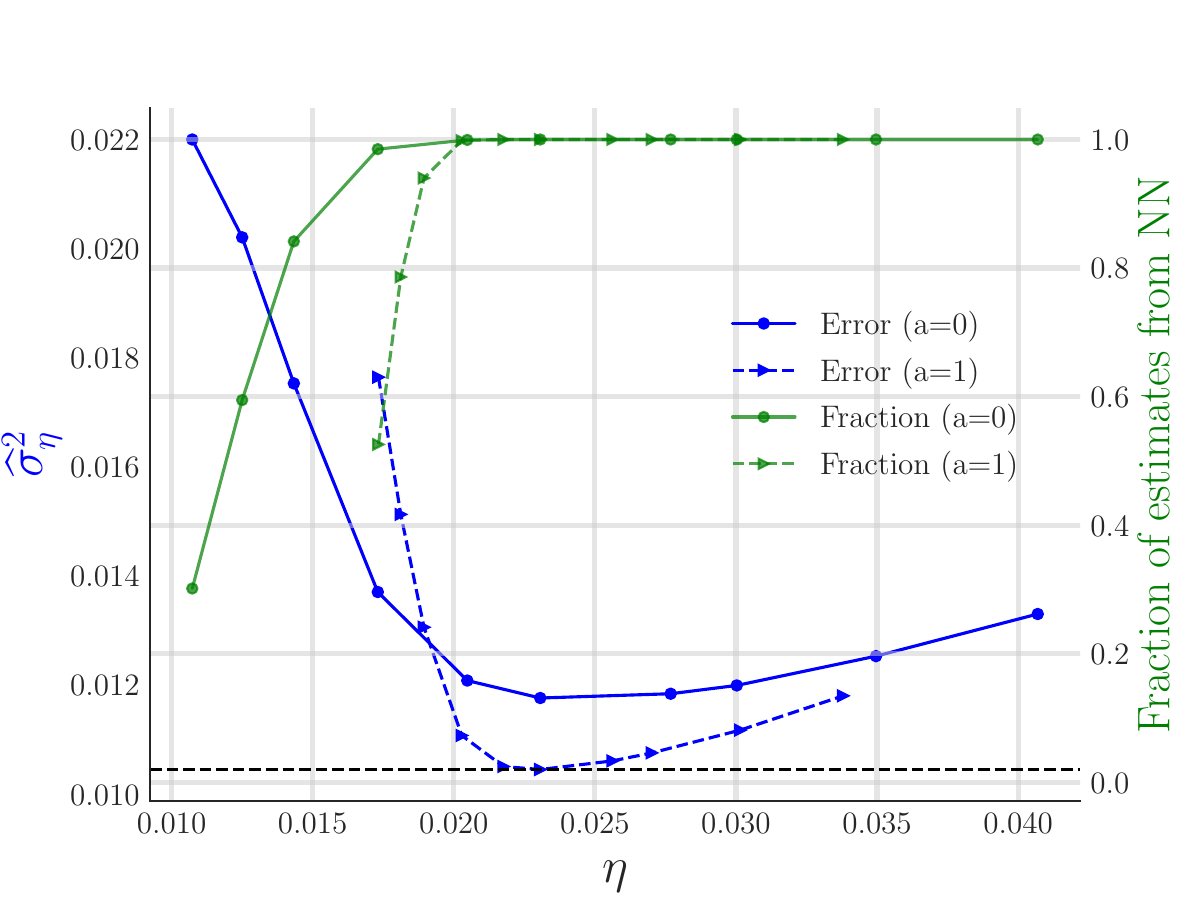} \\ 
        (a) & & (b)
        \end{tabular}
        \caption{\tbf{Hyperparameter $\eta$ tuning in our simulations from \cref{fig:ate_1_plot}.} We present the results for choosing $\eta$ for our nearest neighbors estimates for one trial with $\N=\T=512, d=2, \sigma_{\vareps}=0.1, \mrm{ATE}=0.1$ and $\eps=0.5$ with treatments assigned using pooled $\epsilon$ greedy(same settings as in one of the points in the rightmost plot of top row of \cref{fig:ate_1_plot}). The dashed line denotes the best validation MSE for the better treatment $1$.
         The  y-axis on the left side denotes the averaged value of $(\estobs -\trueobs)^2$ across all validation time points $\t$ and users $\n$ with $\nestnbr > 0$ in panel (a), and across all validation time points and all $\N$ users in panel (b) (where we use the estimate \cref{eq:no_nbr_est} whenever $\nestnbr =0$). The right y-axis in both panels denotes the fraction of estimates in the validation set (out of total $N\times (T-2)/5=52,224$) that have non-zero neighbors, i.e., $\nestnbr > 0$.}
        \label{fig:eta_tune}
    \end{figure}

    \paragraph{Empirical support for the theory-inspired variance heuristic}
    \cref{fig:eta_tune} provides a numerical support to our (theory-inspired) heuristic for setting the variance estimate 
    $\what{\sigma}$ equal to $\min(\what{\sigma}_{\threshold_{\mrm{tuned}}},\sqrt{\threshold_{\mrm{tuned}}/2})$. 
    From \cref{fig:eta_tune}(a), we find that for treatment $1$, the optimal $\eta$, i.e., $\threshold_{\mrm{tuned}}$, is very close to $2\sigma_{\vareps}^2=0.02$. Moreover, the corresponding validation error, $\what{\sigma}^2_{\threshold_{\mrm{tuned}}}$, is also close to $\sigma_{\vareps}^2=0.02=0.01$. In other words, for this treatment both $\what{\sigma}_{\threshold_{\mrm{tuned}}}$ or $\sqrt{\threshold_{\mrm{tuned}}/2}$ are very close to $0.1$. On the other hand, for treatment $0$, we find that $\min(\what{\sigma}_{\threshold_{\mrm{tuned}}},\sqrt{\threshold_{\mrm{tuned}}/2})$ is a good surrogate for the ground truth of $\sigma_{\vareps}=0.1$.

    \begin{remark}
    \label{rem:no_nbr_eta_tune}
        In this paper, all empirical results (except \cref{fig:eta_tune}(b)) present results only across users/timepoints such that $\nestnbr>0$ to ensure that the estimate is reliable. We note that that such a fitering, albeit transparent, might introduce some biases. With users, while the validation MSE typically admits a U-shape curve as in \cref{fig:eta_tune}(a), sometimes the MSE graph might look monotonically increasing with $\eta$ (that is smallest value of $\threshold$ yields the best validation MSE) albeit leading to few estimates with $\nestnbr>0$. As a middle ground to tradeoff MSE with the number of reliable estimates, in all of our simulations, we choose the best $\eta$ that produces at least $70\%$ of the estimates with non-zero neighbors. An alternative choice is to default to \cref{eq:no_nbr_est} for users with no neighbors for a given $\eta$, although such an estimate would be less reliable. For completeness, we present the results for tuning $\eta$ for this alternative choice in \cref{fig:eta_tune}(b). Note that the grid for $\eta$ is the same in two panels of \cref{fig:eta_tune} and furthermore. (The fraction curves would be identical in the two panels.)
    \end{remark}

\subsubsection{Robustness to problem parameters}
\label{sec:robust_params}
In \cref{fig:other_things}, we test the impact on the nearest neighbor performance as a function of various problem parameters; all the unspecified parameter/settings are same as in the simulations for \cref{fig:ate_0_plot}.
We provide the error plots for $\t=\T$ (the analog of \cref{fig:ate_0_plot}(a) except at time $\t=\T$) as we change the total number of units $\N$ in panels (a, b), the dimension of latent factors $d$ in panels (c, d), the minimum sampling probability for the sampling algorithm---the value of $\eps$ for pooled $\eps$-greedy---in panels (e, f). 

As expected, as we decrease the number of units $\N$ or increase the dimensions $d$, the estimation error gets worse (in magnitude) due to lesser number of good neighbors; cf. \cref{fig:ate_0_plot}(a) and \cref{fig:other_things}(a). The decay rate with respect to $\T$ is not affected as long as the bias continues to decrease. However for large values of $\T$ the error begins to flatten, as the variance term (governed largely by $\N$) becomes dominant; this phenomenon leads to an empirical decay rate that is worse than $\T^{-0.25}$ for small $\N$ or large $d$. When $\eps$ is reduced, the bias for the better treatment ($a=1$) is smaller due to more observations under it and thereby it hits an error floor sooner leading to a worsening of decay rate with $\T$. On the contrary, the error for $a=0$ is larger due to lesser number of observations; cf. \cref{fig:other_things}(c).

\subsubsection{Autoregressive time factors} In panels (g, h) of \cref{fig:other_things}, we present results when the time latent factors $\sbraces{\ltime}$ are not drawn \iid and instead from an auto-regressive process:
\begin{align}
\label{eq:ar}
    \ltime[\t+1] = \rho \ltime + \sqrt{1-\rho^2} \nu_{\t+1},
\end{align}
for $\rho \in \sbraces{0.1, 0.9}$, where $\ltime[0], \sbraces{\nu_{t}}_{\t=1}^{\T}$ are drawn \iid from $\mrm{Unif}[-0.5, 0.5]^{d}$. Note that $\rho=0$ in \cref{eq:ar} leads to \iid time factors. 
\cref{fig:other_things}(g,h) when contrasted with \cref{fig:ate_0_plot}(a) reveal that the performance of nearest neighbors is nearly unaffected as $\rho$ varies from $0$ and approaches $1$. This setting is likely to satisfy the sufficient condition \cref{eq:non_iid_suff} stated earlier for non \iid time factors (one can easily verify that \cref{eq:non_iid_suff} holds when in the limit of $\rho=1$). We leave further investigation of the robustness of nearest neighbors to more general time dependent settings for future work.

\addtolength{\tabcolsep}{-0.7em}
\begin{figure}[t!]
    \centering
    \resizebox{\linewidth}{!}{
    \begin{tabular}{cccc}
    \includegraphics[width=0.25\linewidth,trim={0 0 0 1cm},clip]{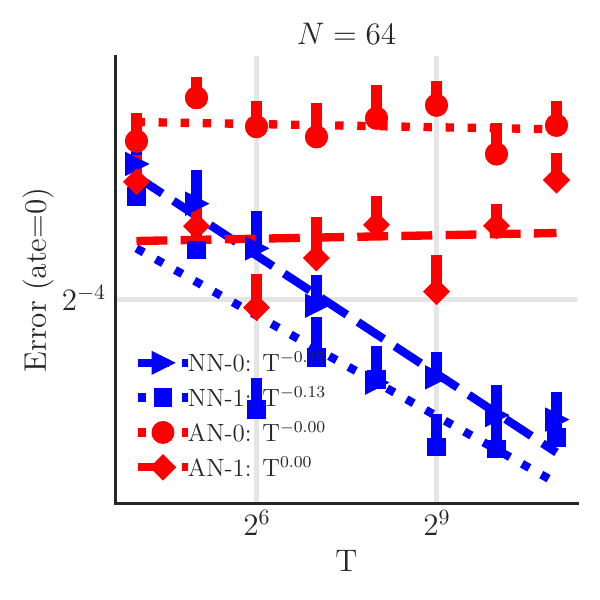}
    &
    \includegraphics[width=0.25\linewidth,trim={0 0 0 1cm},clip]{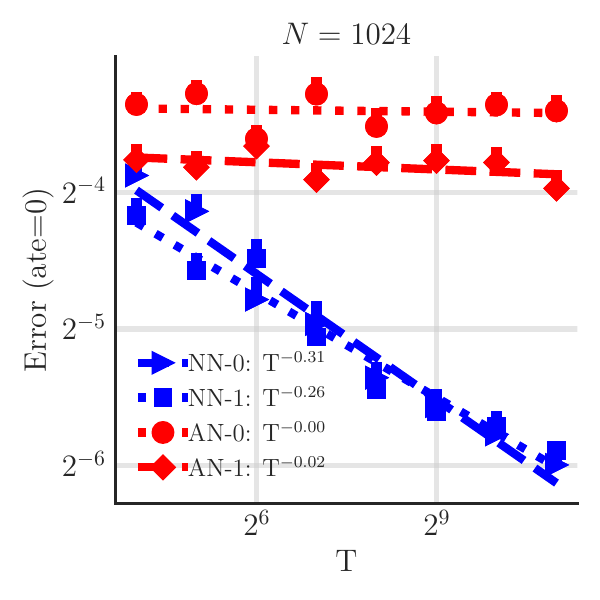}
    &
    \includegraphics[width=0.25\linewidth,trim={0 0 0 1cm},clip]{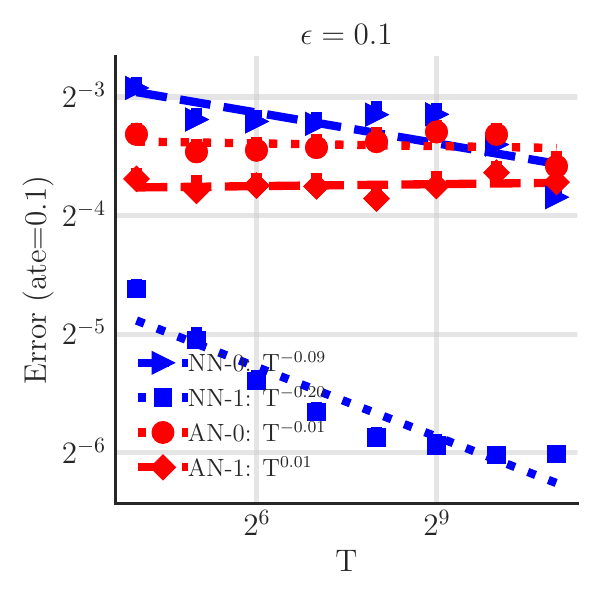}&
    \includegraphics[width=0.25\linewidth,trim={0 0 0 1cm},clip]{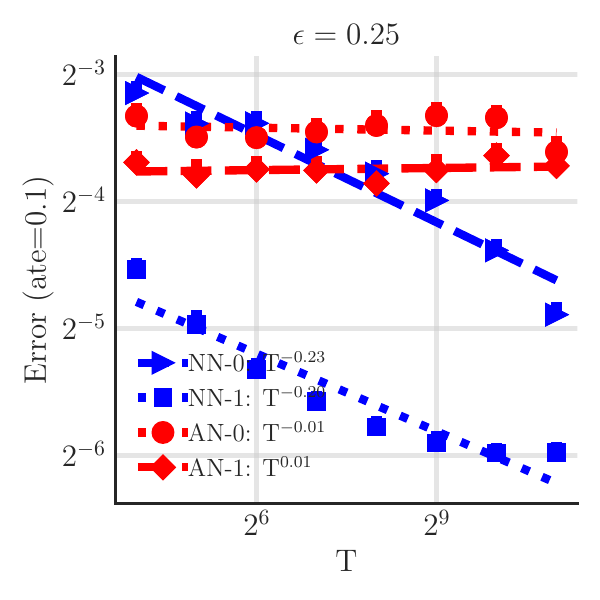}\\[-2mm]
    (a) $\N = 64$  &
    (b) $\N = 1024$ &
    (c) $\epsilon = 0.1$ &
    (d) $\epsilon = 0.25$ \\[2mm] 
    \hline
    \\[-2mm]
    \includegraphics[width=0.25\linewidth,trim={0 0 0 1cm},clip]{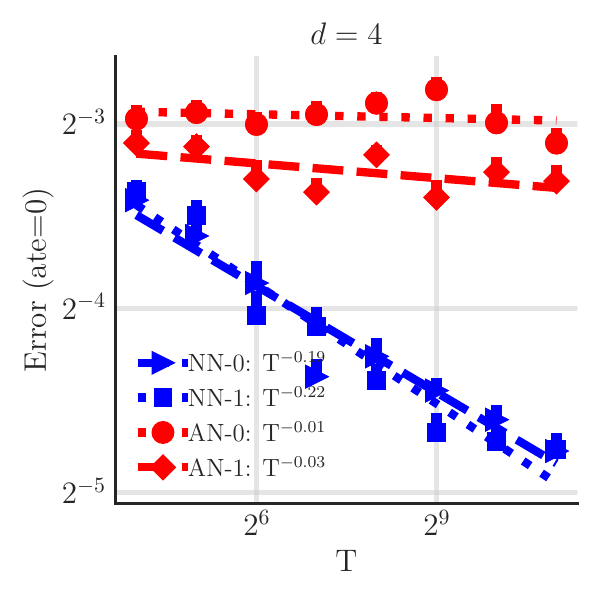}
    &\includegraphics[width=0.25\linewidth,trim={0 0 0 1cm},clip]{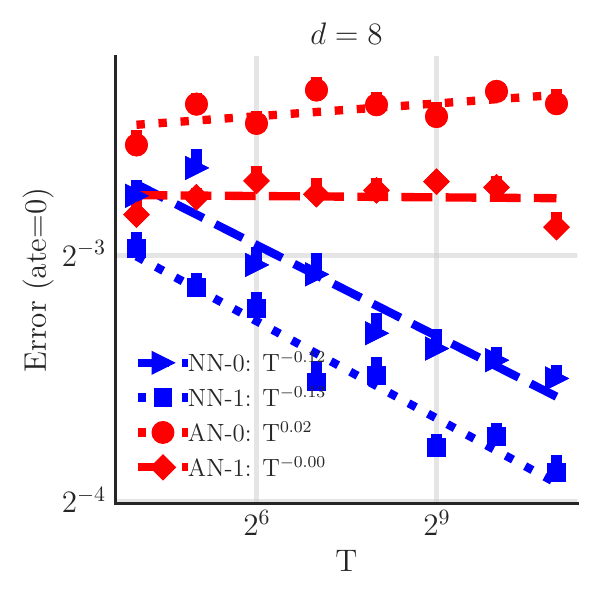}
    &
    \includegraphics[width=0.25\linewidth,trim={0 0 0 1cm},clip]{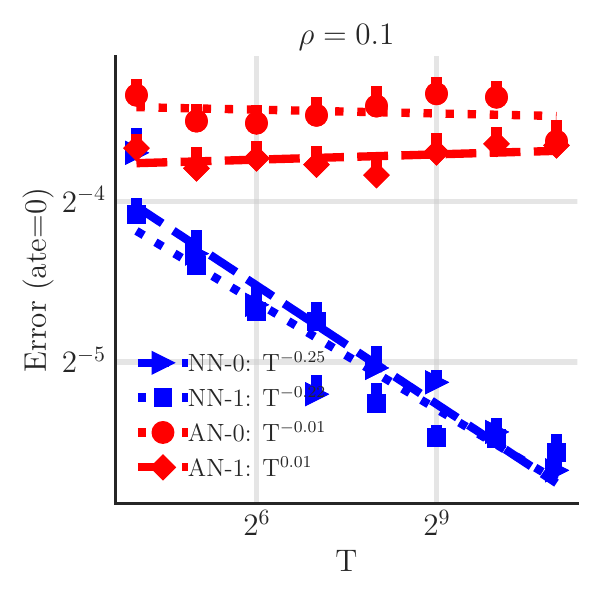}
    &\includegraphics[width=0.25\linewidth,trim={0 0 0 1cm},clip]{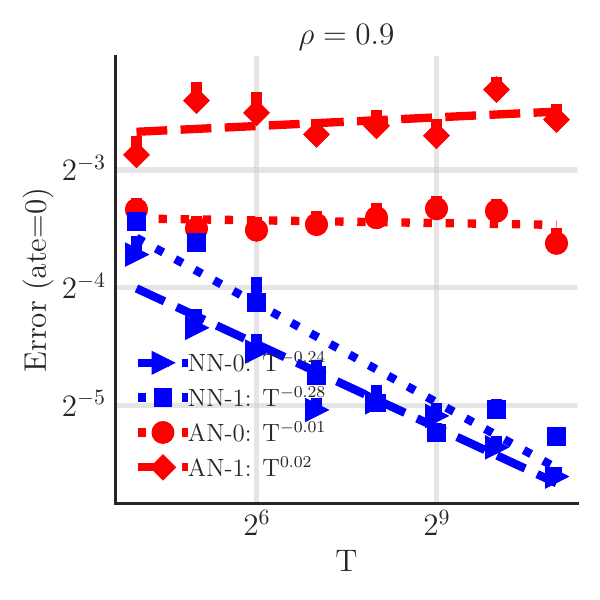}
    \\[-2mm]
    (e) $d=4$ &
    (f) $d=8$ &
    (g) $\rho=0.1$ &
    (h) $\rho=0.9$
    \end{tabular}
    }
    \caption{\tbf{Error as a function of total time points $\T$ when different problem parameters are varied.} We plot the results for estimates using nearest neighbors (NN-a)~\cref{eq:obs_estimate} and a baseline that treats all units as neighbors (AN-a), for treatments $a\in\braces{0, 1}$ at time $\t\! =\!\T$. In all plots, the y-axis denotes the mean of the errors $|\estobs\!-\!\trueobs|$ across all $\N=512$ users for treatment $a$ at time $\t$ when the total number of time points $\T$ in the trial is equal to the value on the x-axis. In panels (a--f), we vary the value of $\N$, $\epsilon$, and $d$ one at a time, while keeping all unspecified parameters at default choices: $\N\!=\! 512, d\!=\!2, \sigma_{\varepsilon} \!=\! 0.1$, and $\epsilon\!=\!0.5$. In panels (g) and (h), we keep all the parameters at default, but generate the latent time factors using AR(1) process with coefficient $\rho$ as in \cref{eq:ar}.
    }
    \label{fig:other_things}
    \vspace{-3mm}
\end{figure}

\subsection{Deferred details for HeartSteps from \cref{sub:mobile_health_study}}
\label{sub:heartsteps}
We now describe additional details and results for the HeartSteps case study.
\subsubsection{Data processing}
We provide a brief overview of our dataset and refer the readers to \cite{liao2020personalized,tomkins2021intelligentpooling} for further details. At each decision time, certain features determined if the user was \emph{available} for randomization, e.g., a suggestion was never sent if the user was driving. We focus on counterfactual inference for such \emph{available times}  and drop all users that are available for less than 20 decision times (out of the maximum possible 450). Such a filtering leaves us with 28 non-pooled, and 7 pooled users, which as a collection are referred to as \emph{\goodusers}. The $5$ user-determined decision times a day represent the user's mid-morning, mid-day, mid-afternoon, mid-evening, and evening. Here, we treat the $5$ decision times on each day to be shared across all users, e.g., we use $\t=$ \{09-14-1,09-14-2, $\ldots$, 09-14-5\} to denote the $5$ decision times on September 14, 2019, for all the users. Overall our data includes $\N=35$ \goodusers (with 27 non-pooled and 8 pooled) and $\T=450$ decision times, where each user is available for a (different) subset of $\T$ decision times.

\subsubsection{Additional results}
We start by illustrating some additional user-specific results and the empirical coverage obtained by our estimates for treatment 0, in \cref{fig:heartsteps_additional_results}. The two users in panels (a) and (b)  have the worst empirical coverage of 60\% and 50\%, respectively, among all 8 pooled and 27 non-pooled \goodusers. Recall that the empirical coverage is defined as the percentage of decision times where the confidence interval covers the user's observed outcome. 
\begin{figure}[t!]
    \centering
    \resizebox{\textwidth}{!}{
    \begin{tabular}{ccc}
    \includegraphics[width=0.33\textwidth]{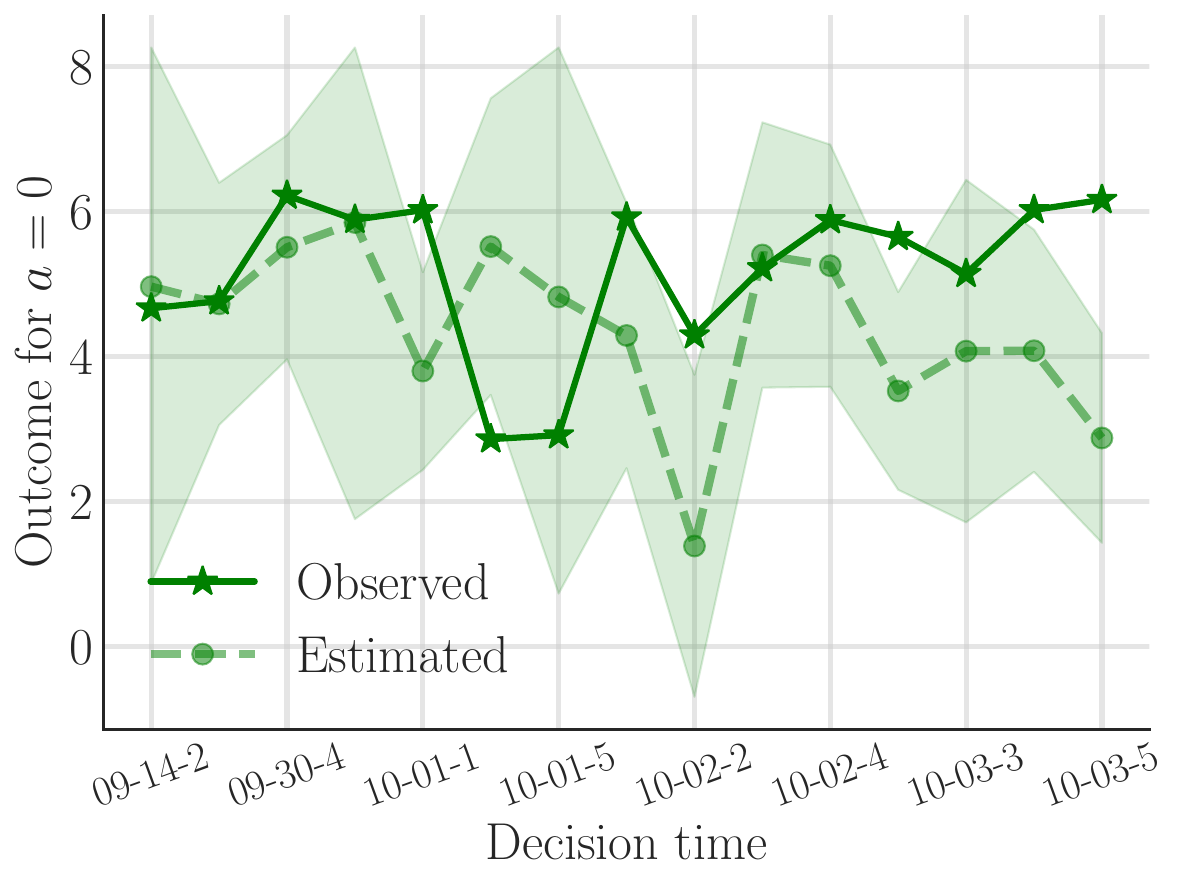}\hspace{2mm}&
    \includegraphics[width=0.33\textwidth]{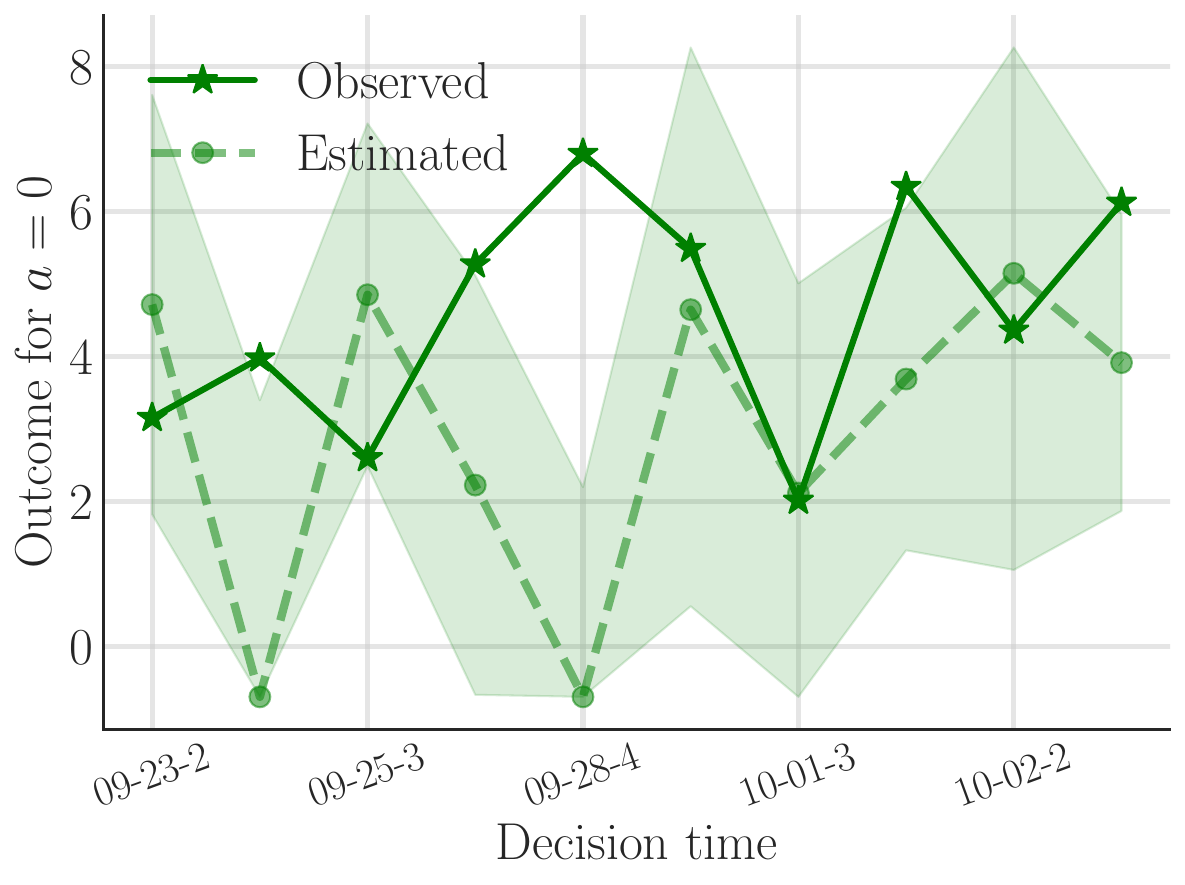}\hspace{2mm}
    & 
    \includegraphics[width=0.33\textwidth]{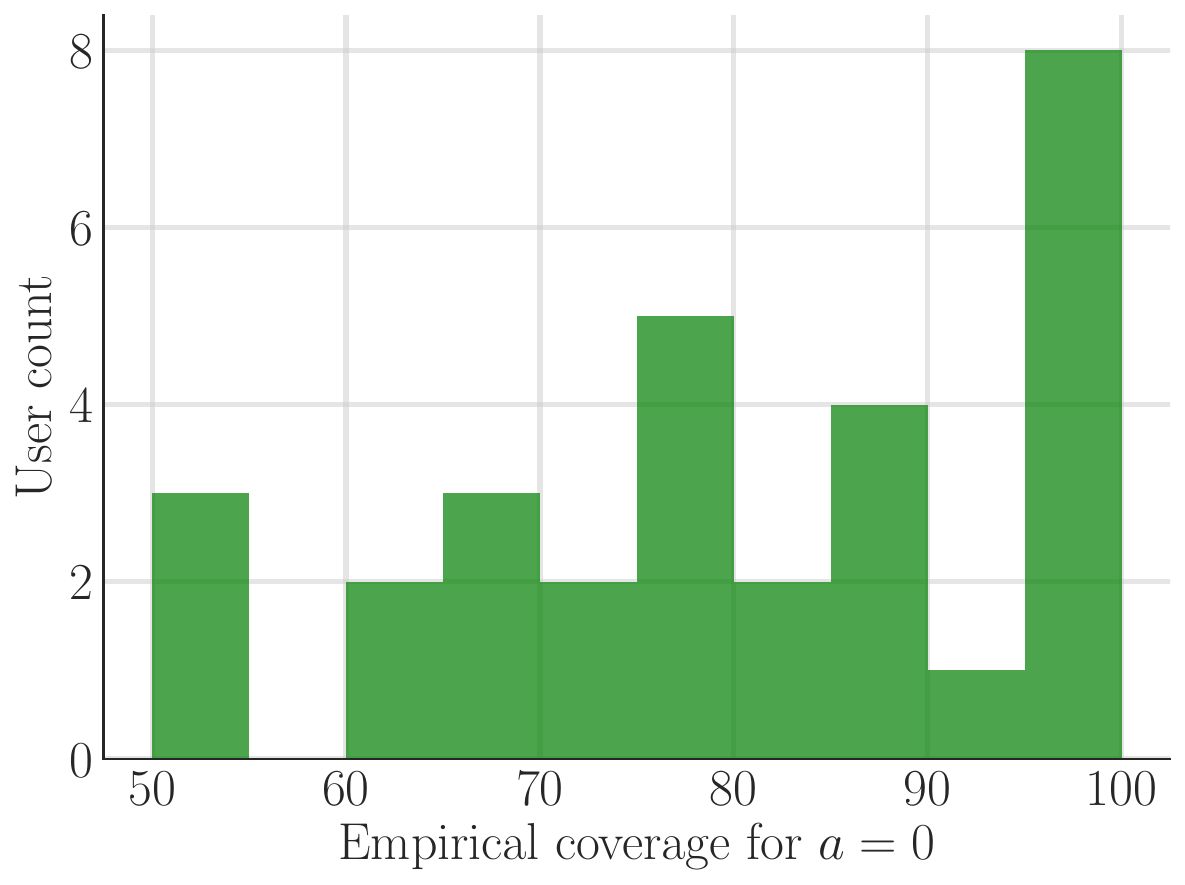}
     \\[-1mm]
    (a) Pooled user 2 & (b) Non-pooled user 2 & (c) All users
    \end{tabular}
    }
    \caption{\tbf{Additional results for $a=0$ for HeartSteps.} Panels (a) and (b) present the outcomes for a pooled and non-pooled user, respectively. The x-axis represents decision time, dark-colored stars represent the observed outcomes, and light-colored circles represent the estimated outcomes. The shaded regions indicate the confidence intervals. Panels (a) and (b) correspond to panels (a) and (b) of \cref{fig:main_results}, respectively, for two different users. Panel (c) shows a histogram of user-specific empirical coverage for all 35 \goodusers where the value on the y-axis represents the user count for which the empirical coverage equals the value on the x-axis.}
    \label{fig:heartsteps_additional_results}
\end{figure}

\paragraph{Relationship to user-heterogeneity} In \cref{fig:corr_plot}, we provide a scatter plot of the user-specific random effect estimates by the IP-TS algorithm at the end of the feasibility study, and the number of nearest neighbors obtained from our NN algorithm, for the eight pooled users. The negative correlation between the two quantities suggests that both IP-TS and nearest neighbors are potentially capturing certain aspects of user heterogeneity.
\begin{figure}[ht!]
    \centering
    \begin{tabular}{c}
    \includegraphics[width=0.45\textwidth]{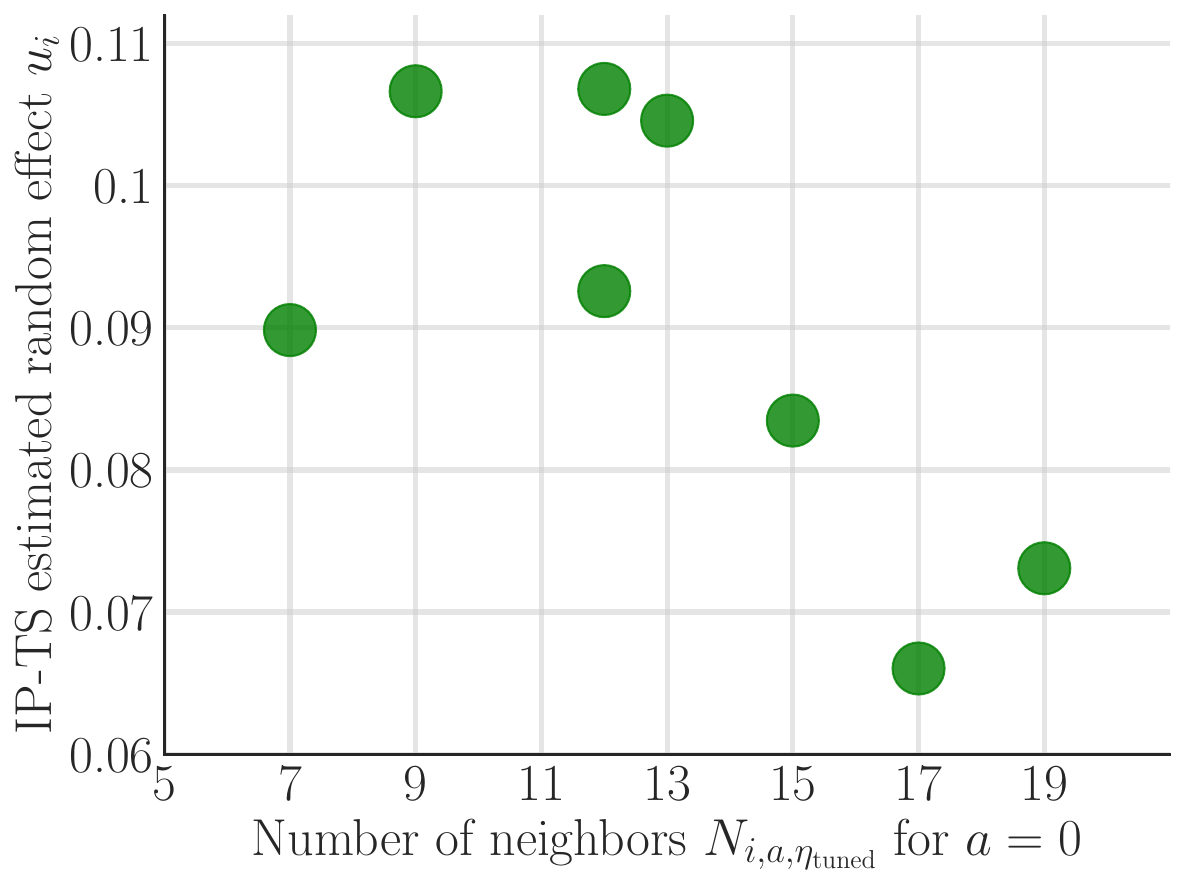}
    \end{tabular}
    \caption{\tbf{Relationship between the user-specific estimated random effects from IP-TS and the number of nearest neighbors for pooled users}. The y-axis plots the value of the user-specific random effect estimated by the IP-TS algorithm by the end of the feasibility study, while the x-axis denotes the number of neighbors $N_{\n, a, \eta_{\mrm{tuned}}} \defeq |\sbraces{j\neq i: \estdist \leq \eta_{\mrm{tuned}}}|$. We observe a correlation of $-0.69$ (with a p-value of 0.06 for the t-test) between the x-axis and y-axis values.}
    \label{fig:corr_plot}
\end{figure}

\paragraph{Hyperparameter tuning} Due to limited data, we do not use a validation set and tune the hyperparameter on the training data. We choose a candidate set for $\eta$ that includes 5 equally-spaced percentiles in the range $[0, 10]$ and 20 equally-spaced percentiles in the range $[10, 100]$ across all the ${35 \choose 2}$ pairwise distances for treatment 0 across the 35 \goodusers. For reference, we provide a histogram of these pairwise distances in \cref{fig:heartsteps_eta_tune}(a). The results for tuning $\eta$ are presented in \cref{fig:heartsteps_eta_tune}(b). We find that the best choice of $\eta$ with at least 70\% valid estimates from nearest neighbor strategy (\cref{rem:no_nbr_eta_tune}) is the $43$-rd percentile value, which is equal to $8.72$. We mark this choice, which is used to compute all the results in \cref{fig:main_results,fig:corr_plot}, as a vertical dashed-dotted line in panels (a) and (b) of \cref{fig:heartsteps_eta_tune}. Note that the tuned choice of $\eta$ yields 3 or more neighbors for the 35 \goodusers as illustrated in \cref{fig:heartsteps_eta_tune}(c).
\begin{figure}[t!]
    \centering
    \resizebox{\textwidth}{!}{
    \begin{tabular}{ccc}
    \includegraphics[width=0.33\textwidth]{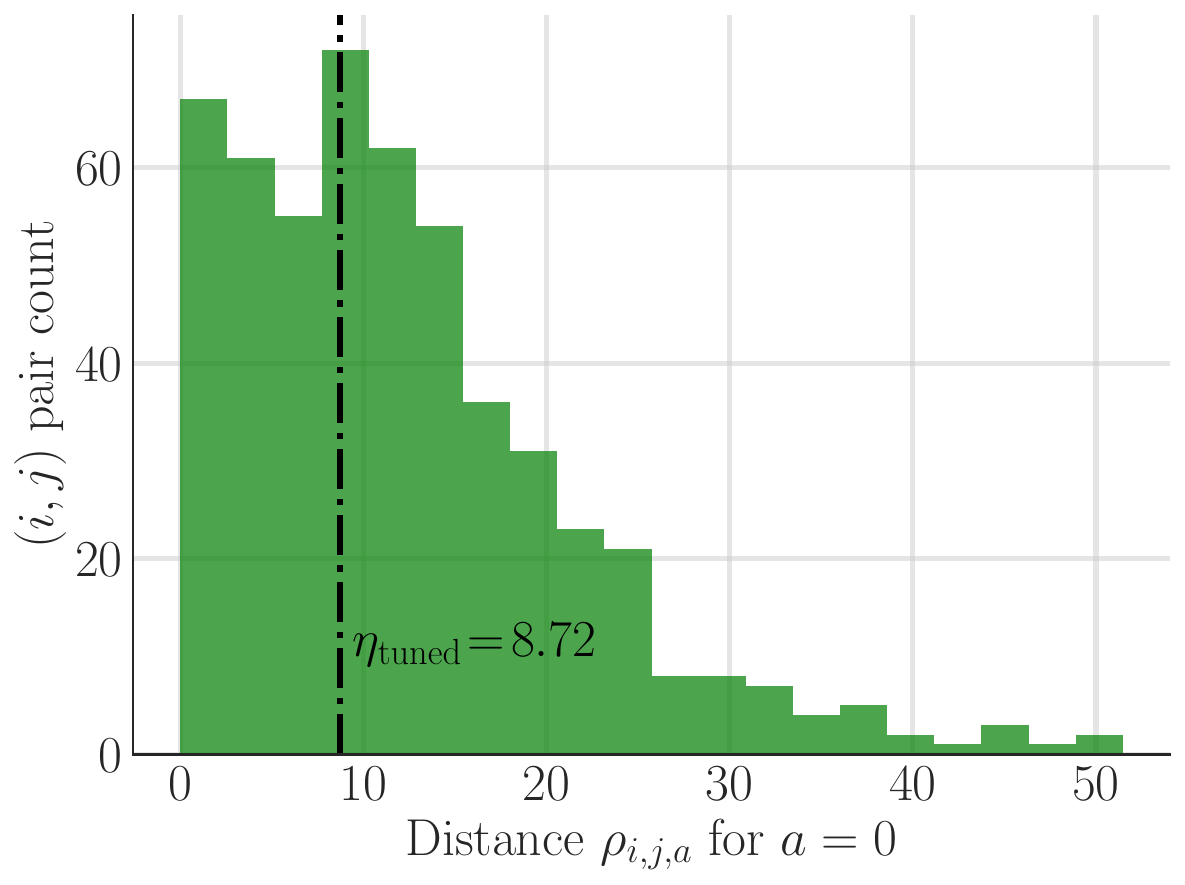}\hspace{2mm}&
    \includegraphics[width=0.33\textwidth]{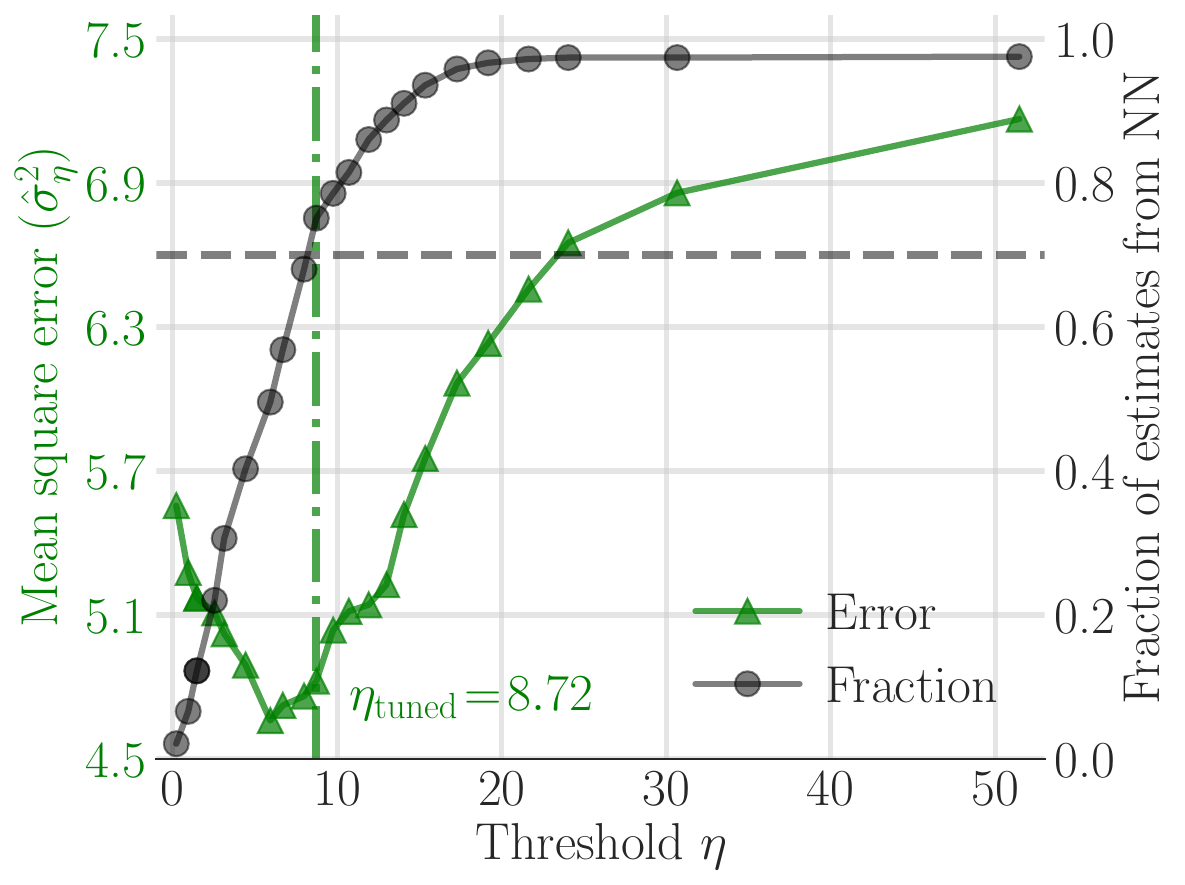}\hspace{2mm}
    & 
    \includegraphics[width=0.33\textwidth]{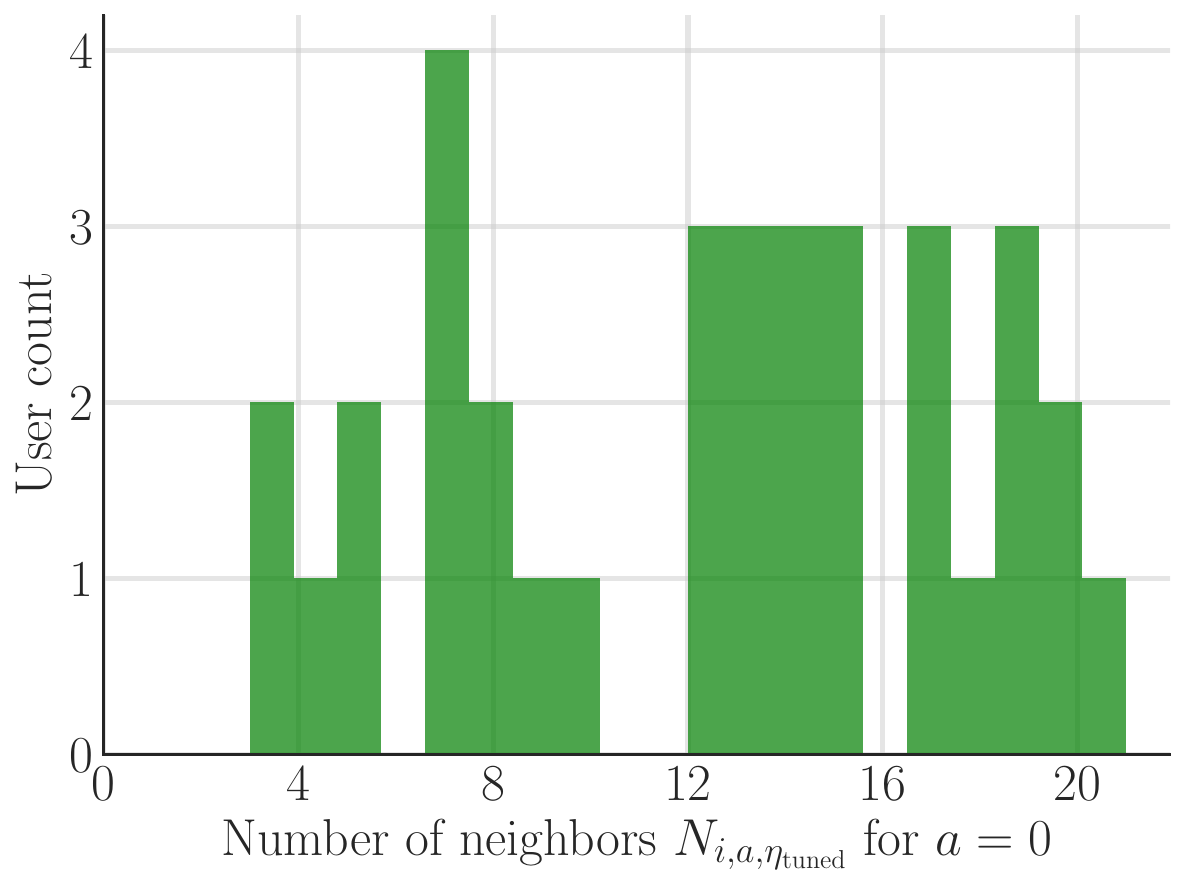}
     \\[-1mm]
    (a) & (b) & (c)
    \end{tabular}
    }
    \caption{\tbf{Hyperparameter $\eta$ tuning and the resulting number of neighbors for treatment $a=0$ in HeartSteps.}  Panel (a) shows a histogram of pair-wise distance across users on the training data with distance on the x-axis and the user-pair count on the y-axis. The left y-axis in panel~(b) shows the mean squared error, same as $\what\sigma^2_{a, \eta}$~\cref{eq:eta_tune} (on training data), for the threshold $\eta$ on the x-axis. The right y-axis in panel (b) displays the fraction of estimates with non-zero neighbors (cf. \cref{fig:eta_tune}). The tuned choice of $\eta$ is marked as a vertical dashed-dotted line in panels (a) and (b).  Panel (c) plots a histogram of the number of nearest neighbors $N_{\n, a, \eta_{\mrm{tuned}}} \defeq |\sbraces{j\neq i: \estdist \leq \eta_{\mrm{tuned}}}|$ across the 35 \goodusers.}
    \label{fig:heartsteps_eta_tune}
\end{figure}

\paragraph{Additional SVD results for different imputation strategy}
Finally, we discuss analogs of \cref{fig:eda_plot}(c), albeit with two alternate strategies for imputing the missing outcomes at non-available times before computing SVD of the outcome matrices (stratified by treatment). However, for missing outcomes at available times, across both strategies, we still use the mean value specific to each user and treatment as in \cref{fig:eda_plot}(c). We present the results for the two strategies, respectively, in the two panels of \cref{fig:svd_plot}. Note that we use the same strategy for the two treatments in a given panel.

For the first strategy, called ``available'' imputation, with results in panel (a), we assign an outcome value equal to the recorded outcome in the original data for non-available times. If no outcome value is available, we assign a value of $\log(0.5)=-0.69$ (equivalent to $0$ step count).  For the second strategy, called ``zero'' imputation, with results in panel (b), we assign a value of $-0.69$ for all non-available times.  

We observe that the first three components in panel (a) account for over 60\% of the variance, while around ten components in panel (b) explain the same amount of variance. 
\begin{figure}[t!]
    \centering
    \resizebox{\textwidth}{!}{
    \begin{tabular}{cc}
    \includegraphics[width=0.5\linewidth]{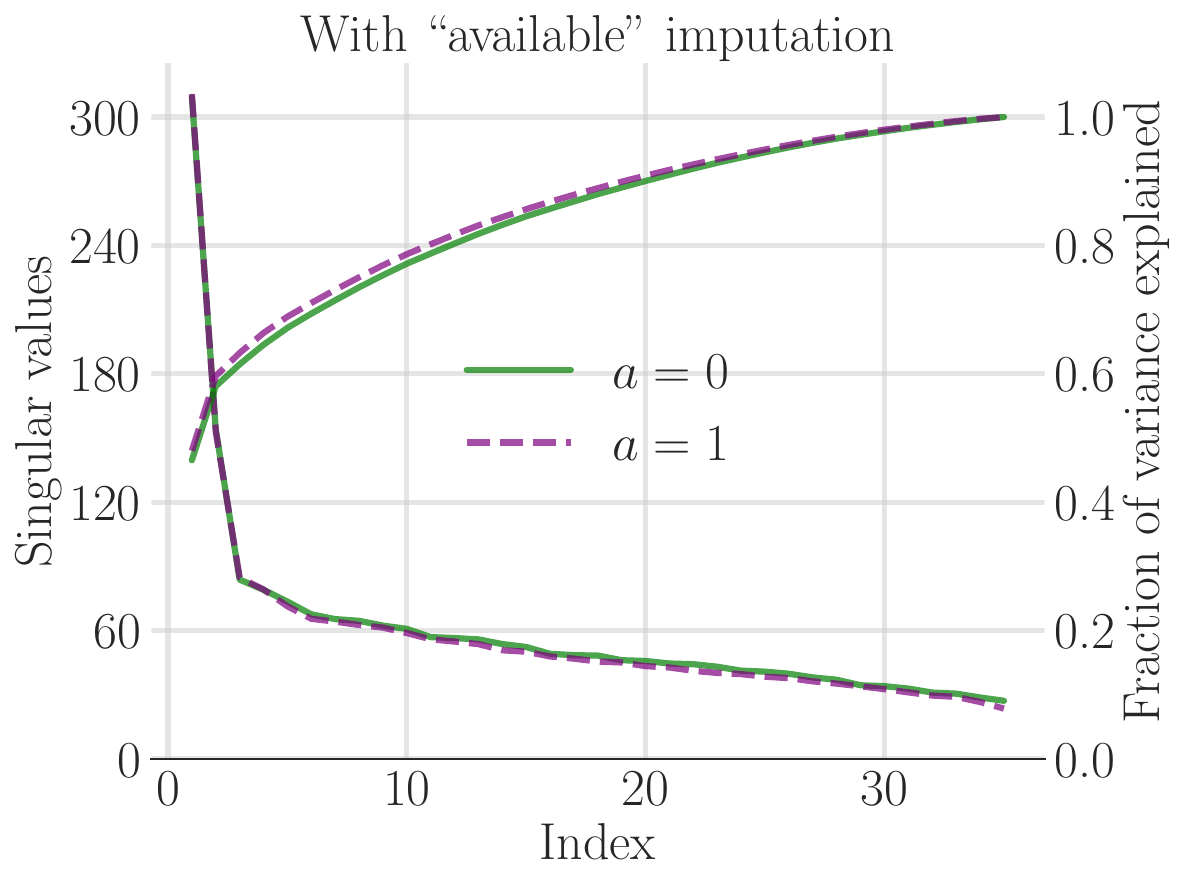}&
    \includegraphics[width=0.5\linewidth]{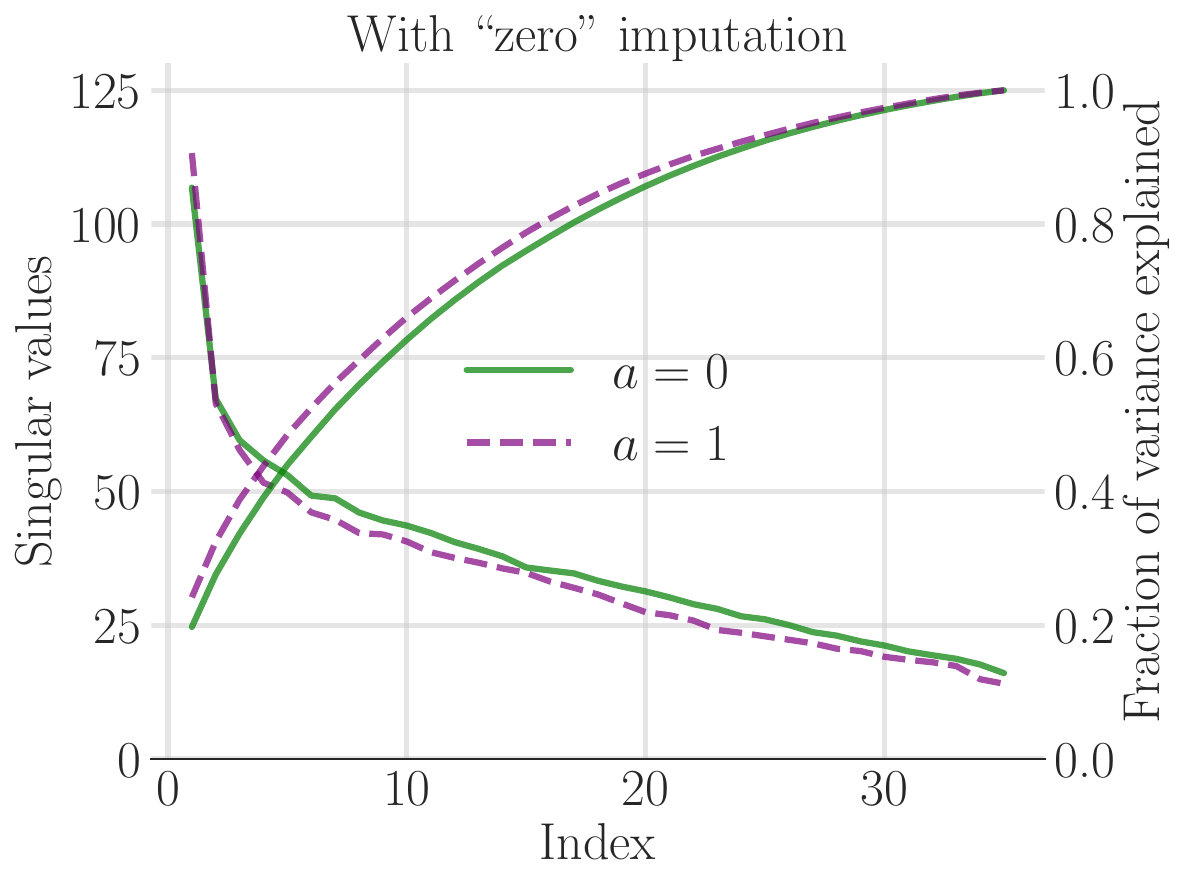} \\ 
    (a) &  (b)
    \end{tabular}
    }
    \caption{\tbf{Singular values for the observed outcomes in HeartSteps.} In each panel, we display the singular values (in decreasing order) for outcomes colored by two treatments, with the index on the x-axis, values on the left y-axis, and the fraction of variance explained on the right y-axis. Note that compared to \cref{fig:eda_plot}(c), the data used here differs only in the imputation strategy to fill missing data at non-available times.}
    \label{fig:svd_plot}
\end{figure}

\section{Other extensions}
\label{sub:other_extensions}

    We now discuss how to extend our theoretical results with possible relaxations of assumptions~\cref{assum:zero_mean_noise,assum:lambda_min,assum:iid_unit_lf} and another setting with spillover effects.

     \subsection{Sub-Gaussian latent factors, and noise variables}
    \label{sub:sub_gauss}
    We now describe how to extend \cref{thm:anytime_bound} to the setting with sub-Gaussian latent factors and noise variables. If $\lunit$, and $\ltime$ are \iid sub-Gaussian random vectors with parameter $\sigma_{u}$, and $\sigma_{v}$ respectively, and the noise is sub-Gaussian with parameter $\sigma_{\vareps}$, then standard sub-Gaussian concentration results yield that
    \begin{align}
        \max_{\n\in[\N]}\twonorm{\lunit} &\leq 2\sigma_{u}\sqrt{d}(2+\sqrt{\log(3\N/\delta)}), \\
        \max_{\t\in[\T]}\twonorm{\ltime} &\leq 2\sigma_{v}\sqrt{d}(2+\sqrt{\log(3\T/\delta)}), \\
        \qtext{and}
        \max_{\n\in[\N],\t\in[\T]}\sabss{\noiseobs} &\leq \sigma_{\vareps}\sqrt{2\log(6\N\T/\delta)}),
    \end{align}
    with at least $1-\delta$. Thus, we can relax the boundedness assumption on the latent factors, and noise variables in \cref{assum:iid_unit_lf,assum:lambda_min,assum:zero_mean_noise} to sub-Gaussian tails for them, and obtain a modified version of \cref{thm:anytime_bound}, wherein the constants $\uconst,\vconst$ and $\nconst$ in the bound~\cref{eq:nn_bnd} are replaced with the appropriate terms appearing in the display above, and this new bound holds with probability at least  $1\!-\!2\delta\!-\!2\!\exp\sparenth{-\frac{\pmint[\t]\probparam_{\lunit}( \threshold')(\N\squash{-}1)}{16}}$. Overall, we find that the error bound  \cref{eq:nn_bnd} is inflated only by logarithmic in $(\N, \T)$ factors, so that the subsequent results (e.g., \cref{thm:anytime_asymp,thm:ate_asymp}) would also hold with logarithmic factor adjustments.
    
    \subsection{Non-positive definite covariance $\Sigv$}
    \label{sub:non_psd}
    In \cref{assum:lambda_min}, we assert that the minimum eigenvalue of $\Sigv$---the non-centered covariance of latent time factors under treatment $a$---is non-zero. We now argue that this assumption is not restrictive. Consider the following eigendecomposition of $\Sigv$ in terms of the sorted eigenvalues $\lambda^{(1)} \geq \lambda^{(2)} \ldots \geq \lambda^{(d)}$, and the corresponding orthonormal eigenvectors $\sbraces{w_{\l}}$:
    \begin{align}
        \Sigv = \sum_{\l=1}^{d} \lambda^{(\l)} w_{\l}w_{\l}\tp.
    \end{align}
    In terms of this notation, the RHS of \cref{eq:nn_bnd} scales with $\lambda^{(d)}$. Suppose that $\lambda^{(d)}=0$ and $\lambda^{(d-1)}>0$, then we claim that \cref{eq:nn_bnd} holds with $\lambda_{a} \gets \lambda^{(d-1)}$. To prove the claim, first we note that
     \begin{align}
        0 =\lambda_{d} = w_{d}\tp \Sigv w_d = w_{d}\tp\E[\ltime^{(a)}(\ltime^{(a)})\tp]w_{d}
         = \E[((\ltime^{(a)})\tp w_d)^2],
    \end{align}
    which implies that $w_{d}\tp \ltime^{(a)}  = 0$ almost surely.
    Thus in this setting, we can repeat the proof of \cref{thm:anytime_bound} with the latent unit factors projected on the subspace spanned by the vectors $\sbraces{w_{\l}}_{\l=1}^{d-1}$. In other words, the effective dimensionality of the problem is $d-1$ for both the latent time factors and latent unit factors (after projecting on to a suitable subspace). Our claim follows.

    \subsection{Non-\iid latent time factors}
    \label{sub:non_iid}
    As noted earlier, the \iid sampling assumption on the latent time factors in \cref{assum:lambda_min} is a bit restrictive, especially for the applications motivating this work. E.g., the latent time factors capture unobserved societal changes over time, and these causes might be correlated, or drawn from a mixing process. We present one illustrative experiment in \cref{sec:simulations} (see \cref{eq:ar,fig:other_things}) that shows that the performance of nearest neighbors does not deteriorate when the \iid assumption is violated.

    For our analysis, the \iid assumption on the latent time factors can be relaxed if we (i) restrict the class of policy (in \cref{assum:policy}) to be unit-wise sequential (i.e., the policy is not pooling data across units) and (ii) use the treatment assignment probabilities to adjust our distance computation. Let $\history[\n, \t\!-\!1] \defeq \braces{\obs[\n, \t'], \miss[\n, \t'], \t'\in [\t\!-\!1]}$ denote the sigma-algebra of the data observed for unit $\n$ till time $\t\!-\!1$. Suppose the treatment policies are sequentially adaptive independently across unit, i.e., at time $t$, $\miss[\n, \t]$ is sampled using policy $\policy_{\t,\n}$, where $\policy_{\t,\n}(a) = \P(\miss = a \vert  \history[i, \t\!-\!1])$ , and that \cref{eq:pmin} holds for some  sequence of non-increasing scalars $\sbraces{\pmint[\t]}_{\t=1}^{\infty}$. Note that under this assumption, treatment assignment to two different units are independent of each other. Next, assuming known $\policy_{\t}$, we replace the distance~\cref{eq:time_nbr_dist} by

    \begin{align}
    \label{eq:time_nbr_dist_pi}
    \estdist^{\pi} \defeq \displaystyle  \frac1{\T}\sumt \frac{\indicator(\miss=a) \indicator(\miss[j, \t]=a)}{\policy_{\t,\n}(a)\policy_{\t,j}(a)} (\obs\!-\!\obs[j, \t])^2, %
    \end{align}
    with neighbors~\cref{eq:reliable_nbr} and the estimate~\cref{eq:obs_estimate} defined using this modified distance~\cref{eq:time_nbr_dist_pi}. Then a straightforward adaptation of our proof would yield a guarantee analogous to \cref{thm:anytime_bound}, now conditional on all time latent factors $\mc V$, the covariance matrix $\Sigv$ replaced by $\frac{1}{\T}\sumt \ltime^{(a)}(\ltime^{(a)})\tp$, and $\errterm$ replaced by $\errterm/\pmint[t]$. (The main change needed in the analysis is to establish an appropriate version of the bound \ox{\cref{eq:event_dist}}\oa{\cite[Eq.~\cref{eq:event_dist}]{supplement}} with $\errterm$ replaced by $\errterm/\pmint[\t]$, which in fact follows directly by using Azuma-Hoeffiding concentration. %

    More generally, the key challenge is to identify whether there exists a deterministic matrix $B$ and a function $g$ (that decreasing with respect to its second argument $\T$) such that
    \begin{align}
    \label{eq:non_iid_suff}
        \sup_{i \neq j}\opnorm{\frac{\sumt\indicator(\miss=a,\miss[j, \t]=a)\ltime^{(a)}(\ltime^{(a)})\tp }{\sumt\indicator(\miss=a,\miss[j, \t]=a)} - B} \leq g(N, \T, \delta)
    \end{align}
     with probability at least $1-\delta$. Under this condition, we can establish a modified version of \cref{thm:anytime_bound} that holds conditional on the latent time factors, with $\Sigv$ and $\errterm$ replaced by $B$ and $g(\N,\T, \delta)$ in all relevant terms and definitions. However, it is non-trivial to identify simple conditions on a sequential policy that pools across users, that would imply \cref{eq:non_iid_suff}, and we leave a further exploration in this direction for future work.

    \subsection{Spillover effects}
    \label{sub:excursion_effects}
    We continue our discussion from \cref{sub:delayed_effects} on how our algorithm can be generalized for inference with spillover effects.

    \paragraph{Bounded spillover effects}
    First we extend the nearest neighbor algorithm to estimate $\trueobsvar_{\n, \t}^{(\overline{a})}$ for any $\overline{a} \in \sbraces{0, 1}^K$ for the model with bounded spillover effects. In essence, effectively the same algorithm from \cref{sub:estimates} works once we redefine the treatments at time $t$ to be the sequence of assignments $\miss[\n, \t-K+1:\t]$ and extend the treatment space to $\sbraces{0, 1}^K$. Put simply, the nearest neighbor estimate for $\trueobsvar_{\n, \t}^{(\overline{a})}$ for the bounded spillover set-up is given by
    \begin{align}
    \label{eq:obs_estimate_K}
        \what{\trueobsvar}_{\n, \t, \eta}^{(\overline{a})} &\defeq 
         \displaystyle\frac{1}{|\estnbr[\n, \t, \overline{a}, \threshold]|} \sum_{j \in 
         \estnbr[\n, \t, \overline{a}, \threshold]} \obs[j, \t]
         \stext{where} 
         \estnbr[\n, \t, \overline{a}, \threshold]
         \defeq 
          \braces{j \neq \n \suchthat \estdist[\n, j, \overline{a}] \leq \threshold, 
        \miss[j, \t-K+1:\t] = \overline{a}} \\ 
        \qtext{and}
        \estdist[\n, j, \overline{a}] &\defeq \displaystyle  \frac{\sum_{s=1}^{\T}\indicator(\miss[\n, s-K+1:s]=\overline{a}) \indicator(\miss[j, s-K+1:s]=\overline{a}) (\obs[\n, s]\!-\!\obs[j, s])^2}{\sum_{s=1}^{\T}\indicator(\miss[\n, s-K+1:s]=\overline{a}) \indicator(\miss[j, s-K+1:s]=\overline{a})}.
    \end{align}
    Under suitable generalization of the assumptions, e.g., a factor model on the counterfactual parameters for $\overline{a}$ and $\policy_{\n,\t}\in[p, 1-p]$, an analogous error guarantee to \cref{thm:anytime_bound} for $(\what{\trueobsvar}_{\n, \t, \eta}^{(\overline{a})}-\trueobsvar_{\n, \t, \eta}^{(\overline{a})})^2$ can be derived for any $\overline{a} \in \sbraces{0, 1}^K$. However, as expected, this squared error would scale with $p^{-K}$ and thus such a strategy would work only when $\min(p^K \T, p^{K} \N) \to \infty $, i.e., $K$ grows very slowly with $\min(\T, \N)$.

    \paragraph{Spillover effects dependent on entire history of treatments}
    Next we consider a setting where the spillover effects at time $t$ are moderated by the lagged outcome at time $\t-1$. In particular, consider a potential outcome and observation model given by,
    \begin{align}
            \obs^{(a|y)} = \alpha^{(a)}y + \trueobsvar_{\n, \t}^{(a)} + \noise_{\n, \t}^{(a)}
            \qtext{and}
            \obs = \obs^{(\miss[\n, \t]\vert \obs[\n, \t-1])}.
    \end{align}    
    In this case, $\obs^{(a|y)}$ is the potential outcome for unit $i$ at time $t$ given treatment $a\in \{0,1\}$ and lagged outcome taking value $y$ at time $\t-1$. This model allows for lagged treatment effects, where the lagged effect is carried over via the auto-regressive term, $\alpha^{(a)} \obs[\n, \t-1]$, with  $\alpha^{(a)} \in (-1, 1)$ being the auto-regressive coefficient for treatment $a\in \{0,1\}$. Note that the dependence of $\obs$ on $ \obs[\n, \t-1]$ makes $\obs$ a function of the entire sequence of treatments till time time $\t$.

    In this case, an important quantity of interest is the following treatment effect for unit $\n$ and time $\t$:
    \begin{align}
        \trm{Exc-Eff}_{\n, t} \defeq \E[ \obs^{(1\vert \obs[\n, \t-1])} - \obs^{(0\vert \obs[\n, \t-1])} \vert \obs[\n, \t-1]] = (\alpha^{(1)} - \alpha^{(0)})\obs[\n, \t-1] + \trueobsvar_{\n, \t}^{(1)} - \trueobsvar_{\n, \t}^{(0)}.
        \label{eq:exc_effect}
    \end{align}
      Such effects are analogous to one-step excursion effects~\cite{qian2019estimating,boruvka2018assessing} for this outcome model. Suppose the outcomes are bounded in magnitude by $C_y$ and we are equipped with an estimate $\alpha^{(a)}$ with an error $\max_{a}|\what{\alpha}^{(a)}-{\alpha}^{(a)}| \leq \Delta(\alpha)$ from historical data (e.g., one can using algorithms from \cite[Sec.~5]{bai2024likelihood} to estimate the auto-regressive coefficients). Then we can estimate the excursion treatment effect via the formula
     \begin{align}
      \hat{\trm{Exc-Eff}}_{\n, t} \defeq (\what\alpha^{(1)} - \what\alpha^{(0)})\obs[\n, \t-1] + \what\trueobsvar_{\n, \t, \eta_1}^{(1)} - \what{\trueobsvar}_{\n, \t, \eta_0}^{(0)} 
     \end{align}
      where $\what\trueobsvar_{\n, \t, \eta_a}^{(a)}$ denote the nearest neighbor estimates obtained from using the residual matrix $\sbraces{\wtil{Y_{\n, t}} = \obs - \what{\alpha}^{(\miss)} \obs[\n, \t-1]}_{\n\in[\N], t\in[\T]}$ as the outcome matrix and the originally observed treatment matrix in constructing the estimates. We remark that our analysis can be generalized to account for deterministic/misspecification errors, so that the overall squared error in estimating one-step excursion effects would scale as the sum of errors obtained from \cref{thm:anytime_bound} for $a=0$ and $1$, plus $4\Delta^2(\alpha) C^2_y$.

\section{Proof of Thm. 4.1: \nonlinearanytimeboundname}
\label{sec:proof_non_linear}
The proof of \cref{thm:non_linear_anytime_bound} mimics most of the steps from the proof of \cref{thm:anytime_bound} from \cref{sec:proof_of_anytime_bound}, except with one major change: calculating the bias term requires new results. 

    We start with the decomposition~\cref{eq:det_final} from the proof of \cref{thm:anytime_bound}.
    \begin{align}
        (\trueobs-\estobs)^2 \leq 2 \parenth{ \frac{\sum_{j \in \estnbr} |\trueobs[\n,\t]- \trueobs[j, \t]|}{\nestnbr}  }^2
        + 2 \parenth{ \frac{\sum_{j \in \estnbr} \noiseobs[j, \t]}{\nestnbr}  }^2
        \defeq 2(\mbb B + \mbb V)
    \end{align}
    Let us replace the definition~\cref{eq:truedist} with
    \begin{align}
    \label{eq:truedist_nonlin}
        \truedist \defeq \E_{\ltime[]^{(a)}}(\lfun^{(a)}(\lunit, \ltime[]^{(a)})-\lfun^{(a)}(\lunit[j], \ltime[]^{(a)}) )^2  + 2\sigma_{a}^2.
    \end{align}
    and redefine the event from \cref{eq:event_dist} to be the event $\event[1]$ such that
    \begin{align}
        \abss{\estdist\!-\!\truedist} \leq \frac{(G_1+\nconst)^2\sqrt{32\log(4\N\T/\delta)}}{\pmint\sqrt{\T}} \seq{\eqref{eq:nonlinear_eta_chi}} \errterm', \stext{for all} j \in [\N]\backslash\braces{\n},
        \label{eq:event_dist_nonlinear}
    \end{align}
    where the last step follows from the definition of $\errterm'$ as stated in \cref{thm:non_linear_anytime_bound}.
    Notably, $\event[1]$ denotes the event that the estimated distance between unit $\n$ and all other units $j$ concentrates around the expected value $\truedist$. Moreover let $\event[2]$ be as defined in the proof of \cref{thm:anytime_bound} (see discussion around \cref{eq:event_noise}). Then with the definitions in place, one can easily derive the following analog of \cref{lemma:dist_noise_conc} by mimicking the proof and simply replacing all the definitions:
    \begin{lemma}
        \label{lemma:dist_noise_conc_non_linear}
        Under the setting of \cref{thm:non_linear_anytime_bound}, the events $\event[1]$~\cref{eq:event_dist_nonlinear} and  $\event[2]$~\cref{eq:event_noise} satisfy
        \begin{align}
        \P(\event[1]\vert \ulf) \sgrt{(a)}  1-\frac{\delta}{2},
        \qtext{and}
        \P(\event[2], \event[1] \vert \lunit) \sgrt{(b)} 1-\frac{\delta}{2}-2\exp\parenth{-\frac{\pmint[\t]\ppau( \threshold')(\N\squash{-}1)}{8}}.
        \label{eq:event_dist_noise_conc_nonlinear}
    \end{align}
    \end{lemma}
    With this result in place, the claimed bound in \cref{thm:non_linear_anytime_bound} follows as long as we can establish that on event $\event[1]$ from \cref{lemma:dist_noise_conc_non_linear}, we have $\mbb B \leq \wtil{\mbb B}^{\trm{non-lin}}$. To this end, we use the following lemma that relates the $\sinfnorm{\cdot}$-norm of a function with its second moment.
     \begin{lemma}
        \label{lem:inf_lip_bound}
        Consider a function $g:[0, 1]^d\to \real$ that is $G$-Lipschitz with respect to $\infnorm{\cdot}$-norm. Let $X \sim \mc P$, where $\mc P$ denotes a distribution on $[0, 1]^d$ with a density that is lower bounded by $c_{\mc P}$ with $\mu_2 \defeq \mbb E[g^2(X)]$. Then, we have
        \begin{align}
             \infnorm{g}  = \order_d(G^{\frac{d}{d+2}} \cdot (\mu_2/c_{\mc P})^{\frac{1}{d+2}}),
             \label{eq:decay_l2_inf}
            \end{align}
            where the big-O notation is indexing sequence of functions such that $\mu_2 \to 0$, and the constant underlying $\order_d$ can be tracked in the proof.
    \end{lemma}
    
    To apply \cref{lem:inf_lip_bound}, note that \cref{assum:nonlinear} implies that the function $g:\ltime[] \mapsto \lfun^{(a)}(\lunit[], \ltime[]) -  \lfun^{(a)}(\lunit[]', \ltime[])$ is $2G_2$-Lipschitz for any $\lunit[],\lunit[]'$ and the distribution on time factors admits a density with a lower bound $c_{\mc P, v}$. Thus we can apply \cref{lem:inf_lip_bound}  for the functions $g_j: v\mapsto \lfun^{(a)}(\lunit, \ltime[]) -  \lfun^{(a)}(\lunit[j], \ltime[])$ for any $j \in [N]$, with $c_{\mc P} = c_{\mc P, v}$ and $G=2G_2$. Putting the pieces together, on event $\event[1]$ we obtain that
    \begin{align}
        \max_{j \in \estnbr}|\lfun^{(a)}(\lunit, \ltime)-\lfun^{(a)}(\lunit[j], \ltime)|
        &\leq \max_{j \in \estnbr} \sinfnorm{g_j} \\
        &\sless{(i)} \max_{j \in \estnbr} c_d(2G_2)^{\frac{d}{d+2}} \parenth{\frac{\truedist-2\sigma_{a}^2}{c_{\mc P, v}}}^{\frac{1}{d+2}}  \\ 
        &\sless{(ii)} c_d'G_2^{\frac{d}{d+2}} \parenth{\frac{\estdist-2\sigma_{a}^2 + \errterm'}{c_{\mc P, v}}}^{\frac{1}{d+2}} \\ 
        &\sless{(ii)} c_d'G_2^{\frac{d}{d+2}} \parenth{\frac{\threshold-2\sigma_{a}^2 + \errterm'}{c_{\mc P, v}}}^{\frac{1}{d+2}}
    \end{align}
    where step~(i) follows from \cref{lem:inf_lip_bound} and the fact that $\E_{\ltime[]} g^2(\ltime[]) = \truedist - 2\sigma_{a}^2$ (see \cref{eq:truedist_nonlin}), step~(ii) follows from the definition~\cref{eq:event_dist_nonlinear} of the event $\event[1]$, and step~(iii) follows from the definition~\cref{eq:reliable_nbr} of $\estnbr$. As a result we immediately have
    \begin{align}
        \mbb B \leq \max_{j \in \estnbr}|\lfun^{(a)}(\lunit, \ltime)-\lfun^{(a)}(\lunit[j], \ltime)|^2 
        \leq c_d''G_2^{\frac{2d}{d+2}} \parenth{\frac{\threshold-2\sigma_{a}^2 + \errterm'}{c_{\mc P, v}}}^{\frac{2}{d+2}}
        \seq{\eqref{eq:nonlinear_bias}} \wtil{\mbb B}^{\trm{non-lin}}
    \end{align}
    on the event $\event[1]$ and our proof is now complete.
\subsection{Proof of \cref{lem:inf_lip_bound}}
        Let $\infnorm{g} = B$ and $x^\star \in [0, 1]^d$ be such that $|g(x^\star)| = B$. Then by Lipschitzness of $g$, for any $x \in [0, 1]^d$ we have
        \begin{align}
            g^2(x) \geq ((B-G\infnorm{x-x^\star})_{+})^2,
        \end{align}
        and consequently also that
        \begin{align}
            \mu_2 \geq c_{\mc P} \int_{[0, 1]^d}  ((B-G\infnorm{x-x^\star})_{+})^2 dx.
        \end{align}
        Next, we claim that for any $x^\star \in [0, 1]^d$,
        \begin{align}
        \label{eq:int_dominance}
            \int_{[0, 1]^d}  ((B-G\infnorm{x-x^\star})_{+})^2 dx
            \geq \int_{[0, 1]^d}  ((B-G\infnorm{x})_{+})^2 dx.
        \end{align}
        We prove this claim at the end of the section and now proceed to provide the remainder of the proof. The claim~\cref{eq:int_dominance} and symmetry of $\infnorm{x}$ implies that
        \begin{align}
            \int_{[0, 1]^d}  ((B-G\infnorm{x})_{+})^2 dx
            &= d \int_{x_2, \ldots, x_d \leq x_1 \cap [0, 1]^d} ((B-Gx_1)_{+})^2  dx\\
            &= d \int_0^1 ((B-Gx_1)_{+})^2 x_1^{d-1} dx_1 \\ 
            &= d \int_0^{\min(1, B/G)} (B-Gx_1)^2 x_1^{d-1} dx_1 \\ 
            &= d \brackets{\frac{B^2x_1^d}{d} + \frac{G^2x_1^{d+2}}{d+2} - \frac{2BGx_1^{d+1}}{d+1}}_{0}^{\min(1, B/G)} \\ 
            &= d\min\bigg\{\frac{B^2}{d} + \frac{G^2}{d+2} - \frac{2BG}{d+1}, 
            \\
            &\qquad\qquad\qquad \frac{B^{d+2}}{G^{d}} \parenth{\frac{1}{d} + \frac{1}{d+2} - \frac{2}{d+1}}
             \bigg\}.
        \end{align}
        Putting the pieces together, we find that
        \begin{align}
            \mu_2 \geq c_{\mc P} d\min \sbraces{c_{d} (B^2+G^2), c_{d}' B^{d+2}/G^d }
        \end{align}
        Noting that $G$ remains fixed, we conclude that as $\mu_2 \to 0$, we must have 
        \begin{align}
            B = \order_d(G^{\frac{d}{d+2}} \cdot (\mu_2/c_{\mc P})^{\frac{1}{d+2}}),
        \end{align}
        as claimed.

        \paragraph{Proof of \cref{eq:int_dominance}} Introduce the shorthand $h(z) = (B-G\sinfnorm{z})_+^2$, Define $x_i' \defeq \min(x_i^\star, 1-x_i^\star) \in [0, 1/2]$ and $x_i'' \defeq \max(x_i^\star, 1-x_i^\star) = 1-x_i' \in [1/2, 1]$. Applying a change of variables with $y\gets x-x^\star$, we find that
        \begin{align}
             \int_{[0, 1]^d} h(x-x^\star) dx &= 
             \int_0^1 \int_0^1\cdots \int_0^1   h(x-x^\star) dx_1dx_2\cdots dx_d \\ 
             &= \int_{-x_1^\star}^{1-x_1^\star} \int_{-x_2^\star}^{1-x_2^\star}\cdots \int_{-x_d^\star}^{1-x_d^\star} h(y) dy_1dy_2\cdots dy_d \\ 
             &= \sum_{z_1 \in \sbraces{x_1', x_1''}}\sum_{z_2 \in \sbraces{x_2', x_2''}}\cdots \sum_{z_d \in \sbraces{x_d', x_d''}} \int_{0}^{z_1} \int_{0}^{z_2} \cdots \int_{0}^{z_d} h(y) dy_1dy_2\cdots dy_d.
        \end{align}
        Noting that $x_1'' = 1-x_1' \geq x_1'$, for any fixed $z_2 \in \sbraces{x_2' ,x_2''}, \ldots, z_d \in \sbraces{x_d', x_d''} $  in the above display, we have
        \begin{align}
            &\sum_{z_1 \in \sbraces{x_1', x_1''}} \int_{0}^{z_1} \int_{0}^{z_2} \cdots \int_{0}^{z_d} h(y) dy_1dy_2\cdots dy_d\\
            &= \int_{0}^{x_1'} \int_{0}^{z_2} \cdots \int_{0}^{z_d} h(y) dy_1dy_2\cdots dy_d \\
            &\qquad + \int_{0}^{x_1''} \int_{0}^{z_2} \cdots \int_{0}^{z_d}h(y) dy_1dy_2\cdots dy_d\\ 
            &= \int_{0}^{x_1'} \int_{0}^{z_2} \cdots \int_{0}^{z_d} h(y) dy_1dy_2\cdots dy_d
            + \int_{x_1'}^{1-x_1'}  \int_{0}^{z_2} \cdots \int_{0}^{z_d} h(y) dy_1dy_2\cdots dy_d
             \\
            &\qquad+ \int_{0}^{x_1'} \int_{0}^{z_2} \cdots \int_{0}^{z_d} h(y) dy_1dy_2\cdots dy_d \\
            &\geq \int_0^1 \int_{0}^{z_2} \cdots \int_{0}^{z_d} h(y) dy_1dy_2\cdots dy_d
            \label{eq:last_arg}
        \end{align}
        where the last inequality follows from the following argument: For any $y=(y_1, y_2, \ldots, y_d)\tp$ and $y'=(1-y_1, y_2, \ldots, y_d)\tp$ such that $y_1 \in [0, x_1']$, since $1-y_1 \geq \half \geq y_1 \geq 0$, we have
        \begin{align}
            \sinfnorm{y'} \geq \sinfnorm{y}
            \implies h(y) \geq h(y')
        \end{align}
        and consequently that
        \begin{align}
            \int_{0}^{x_1'} \int_{0}^{z_2} \cdots \int_{0}^{z_d} h(y) dy_1dy_2\cdots dy_d
            &\geq \int_{0}^{x_1'} \int_{0}^{z_2} \cdots \int_{0}^{z_d} h(y') dy_1dy_2\cdots dy_d \\
            &=\int_{1-x_1'}^{1} \int_{0}^{z_2} \cdots \int_{0}^{z_d} h(y) dy_1dy_2\cdots dy_d,
        \end{align}
        where the last step follows from a change of variables for the first coordinate.
        Repeating the argument from \cref{eq:last_arg} for the summation over other coordinates one-by-one yields that
        \begin{align}
            &\sum_{z_1 \in \sbraces{x_1', x_1''}}\sum_{z_2 \in \sbraces{x_2', x_2''}}\cdots \sum_{z_d \in \sbraces{x_d', x_d''}}  \int_{0}^{z_1} \int_{0}^{z_2}\cdots \int_{0}^{z_d} h(y) dy_1dy_2\cdots dy_d \\
            &\geq \int_0^1 \int_{0}^{z_2} \cdots \int_{0}^{z_d} h(y) dy_1dy_2\cdots dy_d \\
            &\ \ \vdots \\
            &\geq \int_0^1 \int_{0}^{1} \cdots \int_{0}^{1} h(y) dy_1dy_2\cdots dy_d,
        \end{align}
        and the desired claim~\cref{eq:int_dominance} follows.

\section{Asymptotic results for non-linear factor model}
\label{sec:asymp_nonlinear}
    In a manner similar to the generalization of \cref{thm:anytime_bound} to \cref{thm:non_linear_anytime_bound} for non-asymptotic guarantees under a non-linear factor model, even our asymptotic guarantees from \cref{thm:anytime_asymp} for the bilinear model can be extended to the non-linear factor model setting. As noted in \cref{sub:non_linear_lf}, the primary difference between the bilinear and non-linear model results lies in the scaling of the bias term, where in the bias goes to $0$ much more slowly for the non-linear case. 

    Nevertheless, one can see that given \cref{thm:non_linear_anytime_bound}, we can generalize \cref{thm:anytime_asymp} to the non-linear setting by making the following changes: (i) Replacing $\uconst\vconst$ by $G_1$ while defining $\errtwo$ in \cref{eq:errtwo}, (ii) enforcing identically the same regularity conditions as in \cref{eq:regularity_consistency} to derive asymptotic consistency, and (ii) enforcing the regularity conditions same as in \cref{eq:regularity_clt} to derive asymptotic normality, albeit with one change---replacing the condition $L_T \threstwot \to 0$ with $L_T(\threstwot)^{\frac{2}{d+2}}\to 0$, to account for the different scaling of bias for the two settings. Moreover, analogous generalizations for estimating the average treatment effects (like in \cref{thm:ate_asymp}) can be also derived. Given \cref{thm:non_linear_anytime_bound}, the results and proofs for each case are easily derivable and thus we omit stating and proving formal guarantees.
\end{appendix}

\bibliographystyle{imsart-number} %
\bibliography{refs}       %

\begin{thebibliography}{60}
% BibTex style file: imsart-number.bst, 2017-11-03
% Default style options (sort=1,type=number).
% Used options (sort=1,type=number).

\bibitem{abadie2}
\begin{barticle}[author]
\bauthor{\bsnm{Abadie},~\bfnm{A.}\binits{A.}},
  \bauthor{\bsnm{Diamond},~\bfnm{A.}\binits{A.}} \AND
  \bauthor{\bsnm{Hainmueller},~\bfnm{J.}\binits{J.}}
(\byear{2010}).
\btitle{Synthetic Control Methods for Comparative Case Studies: Estimating the
  Effect of California's Tobacco Control Program.}
\bjournal{Journal of the American Statistical Association}.
\end{barticle}
\endbibitem

\bibitem{agarwal2021causal}
\begin{barticle}[author]
\bauthor{\bsnm{Agarwal},~\bfnm{Anish}\binits{A.}},
  \bauthor{\bsnm{Dahleh},~\bfnm{Munther}\binits{M.}},
  \bauthor{\bsnm{Shah},~\bfnm{Devavrat}\binits{D.}} \AND
  \bauthor{\bsnm{Shen},~\bfnm{Dennis}\binits{D.}}
(\byear{2021}).
\btitle{Causal Matrix Completion}.
\bjournal{arXiv preprint arXiv:2109.15154}.
\end{barticle}
\endbibitem

\bibitem{agarwal2021synthetic}
\begin{barticle}[author]
\bauthor{\bsnm{Agarwal},~\bfnm{Anish}\binits{A.}},
  \bauthor{\bsnm{Shah},~\bfnm{Devavrat}\binits{D.}} \AND
  \bauthor{\bsnm{Shen},~\bfnm{Dennis}\binits{D.}}
(\byear{2021}).
\btitle{Synthetic interventions}.
\bjournal{arXiv preprint arXiv:2006.07691}.
\end{barticle}
\endbibitem

\bibitem{stanford:pre-trialnudge}
\begin{bmisc}[author]
\bauthor{\bsnm{Allen},~\bfnm{Sophie}\binits{S.}}
(\byear{2022}).
\btitle{{Stanford Computational Policy Lab} Pretrial Nudges}.
\bhowpublished{\url{https://policylab.stanford.edu/projects/nudge.html}}.
\end{bmisc}
\endbibitem

\bibitem{arkhangelsky2019synthetic}
\begin{barticle}[author]
\bauthor{\bsnm{Arkhangelsky},~\bfnm{Dmitry}\binits{D.}},
  \bauthor{\bsnm{Athey},~\bfnm{Susan}\binits{S.}},
  \bauthor{\bsnm{Hirshberg},~\bfnm{David~A.}\binits{D.~A.}},
  \bauthor{\bsnm{Imbens},~\bfnm{Guido~W.}\binits{G.~W.}} \AND
  \bauthor{\bsnm{Wager},~\bfnm{Stefan}\binits{S.}}
(\byear{2021}).
\btitle{Synthetic Difference-in-Differences}.
\bjournal{American Economic Review}
\bvolume{111}
\bpages{4088-4118}.
\bdoi{10.1257/aer.20190159}
\end{barticle}
\endbibitem

\bibitem{athey2021matrix}
\begin{barticle}[author]
\bauthor{\bsnm{Athey},~\bfnm{Susan}\binits{S.}},
  \bauthor{\bsnm{Bayati},~\bfnm{Mohsen}\binits{M.}},
  \bauthor{\bsnm{Doudchenko},~\bfnm{Nikolay}\binits{N.}},
  \bauthor{\bsnm{Imbens},~\bfnm{Guido}\binits{G.}} \AND
  \bauthor{\bsnm{Khosravi},~\bfnm{Khashayar}\binits{K.}}
(\byear{2021}).
\btitle{Matrix completion methods for causal panel data models}.
\bjournal{Journal of the American Statistical Association}
\bvolume{116}
\bpages{1716--1730}.
\end{barticle}
\endbibitem

\bibitem{avadhanula2021stochastic}
\begin{binproceedings}[author]
\bauthor{\bsnm{Avadhanula},~\bfnm{Vashist}\binits{V.}},
  \bauthor{\bsnm{Colini~Baldeschi},~\bfnm{Riccardo}\binits{R.}},
  \bauthor{\bsnm{Leonardi},~\bfnm{Stefano}\binits{S.}},
  \bauthor{\bsnm{Sankararaman},~\bfnm{Karthik~Abinav}\binits{K.~A.}} \AND
  \bauthor{\bsnm{Schrijvers},~\bfnm{Okke}\binits{O.}}
(\byear{2021}).
\btitle{Stochastic bandits for multi-platform budget optimization in online
  advertising}.
In \bbooktitle{Proceedings of the Web Conference 2021}
\bpages{2805--2817}.
\end{binproceedings}
\endbibitem

\bibitem{bai2003inferential}
\begin{barticle}[author]
\bauthor{\bsnm{Bai},~\bfnm{Jushan}\binits{J.}}
(\byear{2003}).
\btitle{Inferential theory for factor models of large dimensions}.
\bjournal{Econometrica}
\bvolume{71}
\bpages{135--171}.
\end{barticle}
\endbibitem

\bibitem{bai2024likelihood}
\begin{barticle}[author]
\bauthor{\bsnm{Bai},~\bfnm{Jushan}\binits{J.}}
(\byear{2024}).
\btitle{Likelihood approach to dynamic panel models with interactive effects}.
\bjournal{Journal of Econometrics}
\bvolume{240}
\bpages{105636}.
\bdoi{https://doi.org/10.1016/j.jeconom.2023.105636}
\end{barticle}
\endbibitem

\bibitem{bai2021matrix}
\begin{barticle}[author]
\bauthor{\bsnm{Bai},~\bfnm{Jushan}\binits{J.}} \AND
  \bauthor{\bsnm{Ng},~\bfnm{Serena}\binits{S.}}
(\byear{2021}).
\btitle{Matrix completion, counterfactuals, and factor analysis of missing
  data}.
\bjournal{Journal of the American Statistical Association}
\bvolume{116}
\bpages{1746--1763}.
\end{barticle}
\endbibitem

\bibitem{bakshy2013uncertainty}
\begin{binproceedings}[author]
\bauthor{\bsnm{Bakshy},~\bfnm{Eytan}\binits{E.}} \AND
  \bauthor{\bsnm{Eckles},~\bfnm{Dean}\binits{D.}}
(\byear{2013}).
\btitle{Uncertainty in online experiments with dependent data: An evaluation of
  bootstrap methods}.
In \bbooktitle{Proceedings of the 19th ACM SIGKDD international conference on
  Knowledge discovery and data mining}
\bpages{1303--1311}.
\end{binproceedings}
\endbibitem

\bibitem{beretta2016nearest}
\begin{barticle}[author]
\bauthor{\bsnm{Beretta},~\bfnm{Lorenzo}\binits{L.}} \AND
  \bauthor{\bsnm{Santaniello},~\bfnm{Alessandro}\binits{A.}}
(\byear{2016}).
\btitle{Nearest neighbor imputation algorithms: a critical evaluation}.
\bjournal{BMC medical informatics and decision making}
\bvolume{16}
\bpages{197--208}.
\end{barticle}
\endbibitem

\bibitem{bertsimas2017personalized}
\begin{barticle}[author]
\bauthor{\bsnm{Bertsimas},~\bfnm{Dimitris}\binits{D.}},
  \bauthor{\bsnm{Kallus},~\bfnm{Nathan}\binits{N.}},
  \bauthor{\bsnm{Weinstein},~\bfnm{Alexander~M}\binits{A.~M.}} \AND
  \bauthor{\bsnm{Zhuo},~\bfnm{Ying~Daisy}\binits{Y.~D.}}
(\byear{2017}).
\btitle{Personalized diabetes management using electronic medical records}.
\bjournal{Diabetes care}
\bvolume{40}
\bpages{210--217}.
\end{barticle}
\endbibitem

\bibitem{bertsimas2017predictive}
\begin{barticle}[author]
\bauthor{\bsnm{Bertsimas},~\bfnm{Dimitris}\binits{D.}},
  \bauthor{\bsnm{Pawlowski},~\bfnm{Colin}\binits{C.}} \AND
  \bauthor{\bsnm{Zhuo},~\bfnm{Ying~Daisy}\binits{Y.~D.}}
(\byear{2017}).
\btitle{From predictive methods to missing data imputation: an optimization
  approach.}
\bjournal{J. Mach. Learn. Res.}
\bvolume{18}
\bpages{7133--7171}.
\end{barticle}
\endbibitem

\bibitem{bian2023off}
\begin{barticle}[author]
\bauthor{\bsnm{Bian},~\bfnm{Zeyu}\binits{Z.}},
  \bauthor{\bsnm{Shi},~\bfnm{Chengchun}\binits{C.}},
  \bauthor{\bsnm{Qi},~\bfnm{Zhengling}\binits{Z.}} \AND
  \bauthor{\bsnm{Wang},~\bfnm{Lan}\binits{L.}}
(\byear{2023}).
\btitle{Off-policy evaluation in doubly inhomogeneous environments}.
\bjournal{arXiv preprint arXiv:2306.08719}.
\end{barticle}
\endbibitem

\bibitem{bibaut2021post}
\begin{barticle}[author]
\bauthor{\bsnm{Bibaut},~\bfnm{Aur{\'e}lien}\binits{A.}},
  \bauthor{\bsnm{Chambaz},~\bfnm{Antoine}\binits{A.}},
  \bauthor{\bsnm{Dimakopoulou},~\bfnm{Maria}\binits{M.}},
  \bauthor{\bsnm{Kallus},~\bfnm{Nathan}\binits{N.}} \AND
  \bauthor{\bparticle{van~der} \bsnm{Laan},~\bfnm{Mark}\binits{M.}}
(\byear{2021}).
\btitle{Post-Contextual-Bandit Inference}.
\bjournal{NeurIPS 2021}.
\end{barticle}
\endbibitem

\bibitem{bojinov2021panel}
\begin{barticle}[author]
\bauthor{\bsnm{Bojinov},~\bfnm{Iavor}\binits{I.}},
  \bauthor{\bsnm{Rambachan},~\bfnm{Ashesh}\binits{A.}} \AND
  \bauthor{\bsnm{Shephard},~\bfnm{Neil}\binits{N.}}
(\byear{2021}).
\btitle{Panel experiments and dynamic causal effects: A finite population
  perspective}.
\bjournal{Quantitative Economics}
\bvolume{12}
\bpages{1171--1196}.
\end{barticle}
\endbibitem

\bibitem{boruvka2018assessing}
\begin{barticle}[author]
\bauthor{\bsnm{Boruvka},~\bfnm{Audrey}\binits{A.}},
  \bauthor{\bsnm{Almirall},~\bfnm{Daniel}\binits{D.}},
  \bauthor{\bsnm{Witkiewitz},~\bfnm{Katie}\binits{K.}} \AND
  \bauthor{\bsnm{Murphy},~\bfnm{Susan~A}\binits{S.~A.}}
(\byear{2018}).
\btitle{Assessing time-varying causal effect moderation in mobile health}.
\bjournal{Journal of the American Statistical Association}
\bvolume{113}
\bpages{1112--1121}.
\end{barticle}
\endbibitem

\bibitem{cai2021bandit}
\begin{barticle}[author]
\bauthor{\bsnm{Cai},~\bfnm{William}\binits{W.}},
  \bauthor{\bsnm{Grossman},~\bfnm{Josh}\binits{J.}},
  \bauthor{\bsnm{Lin},~\bfnm{Zhiyuan~Jerry}\binits{Z.~J.}},
  \bauthor{\bsnm{Sheng},~\bfnm{Hao}\binits{H.}},
  \bauthor{\bsnm{Wei},~\bfnm{Johnny Tian-Zheng}\binits{J.~T.-Z.}},
  \bauthor{\bsnm{Williams},~\bfnm{Joseph~Jay}\binits{J.~J.}} \AND
  \bauthor{\bsnm{Goel},~\bfnm{Sharad}\binits{S.}}
(\byear{2021}).
\btitle{Bandit algorithms to personalize educational chatbots}.
\bjournal{Machine Learning}
\bvolume{110}
\bpages{2389--2418}.
\end{barticle}
\endbibitem

\bibitem{caria2020adaptive}
\begin{barticle}[author]
\bauthor{\bsnm{Caria},~\bfnm{Stefano}\binits{S.}},
  \bauthor{\bsnm{Kasy},~\bfnm{Maximilian}\binits{M.}},
  \bauthor{\bsnm{Quinn},~\bfnm{Simon}\binits{S.}},
  \bauthor{\bsnm{Shami},~\bfnm{Soha}\binits{S.}},
  \bauthor{\bsnm{Teytelboym},~\bfnm{Alex}\binits{A.}} \betal{et~al.}
(\byear{2020}).
\btitle{An Adaptive Targeted Field Experiment: Job Search Assistance for
  Refugees in Jordan}.
\end{barticle}
\endbibitem

\bibitem{chen2019inference}
\begin{barticle}[author]
\bauthor{\bsnm{Chen},~\bfnm{Yuxin}\binits{Y.}},
  \bauthor{\bsnm{Fan},~\bfnm{Jianqing}\binits{J.}},
  \bauthor{\bsnm{Ma},~\bfnm{Cong}\binits{C.}} \AND
  \bauthor{\bsnm{Yan},~\bfnm{Yuling}\binits{Y.}}
(\byear{2019}).
\btitle{Inference and uncertainty quantification for noisy matrix completion}.
\bjournal{Proceedings of the National Academy of Sciences}
\bvolume{116}
\bpages{22931--22937}.
\end{barticle}
\endbibitem

\bibitem{chung2006concentration}
\begin{barticle}[author]
\bauthor{\bsnm{Chung},~\bfnm{Fan}\binits{F.}} \AND
  \bauthor{\bsnm{Lu},~\bfnm{Linyuan}\binits{L.}}
(\byear{2006}).
\btitle{Concentration inequalities and martingale inequalities: a survey}.
\bjournal{Internet mathematics}
\bvolume{3}
\bpages{79--127}.
\end{barticle}
\endbibitem

\bibitem{ieee_matrix_completion_overview}
\begin{barticle}[author]
\bauthor{\bsnm{Davenport},~\bfnm{Mark~A.}\binits{M.~A.}} \AND
  \bauthor{\bsnm{Romberg},~\bfnm{Justin}\binits{J.}}
(\byear{2016}).
\btitle{An Overview of Low-Rank Matrix Recovery From Incomplete Observations}.
\bjournal{IEEE Journal of Selected Topics in Signal Processing}
\bvolume{10}
\bpages{608-622}.
\bdoi{10.1109/JSTSP.2016.2539100}
\end{barticle}
\endbibitem

\bibitem{davidian2017nonlinear}
\begin{bbook}[author]
\bauthor{\bsnm{Davidian},~\bfnm{Marie}\binits{M.}} \AND
  \bauthor{\bsnm{Giltinan},~\bfnm{David~M}\binits{D.~M.}}
(\byear{2017}).
\btitle{Nonlinear models for repeated measurement data}.
\bpublisher{Routledge}.
\end{bbook}
\endbibitem

\bibitem{durrett2019probability}
\begin{bbook}[author]
\bauthor{\bsnm{Durrett},~\bfnm{Rick}\binits{R.}}
(\byear{2019}).
\btitle{Probability: theory and examples}
\bvolume{49}.
\bpublisher{Cambridge university press}.
\end{bbook}
\endbibitem

\bibitem{supplement}
\begin{barticle}[author]
\bauthor{\bsnm{Dwivedi},~\bfnm{Raaz}\binits{R.}},
  \bauthor{\bsnm{Tian},~\bfnm{Katherine}\binits{K.}},
  \bauthor{\bsnm{Tomkins},~\bfnm{Sabina}\binits{S.}},
  \bauthor{\bsnm{Predrag},~\bfnm{Klasnja}\binits{K.}},
  \bauthor{\bsnm{Murphy},~\bfnm{Susan}\binits{S.}} \AND
  \bauthor{\bsnm{Shah},~\bfnm{Devavrat}\binits{D.}}
(\byear{2025}).
\btitle{Supplement to ``Counterfactual inference in sequential experiments''}.
\end{barticle}
\endbibitem

\bibitem{erraqabi2017trading}
\begin{binproceedings}[author]
\bauthor{\bsnm{Erraqabi},~\bfnm{Akram}\binits{A.}},
  \bauthor{\bsnm{Lazaric},~\bfnm{Alessandro}\binits{A.}},
  \bauthor{\bsnm{Valko},~\bfnm{Michal}\binits{M.}},
  \bauthor{\bsnm{Brunskill},~\bfnm{Emma}\binits{E.}} \AND
  \bauthor{\bsnm{Liu},~\bfnm{Yun-En}\binits{Y.-E.}}
(\byear{2017}).
\btitle{Trading off rewards and errors in multi-armed bandits}.
In \bbooktitle{Artificial Intelligence and Statistics}
\bpages{709--717}.
\bpublisher{PMLR}.
\end{binproceedings}
\endbibitem

\bibitem{forman2019can}
\begin{barticle}[author]
\bauthor{\bsnm{Forman},~\bfnm{Evan~M}\binits{E.~M.}},
  \bauthor{\bsnm{Kerrigan},~\bfnm{Stephanie~G}\binits{S.~G.}},
  \bauthor{\bsnm{Butryn},~\bfnm{Meghan~L}\binits{M.~L.}},
  \bauthor{\bsnm{Juarascio},~\bfnm{Adrienne~S}\binits{A.~S.}},
  \bauthor{\bsnm{Manasse},~\bfnm{Stephanie~M}\binits{S.~M.}},
  \bauthor{\bsnm{Onta{\~n}{\'o}n},~\bfnm{Santiago}\binits{S.}},
  \bauthor{\bsnm{Dallal},~\bfnm{Diane~H}\binits{D.~H.}},
  \bauthor{\bsnm{Crochiere},~\bfnm{Rebecca~J}\binits{R.~J.}} \AND
  \bauthor{\bsnm{Moskow},~\bfnm{Danielle}\binits{D.}}
(\byear{2019}).
\btitle{Can the artificial intelligence technique of reinforcement learning use
  continuously-monitored digital data to optimize treatment for weight loss?}
\bjournal{Journal of behavioral medicine}
\bvolume{42}
\bpages{276--290}.
\end{barticle}
\endbibitem

\bibitem{hadad2021confidence}
\begin{barticle}[author]
\bauthor{\bsnm{Hadad},~\bfnm{Vitor}\binits{V.}},
  \bauthor{\bsnm{Hirshberg},~\bfnm{David~A}\binits{D.~A.}},
  \bauthor{\bsnm{Zhan},~\bfnm{Ruohan}\binits{R.}},
  \bauthor{\bsnm{Wager},~\bfnm{Stefan}\binits{S.}} \AND
  \bauthor{\bsnm{Athey},~\bfnm{Susan}\binits{S.}}
(\byear{2021}).
\btitle{Confidence intervals for policy evaluation in adaptive experiments}.
\bjournal{Proceedings of the National Academy of Sciences}
\bvolume{118}.
\end{barticle}
\endbibitem

\bibitem{kallus2020double}
\begin{barticle}[author]
\bauthor{\bsnm{Kallus},~\bfnm{Nathan}\binits{N.}} \AND
  \bauthor{\bsnm{Uehara},~\bfnm{Masatoshi}\binits{M.}}
(\byear{2020}).
\btitle{Double reinforcement learning for efficient off-policy evaluation in
  markov decision processes}.
\bjournal{Journal of Machine Learning Research}
\bvolume{21}
\bpages{1--63}.
\end{barticle}
\endbibitem

\bibitem{kasy2021adaptive}
\begin{barticle}[author]
\bauthor{\bsnm{Kasy},~\bfnm{Maximilian}\binits{M.}} \AND
  \bauthor{\bsnm{Sautmann},~\bfnm{Anja}\binits{A.}}
(\byear{2021}).
\btitle{Adaptive treatment assignment in experiments for policy choice}.
\bjournal{Econometrica}
\bvolume{89}
\bpages{113--132}.
\end{barticle}
\endbibitem

\bibitem{laird1982random}
\begin{barticle}[author]
\bauthor{\bsnm{Laird},~\bfnm{Nan~M}\binits{N.~M.}} \AND
  \bauthor{\bsnm{Ware},~\bfnm{James~H}\binits{J.~H.}}
(\byear{1982}).
\btitle{Random-effects models for longitudinal data}.
\bjournal{Biometrics}
\bpages{963--974}.
\end{barticle}
\endbibitem

\bibitem{LeeLiShahSong16}
\begin{binproceedings}[author]
\bauthor{\bsnm{Lee},~\bfnm{Christina~E.}\binits{C.~E.}},
  \bauthor{\bsnm{Li},~\bfnm{Yihua}\binits{Y.}},
  \bauthor{\bsnm{Shah},~\bfnm{Devavrat}\binits{D.}} \AND
  \bauthor{\bsnm{Song},~\bfnm{Dogyoon}\binits{D.}}
(\byear{2016}).
\btitle{Blind Regression: Nonparametric Regression for Latent Variable Models
  via Collaborative Filtering}.
In \bbooktitle{Advances in Neural Information Processing Systems 29}
\bpages{2155--2163}.
\end{binproceedings}
\endbibitem

\bibitem{li2010contextual}
\begin{binproceedings}[author]
\bauthor{\bsnm{Li},~\bfnm{Lihong}\binits{L.}},
  \bauthor{\bsnm{Chu},~\bfnm{Wei}\binits{W.}},
  \bauthor{\bsnm{Langford},~\bfnm{John}\binits{J.}} \AND
  \bauthor{\bsnm{Schapire},~\bfnm{Robert~E}\binits{R.~E.}}
(\byear{2010}).
\btitle{A contextual-bandit approach to personalized news article
  recommendation}.
In \bbooktitle{Proceedings of the 19th international conference on World wide
  web}
\bpages{661--670}.
\end{binproceedings}
\endbibitem

\bibitem{li2019nearest}
\begin{barticle}[author]
\bauthor{\bsnm{Li},~\bfnm{Yihua}\binits{Y.}},
  \bauthor{\bsnm{Shah},~\bfnm{Devavrat}\binits{D.}},
  \bauthor{\bsnm{Song},~\bfnm{Dogyoon}\binits{D.}} \AND
  \bauthor{\bsnm{Yu},~\bfnm{Christina~Lee}\binits{C.~L.}}
(\byear{2019}).
\btitle{Nearest neighbors for matrix estimation interpreted as blind regression
  for latent variable model}.
\bjournal{IEEE Transactions on Information Theory}
\bvolume{66}
\bpages{1760--1784}.
\end{barticle}
\endbibitem

\bibitem{liao2020personalized}
\begin{barticle}[author]
\bauthor{\bsnm{Liao},~\bfnm{Peng}\binits{P.}},
  \bauthor{\bsnm{Greenewald},~\bfnm{Kristjan}\binits{K.}},
  \bauthor{\bsnm{Klasnja},~\bfnm{Predrag}\binits{P.}} \AND
  \bauthor{\bsnm{Murphy},~\bfnm{Susan}\binits{S.}}
(\byear{2020}).
\btitle{Personalized HeartSteps: A Reinforcement Learning Algorithm for
  Optimizing Physical Activity}.
\bjournal{Proceedings of the ACM on Interactive, Mobile, Wearable and
  Ubiquitous Technologies}
\bvolume{4}
\bpages{1--22}.
\end{barticle}
\endbibitem

\bibitem{liu2014trading}
\begin{binproceedings}[author]
\bauthor{\bsnm{Liu},~\bfnm{Yun-En}\binits{Y.-E.}},
  \bauthor{\bsnm{Mandel},~\bfnm{Travis}\binits{T.}},
  \bauthor{\bsnm{Brunskill},~\bfnm{Emma}\binits{E.}} \AND
  \bauthor{\bsnm{Popovic},~\bfnm{Zoran}\binits{Z.}}
(\byear{2014}).
\btitle{Trading Off Scientific Knowledge and User Learning with Multi-Armed
  Bandits.}
In \bbooktitle{EDM}
\bpages{161--168}.
\end{binproceedings}
\endbibitem

\bibitem{motwani1995randomized}
\begin{bbook}[author]
\bauthor{\bsnm{Motwani},~\bfnm{Rajeev}\binits{R.}} \AND
  \bauthor{\bsnm{Raghavan},~\bfnm{Prabhakar}\binits{P.}}
(\byear{1995}).
\btitle{Randomized algorithms}.
\bpublisher{Cambridge university press}.
\end{bbook}
\endbibitem

\bibitem{neyman}
\begin{barticle}[author]
\bauthor{\bsnm{Neyman},~\bfnm{Jerzy}\binits{J.}}
(\byear{1923}).
\btitle{Sur les applications de la theorie des probabilites aux experiences
  agricoles: Essai des principes}.
\bjournal{Master's Thesis}.
\end{barticle}
\endbibitem

\bibitem{qi2018bandit}
\begin{barticle}[author]
\bauthor{\bsnm{Qi},~\bfnm{Yi}\binits{Y.}},
  \bauthor{\bsnm{Wu},~\bfnm{Qingyun}\binits{Q.}},
  \bauthor{\bsnm{Wang},~\bfnm{Hongning}\binits{H.}},
  \bauthor{\bsnm{Tang},~\bfnm{Jie}\binits{J.}} \AND
  \bauthor{\bsnm{Sun},~\bfnm{Maosong}\binits{M.}}
(\byear{2018}).
\btitle{Bandit learning with implicit feedback}.
\bjournal{Advances in Neural Information Processing Systems}
\bvolume{31}.
\end{barticle}
\endbibitem

\bibitem{qian2019estimating}
\begin{barticle}[author]
\bauthor{\bsnm{Qian},~\bfnm{Tianchen}\binits{T.}},
  \bauthor{\bsnm{Yoo},~\bfnm{Hyesun}\binits{H.}},
  \bauthor{\bsnm{Klasnja},~\bfnm{Predrag}\binits{P.}},
  \bauthor{\bsnm{Almirall},~\bfnm{Daniel}\binits{D.}} \AND
  \bauthor{\bsnm{Murphy},~\bfnm{Susan~A}\binits{S.~A.}}
(\byear{2019}).
\btitle{Estimating Time-Varying Causal Excursion Effect in Mobile Health with
  Binary Outcomes}.
\bjournal{arXiv preprint arXiv:1906.00528}.
\end{barticle}
\endbibitem

\bibitem{rubin1976}
\begin{barticle}[author]
\bauthor{\bsnm{Rubin},~\bfnm{Donald~B.}\binits{D.~B.}}
(\byear{1976}).
\btitle{Inference and Missing Data}.
\bjournal{Biometrika}
\bvolume{63}
\bpages{581--592}.
\end{barticle}
\endbibitem

\bibitem{rubin2004multiple}
\begin{bbook}[author]
\bauthor{\bsnm{Rubin},~\bfnm{Donald~B}\binits{D.~B.}}
(\byear{2004}).
\btitle{Multiple imputation for nonresponse in surveys}
\bvolume{81}.
\bpublisher{John Wiley \& Sons}.
\end{bbook}
\endbibitem

\bibitem{DBLP:journals/ftml/RussoRKOW18}
\begin{barticle}[author]
\bauthor{\bsnm{Russo},~\bfnm{Daniel}\binits{D.}},
  \bauthor{\bsnm{Roy},~\bfnm{Benjamin~Van}\binits{B.~V.}},
  \bauthor{\bsnm{Kazerouni},~\bfnm{Abbas}\binits{A.}},
  \bauthor{\bsnm{Osband},~\bfnm{Ian}\binits{I.}} \AND
  \bauthor{\bsnm{Wen},~\bfnm{Zheng}\binits{Z.}}
(\byear{2018}).
\btitle{A Tutorial on Thompson Sampling}.
\bjournal{Found. Trends Mach. Learn.}
\bvolume{11}
\bpages{1--96}.
\bdoi{10.1561/2200000070}
\end{barticle}
\endbibitem

\bibitem{sawant2018contextual}
\begin{barticle}[author]
\bauthor{\bsnm{Sawant},~\bfnm{Neela}\binits{N.}},
  \bauthor{\bsnm{Namballa},~\bfnm{Chitti~Babu}\binits{C.~B.}},
  \bauthor{\bsnm{Sadagopan},~\bfnm{Narayanan}\binits{N.}} \AND
  \bauthor{\bsnm{Nassif},~\bfnm{Houssam}\binits{H.}}
(\byear{2018}).
\btitle{Contextual multi-armed bandits for causal marketing}.
\bjournal{arXiv preprint arXiv:1810.01859}.
\end{barticle}
\endbibitem

\bibitem{schwartz2017customer}
\begin{barticle}[author]
\bauthor{\bsnm{Schwartz},~\bfnm{Eric~M}\binits{E.~M.}},
  \bauthor{\bsnm{Bradlow},~\bfnm{Eric~T}\binits{E.~T.}} \AND
  \bauthor{\bsnm{Fader},~\bfnm{Peter~S}\binits{P.~S.}}
(\byear{2017}).
\btitle{Customer acquisition via display advertising using multi-armed bandit
  experiments}.
\bjournal{Marketing Science}
\bvolume{36}
\bpages{500--522}.
\end{barticle}
\endbibitem

\bibitem{shaikh2019balancing}
\begin{binproceedings}[author]
\bauthor{\bsnm{Shaikh},~\bfnm{Hammad}\binits{H.}},
  \bauthor{\bsnm{Modiri},~\bfnm{Arghavan}\binits{A.}},
  \bauthor{\bsnm{Williams},~\bfnm{Joseph~Jay}\binits{J.~J.}} \AND
  \bauthor{\bsnm{Rafferty},~\bfnm{Anna~N}\binits{A.~N.}}
(\byear{2019}).
\btitle{Balancing Student Success and Inferring Personalized Effects in Dynamic
  Experiments.}
In \bbooktitle{EDM}.
\end{binproceedings}
\endbibitem

\bibitem{tomkins2021intelligentpooling}
\begin{barticle}[author]
\bauthor{\bsnm{Tomkins},~\bfnm{Sabina}\binits{S.}},
  \bauthor{\bsnm{Liao},~\bfnm{Peng}\binits{P.}},
  \bauthor{\bsnm{Klasnja},~\bfnm{Predrag}\binits{P.}} \AND
  \bauthor{\bsnm{Murphy},~\bfnm{Susan}\binits{S.}}
(\byear{2021}).
\btitle{IntelligentPooling: practical Thompson sampling for mHealth}.
\bjournal{Machine Learning}
\bpages{1--43}.
\end{barticle}
\endbibitem

\bibitem{uehara2022review}
\begin{barticle}[author]
\bauthor{\bsnm{Uehara},~\bfnm{Masatoshi}\binits{M.}},
  \bauthor{\bsnm{Shi},~\bfnm{Chengchun}\binits{C.}} \AND
  \bauthor{\bsnm{Kallus},~\bfnm{Nathan}\binits{N.}}
(\byear{2022}).
\btitle{A review of off-policy evaluation in reinforcement learning}.
\bjournal{arXiv preprint arXiv:2212.06355}.
\end{barticle}
\endbibitem

\bibitem{vonesh1992mixed}
\begin{barticle}[author]
\bauthor{\bsnm{Vonesh},~\bfnm{Edward~F}\binits{E.~F.}} \AND
  \bauthor{\bsnm{Carter},~\bfnm{Randy~L}\binits{R.~L.}}
(\byear{1992}).
\btitle{Mixed-effects nonlinear regression for unbalanced repeated measures}.
\bjournal{Biometrics}
\bpages{1--17}.
\end{barticle}
\endbibitem

\bibitem{wainwright2019high}
\begin{bbook}[author]
\bauthor{\bsnm{Wainwright},~\bfnm{Martin~J}\binits{M.~J.}}
(\byear{2019}).
\btitle{High-dimensional statistics: A non-asymptotic viewpoint}
\bvolume{48}.
\bpublisher{Cambridge University Press}.
\end{bbook}
\endbibitem

\bibitem{ijcai2019sequential}
\begin{binproceedings}[author]
\bauthor{\bsnm{Wang},~\bfnm{Shoujin}\binits{S.}},
  \bauthor{\bsnm{Hu},~\bfnm{Liang}\binits{L.}},
  \bauthor{\bsnm{Wang},~\bfnm{Yan}\binits{Y.}},
  \bauthor{\bsnm{Cao},~\bfnm{Longbing}\binits{L.}},
  \bauthor{\bsnm{Sheng},~\bfnm{Quan~Z.}\binits{Q.~Z.}} \AND
  \bauthor{\bsnm{Orgun},~\bfnm{Mehmet}\binits{M.}}
(\byear{2019}).
\btitle{Sequential Recommender Systems: Challenges, Progress and Prospects}.
In \bbooktitle{Proceedings of the Twenty-Eighth International Joint Conference
  on Artificial Intelligence, {IJCAI-19}}
\bpages{6332--6338}.
\bpublisher{International Joint Conferences on Artificial Intelligence
  Organization}.
\bdoi{10.24963/ijcai.2019/883}
\end{binproceedings}
\endbibitem

\bibitem{xu2017generalized}
\begin{barticle}[author]
\bauthor{\bsnm{Xu},~\bfnm{Yiqing}\binits{Y.}}
(\byear{2017}).
\btitle{Generalized synthetic control method: Causal inference with interactive
  fixed effects models}.
\bjournal{Political Analysis}
\bvolume{25}
\bpages{57--76}.
\end{barticle}
\endbibitem

\bibitem{yom2017encouraging}
\begin{barticle}[author]
\bauthor{\bsnm{Yom-Tov},~\bfnm{Elad}\binits{E.}},
  \bauthor{\bsnm{Feraru},~\bfnm{Guy}\binits{G.}},
  \bauthor{\bsnm{Kozdoba},~\bfnm{Mark}\binits{M.}},
  \bauthor{\bsnm{Mannor},~\bfnm{Shie}\binits{S.}},
  \bauthor{\bsnm{Tennenholtz},~\bfnm{Moshe}\binits{M.}} \AND
  \bauthor{\bsnm{Hochberg},~\bfnm{Irit}\binits{I.}}
(\byear{2017}).
\btitle{Encouraging physical activity in patients with diabetes: treatment
  using a reinforcement learning system}.
\bjournal{Journal of medical Internet research}
\bvolume{19}
\bpages{e338}.
\end{barticle}
\endbibitem

\bibitem{young1997fieller}
\begin{barticle}[author]
\bauthor{\bsnm{Young},~\bfnm{David~A}\binits{D.~A.}},
  \bauthor{\bsnm{Zerbe},~\bfnm{Gary~O}\binits{G.~O.}} \AND
  \bauthor{\bsnm{Hay~Jr},~\bfnm{William~W}\binits{W.~W.}}
(\byear{1997}).
\btitle{Fieller's theorem, Scheff{\'e} simultaneous confidence intervals, and
  ratios of parameters of linear and nonlinear mixed-effects models}.
\bjournal{Biometrics}
\bpages{838--847}.
\end{barticle}
\endbibitem

\bibitem{yu2022nonparametric}
\begin{binproceedings}[author]
\bauthor{\bsnm{Yu},~\bfnm{Christina~Lee}\binits{C.~L.}}
(\byear{2022}).
\btitle{Nonparametric Matrix Estimation with One-Sided Covariates}.
In \bbooktitle{2022 IEEE International Symposium on Information Theory (ISIT)}
\bpages{892--897}.
\bpublisher{IEEE}.
\end{binproceedings}
\endbibitem

\bibitem{zhan2021off}
\begin{barticle}[author]
\bauthor{\bsnm{Zhan},~\bfnm{Ruohan}\binits{R.}},
  \bauthor{\bsnm{Hadad},~\bfnm{Vitor}\binits{V.}},
  \bauthor{\bsnm{Hirshberg},~\bfnm{David~A}\binits{D.~A.}} \AND
  \bauthor{\bsnm{Athey},~\bfnm{Susan}\binits{S.}}
(\byear{2021}).
\btitle{Off-Policy Evaluation via Adaptive Weighting with Data from Contextual
  Bandits}.
\bjournal{Proceedings of the 27th ACM SIGKDD Conference on Knowledge Discovery
  and Data Mining}.
\end{barticle}
\endbibitem

\bibitem{NEURIPS2020_6fd86e0a}
\begin{binproceedings}[author]
\bauthor{\bsnm{Zhang},~\bfnm{Kelly}\binits{K.}},
  \bauthor{\bsnm{Janson},~\bfnm{Lucas}\binits{L.}} \AND
  \bauthor{\bsnm{Murphy},~\bfnm{Susan}\binits{S.}}
(\byear{2020}).
\btitle{Inference for Batched Bandits}.
In \bbooktitle{Advances in Neural Information Processing Systems}
(\beditor{\bfnm{H.}\binits{H.}~\bsnm{Larochelle}},
  \beditor{\bfnm{M.}\binits{M.}~\bsnm{Ranzato}},
  \beditor{\bfnm{R.}\binits{R.}~\bsnm{Hadsell}},
  \beditor{\bfnm{M.~F.}\binits{M.~F.}~\bsnm{Balcan}} \AND
  \beditor{\bfnm{H.}\binits{H.}~\bsnm{Lin}}, eds.)
\bvolume{33}
\bpages{9818--9829}.
\bpublisher{Curran Associates, Inc.}
\end{binproceedings}
\endbibitem

\bibitem{zhang2020inference}
\begin{binproceedings}[author]
\bauthor{\bsnm{Zhang},~\bfnm{Kelly~W}\binits{K.~W.}},
  \bauthor{\bsnm{Janson},~\bfnm{Lucas}\binits{L.}} \AND
  \bauthor{\bsnm{Murphy},~\bfnm{Susan~A}\binits{S.~A.}}
(\byear{2020}).
\btitle{Inference for Batched Bandits}.
In \bbooktitle{Advances in Neural Information Processing Systems}.
\end{binproceedings}
\endbibitem

\bibitem{zhang2021statistical}
\begin{barticle}[author]
\bauthor{\bsnm{Zhang},~\bfnm{Kelly~W.}\binits{K.~W.}},
  \bauthor{\bsnm{Janson},~\bfnm{Lucas}\binits{L.}} \AND
  \bauthor{\bsnm{Murphy},~\bfnm{Susan~A.}\binits{S.~A.}}
(\byear{2021}).
\btitle{Statistical Inference with {M}-Estimators on Adaptively Collected
  Data}.
\end{barticle}
\endbibitem

\end{thebibliography}

\end{document}